\documentclass{article}

\usepackage[preprint,nonatbib]{neurips_2020}

\usepackage{graphicx}
\usepackage{amssymb,amsmath}
\usepackage{latexsym}
\usepackage{pst-node,pst-tree}
\usepackage{xcolor}
\usepackage{algorithmic}
\usepackage[linesnumbered,ruled]{algorithm2e}
\usepackage{bm}
\usepackage{setspace}
\usepackage{subfig}
\usepackage{booktabs}
\usepackage{amsthm}
\usepackage{url}
\usepackage[version=4]{mhchem}
\usepackage{wrapfig}
\usepackage{textcomp}
\usepackage{changepage}
\usepackage{caption}
\usepackage{romannum}
\usepackage{tikz}
\usepackage{cancel}
\usepackage{enumitem}
\usepackage{mathtools}
\usepackage{float}
\usepackage{dsfont}
\usepackage{titletoc}
\usepackage{upgreek}

\usepackage[super]{nth}

\usepackage[utf8]{inputenc} \usepackage[T1]{fontenc}    \usepackage{hyperref}       \usepackage{amsfonts}       \usepackage{nicefrac}       \usepackage{microtype}

\definecolor{scarlet}{RGB}{190, 1, 25}
\definecolor{crimson}{RGB}{153, 0, 0}
\definecolor{waterblue}{RGB}{55, 120, 191}
\definecolor{tangerine}{RGB}{249, 115, 6}
\definecolor{grassgreen}{RGB}{77, 164, 9}

\newcommand{\comment}[1]{}
\newcommand{\techreport}[1]{}
\newcommand{\mmp}[1]{\textcolor{crimson}{\em MMP: #1}}
\newcommand{\yuchaz}[1]{\textcolor{waterblue}{\em Yuchia: #1}}

\newcommand{\boundarymap}{{\sc BoundaryMap}}
\newcommand{\manifoldhelmholtzian}{{\sc ManifoldHelmholtzian}}

\newcommand{\pdbootstrap}{{\sc BootstrapPD}}
\newcommand{\vrcomplex}{{\sc VRComplex}}

\renewcommand{\vec}[1]{\bm{\mathbf{#1}}}
\newcommand{\inv}[1]{\frac{1}{#1}}
\DeclareMathOperator*{\argmin}{argmin}  \newcommand{\img}{\mathrm{im}}
\newcommand{\scx}{\mathrm{SC}}
\newcommand{\diag}{\mathrm{diag}}

\newcommand{\pequal}{\mathrel{\phantom{=}}}
\newcommand{\rrr}{\mathbb R}

\newcommand{\vecxyz}{\vec x, \vec y, \vec z}
\newcommand{\vecxy}{\vec x, \vec y}

\newcommand{\dd}{\mathsf{d}}

\newcommand{\LL}{\vec{\mathcal L}}
\newcommand{\M}{{\mathcal M}}  \newcommand{\T}{{\mathcal T}}

\newcommand{\romanenu}[1]{(\romannum{#1})}

\newcommand{\bit}{\begin{itemize}}
\newcommand{\eit}{\end{itemize}}
\newcommand{\benum}{\begin{enumerate}}
\newcommand{\eenum}{\end{enumerate}}
\newcommand{\beq}{\begin{equation}}
\newcommand{\eeq}{\end{equation}}

\theoremstyle{definition}
\newtheorem{definition}{Definition}
\newtheorem*{remark}{Remark}
\newtheorem*{example}{Example}

\theoremstyle{plain}

\newtheorem{theorem}[definition]{Theorem}
\newtheorem{corollary}[definition]{Corollary}
\newtheorem{lemma}[definition]{Lemma}
\newtheorem{proposition}[definition]{Proposition}
\newtheorem*{claim}{Claim}

\newtheorem{assumption}{Assumption}

\renewenvironment{proof}{\paragraph{Proof.}}{\hfill\qed}
\newenvironment{proofsk}{\paragraph{Sketch of proof.}}{\hfill\qed}
\newenvironment{proofof}[1]{\paragraph{Proof of #1.}}{\hfill\qed}

\newcommand{\iksetremovej}{i_0\cdots \hat{i_j}\cdots i_k}
\newcommand{\iksetinsertj}{i_0\cdots \check{i_j}\cdots i_k}

\newcommand{\suppnumberprefix}{S}
\newcommand{\setupsupp}{
	\clearpage
	\newpage
	\appendix

	\renewcommand{\thefigure}{\suppnumberprefix\arabic{figure}}
	\setcounter{figure}{0}

	\renewcommand{\thetable}{\suppnumberprefix\arabic{table}}
	\setcounter{table}{0}

	\renewcommand{\theequation}{\suppnumberprefix\arabic{equation}}
	\setcounter{equation}{0}

	\renewcommand{\thealgocf}{\suppnumberprefix\arabic{algocf}}
	\setcounter{algocf}{0}

	\renewcommand{\thedefinition}{\suppnumberprefix\arabic{definition}}
	\setcounter{definition}{0}
	\section*{\centering Supplementary Material of \\ \titlename}

	\pagenumbering{roman}
	\setcounter{page}{1}
	\section*{Table of Contents}
	\startcontents[appendix]
	\printcontents[appendix]{l}{1}{\setcounter{tocdepth}{2}}
	\newpage

\pagenumbering{arabic}
	\setcounter{page}{1}
}

\let\vec\mathbf
\newcommand{\va}{{\vec a}}
\newcommand{\vb}{{\vec b}}

\newcommand{\ve}{{\vec e}}
\newcommand{\vf}{{\vec f}}
\newcommand{\vg}{{\vec g}}

\newcommand{\vn}{{\vec n}}

\newcommand{\vs}{{\vec s}}

\newcommand{\vu}{{\vec u}}
\newcommand{\vv}{{\vec v}}
\newcommand{\vw}{{\vec w}}
\newcommand{\vx}{{\vec x}}
\newcommand{\vy}{{\vec y}}
\newcommand{\vz}{{\vec z}}

\newcommand{\vB}{{\vec B}}

\newcommand{\vF}{{\vec F}}

\newcommand{\vI}{{\vec I}}

\newcommand{\vL}{{\vec L}}

\newcommand{\vW}{{\vec W}}
\newcommand{\vX}{{\vec X}}

\newcommand{\rotvert}{\rotatebox[origin=c]{90}{$\vert$}}
\newcommand{\rowsvdots}{\multicolumn{1}{@{}c@{}}{\vdots}}
 \renewcommand{\mmp}[1]{} \renewcommand{\yuchaz}[1]{}

\def\titlename{Helmholtzian Eigenmap: Topological feature discovery \& edge flow learning from point cloud data}

\title{\titlename}

\author{
  Yu-Chia Chen \\
  Meta Platforms, Inc. \\
  Seattle, WA 98109 \\
  \texttt{yuchaz@uw.edu} \\
  \And
  Weicheng Wu\\
  Department of Statistics \\
  University of Washington \\
  Seattle, WA 98195 \\
  \texttt{weich555@uw.edu} \\
  \And
  Marina Meil\u{a} \\
  Department of Statistics \\
  University of Washington \\
  Seattle, WA 98195 \\
  \texttt{mmp2@uw.edu} \\
  \And
  Ioannis G. Kevrekidis \\
  Chemical and Biomolecular Engineering \& \\
  Applied Mathematics and Statistics \\
Johns Hopkins University\\
  Baltimore, MD 21218\\
  \texttt{yannisk@jhu.edu}
}

\begin{document}
\pagenumbering{arabic}

\maketitle

\begin{abstract}
The manifold Helmholtzian (1-Laplacian) operator $\Delta_1$ elegantly
generalizes the Laplace-Beltrami operator to vector fields on a
manifold $\M$.  In this work, we propose the estimation of the manifold
Helmholtzian from point cloud data by a weighted 1-Laplacian
$\LL_1$. While higher order Laplacians have been introduced and studied,
this work is the first to present a graph Helmholtzian constructed from a simplicial complex as a {\em consistent} estimator for
the continuous operator in a non-parametric setting.
Equipped with the geometric and topological information about $\M$,
the Helmholtzian is a useful tool for the analysis of flows and vector
fields on $\M$ via the Helmholtz-Hodge theorem. In addition, the $\LL_1$
allows the smoothing, prediction, and 
feature extraction of the flows. We
demonstrate these possibilities on substantial sets of synthetic and real
point cloud datasets with non-trivial topological structures;
and provide theoretical results on the limit of $\LL_1$ to $\Delta_1$.
\end{abstract}

\section{Motivation}
In this paper we initiate the estimation of higher order Laplacian
operators from point cloud data, with a focus on the first order
Laplacian operator $\Delta_1$ of a manifold. Laplacians are known to
be intimately tied to a manifold's topology and geometry. While the
{\em Laplace-Beltrami} operator $\Delta_0$, an operator acting on
functions ($0$-forms), is well studied and pivotal in classical
manifold learning; estimating the $1$-Laplacian $\Delta_1$, an
operator acting on vector fields ($1$-forms), for a manifold has rarely been
attempted yet. Order $k$ Laplacian operators (on $k$-forms), denoted $\Delta_k$, exist as well, for $k=1,2,\ldots$.

The discrete operator analogue of $\Delta_k$, known as the {\em $k$-Hodge Laplacian matrix} $\vec L_k$,
has been proposed more than 7 decades ago \cite{EckmannB:44}.
The beauty of the aforementioned framework generated numerous applications
in areas such as numerical analysis \cite{ArnoldD.F.W+:10,DodziukJ:76},
edge flow learning on graphs \cite{SchaubMT.B.H+:20,JiaJ.S.S+:19},
pairwise ranking \cite{JiangX.L.Y+:11}, and game theory \cite{CandoganO.M.O+:11}.

Being able to estimate $\Delta_1$ by a discrete Helmholtzian $\LL_1$ acting on the edges of a graph can support many applications, just as the $\Delta_0$ estimator by weighted Laplacians successfully did.
For instance, \romanenu{1} topological information, i.e.,
the first Betti number $\beta_1$ \cite{LimLH:20}, can be obtained by the dimension of the null space of $\LL_1$;
\romanenu{2} low dimensional representation of the space of vector fields on a manifold are made possible, similar to the dimensionality
reduction algorithms such as Laplacian Eigenmap from the discrete estimates of $\Delta_0$;
\romanenu{3} the well known {\em Helmholtz-Hodge decompositon} (HHD) \cite{BhatiaH.N.P+:13,LimLH:20}
allows us to test, e.g., if a vector field on the manifold $\M$ is approximately a gradient
or a rotational field;
lastly \romanenu{4}, edge flow semi-supervised learning (SSL) and unsupervised learning algorithms, i.e., flow prediction and flow smoothing in edge space,
can be easily derived from the well-studied node based learning models
\cite{BelkinM.N.S+:06,OrtegaA.F.K+:18} with the aid of $\LL_1$.

In this work, we propose a discrete estimator ($\LL_1$) of the Helmholtzian $\Delta_1$
with proper triangular weights which resembles the well known Vietoris-Rips (VR) complex in
the persistent homology theory. We show separately the (pointwise) convergence
of the down and the up components of the discrete graph Helmholtzian $\LL_1$ to the continuous operators $\Delta_1^{\rm down}$  and $\Delta_1^{\rm up}$.
In addition, we present several applications to graph signal
processing and semi-supervised learning algorithms on the edge flows
with the constructed $\LL_1$.
We support our theoretical claims and illustrate the versatility of
the proposed $\Delta_1$ estimator with extensive empirical results on
synthetic and real datasets with non-trivial manifold structures.

In the next section we briefly introduce Hodge theory and higher
order Laplacians. Section \ref{sec:algorithms} presents the $\LL_1$
construction algorithm. The theoretical results are in Section
\ref{sec:consistency-analysis}.  Sections
\ref{sec:application-laplacian} and \ref{sec:experiments} provide
several applications of the estimated Helmholtzian to the analysis of
vector fields.  Those Figure/Table/Equation/Theorem references with
prefix \suppnumberprefix{} are in the Supplement.

\section{Background: Hodge theory}
\label{sec:background}
\paragraph{Simplicial complex}
A natural extension of a graph to higher dimensional relations is
called a {\em simplicial complex}.  Define a $k$-simplex to be a
$k$-dimensional polytope which is the convex hull of its $k+1$
(affinely independent) vertices.  A simplicial complex $\scx$ is a set
of simplices such that every {\em face} of a simplex from $\scx$ is also in
$\scx$.  Let $\Sigma_k$ be the collection of $k$-simplices $\sigma_k$
of $\scx$; we write $\scx = (\Sigma_0,\Sigma_1,\cdots,\Sigma_\ell)$, with $n_k = |\Sigma_k|$. A graph is $G = (V, E) = (\Sigma_0, \Sigma_1)$, with $n_0=|V|=n$. In this work, we focus on dimension $k\leq 2$ and simplicial complexes of the form  $\scx_2=(V, E, T)\equiv (\Sigma_0, \Sigma_1, \Sigma_2)$, where $T$ is the set of triangles of $\scx$.
\paragraph{$k$ (co-)chain}
Given an arbitrary orientation for each simplex $\sigma_i^k\in
\Sigma_k$, one can define the finite-dimensional vector space
$\mathcal C_k$ over $\Sigma_k$ with coefficients in $\mathbb R$. An
element $\vec{\omega}_k\in \mathcal C_k$ is called a {\em $k$-chain}
and can be written as $\vec{\omega}_k = \sum_i \omega_{k,i}
\sigma^k_i$.
Since $\mathcal C_k$ is isomorphic to $\rrr^{n_k}$,
$\vec{\omega}_k$ can be represented by a vector of coordinates
$\vec{\omega}_k = (\omega_{k,1},\cdots,\omega_{k, n_k})^\top\in\mathbb R^{n_k}$.
$\mathcal C^k$ denotes the dual space of $\mathcal C_k$; an element of $\vec{\omega}^k\in \mathcal C^k$ is called a {\em $k$-cochain}.
Even though they are intrinsically different, we will use {\em chains} and {\em cochains} interchangeably
for simplicity in this work. 
Readers are encouraged to consult \cite{LimLH:20} for thorough discussions on
the distinction between these two terms.

\paragraph{(Co-)boundary map}
The {\em boundary map (operator)} $\mathcal B_k: \mathcal C_k \to\mathcal C_{k-1}$ (defined rigorously in the Supplement
\ref{sec:background-ext-calculus}) maps a simplex $\sigma_k$ to the $k-1$-chain of its faces, with signs given by the
orientation of the simplex w.r.t. each face. For example,
let $x, y, z\in V$, edges $[x,y],[y,z],[x,z] \in E$, and a triangle $t=[x,y,z]\in T$, we have $B_2(t)=[x,y]+[y,z]-[x,z]$.
Since $\mathcal C_k$ is isomorphic to $\mathbb R^{n_k}$, one can represent $\mathcal B_k$ by a {\em boundary map (matrix)}
$\vec{B}_k \in\{0,\pm 1\}^{n_{k-1}\times n_k}$. The entry
$(\vec{B}_k)_{\sigma_{k-1},\sigma_k}$ represents the orientation of $\sigma_{k-1}$ as a face of $\sigma_k$,
or equals 0 when the two are not adjacent.
For $k=1$, the boundary map is the node to edge {\em graph incidence matrix}, i.e.,
$(\vB_1)_{[a], [xy]} = 1$ if $a = x$, $(\vB_1)_{[a], [x,y]}=-1$ if $a=y$, and zero otherwise;
for $k=2$, each column of $\vec{B}_2$ contains
the orientation of a triangle w.r.t. its edges.
In other words, $(\vB_2)_{[a,b], [x,y,z]} = 1$ if $[a,b] \in \{[x,y], [y,z]\}$, $(\vB_2)_{[a,b], [x,y,z]} = -1$
if $[a,b] = [x,z]$, and $0$ otherwise.
In this paper, we will only work with simplices up to dimension $k=2$.
Closely related to {\em boundary map} is the {\em co-boundary map}. This operator
is the adjoint of the {\em boundary map} and it maps a simplex to its co-faces, i.e.,
$\mathcal B^k: \mathcal C^{k-1}\to \mathcal C^{k}$.
The corresponding {\em co-boundary matrix} is simply the transpose of the
{\em boundary matrix}, i.e., $\vec{B}_k^\top$.
Pseudocode for constructing $\vec{B}_k$ can be found in Algorithm \ref{alg:boundary-map}.

\paragraph{The discrete $k$-Laplacian}
The unnormalized $k$-Laplacian $\vec{L}_k = \vec{B}_k^\top\vec{B}_k + \vec{B}_{k+1}\vec{B}_{k+1}^\top$
was first introduced by \cite{EckmannB:44} as a discrete analog to $\Delta_k$.
One can verify that $\vec{L}_0 = \vec{B}_1\vec{B}_1^\top$ represents the {\em unnormalized} graph
Laplacian \cite{ChungF:96}.  To extend the aforementioned construction
to a weighted $k$-Laplacian, one can introduce a diagonal non-negative
{\em weight matrix} $\vW_k$ of dimension $n_k$, with
$(\vec{W}_k)_{[\sigma_k, \sigma_k]}$ being the weight for simplex
$\sigma_k$.  For unweighted $k$-Laplacians, $\vec{W}_{k-1}, \vec{W}_{k}, \vec{W}_{k+1}$ are equal to
the unit matrices.
By analogy to the {\em random walk} graph Laplacian,
\cite{HorakD.J:13} define a (weighted) {\em random walk} $k$-Hodge Laplacian by
$
	\vec{\mathcal L}_k = \vec{B}_k^\top\vec{W}_{k-1}^{-1}\vec{B}_k\vec{W}_k + \vec{W}_k^{-1}\vec{B}_{k+1}\vec{W}_{k+1}\vec{B}_{k+1}^\top.
$
Of specific interest to us are the weighted Laplacians for $k=0$
(graph Laplacian) and $k=1$ (graph Helmholtzian). The operator
$\vec{\mathcal
L}_0=\vec{W}_0^{-1}\vec{B}_{1}\vec{W}_1\vec{B}_{1}^\top$ coincides
with the {\em random walk} graph Laplacian \cite{ChungF:96}. The Helmholtzian is defined as
\begin{equation}
\label{eq:weighted-laplacian}
\LL_1 =  a\cdot \underbrace{\vB_1^\top\vW_0^{-1}\vB_1\vW_1}_{\LL_1^\mathrm{down}} + b\cdot \underbrace{\vW_1^{-1}\vB_2\vW_2\vB_2^\top}_{\LL_1^\mathrm{up}}.
\end{equation}
Here $a, b$ are non-negative constants which were usually set to $1$ in the previous studies, we will show that $a=\frac{1}{4}$ and $b=1$ in Section \ref{sec:consistency-analysis}.
Since $\vec{\mathcal
L}_k$ is an asymmetric matrix, one can symmetrize it by
$\vec{\mathcal L}_1^s = \vec{W}_1^{1/2} \vec{\mathcal
L}_1\vec{W}_1^{-1/2}$
while preserving the spectral properties \cite{SchaubMT.B.H+:20};
$\LL_1^s$ is called the {\em
symmetrized} or {\em renormalized} $1$-Hodge Laplacian.
For more information about the boundary map and $k$-Laplacian, please
refer to \cite{LimLH:20,HorakD.J:13,SchaubMT.B.H+:20}.

\begin{flalign}
	\hspace{-10pt}\mathcal C_1 \simeq \mathbb R^{n_1} = \lefteqn{\overbracket{\phantom{\overbrace{\img\left(\vec{W}_1^{1/2}\vec{B}_{1}^\top\right)}^{temp} \oplus \ker\left(\vec{\mathcal L}_1\right)}}^{\ker\left(\vec{B}_{2}^\top\vec{W}_1^{-1/2}\right) = \ker(\text{curl})}}
	\underbrace{\img\left(\vec{W}_1^{1/2}\vec{B}_{1}^\top\right)}_{\text{gradient}} \oplus \underbracket{\overbrace{\ker\left(\vec{\mathcal L}_1\right)}^{\text{harmonic}} \oplus \underbrace{\img\left(\vec{W}_1^{-1/2}\vec{B}_{2}\right)}_{\text{curl}}}_{\ker\left(\vec{B}_1\vec{W}_1^{1/2}\right) = \ker(\text{gradient})}.
\label{eq:hhd-equation}
\end{flalign}

\paragraph{(Normalized) Hodge decomposition}
This celebrated result expresses a vector field as the direct sum of a
{\em gradient}, a {\em harmonic}, and a {\em rotational} vector field.
The eigenvectors of the $k$-Laplacian, which span the vector space $\mathcal C_k$
of $k$-chains, specifically form the bases of three different subspaces (the
image of $\LL_k^\mathrm{down}$, the image of $\LL_k^\mathrm{up}$, and the
kernel of both operators).
Here, we mainly discuss the
{\em normalized Hodge decomposition} defined by \cite{SchaubMT.B.H+:20}
for $0$ or $1$-cochains as well as the Laplacian operators $\LL_0$ or $\vec{\mathcal L}_{1}^s$;
but they can be generalized to higher order cochains and Laplacians.
For $k=0$ (graph Laplacian), the first term in the decomposition vanishes and we have
only the decomposition of the null and image of $\vec{\mathcal L}_0$.
As for $k=1$ (Helmholtzian), one can obtain the decomposition as in \eqref{eq:hhd-equation}.
Here the symbols $\ker, \img$ denote respectively
the null space and image of an operator.
For any 1-cochain $\vec{\omega} \in \mathbb R^{n_1}$ we can write
$\vec{\omega} = \vec{g} \oplus \vec{r} \oplus \vec{h}$, with
$\vec{g} = \vW_1^{1/2}\vec{B}_1^\top\vec{p}$  the gradient,
$\vec{r} = \vW_1^{-1/2}\vec{B}_2 \vec{v}$ the curl, and $\vec{h}$ the harmonic flow component.
The flows $\vec{g}, \vec{r}, \vec{h}$ can be estimated by least squares, i.e., $\hat{\vec{p}} = \mathrm{argmin}_{\vec{p}\in\mathbb R^{n_0}} \|\vW_1^{1/2}\vec{B}_1^\top\vec{p} - \vec{\omega}\|^2$,
$\hat{\vec{v}} = \mathrm{argmin}_{\vec{v}\in\mathbb R^{n_2}} \|\vW_1^{-1/2}\vec{B}_2 \vec{v} - \vec{\omega}\|^2$, and finally
$\hat{\vec{h}} = \vec{\omega} - \vW_1^{1/2}\vec{B}_1^\top \hat{\vec{p}} - \vW_1^{-1/2}\vec{B}_2\hat{\vec{w}}$.

\paragraph{$k$-Laplacian operators on manifolds} These operators act on {\em differential forms} of order $k$ \cite{DesbrunM.K.T+:08, WhitneyH:05}, the continuous  analogues to $k$-cochains. For instance,  a $0$-form is a scalar function, and an $1$-form a vector field.
The $k$-Hodge Laplacian is defined to be
$\Delta_k \coloneqq \mathsf{d}_{k-1}\updelta_k + \updelta_{k+1}\mathsf{d}_{k} = (\mathsf d + \updelta)^2$.
Similar to the discrete operator, $\Delta_k$ can be written as the sum of down
$k$-Laplacian ($\Delta_k^\mathrm{down} = \dd_{k-1}\updelta_k$) and up
$k$-Laplacian ($\Delta_k^\mathrm{up} = \updelta_{k+1}\mathsf{d}_{k}$).
The operators $\updelta_k$ and $\mathsf{d}_k$, called respectively {\em exterior derivative} and {\em co-derivative} are the differential analogues of the boundary $\vB_k$ and co-boundary operators $\vB_{k-1}$ (definitions in Supplement \ref{sec:background-ext-calculus}).
The well-known Laplace-Beltrami operator is $\Delta_0 = \updelta_1\mathsf d_0$. For our paper, the main object of interest is the 1-Hodge Laplacian $\Delta_1=\mathsf{d}_0\updelta_1+\updelta_2\mathsf{d}_1$, also known as {\em Helmholtzian}.
For $d = 3$, one has $\mathsf{d}_0 = \mathrm{grad}$, and $\updelta_1 = -\mathrm{div}$, and $\mathsf{d}_1$ corresponds to
$\mathrm{curl}$. The vector Laplacian in 3D (coordinate-wise $\Delta_0$) corresponds to $\Delta_1$,
i.e., $\Delta_1 = -\nabla^2 = -\mathrm{grad}\,\mathrm{div} + \mathrm{curl}\,\mathrm{curl}$ \cite{BhatiaH.N.P+:13}.
For a 1-form $\zeta_1 = (f_1,\cdots, f_d)$, the expression of $\Delta_1\zeta_1$ in local coordinates can be found in Proposition \ref{thm:hodge-1-laplacian-local-coord} in Supplement \ref{sec:lemmas-exterior-calculus}. In particular, if $\zeta$ is purely {\em curl} ($\dd_0\updelta_1\zeta = 0$) or {\em gradient flow} ($\updelta_0\dd_1\zeta = 0$), then $\Delta_1\zeta$ is a coordinate-wise 0-Laplacian, i.e., $\Delta_1\zeta \propto (\Delta_0 f_1,\cdots \Delta_0 f_d)$, as shown in Corollary \ref{thm:pure-curl-grad-1-Laplacian} or in \cite{LeeJM:06}.
 \section{Problem formulation and main algorithm}
\label{sec:algorithms}
We now formally describe the aim of this work.
Suppose we observe data $\vec{X}\in\rrr^{n\times D}$,
with data points denoted by $\vec{x}_i \in \rrr^D \,\forall\, i\in [n]$, that are sampled from a smooth $d$-dimensional submanifold $\mathcal M \subset \rrr^D$; and the sampling density is uniform on $\M$. In this paper, a point on $\M$ has two notations: $\vec x$ is its $\rrr^D$ coordinate vector, while $x$ is the coordinate free point on $\M$. For instance, computations are always in $\rrr^D$ coordinates; whereas $V,E,T$, and geodesics refer to the coordinate-free representations.

\begin{algorithm}[H]
    \setstretch{1.15}
    \SetKwInOut{Input}{Input}
    \SetKwInOut{Output}{Return}
    \SetKwComment{Comment}{$\triangleright $\ }{}
\Input{data $\vec{X}\in\mathbb R^{n\times D}$, radius $\delta$, kernel bandwidth $\varepsilon$}
$\scx_2 = (V,E,T) \gets \text{\vrcomplex}$$(\texttt{data}=\vec{X},\, \texttt{max\_dist}=\delta,\, \texttt{max\_dim}=2)$ \\
    $\vec{B}_1 = \textsc{BoundaryMap}(V, E)$ \,\Comment{Algorithm \ref{alg:boundary-map}}
    $\vec{B}_2 = \textsc{BoundaryMap}(E, T)$ \label{step:b2}\\
$\vec{W}_2 \gets \diag\{w^{(2)}(t),\,t\in T\}$ as in \eqref{eq:vr-kernel-general-form}\label{step:w2}\\
    $\vec{W}_1 \gets \diag\{|\vec{B}_2|\vec{W}_2\vec{1}_{n_2}\}$
\label{step:w1}\\
    $\vec{W}_0 \gets \diag\{|\vec{B}_1|\vec{W}_1\vec{1}_{n_1}\}$
\label{step:w0}\\
    $\vec{\mathcal L}_1 \gets \inv{4} \vec{B}_1^\top \vec{W}_0^{-1}\vec{B}_1\vec{W}_1 + \vec{W}_1^{-1}\vec{B}_2\vec{W}_2\vec{B}_2^\top$ \label{step:ll1} \Comment{Set $a = \inv{4}; b=1$ in \eqref{eq:weighted-laplacian}}
    $\vec{\mathcal L}_1^s \gets \vec{W}_1^{1/2} \vec{\mathcal L}_1\vec{W}_1^{-1/2}$ \label{step:ll1s}\\
\Output{Helmholtzian $\vec{\mathcal L}_1$, symmetrized Helmholtzian $\vec{\mathcal L}_1^s$}
\caption{{\manifoldhelmholtzian}}
    \label{alg:manifold-helmholtzian-learning}
\end{algorithm}

Our aim is to approximate $\Delta_1$ by a suitably weighted Helmholtzian $\LL_1$ on a 2-simplicial complex $\scx_2=(V,E,T)$, with nodes located at the data points.
The steps of this construction are given in Algorithm \ref{alg:manifold-helmholtzian-learning}.
The first \ref{step:b2} steps produce
the simplicial complex $\scx_2$ and the boundary matrices $\vB_1, \vB_2$ from $\vX$. There are
multiple ways to build an $\scx_\ell$ from point
cloud data \cite{OtterN.P.T+:17}; here the {\em Vietoris-Rips} (VR) \cite{ChazalF.M:17}
complex is used for its efficient runtime
and natural fit with the chosen triangular kernel
which will be described below.
The VR complex is an {\em abstract simplicial complex} defined on the finite
metric space. An $\ell$-simplex $\sigma_\ell$ is included in the complex
if all the pairwise distances between its vertices are smaller than some
radius $\delta$.
The VR complex is a special case of the {\em clique complex};
one can easily build such complex from a $\delta$-radius graph by including
all the cliques in the graph.
Note that a VR complex built from a point cloud dataset $\vX\in\rrr^D$ cannot always be
embedded in $\rrr^D$ due to the possible crossings between
simplices. For the VR 2-complex $\scx_2$ constructed from a point cloud,
the vertex set is the
data $\vec{X}$, and two vertices are connected if they are at distance
$\delta$ or less. A triangle $t$ in $\scx_2$ is formed when 3 edges of $t$ are
all connected.
The edges $E$ and triangles $T$ are represented as
lists of tuples of lengths $n_1$ and $n_2$, respectively.  From them,
$\vec{B}_1$, $\vec{B}_2$ are constructed in linear time
w.r.t. $n_1$ and $n_2$. In the worst case scenario, one needs $n_1 =
\mathcal O(n^2)$ in memory. Luckily, this is oftentimes a corner case
due to the manifold assumption. The memory size can further be reduced
by different approximation methods
\cite{DeyTK.F.W+:13,SheehyDR:13,KerberM.S:13} in building $\scx_2$.
We implemented Algorithm \ref{alg:manifold-helmholtzian-learning} in
\texttt{python}.  The $\scx_2$ is built with \texttt{gudhi}
\cite{MariaC.B.G+:14}.  Upon constructing the $\scx_2$,
Algorithm \ref{alg:boundary-map} $\textsc{BoundaryMap}$ is implemented
by \texttt{numba} \cite{LamSK.P.S+:15} to speed up the for-loop
operation by multi-threading.

Steps \ref{step:w2}--\ref{step:w0} construct the weigths matrices $\vec{W}_2,\vec{W}_1,\vec{W}_0$. The crucial step is the weighting of the triangles, which is described below. Once the weights of the triangles are given, the weights of the lower dimensional simplices are determined by the consistency conditions required by the boundary operator. It then follows that the weight of vertex
$v\in V$ equals its degree $[\vec{W}_0]_{vv}=\sum_{e\in E} |[\vec{B}_1]_{ve}|[\vec{W}_1]_{ee}$ and similarly, $[\vec{W}_1]_{ee}=\sum_{t\in T} |[\vec{B}_2]_{et}|[\vec{W}_2]_{tt}$. Finally, $\LL_1$ and $\LL_1^s$ are obtained by directly applying the definitions in Section \ref{sec:background}.

\paragraph{The triangle kernel $w^{(2)}$}
There are multiple choices of kernels, e.g., constant values on triangles \cite{SchaubMT.B.H+:20}, or
weights based on $\vB_2^\top$ \cite{GradyLJ.P:10}. The former fails to capture the size and geometry of a triangle while the latter
violates the assumption that we are building $\scx_2$ from a point cloud.
Here we introduce a kernel which weighs triangles by the product of the pairwise edge kernels.
\begin{equation}
\begin{gathered}
    w^{(2)}(\vecxyz) = \kappa(\vecxy)\cdot\kappa(\vec x, \vec z)\cdot\kappa(\vec y,\vec z) \,\text{ for }\, x, y, z \in T, \\
    \,\text{ where }\, \kappa(\vecxy) = \kappa\left(\|\vec x-\vec y\|^2 / \varepsilon^2\right). \\
\end{gathered}
\label{eq:vr-kernel-general-form}
\end{equation}
with $\kappa(\cdot)$ any exponentially decreasing function. In this work,
we use the exponential kernel $\kappa(u) = \exp(-u)$.
With the aid of an exponentially decreasing function $\kappa(\cdot)$ in
\eqref{eq:vr-kernel-general-form},
one can filter out structures that are topological noise in $\LL_1$, as
we will show in Section \ref{sec:experiments}.
Note that there is a resemblance between \eqref{eq:vr-kernel-general-form}
and the VR complex. By definition, a triangle $x, y, z$ in the VR complex
is formed if and only if
$\mathds{1}(\|\vec x-\vec y\|<\varepsilon)\mathds{1}(\|\vec y-\vec z\|<\varepsilon)\mathds{1}(\|\vec z-\vec x\|<\varepsilon)$ equals $1$.
Hence the VR complex itself is given by using $\kappa(u) = \mathds{1}(u < 1)$
in \eqref{eq:vr-kernel-general-form}.

\paragraph{Selecting the parameters $\varepsilon$ and $\delta$}
Asymptotically, as $n\rightarrow \infty$, the kernel widths
corresponding to $\vec{W}_1$, $\vec{W}_2$ must decrease towards 0 at rates that are mutually
compatible.
The consistency analysis in
Proposition \ref{thm:asymptotic-expansion-L-1-up} of Section \ref{sec:consistency-analysis} suggests a choice of
$\varepsilon = \mathcal O(\delta^{\frac{2}{3}})$.
From Section
\ref{sec:consistency-analysis} we can also see that, since the down Helmholtzian is
consistent if the corresponding graph Laplacian is consistent, one can
choose $\delta$ by
\cite{Joncas2017-kn}.
A data-driven approach for choosing the $\varepsilon$ parameter
is currently lacking, and we leave it as future work.

\paragraph{Choice of $a, b$}
In Section \ref{sec:consistency-analysis}, we will analyze separately the convergence
of the up and down Laplacian, i.e.,  $\LL_1^\mathrm{down}\to a\cdot\Delta_1^\mathrm{down}$ 
and $\LL_1^\mathrm{up}\to b\cdot\Delta_1^\mathrm{up}$. From the proof, it follows that $a = \inv{4}; b = 1$. Please refer to Supplement
\ref{sec:discussion-choice-of-a-b} for more details and discussions.
 \section{Consistency results for graph Helmholtzian}
\label{sec:consistency-analysis}
This section investigates the continuous limits of the discrete operators
$\LL_1^\mathrm{down} = \vec{B}_1^\top\vec{W}_0^{-1}\vec{B}_1\vec{W}_1$ and $\LL_1^\mathrm{up} =
\vec{W}_1^{-1}\vec{B}_{2}\vec{W}_2\vec{B}_{2}^\top$.
We assume the following for our analysis.
\begin{assumption}
The data $\vec{X}$ are sampled i.i.d. from a uniform density supported
on a $d$ dimensional manifold $\M\subseteq \rrr^D$ that is of
class $C^3$ and has bounded curvature. W.l.o.g., we assume that the volume of $\M$ is 1; and we denote by $\mu$ the Lebesgue measure on $\M$.
\label{assu:manifold-data-assu}
\end{assumption}
\begin{assumption}
The kernel $\kappa(\vx,\vy)$ of $w^{(2)}(\vecxyz)$ in
\eqref{eq:vr-kernel-general-form} is of class $C^3$ and has exponential decay.
\label{assu:kernel-assu}
\end{assumption}
We proofs of the pointwise consistency of $\LL_1^\mathrm{up}$ and $\LL_1^\mathrm{down}$; in addition, we also show the spectral consistency of $\LL_1^\mathrm{down}$ based on results for the spectral consistency of the related $\LL_0$.

\paragraph{Pointwise convergence of the up Laplacian $\LL_1^{\rm up}$}
Let $\vec\gamma(t)$ for $t\in[0, 1]$ be
the geodesic curve connecting $x, y$ with $\vec\gamma(0) = \vx,
\vec\gamma(1) = \vy$, and $\vec\gamma'(t) = \dd\vec\gamma(t) / \dd t$. A
1-form (vector field) $\zeta$ on $\M$ induces the $1$-cochain
$\vec{\omega}$ on $E$ by  $\vec{\omega}([x,y]) = \omega_{xy} =
\int \zeta(\vec\gamma(t))^\top\vec\gamma'(t) \dd t$ for any edge
$[x,y]\in E$.
For notational simplicity, let $f_{xyz} = {\omega}_{xy} + {\omega}_{yz} + {\omega}_{zx}$.
The goal is to show the consistency of $\LL_1^{\mathrm{up}}$ for a
fixed edge $[x,y]$, i.e., to show that $\vec{\mathcal
  L}_1^{\mathrm{up}} \vec\omega_{xy} \to c\int_0^1 (\Delta_1
\zeta)(\vec\gamma(t))^\top\vec\gamma'(t)\dd t$.  First we obtain the
discrete form of the unnormalized (weighted) up-Laplacian operating on
a 1-cochain $\vec\omega$.
\begin{lemma}\label{thm:discrete-form-L-up}
Let $\vec{\omega}\in \rrr^{n_1}$ be a $1$-cochain induced on $\scx_2$ by  vector field $\zeta$. For any $x,y,z\in V$, we denote by $[x', y', z']$ the canonical ordering of the triangle $t\in T$ with vertex set $\{x, y, z\}$ (if one exists).
Then, $[ \vec{B}_2 \vec{W}_2\vec{B}_2^\top\vec\omega]_{[x, y]} = \sum_{z\in\{v\in V: [x', y', v'] \in T\}}w^{(2)}(\vec x, \vec y, \vec z) f_{xyz}$.
\end{lemma}
Lemma \ref{thm:discrete-form-L-up} is proved in Supplement
\ref{sec:proof-thm-discrete-form-L-up}. From Lemma
\ref{thm:discrete-form-L-up}, it is enough to consider {\em divergence
  free} 1-forms for the limit of $\LL_1^{\mathrm{up}}$, because $f_{xyz} = 0$ if $\zeta$ is curl free.
The following proposition shows the
asymptotic expansion for the integral form of $\sum_z w_T(\vecxyz)f_{xyz}$ (Lemma \ref{thm:discrete-form-L-up}) when $n$ is large.

\begin{proposition}
If Assumptions \ref{assu:manifold-data-assu}--\ref{assu:kernel-assu} hold,
 $\vec{\upsilon}$ is of class $C^4(\M)$,
then
\begin{equation}
	\begin{split}
	\int_\M w_T(\vecxyz)f_{xyz}\mathsf{d}\mu(\vec z) =& 
\frac{2}{3}\varepsilon^{2+d}C_2\int_0^1 [\updelta \dd\vec{\upsilon}(\vec\gamma_\M(t))]^\top \vec\gamma_\M'(t) \dd t \\&+\mathcal O(\varepsilon^{4+d}) + \mathcal O(\varepsilon^{d+2}\delta^2)+\mathcal O(\varepsilon^{d}\delta^3).
 \end{split}
\label{eq:asymptotic-expansion-L-1-up-result}
\end{equation}
with 
$C_2 = \varepsilon^{-(4+d)}\int_{\vz\in \rrr^d}\kappa^2(\frac{\|\vz\|^2}{\varepsilon^2})\vz_1^2\vz_2^2\dd\vz$ and $z_i$ representing the i-th coordinate of $\vz$.
\label{thm:asymptotic-expansion-L-1-up}
\end{proposition}

\begin{proofsk}
We first prove the case when $\M = \rrr^d$ (Lemma
\ref{thm:asymptotic-expansion-L-1-up-Rd}).  Consider a triangle
$[x,y,z]\in T$, where $z$ will be integrated over. We parametrize the
path segments $x\to y$, $y\to z$, and $z\to x$ by $\vec u(t)$, $\vec
v(t)$, and $\vec w(t)$, respectively. By changing the variables from 
$\vec w(t)$ and $\vec v(t)$ to $\vec u(t)$, we express all three line integrals
as integrals along the segment $[x,y]$. Next, when $\M\subseteq\rrr^D$, we
bound the error terms of approximating the integration from $\M$ to the
tangent plane $\T_\vx\M$ at $\vx$ with $\mathcal O(\varepsilon^4)$ in
Lemma \ref{thm:coordinate-change-errors}. Combining this Lemma with
Lemma \ref{thm:asymptotic-expansion-L-1-up-Rd} concludes the proof.
\end{proofsk}

Proposition \ref{thm:asymptotic-expansion-L-1-up} implies $\delta$ should decay faster than $\varepsilon$. So one can
choose $\delta=\mathcal{O}(\varepsilon^{3/2})$.Now we can analyze the pointwise bias of the estimator.
\begin{theorem}
Under the assumptions of Proposition \ref{thm:asymptotic-expansion-L-1-up}, let $\delta=\mathcal{O}(\varepsilon^{3/2})$, then, for any fixed $x,y\in \M$, we have
\begin{equation}
 	\mathbb E\left[(\LL_1^\mathrm{up})_{[x,y]}\right] = \frac{2}{3}\varepsilon^{2}\frac{C_2}{C_0}\int_0^1 [\updelta \dd\vec{\upsilon}(\vec\gamma_\M(t))]^\top \vec\gamma_\M'(t) \dd t +\mathcal O(\varepsilon^{4})+\mathcal O(\varepsilon^{2}\delta n^{-1}).
	\label{eq:pointwise-bias-L1-up}
\end{equation}
where  we define $C_0 = \frac{1}{\varepsilon^d}\int_{\vz\in \rrr^d}\exp\left(-\frac{2|\vz\|^2}{\varepsilon^2}\right)\dd\vz$.
In the above, the expectation is taken over samples $\vec{X}$ of size $n$, to which the points $\vx,\vy$ are added.
\label{thm:pointwise-bias-L1-up}
\end{theorem}
\begin{proofsk}
The proof follows from the {\em Monte Carlo} approximation \cite{NadlerB.L.C+:06,BerryT.S:19} of the RHS of Lemma \ref{thm:discrete-form-L-up}, i.e.,
$$\mathbb E\left[\inv{n} \sum_z w_T(\vec x, \vec y, \vec z) f_{xyz}\right] = \int_\mathcal{M} w_T(\vec x, \vec y, \vec z) f_{xyz} \dd\mu(\vz).$$
Combining the result of Proposition \ref{thm:asymptotic-expansion-L-1-up}
(the $\varepsilon$ bias)
and the standard ratio estimator (the $n^{-1}$ bias) completes the proof.
The  proof details are in \ref{sec:proof-of-pointwise-bias-L1-up}.
\end{proofsk}

Note that the rate for $\varepsilon$ for the $\Delta_0$ estimator ($\LL_0$) is
slower than $n^{-1}$, see e.g., \cite{SingerA:06,BerryT.S:19}.
One can thus drop the $n^{-1}$ term in \eqref{eq:pointwise-bias-L1-up} using the
similar bandwidth parameter as the $\Delta_0$ estimator.
%
\paragraph{Pointwise convergence of the down Laplacian $\LL_1^{\rm down}$}
For the pointwise convergence of the down Helmholtzian, 
the following proposition shows the asymptotic expansion for the integral when n is large

\begin{proposition}
 \label{thm:asymptotic-expansion-L-1-down}
    If Assumptions \ref{assu:manifold-data-assu}--\ref{assu:kernel-assu} hold,
 $\vec{\upsilon}$ is of class $C^4(\M)$,
then
\begin{equation}
    \begin{split}
        \,&\pequal\,\int_{\vz\in\M} (w_E(\vy,\vz) f_{yz}-w_E(\vx,\vz) f_{xz}) \dd\vz \\
        &=\frac{8}{3}\varepsilon^{2d+2}C_0C_1\int_0^1 (\dd\updelta \vec{\upsilon})(\vec\gamma_\M'(t))^\top \vec\gamma_\M'(t) \dd t +\mathcal O(\varepsilon^{4+2d})+\mathcal O(\varepsilon^{1+2d}\delta^2),
    \end{split}
\end{equation}
where $\int_{\vz\in \rrr^d}\exp\left(-\frac{\|\vz-\vx\|^2}{2\varepsilon^2}\right)(\vz-\vx)_j^2\dd\vz=4\varepsilon^{d+2}C_1$.
\end{proposition}
With this proposition, now we can analyze the pointwise convergence of the down Helmholtzian

\begin{theorem}
\label{thm:L_1-down}
Under the assumptions of Proposition \ref{thm:asymptotic-expansion-L-1-down}, let $\delta=\mathcal{O}(\varepsilon^{3/2})$, then, for any fixed $x,y\in \M$, we have
\begin{equation}
	\mathbb E\left[(\LL_1^\mathrm{down})_{[x,y]}\right]= \frac{2}{3}\varepsilon^{2}\frac{C_1}{C_0}\int_0^1 [\dd\updelta \vec{\upsilon}(\vec\gamma_\M(t))]^\top \vec\gamma_\M'(t) \dd t \\
        + \mathcal O(\varepsilon^4)+\mathcal O(\varepsilon\delta^2)+\mathcal O(n^{-1}\varepsilon^2\delta).
	\label{eq:pointwise-bias-L1-down}
\end{equation}
In the above, the expectation is taken over samples $\vec{X}$ of size $n$, to which the points $\vx,\vy$ are added.

\end{theorem}
\begin{figure}[!htb]
\centering
\pgfdeclarelayer{nodelayer}
\pgfdeclarelayer{edgelayer}
\pgfsetlayers{nodelayer,edgelayer}   \usetikzlibrary{arrows.meta}

\begin{tikzpicture}
	\begin{pgfonlayer}{nodelayer}
		\node (0) at (-3, 1) {Down Helmholtzian $\vec{\mathcal L}_1^\mathrm{down}$};
		\node (1) at (-3, -1) {Graph Laplacian $\vec{\mathcal L}_0$};
		\node (2) at (3, 1) {$\Delta_1^\mathrm{down} = \mathsf{d}_0\updelta_1$};
		\node (3) at (3, -1) {$\Delta_0 = \updelta_1\mathsf{d}_0$};
		\node [anchor=west] (4) at (-3, 0) {
    \begin{tabular}{l}
        Lemma \ref{thm:spectral-dependency-agree-spectrum}
    \end{tabular}};
		\node [anchor=east] (5) at (3, 0) {
    \begin{tabular}{r}
        \cite{DodziukJ.M:95,ColboisB:06}
    \end{tabular}};
    \node [anchor=south] (6) at (0, -1) {Proposition \ref{thm:consistency-L0-with-w1}};
	\end{pgfonlayer}
	\begin{pgfonlayer}{edgelayer}
		\draw [-{Stealth[length=7pt]}] (0) to (1);
		\draw [dashed,line width=.8pt, {Stealth[length=7pt,open]}-{Stealth[length=7pt,open]}] (0) to (2);
		\draw [-{Stealth[length=7pt]}] (1) to (3);
		\draw [-{Stealth[length=7pt]}] (3) to (2);
	\end{pgfonlayer}
\end{tikzpicture}

\caption{Outline of the (spectral) consistency proof of $\LL_1^\mathrm{down}$.}
\label{fig:proof-flow-down-laplacian}
\end{figure}
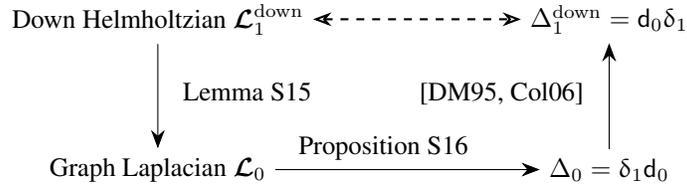

\paragraph{Spectral consistency of the down Laplacian $\LL_1^{\mathrm{down}}$}
The proof for spectral consistency of the 1 down-Laplacian is outlined in Figure
\ref{fig:proof-flow-down-laplacian}.
In short, by linking the
spectra/eigenforms of $\Delta_1^{\rm down}$ to $\Delta_0$ as well as their discrete counterparts
(two vertical arrows), one can show the consistency of $\LL_1^{\rm down}$ (horizontal dashed line)
using the known spectral convergence of the discrete graph Laplacian 
$\LL_0$ to the the Laplace-Beltrami operator $\Delta_0$
\cite{CoifmanRR.L:06,BerryT.S:19} (horizontal solid arrow). 
The details and proofs are in Supplement
\ref{sec:spectral-consistency-down-lapla}.

Here we have derived the continuous operator limits of the up and down Laplacian terms. We have shown that $\LL_1^{\rm down}$ converges to (a constant times) $\Delta_1^{\rm down}$ spectrally. For $\LL_1^{\rm up}$ and $\LL_1^{\rm down}$, we have shown that the pointwise limit exists, and that it equals $\Delta_1^{\rm up}$  and $\Delta_1^{\rm down}$ multiplied with a constant. 

Since the limits of $\LL_1^{\rm up}$ and $\LL_1^{\rm down}$ have different scalings, the estimator $\LL_1$ of $\Delta_1$ is a weighted sum of the two terms with coefficients $a,b$ as in \eqref{eq:weighted-laplacian}. 
From \eqref{eq:c1-c2}, we use $a = \inv{4}; b=1$ based on the pointwise limit,  and empirical observations on the synthetic datasets validate it. Please refer to Section \ref{sec:experiments} and Supplement \ref{sec:discussion-choice-of-a-b} for more details.
 \section{Related works and discussion}
\label{sec:related-works}
\paragraph{Consistency results of Laplace type operators on a manifold}
Numerous {\em non-linear dimensionality reduction} algorithms from point cloud data, e.g.,
\cite{CoifmanRR.L:06,belkin:07,TingD.H.J+:10},
investigated the consistency of functions ($0$-forms) on $\M$.
Spectral exterior calculus (SEC) \cite{BerryT.G:20}
extended the existing consistency results of
$0$-forms to $1$-forms
by building a frame (overcomplete set) to approximate the spectrum of $\Delta_1$
from $\LL_0$.
The SEC only has $\mathcal O(n^3)$ dependency in computing the 
eigenvectors of $\LL_0$. Therefore, it is well-suited for topological feature
discovery when large number of points are sampled from $\M$.
Nevertheless, the algorithm involves several fourth order {\em dense}
tensor computations with size $m$, the number of the 
eigenvectors of 1-Laplacian to estimate, and results in a
$\mathcal O(m^4)$ dependency in memory and $\mathcal O(m^6)$
in runtime. These dependencies may cause difficulties in applying SEC to
the edge flow learning scenarios in real datasets, since higher
frequency terms of the 1-Laplacian are oftentimes needed ($m\geq 100$).
On the other end, \cite{SingerA.W:11} studied the discrete approximation of the
{\em Connection Laplacian}, 
a differential operator acting on tensor
bundles of a manifold. This is intrinsically different from the $1$ Hodge Laplacian
we discussed.

\paragraph{Random walks on discrete $k$-Laplacian operator}
\cite{SchaubMT.B.H+:20} studied random walks on the edges of
normalized 1-Hodge Laplacian of the pre-determined graph.
For points sampled from $\M$, they proposed an {\em ad hoc}
hexagonal binning method to construct the $\scx_2$ from the trajectories;
theoretical aspects of the binning method were not discussed.
On the theoretical front, frameworks of random walks on simplices of down \cite{MukherjeeS.S:16}
and up \cite{ParzanchevskiO.R:17} $k$-Laplacian have also been visited.
These works focused on the connection between random walks on a simplicial complex and
spectral graph theory.
Our graph Helmholtzian $\LL_1^s$ based on pairwise triangle weights 
makes it possible to extend their frameworks to the point cloud datasets.

\paragraph{Persistent homology}
Persistent Homology (PH) theory enables us to study the topological
features of a point cloud in multiple spatial resolutions. The direct
application of PH is the estimation of the $k$-th Betti number,
i.e., the number of $k$ dimensional {\em holes}, from
$\vX\subseteq \M$.  PH algorithms applied to real data typically output large numbers of $k$-holes with low persistences. Therefore, one selects the statistically significant topological
features by some bootstrapping-based methods, e.g., the quantile of the {\em bottleneck
  distances} between estimated {\em persistent diagram} (PD) and the
bootstrapped PDs.
Readers are encouraged to refer to \cite{WassermanL:18,ChazalF.M:17} for more details. 
The PH theories are powerful in finding $\beta_k$ when $k\geq 2$;
in contrast, the Laplacian based methods are found effective in edge 
flow learning and smoothing for their abilities to keep track of the
orientations
(see Section \ref{sec:application-laplacian}).

\paragraph{Edge flow learning}
\cite{JiaJ.S.S+:19} proposed a graph based edge flow SSL algorithm for
 divergence free flows using the
unnormalized down-Laplacian $\vec{B}_1^\top\vec{B}_1$.
Another method, \cite{GodsilC.R:13}, transforms the edge data into
node space via line graph transformation; the problem is thus
turned into a vertex based SSL and solved by well studied tools, e.g.,
\cite{ZhuX.G.L+:03,BelkinM.N.S+:06}.
These algorithms assume the data is a graph, and therefore special techniques are needed
to convert the point clouds to graphs. Moreover, both methods 
are designed for flows that are (approximately) divergence free.
We provide experimental evaluations of these algorithms in Supplement
\ref{sec:urban-traffic-and-comp-jiaj}.

 \section{Applications of the constructed graph Helmholtzian}
\label{sec:application-laplacian}

\paragraph{$k$-th Betti number}
The spectrum of $\LL_1$ contains information about the manifold topology. 
Since $\vec{B}_k\vec{B}_{k+1} = 0$ \cite{LimLH:20},
$\img(\vec{B}_{k+1})$ is a subspace of $\ker(\vec{B}_k)$. One can
define the $k$-th homology vector space $\mathcal H_k$ as the space of $k$
dimensional simplices that are not the face of
any $k+1$-simplex in $\scx_\ell$. In mathematical terms, $\mathcal H_k
\coloneqq \ker(\vec{B}_k) / \img(\vec{B}_{k+1})$. The dimension of the
$k$-th homology space, or the number of the $k$ dimensional {\em
  holes}, is defined to be the $k$-th {\em Betti number} $\beta_k$.
It can be shown that $\beta_k$ is the dimension of the null space of
$k$-th Laplacian, i.e., $\beta_k = \dim(\ker(\vec{L}_k))$. Note that
$\beta_0 = \dim(\ker(\vec{L}_0))$, the number of zero eigenvalues of the
graph Laplacian, corresponds to the number of connected components in
the graph. Similarly, $\beta_1$, the dimension of the null space of the
Helmholtzian, represents the number of loops in $\scx_\ell$ (for $l\geq 2$).
The $k$-th Betti number
$\beta_k$ can also be obtained from the {\em random walk} $k$
Laplacian $\vec{\mathcal L}_k$ as $\beta_k =
\dim(\ker(\vec{\mathcal L}_k))$, or by persistent homology (PH) theories \cite{WassermanL:18,ChazalF.M:17}.

\paragraph{Edge flow smoothing and graph signal processing from point cloud data}
The graph Laplacian has long been used in node-based signal smoothing
on graphs \cite{OrtegaA.F.K+:18}.  \cite{SchaubMT.S:18} proposed
an edge flow smoothing algorithm using only the 
{\em unnormalized down Laplacian} $\vec L_1^\mathrm{down} =
\vec B_1^\top\vec B_1$.
These ideas apply naturally to the regularization with $\LL_1^s$ as follows. 
The smoothed flow $\hat{\vec{\omega}}$ is obtained by projecting
the noisy flow $\vec{\omega}$ to the low frequency eigenspace of $\LL_1^s$,
i.e., by solving the damped least squares problem 
\beq\label{eq:edge-flow-smooth}
\hat{\vec{\omega}} = \argmin_\vv \|\vv-\vec\omega\|^2 + \alpha \vv^\top\LL_1^s\vv.
\eeq
The optimal solution to the aforementioned minimization problem is
$\hat{\vec{\omega}} = (\vI_{n_1} + \alpha \LL_1^s)^{-1} \vec\omega$.
Note that the $\vL_1^\mathrm{down}$-based smoothing algorithm proposed by
\cite{SchaubMT.S:18} would fail to filter noise in the {\em curl} space,
for curl flows are in the space of $\ker(\vec L_1^\mathrm{down})$.
In contrast, the proposed algorithm can successfully smooth out the high 
frequency noise in either the {\em curl} or the {\em gradient} spaces
by using the weighted Laplacian $\LL_1^s$, which encodes the information from
both the up and down Laplacians.

\paragraph{Semi-supervised edge flow learning by 1-Laplacian regularizer}
Similar to the SSL by {\em Laplacian Regularized Least Squares} (LaplacianRLS) on $0$-forms (node-based SSL)
\cite{BelkinM.N.S+:06}, here we propose a framework
for SSL on discrete $1$-cochains (edge-based SSL).
Define the kernel between two edges to be $\mathcal K(e_i, e_j) = 1$ if $e_i, e_j$ share the same coface (triangle) or face (node),
and $0$ otherwise. Let $\mathfrak H$ be the corresponding reproducing kernel Hilbert space (RKHS).
Given the set of known edges $S\subset [n_1]$ (training set), we optimize over edge flows $g\in \mathfrak H$ with the loss function 
\begin{equation}
	\inv{|S|} \sum_{e\in S} \left(g(e) - \omega_e\right)^2 + \lambda_1 \|g\|^2_{\mathfrak H} +  \frac{\lambda_2}{n_1^2} g^\top \vec{\mathcal L}_1^s g.
	\label{eq:laprls-loss}
\end{equation}
From the {\em representer theorem}, the optimal solution is
$g^*(e) = \sum_{e'\in E} \alpha_{e'}^* \mathcal K(e', e)$, with
$\vec{\alpha}^* = \left(\diag\left(\vec{1}_S\right)\vec{\mathcal K} + \lambda_1 |S| \vec{I}_{n_1} + 
	\frac{\lambda_2 |S|}{n_1^2} \vec{\mathcal L}_k^s \vec{\mathcal K}  \right)^{-1} \vec{\omega}$,
where $\vec{\mathcal K}$ is a $n_1\times n_1$ kernel matrix with $[\vec{\mathcal K}]_{e_i, e_j} = \mathcal K(e_i, e_j)$.
A possible extension of the proposed $\LL_1$-RLS is to use HHD as in
\eqref{eq:hhd-equation}. More specifically,
the eigenspace of $\LL_1$ can be decomposed into two subspaces corresponding to
the gradient and curl spaces, respectively. One can weight differently the
importance between $\LL_1^\mathrm{up}$ and $\LL_1^\mathrm{down}$.
To achieve this, we apply a slight change to the second regularization term, i.e.,
$\frac{\lambda_2^\mathrm{up}}{n_1^2} g^\top\LL_1^{s, \mathrm{up}} g + \frac{\lambda_2^\mathrm{down}}{n_1^2} g^\top\LL_1^{s, \mathrm{down}} g$.
We call the proposed variant UpDownLaplacianRLS,
which will become to LaplacianRLS if
$\lambda_2^\mathrm{down} = \lambda_2^\mathrm{up}$.
Note that it is possible to extend other variants of node-based SSL algorithms
(e.g., the label propagation algorithm \cite{ZhuX.G:02} for classification) to
edge-based SSL; or introduce more powerful kernels $\mathcal K$ that capture the
orientations and similarities of edges.
Here we simply use manifold regularization regression
with a simple binary $\mathcal K$
to illustrate the effectiveness of the proposed Helmholtzian.

 \section{Experiments}
\label{sec:experiments}

We demonstrate the proposed manifold Helmholtzian construction and its
applications to edge flow learning (SSL and smoothing) on four synthetic datasets: 
{\em circle}, {\em torus}, {\em flat torus}, and {\em 2 dimensional strip}.
Additionally, we analyze several real datasets from chemistry, oceanography, and
RNA single cell sequencing. All datasets are described in Supplement
\ref{sec:datasets}.
We use $a = \inv{4}; b = 1$ everywhere, except in the experiments involving UpDownLaplacianRLS, where the scaling constants are set by cross validation.

\begin{figure}[!htb]
    \hfill
    \subfloat[][True vs. estimated $\lambda$'s of circle]
    {\includegraphics[height=2.7cm]{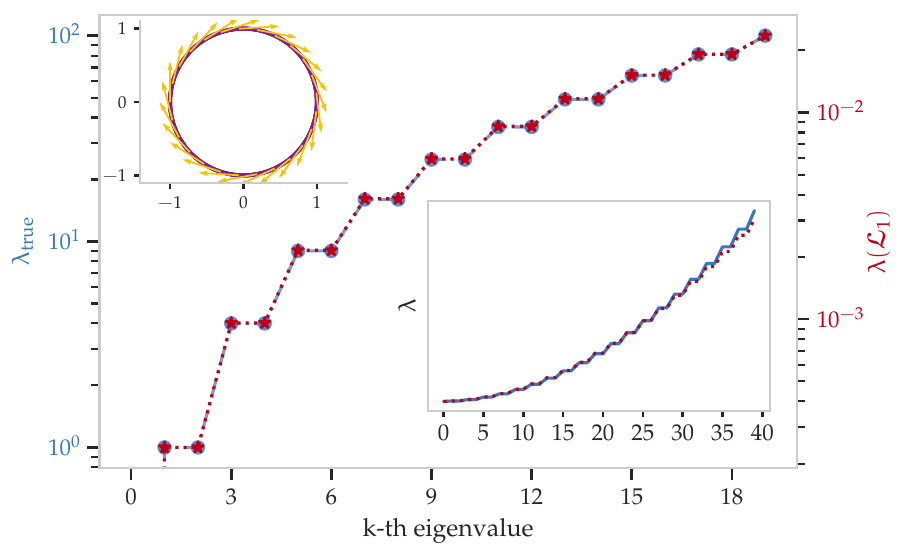}
    \label{fig:lambda-true-esti-circle}}\hfill
\subfloat[][Estimated $\lambda$'s of torus]
    {\includegraphics[height=2.7cm]{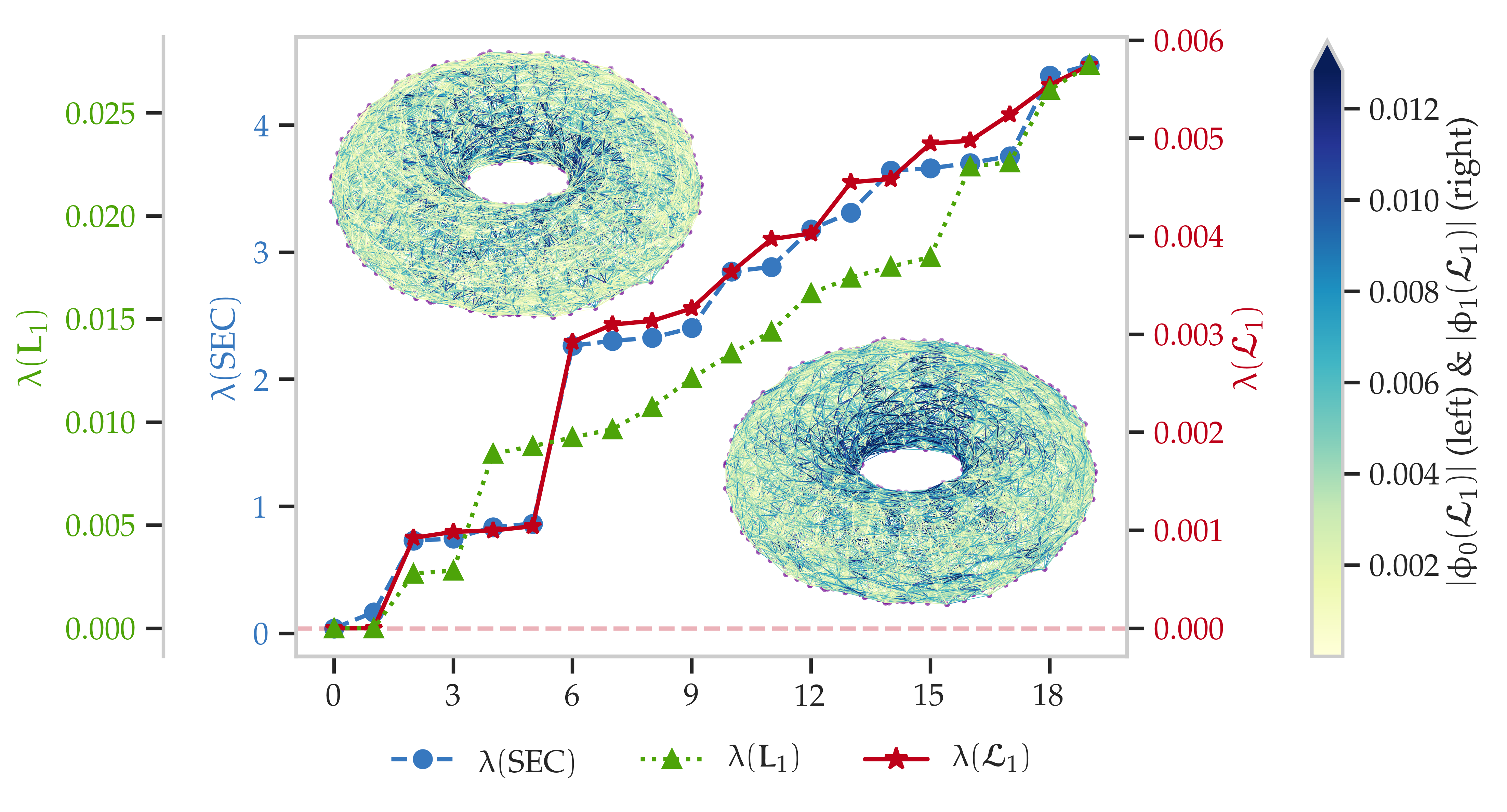}
    \label{fig:lambda-baselines-torus}}\hfill
\subfloat[][Estimated $\lambda$'s of a flat torus]
    {\includegraphics[height=2.7cm]{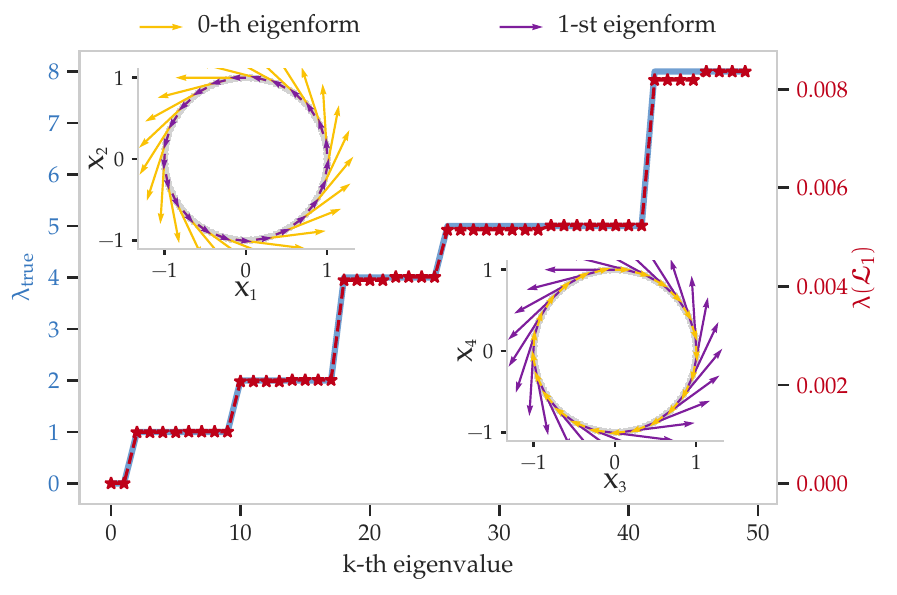}
    \label{fig:lambda-baselines-flat-torus}}\hfill

    \subfloat[][Estimated $\lambda$'s of ethanol dataset]
    {\includegraphics[height=2.85cm]{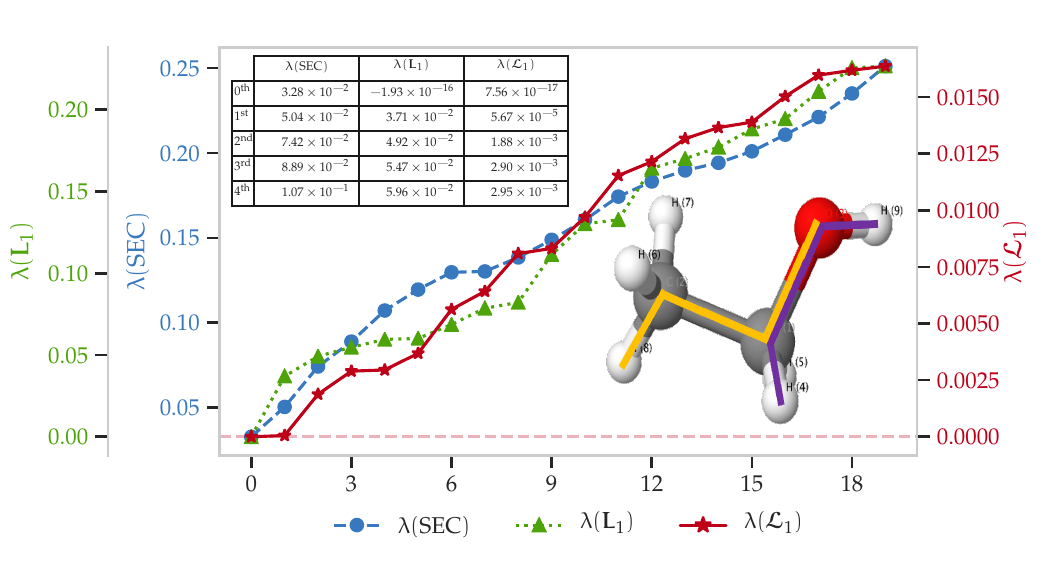}
    \label{fig:lambda-baselines-ethanol}}\hfill
\subfloat[][PD of ethanol]
    {\includegraphics[height=2.85cm]{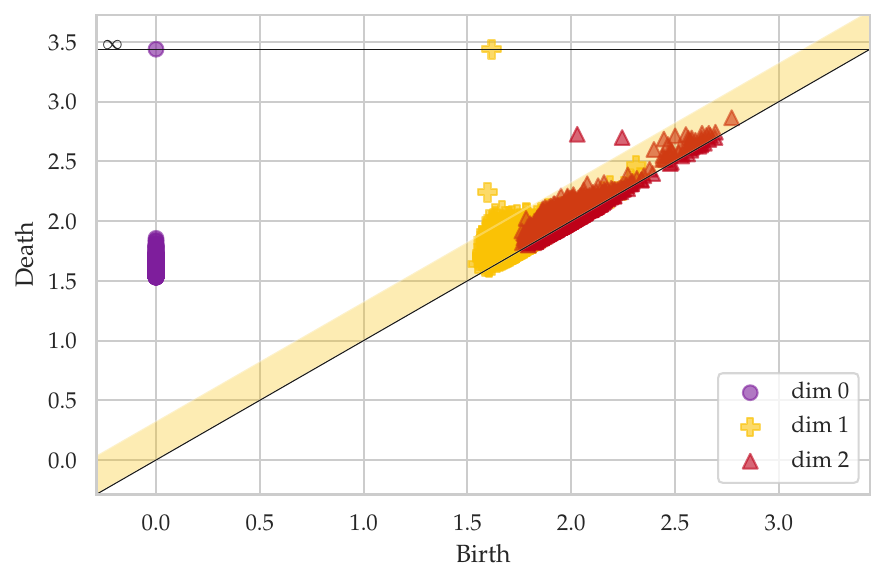}
    \label{fig:ethanol-persist-diag}}\hfill
\subfloat[][$\vec\phi_0,\vec\phi_1$ in the torsion space]
    {\includegraphics[height=2.85cm]{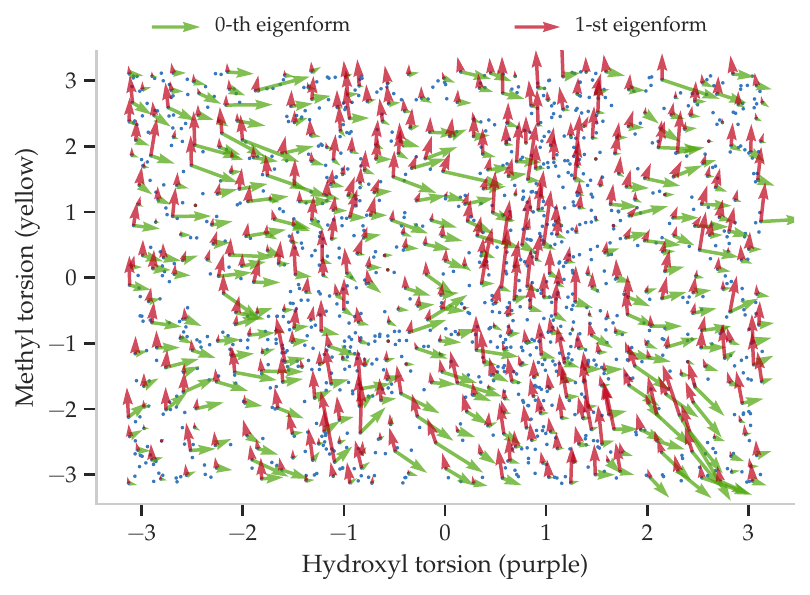}
    \label{fig:ethanol-first-two-eigenform-tau}}\hfill

    \subfloat[][Estimated $\lambda$'s of MDA dataset]
    {\includegraphics[height=2.9cm]{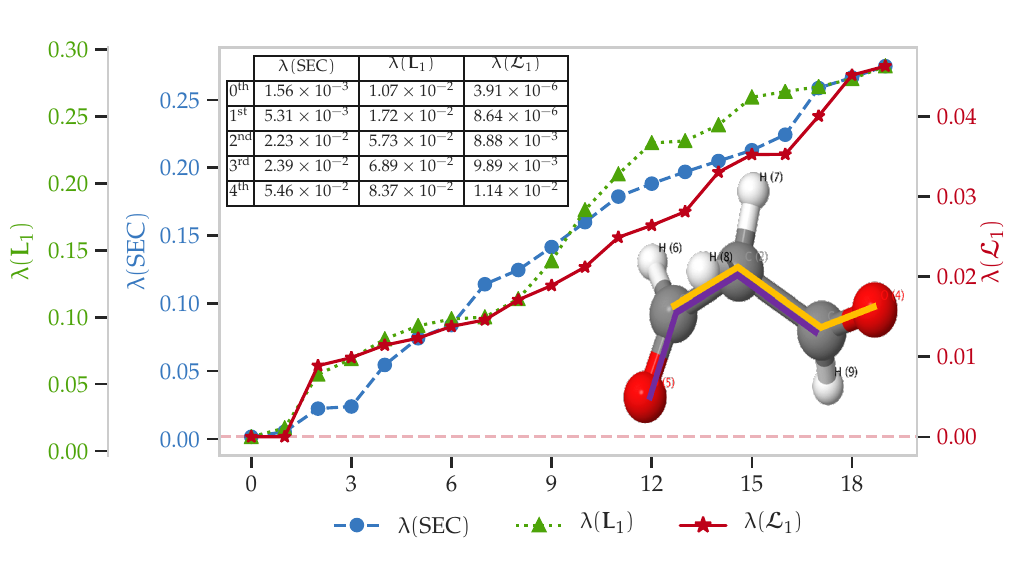}
    \label{fig:lambda-baselines-mda}}\hfill
\subfloat[][PD of MDA]
    {\includegraphics[height=2.9cm]{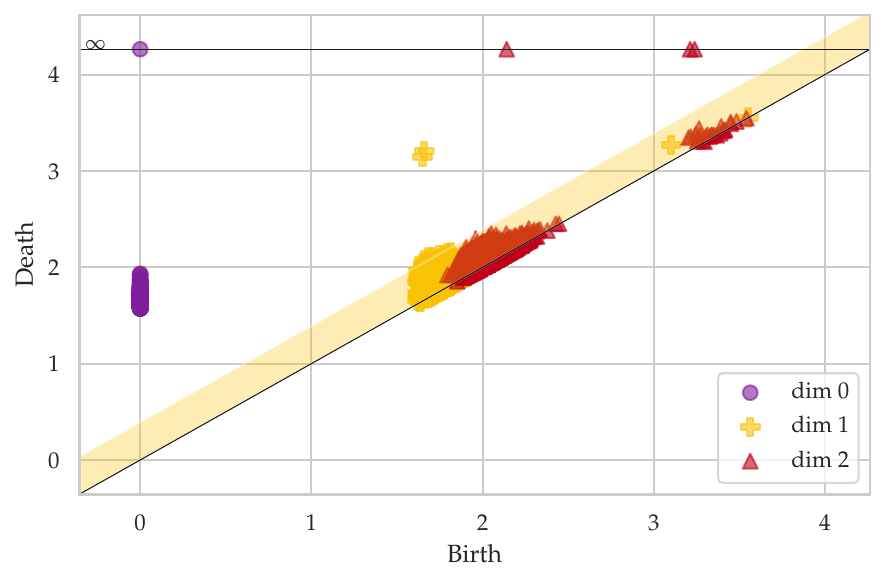}
    \label{fig:mda-persist-diag}}\hfill
\subfloat[][$\vec\phi_0,\vec\phi_1$ in the torsion space]
    {\includegraphics[height=2.9cm]{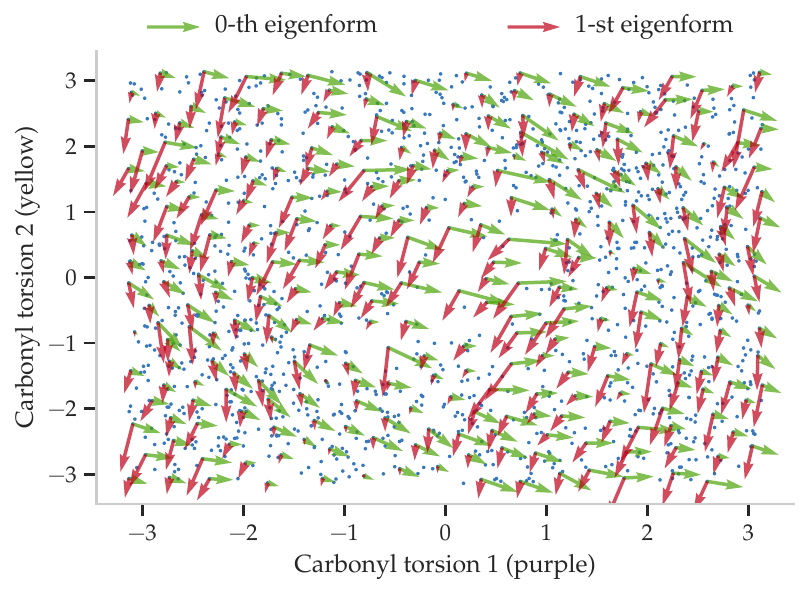}
    \label{fig:mda-first-two-eigenform-tau}}\hfill

    \caption{The first Betti number $\beta_1$ estimation for the synthetic manifolds (first row, left to right are unit circle, torus, and flat torus), ethanol (second row), and malondialdehyde (third row) datasets. The estimated harmonic eigenforms of the synthetic datasets can be found in the inset plots of (a--c). Readers are encouraged to zoom in on these plots for better views of the vector fields and cochains. For the second and the third row, subfigures from left to right correspond to the estimated $\lambda$'s from different methods, persistent diagram of the point cloud data, and the two harmonic flows in the torsion space, respectively.}
    \label{fig:dim-ker-L1-exp}
\end{figure}

\subsection{Dimension of loop space $\beta_1$}
For the first Betti number $\beta_1$, we report the eigenvalues of
$\Delta_1$ estimated by SEC \cite{BerryT.G:20} (in blue), unweighted
random walk Laplacian by letting $\vec{W}_2 = \vec I_{n_2}$ (green curve
in Figure \ref{fig:dim-ker-L1-exp}), and the proposed weighted
1-Laplacian $\vec{\mathcal L}_1$ (in red). Betti numbers can also be
estimated from a {\em Persistence Diagram (PD)}. We present the PD
with 95\% confidence interval estimated from 7,000 bootstrap samples
(see also \cite{WassermanL:18} and Algorithm \ref{alg:bootstrap-pd})
for two chemistry datasets, i.e., the ethanol and malondialdehyde (MDA)
data.  All experiments are replicated at least 5 times with very
similar results.
The eigenvalues of the circle $\Delta_1$ are $\lambda_k = \left(\lceil k / 2\rceil\right)^2$ for $k = 0, 1,\cdots$.
In Figure \ref{fig:lambda-true-esti-circle} we overlay the ground truth eigenvalues (blue) and
the estimated eigenvalues (red) in log scale. The zeroth eigenvalue of $\LL_1^s$ is close to zero, and therefore is
clipped from the plot.
The lower right inset plot shows the first 40 eigenvalues
in linear scale. The spectrum started to diverge when $k \approx 30$. The upper left inset plot is the
original data (purple) and the estimated vector field\footnote{To estimate a vector field from a 1 co-chain, one can solve a linear system as in \eqref{eq:linear-reverse-interpolation-edge-flow-damped} in Supplement
\ref{sec:vector-field-cochain-map} in the Least-Squares sense.}
 (yellow arrow) corresponding to the first eigenflow, which is a harmonic flow.
Figure \ref{fig:lambda-baselines-torus} shows the computed eigenvalues
of different algorithms on the synthetic torus dataset. A torus has two 1
dimensional loops. The first two eigenvalues of $\vec{L}_1$ or $\LL_1$ are close to zero, but the null space estimated by SEC has dimension 1. 
Inset
plots show the $\scx_2$ edges colored with the intensity of the first two
eigenflows of $\vec{\mathcal L}_1$, indicating that the
first eigenflow (upper left) is along the outer (bigger) circle
while the second eigenvector (lower right) belongs to the inner (smaller)
circle.
Figure \ref{fig:lambda-baselines-flat-torus} shows the first fifty estimated
spectrum (red) of the flat torus dataset overlaid with the ground truth (blue).
The first two eigenflows correspond to harmonic flows ($\beta_1 = 2$).
The spectrum of {\it flat torus} is piecewise constant, with half of the
eigenflows for each distinct eigenvalue corresponding to gradient and the other
half to curl flows. The two inset plots present the first two harmonic flows
in the original space, which are shown to be parameterizing two different loops
in the flat torus.

The second row of Figure \ref{fig:dim-ker-L1-exp} shows the experiment
on the ethanol dataset, whose ambient and (estimated) intrinsic dimension are
$D=102$ and $d=2$, respectively.
The dataset is known to be a noisy non-uniformly
sampled torus (see e.g., Figure \ref{fig:eth-pca-0}) with the
second (inner) loop difficult to detect due to an asymmetric topological structure.
The two harmonic eigenflows correspond to the
relative rotations of the Hydroxyl (purple) and Methyl (yellow)
functional groups as shown in the inset plot of
Figure \ref{fig:lambda-baselines-ethanol}.
To remove the non-uniformity effect,
we subsample $n = 1,500$ points that are farthest to each other.
However, as shown in Figure \ref{fig:eth-pca-0},
the non-uniformity sampling effect is still drastic.
For this
dataset, only the spectrum of the proposed Helmholtzian $\vec{\mathcal
  L}_1$ captures the correct $\beta_1=2$, which can be confirmed
in Figure
\ref{fig:lambda-baselines-ethanol}
(see Supplement \ref{sec:exp-detail-ethanol} for the
discussion on the first 10 eigenflows); by contrast, SEC, $\mathbf L_1$ and the PD
estimate $\beta_1$ to 1, 1, and more than 2, respectively.
Due to the proposed triangle weight as in \eqref{eq:vr-kernel-general-form},
one can successfully remove the topological noise while preserving the true
signal (smaller inner loop) in the weighted $\LL_1^s$.
Without the weighting function with exponential decay, the SEC and $\vL_1$ fail to detect
the second inner loop thus reporting $\beta_1 = 1$.
Apart from $\beta_1$, one can also obtain estimates of the two harmonic eigenflows
from the first two eigenvectors of $\LL_1^s$. These two harmonic flows correspond to the two independent loops of this manifold. The first two eigenforms estimated from
the eigenvectors cochain can be found in Figures \ref{fig:eth-pca-0} and
\ref{fig:eth-pca-1} in Supplement. With our prior knowledge that the purple
and yellow rotors parametrize the loops in the dataset, we map the two eigenforms
that reside in the PCA space $\vX$ to the torsion space using
\eqref{eq:vec-field-mapping-jacobian} (see more discussions in Supplement
\ref{sec:vec-field-mapping}) as illustrated in Figure
\ref{fig:ethanol-first-two-eigenform-tau}. As clearly shown in the figure, the
zeroth eigenform (green) aligns with the direction of increase of the Hydroxyl rotor
(purple in the inset of \ref{fig:lambda-baselines-ethanol}), while the first
eigenform (red) matches perfectly with the derivative of Methyl rotor (yellow
in the inset of \ref{fig:lambda-baselines-ethanol}). Note also that one can clearly see
the non-uniform sampling effect in Figure \ref{fig:ethanol-first-two-eigenform-tau}.
More specifically, more points are sampled when Hydroxyl torsion value
is around -1 or 1 compared to other values.

The third row of Figure \ref{fig:dim-ker-L1-exp} shows the $\beta_1$
estimation result on the MDA dataset. This dataset has similar
topological structure as the ethanol dataset; that is to say, they are
both non-uniformly sampled tori. The two loops are parametrized by
the two Carbonyl bond torsions as illustrated in the inset plot of
Figure \ref{fig:lambda-baselines-mda}.  Compared to the ethanol
dataset, the MDA dataset is easier in a sense that the two loops are
more symmetric to each other. However, this dataset is harder to
visualize for the torus is embedded in a 4 dimensional space.  A clear
separation between the zeroth and the first eigenvalues of $\LL_1$ can be
seen in the estimated spectrum (Figure
\ref{fig:lambda-baselines-ethanol}, see also the inset table) with the help of
triangle weights \eqref{eq:vr-kernel-general-form}.  Even
though the estimated dimensions of the null space of SEC and $\vL_1$ are both two,
we do not observe such clear gaps between the first two estimated
eigenvalues compared to that of $\LL_1$.  The bootstrapped PD with
95\% confidence interval shows that there are at least two
loops. However, statistically significant loops generated by the topological
noise, which are close to the diagonal of the PD, are still visible.
Similar to the ethanol dataset, we map the first two eigenforms to the two
Carbonyl torsion space, as shown in Figure \ref{fig:mda-first-two-eigenform-tau}.
It confirms our prior knowledge that these two eigenforms parametrize the
yellow and the purple torsions.

\begin{figure}[!htb]
    \subfloat[][]
    {\includegraphics[height=2.25cm]{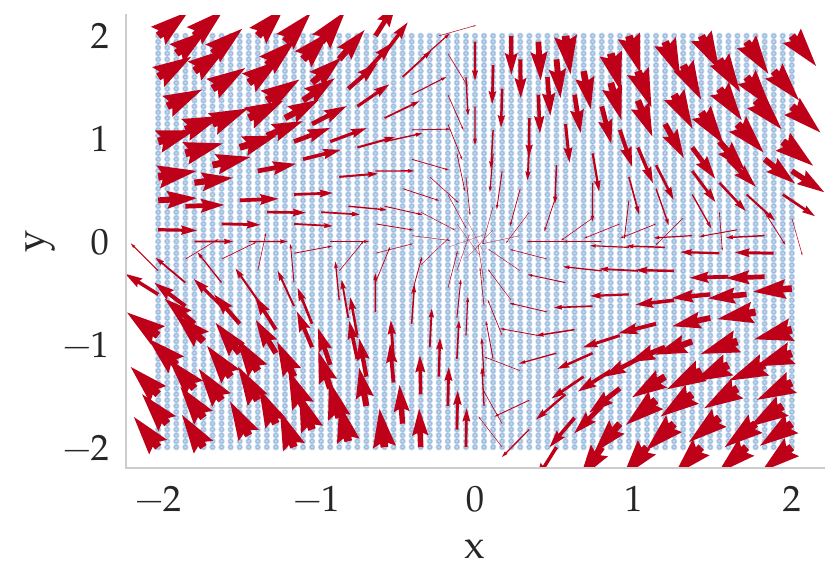}
    \label{fig:synth-field-2D-strip}}\hfill
\subfloat[][]
    {\includegraphics[height=2.25cm]{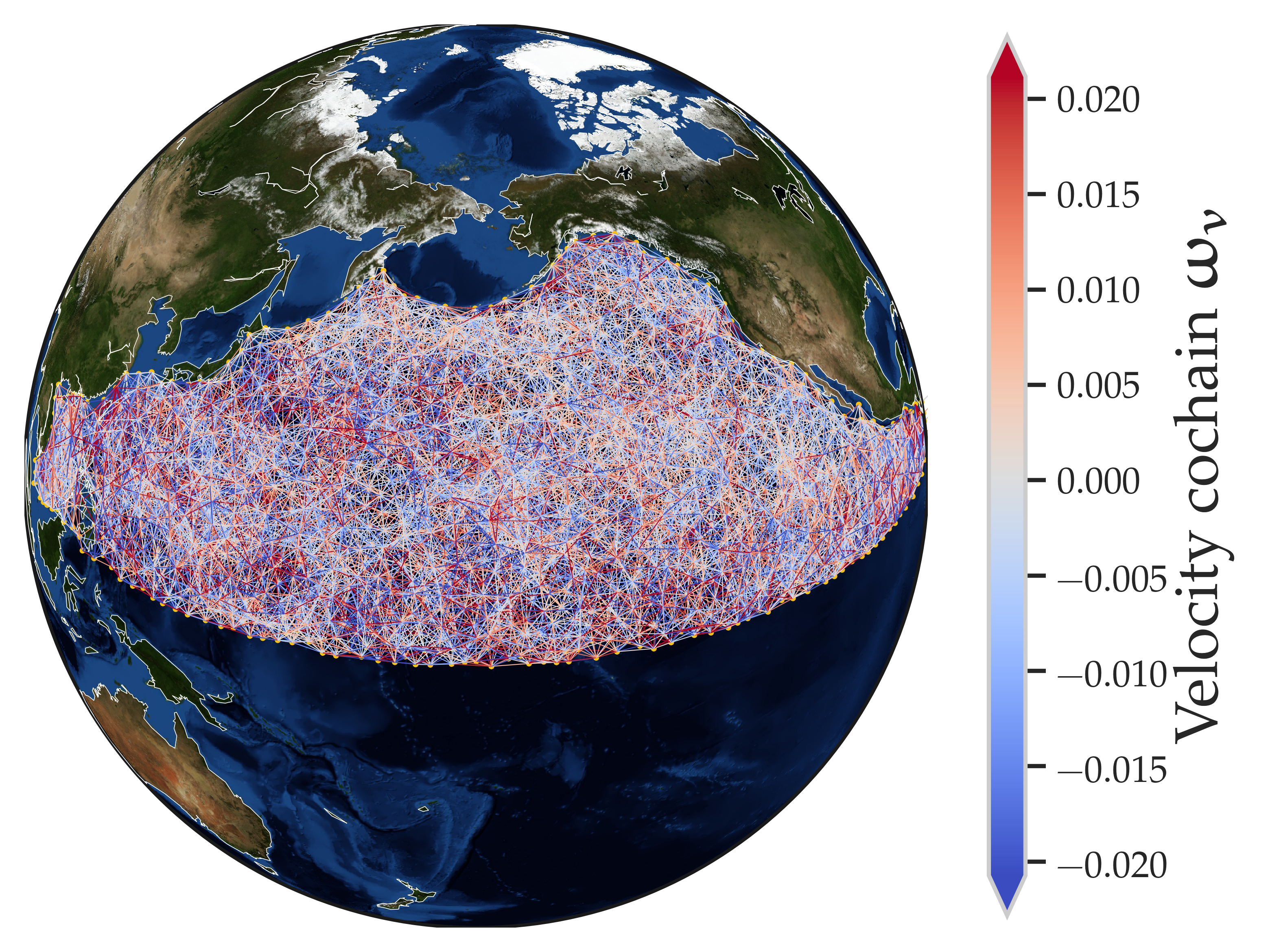}
    \label{fig:buoys-velocity-field}}\hfill
\subfloat[][]
    {\includegraphics[height=2.25cm]{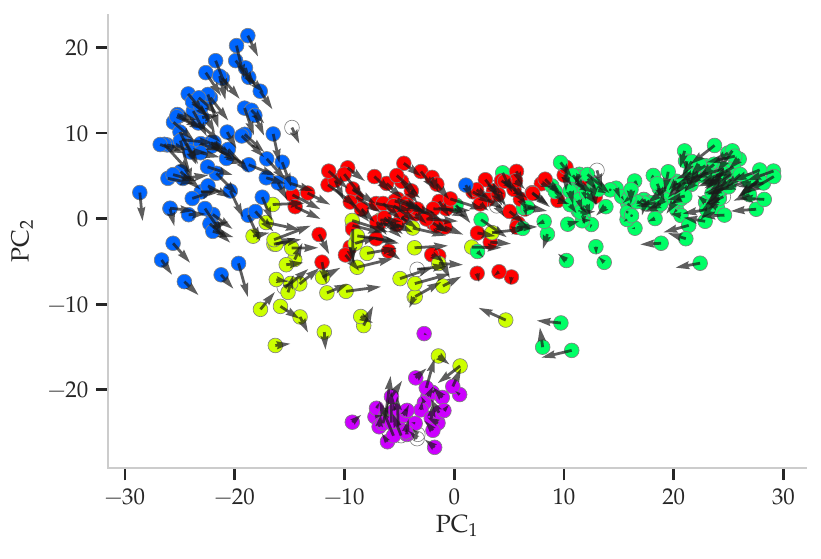}
    \label{fig:rna-chromaffin-velocity-field}}\hfill
\subfloat[][]
    {\includegraphics[height=2.25cm]{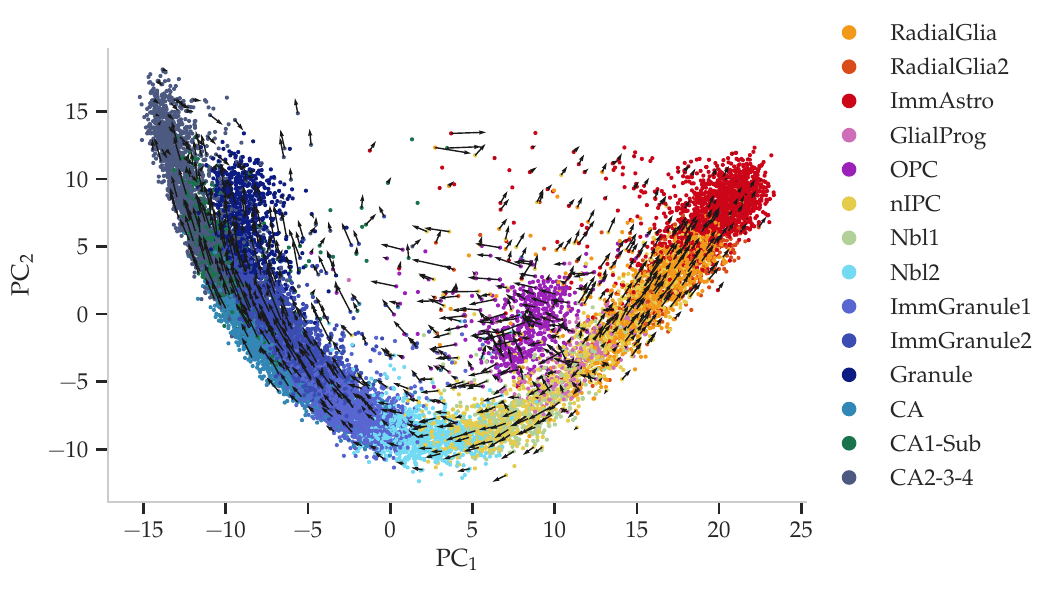}
    \label{fig:rna-hippocampus-500-velocity-field}}\hfill

    \subfloat[][]
    {\includegraphics[height=2.5cm]{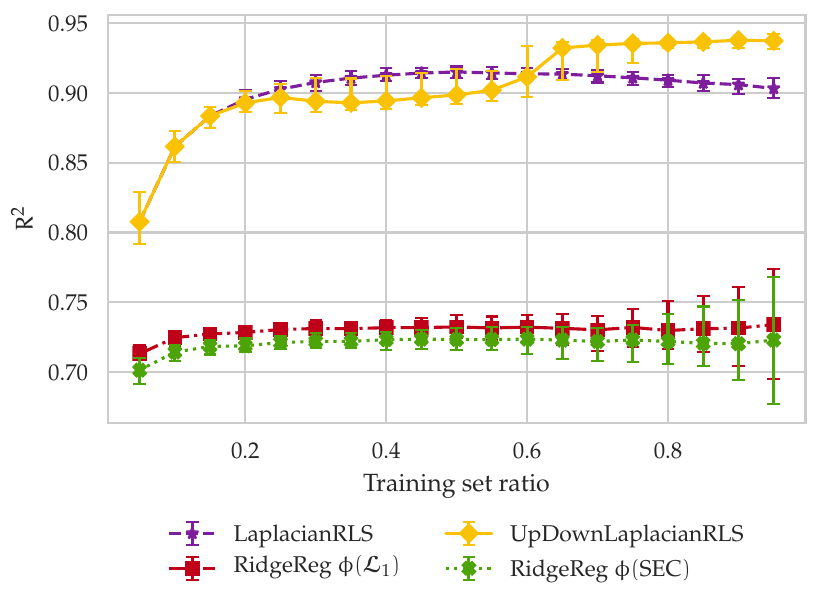}
    \label{fig:ssl-r2-2D-strip}}\hfill
\subfloat[][]
    {\includegraphics[height=2.5cm]{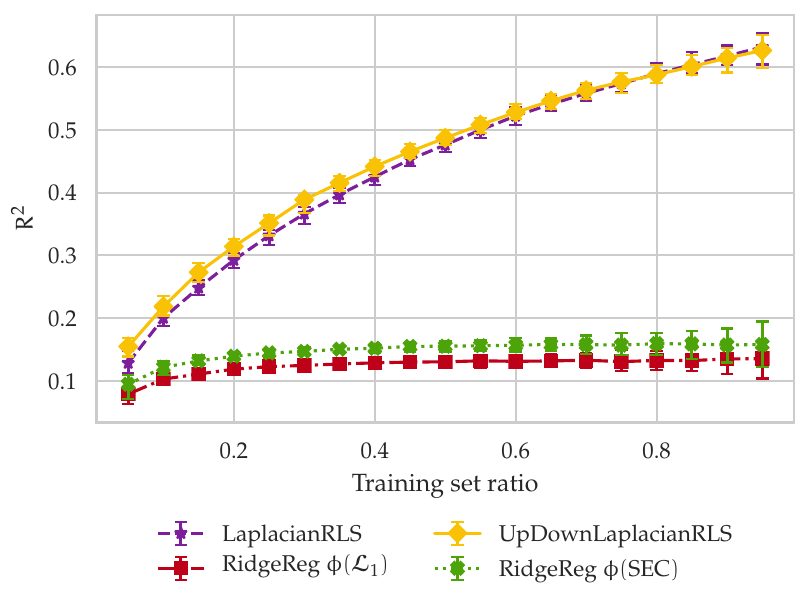}
    \label{fig:ssl-r2-buoys}}\hfill
\subfloat[][]
    {\includegraphics[height=2.5cm]{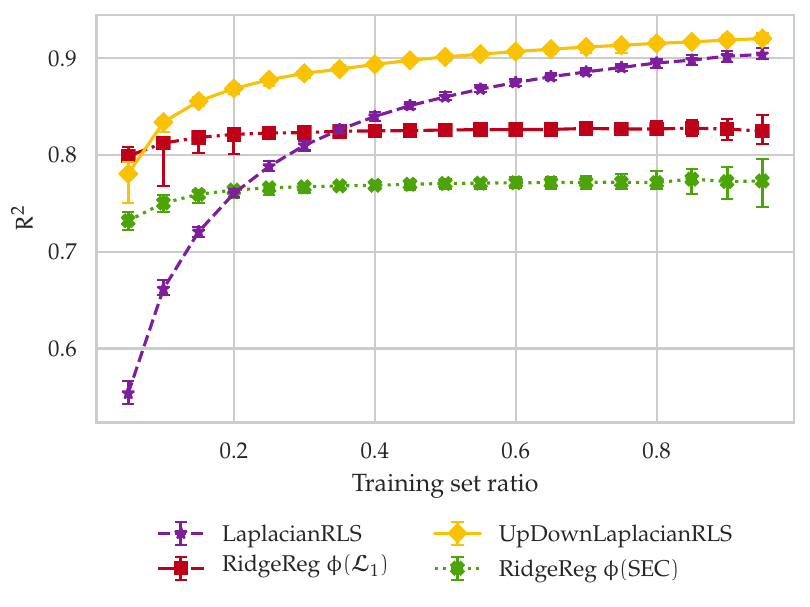}
    \label{fig:ssl-r2-chromaffin}}\hfill
\subfloat[][]
    {\includegraphics[height=2.5cm]{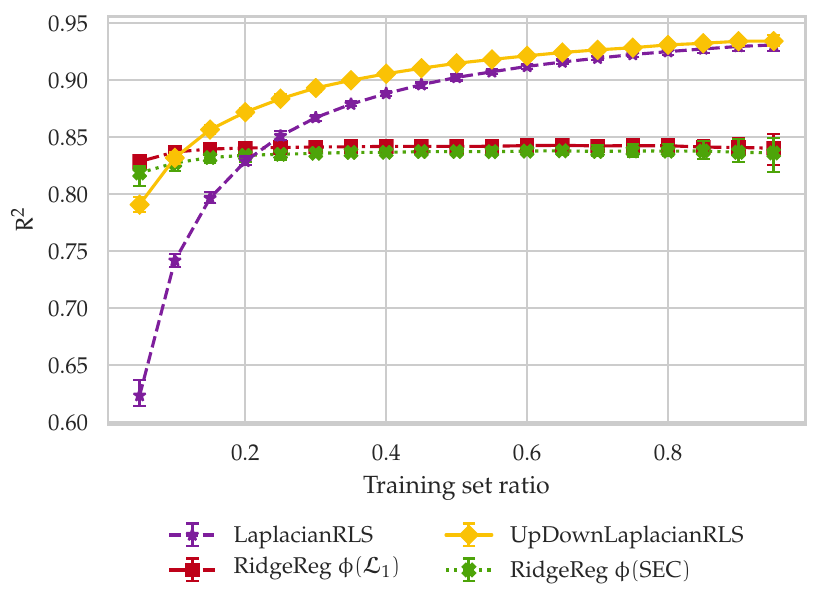}
    \label{fig:ssl-r2-hippocampus-k500}}\hfill

    \caption{Edge flow SSL results for (columns from left to right) synthetic field on 2D strip, velocity field of ocean buoy, RNA velocity field of the chromaffin cells, and RNA velocity of the mouse hippocampus cell differentiation dataset. The first row is the velocity field/1-cochain overlaid on the original point cloud data, while the second row represents the $R^2$ score of different edge flow SSL algorithms.}
    \label{fig:ssl-result}
\end{figure}

\subsection{Edge flow prediction by semi-supervised learning}
For each of the datasets in Figure \ref{fig:ssl-result}, we construct the
$\LL_1^{\rm up}$, $\LL_1^{\rm down}$, and the symmetrized Laplacian $\LL_1^s$
as described in Section \ref{sec:application-laplacian}.
We then predict the flows on a
fraction of the edges (test set) from the flows of the remaining edges
(training set), with edges randomly split into train and test set, for
train set ratio ranges from $0.05$ to $0.95$. We report the coefficient of 
determination, $R^2$, as our performance metrics. The hyperparameters (all in range
$[10^{-5}, 10^5]$) of the LaplacianRLS (in purple) and
UpDownLaplacianRLS (in yellow) on the 2D strip and ocean datasets are
selected by a 5-fold cross validation (CV) for each train/test
split. To reduce the computation time, the hyperparameters for the
larger datasets (RNA velocity datasets) are chosen by a 5-fold CV when
the train set ratio is $0.2$, and are used for all the other train sample
sizes.  Two other baselines are ridge regression on the first 100
eigenvectors estimated by $\LL_1^s$ (in red) and by SEC (in green).
The $\ell_2$ regularization parameter (in range
$[10^{-5}, 10^5]$) for these two baselines is chosen by a 5-fold CV for
different train set size.  The experiments are repeated fifty times
and we report the median $R^2$ value.  The lower and upper error bars
correspond to the 5th and 95th percentile of the $R^2$ values for
different train/test split, respectively.

Figure \ref{fig:ssl-r2-2D-strip} shows the results of predicting the
simple synthetic field shown in Figure \ref{fig:synth-field-2D-strip},
which composed of 70\% curl flow and 30\% gradient flow. The results of
ridge regression on the low frequency eigenvectors indicate that 
100 eigenvectors estimated by either $\LL_1^s$ (red) or SEC (green) are not enough
even for predicting a simple field in Figure
\ref{fig:synth-field-2D-strip}.

The first real dataset we used contains ocean buoy trajectories across the globe.\footnote{Data from the AOML/NOAA Drifter Data Assembly \url{http://www.aoml.noaa.gov/envids/gld/}} The ambient dimension of the data is $D = 3$ (earth surface)
and intrinsic dimension is $d = 2$.
In this paper we subsample $n=1,500$ farthest buoys located in the North Pacific ocean.
The $\scx_2$ is constructed in the {\em earth-centered, earth-fixed} (ECEF) coordinate
system. Supplement \ref{sec:exp-detail-ocean} has more details about the data and how to
preprocess the edge flow. Figure \ref{fig:buoys-velocity-field} shows
the constructed $\scx_2$ of the buoys, with edges colored by the
velocity cochain $\vec \omega$.
Figure \ref{fig:ssl-r2-buoys} reports the $R^2$ scores of the edge flow
prediction. As clearly shown in the plot, higher frequency terms are
needed to successfully predict the edge flow. However, this is clearly
infeasible using the SEC approach as discussed in Section \ref{sec:related-works}.

Next, we investigate the edge flow SSL on the RNA single cell
sequencing manifold equipped with the RNA velocity
\cite{LaMannoG.S.Z+:18}\footnote{All RNA datasets are from
  \url{http://velocyto.org} (with preprocessing codes)}, as shown in
the third and the fourth column of Figure \ref{fig:ssl-result}. These
datasets have non-trivial manifold structures, which make the SSL problem
more challenging.  The Chromaffin cell differentiation dataset has $n
= 384$ and $D = 5$, while the mouse hippocampus dataset has $18,140$
cells in total and $D = 10$ in the PCA space.  We subsample the
farest $n=800$ cells in the mouse hippocampus while using all the
cells in the chromaffin dataset. The RNA velocity fields of the Chromaffin
and the mouse hippocampus datasets in the first
two principal components are presented in Figure
\ref{fig:rna-chromaffin-velocity-field} and
\ref{fig:rna-hippocampus-500-velocity-field}, respectively.  As
expected, the LaplacianRLS and UpDownLaplacianRLS algorithms
outperform the SSL algorithms using only low frequency terms of the
estimated $\Delta_1$.  Note also that compared with the SSL in the
simple manifold structure in the first two columns, the
UpDownLaplacianRLS for the RNA velocity data has more performance
gains when the training set sizes are small.

\subsection{Inverse problem: estimate the underlying velocity field from the trajectory data using SSL}
\begin{figure}[!htb]
    \subfloat[][]
    {\includegraphics[width=.35\linewidth]{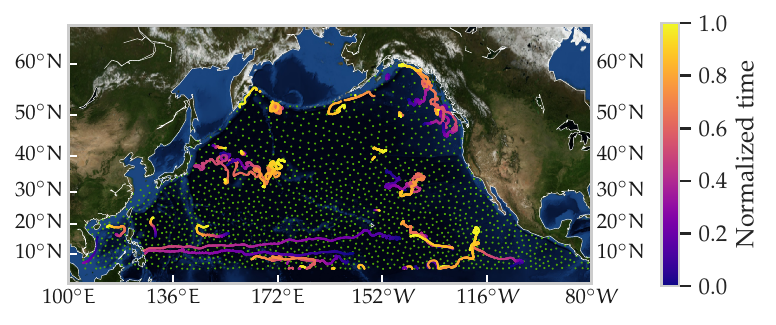}
    \label{fig:inv-prob-ocean-traj}}\hfill
\subfloat[][]
    {\includegraphics[width=.3\linewidth]{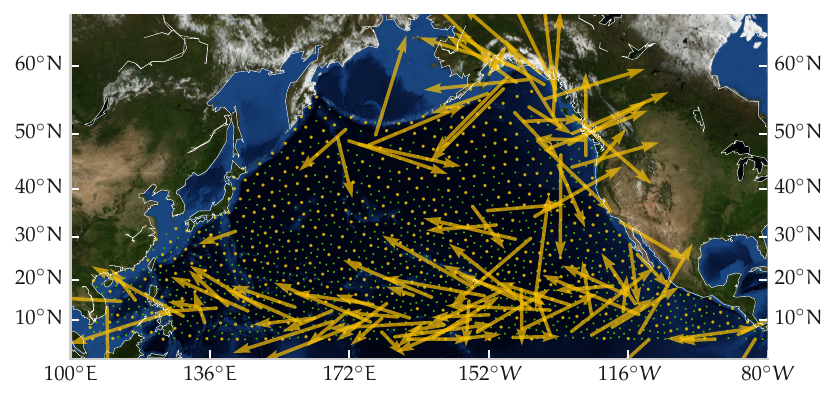}
    \label{fig:inv-prob-ocean-vec-count}}\hfill
\subfloat[][]
    {\includegraphics[width=.3\linewidth]{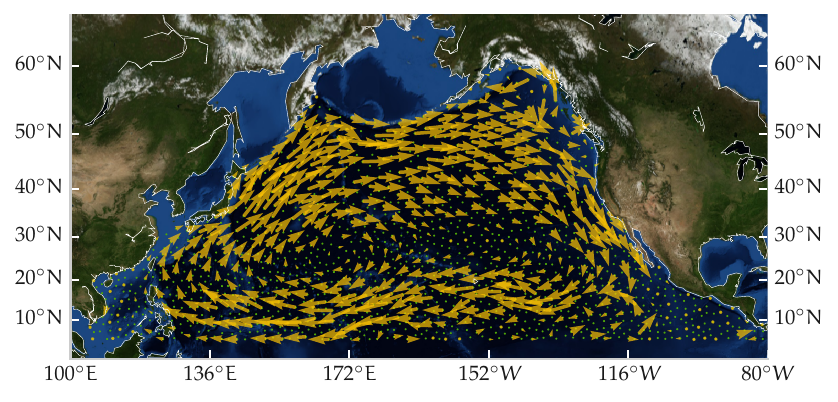}
    \label{fig:inv-prob-ocean-vec-ssl}}\hfill

    \subfloat[][]
    {\includegraphics[width=.32\linewidth]{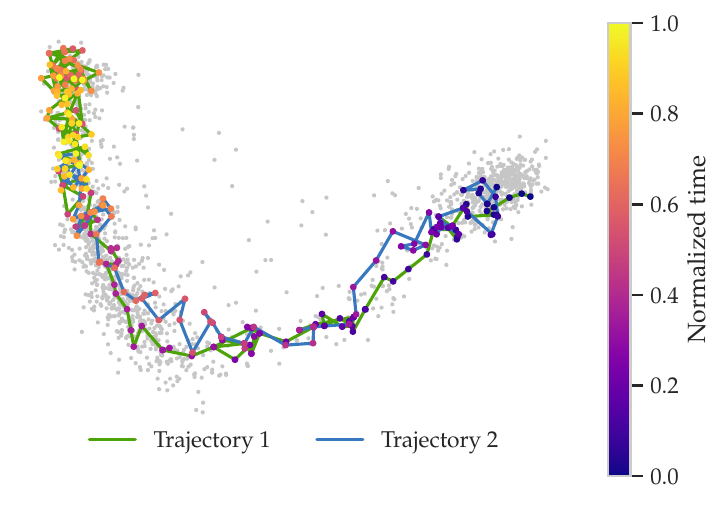}
    \label{fig:inv-prob-rna-emb-traj}}\hfill
\subfloat[][]
    {\includegraphics[width=.31\linewidth]{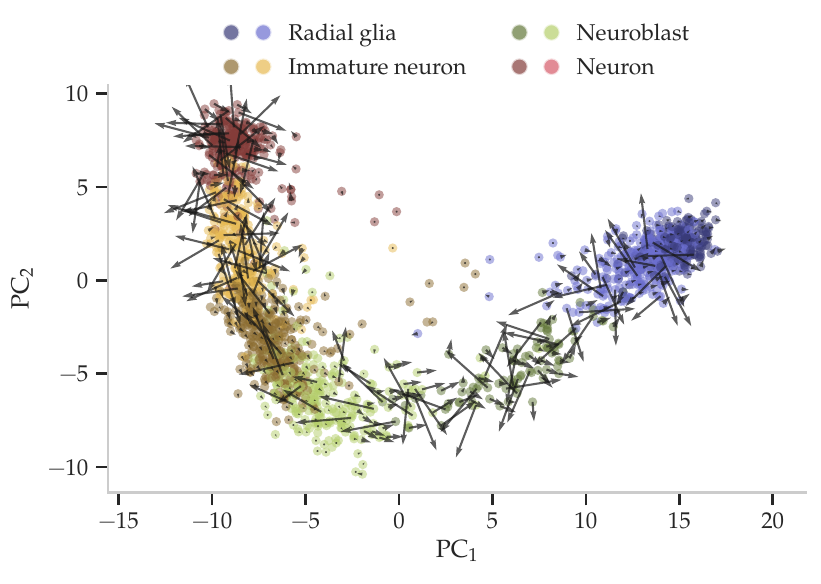}
    \label{fig:inv-prob-rna-emb-vec-count}}\hfill
\subfloat[][]
    {\includegraphics[width=.31\linewidth]{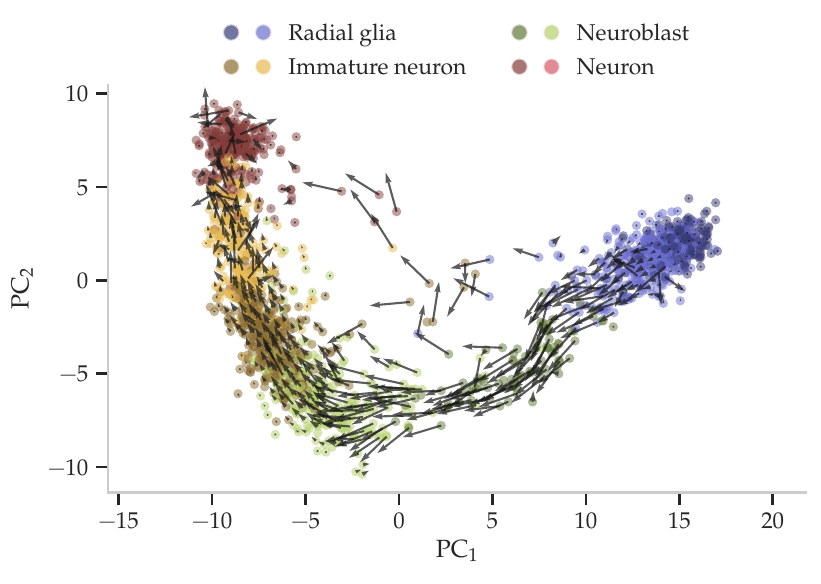}
    \label{fig:inv-prob-rna-emb-vec-ssl}}\hfill
    \caption{Results of estimating underlying velocity field from the partially observed trajectories. The first and second row are the ocean and the human glutamatergic neuron cell differentiation dataset, respectively. Columns from left to right present the observed trajectories, the estimated velocity field from a zero-padded 1-cochain, and the estimated field from the SSL interpolated 1-cochain.}
    \label{fig:inv-prob}
\end{figure}

One application of the Edge flow SSL is to estimate the underlying
velocity field 
given a sparse set of observations (trajectories).
A trajectory can be thought of as a partially
observed 1-cochain with the value of $e = (i, j)$ constructed by
counting number of times the trajectory goes from node $i$ to node $j$
(counted as negative if pass from $j$ to $i$).
We use the UpDownLaplacianRLS algorithm described in Section
\ref{sec:application-laplacian} and show the estimated fields from the
interpolated 1-cochains in Figure \ref{fig:inv-prob}.
The parameters of the SSL algorithm are chosen using a 5-fold CV when test
set ratio is
0.6 of all the observed edges.
For comparison, we also present the velocity fields estimated by the zero-padded
1-cochains.
The ocean
drifter data themselves are trajectories, we sampled 20 trajectories
from the dataset as shown in Figure \ref{fig:inv-prob-ocean-traj}.
Since we subsample the furtherest $n=1,500$ points from the original dataset
to construct the $\scx_2 = (V, E, T)$,
the trajectories will contain points that are not in the vertex set $V$.
To address this issue, one can treat the vertex set $V$ as landmarks of the
original dataset and map each point in the sampled trajectories to its
nearest point in $V$.
The 1-cochain is then constructed from the mapped trajectories.
The figure shows that we can obtain a highly interpretable velocity field that corresponds to the
North pacific Gyre as in Figure \ref{fig:inv-prob-ocean-vec-ssl},
compared with the estimated velocity field from the zero-padded
1-cochain (Figure
\ref{fig:inv-prob-ocean-vec-count}). The algorithm is surprisingly
powerful in the sense that the sampled trajectories do not even cover
the west-traveling buoys at 40$^\circ$ N nor the south-bounding
drifters near the west coast of the U.S.

The human glutamatergic
neuron cell differentiation dataset do not have the temporal
information. To illustrate our method,
we estimate the transition probability matrix of a Markov chain from
the RNA velocity field, which is
computed from a $550$-nearest neighbor (NN) graph. We then
sampled two random walk trajectories
(Figure \ref{fig:inv-prob-rna-emb-traj}) from the constructed Markov chain
on the $550$-NN graph.
Note that the $k$-NN graph above, which is used to estimate the
RNA velocity field, can be different from the 1-skeleton of the
$\scx_2$ constructed in Algorithm \ref{alg:manifold-helmholtzian-learning}.
The estimated field using the UpDownLaplacianRLS in Figure
\ref{fig:inv-prob-ocean-vec-ssl} shows a smooth velocity field
compared with the field estimated from the zero-padded 1-cochain in Figure
\ref{fig:inv-prob-ocean-vec-count}.

\subsection{Edge flow smoothing}
\begin{figure}[!htb]
    \subfloat[][Original vector field]
    {\includegraphics[width=.48\linewidth]{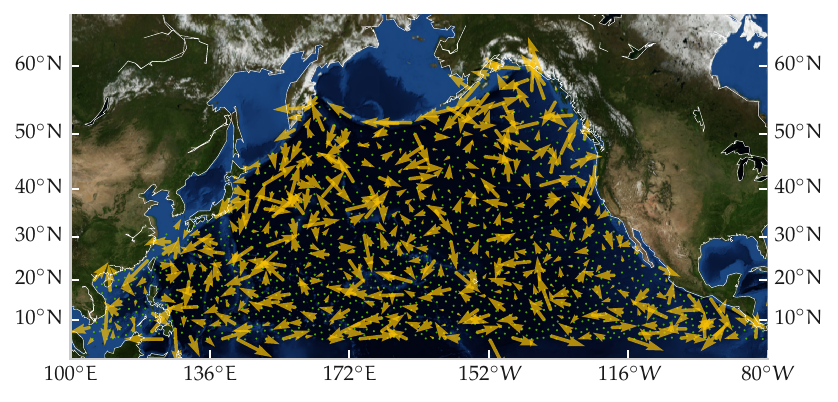}
    \label{fig:orig-vec-field}}\hfill
\subfloat[][$\alpha = 5$]
    {\includegraphics[width=.48\linewidth]{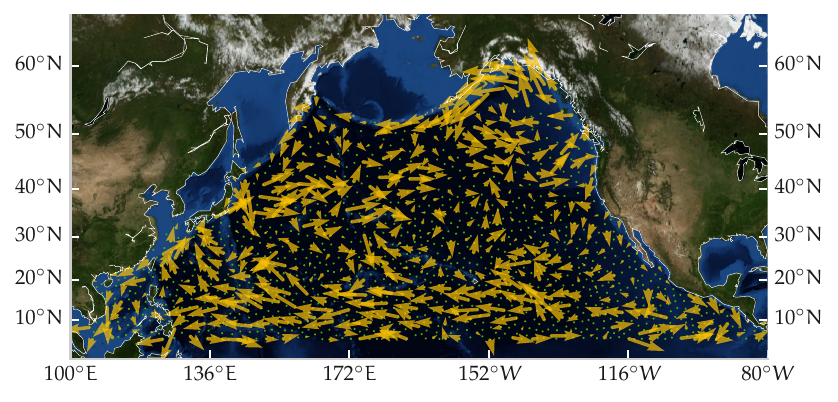}
    \label{fig:smoothed-vec-field-alpha-5}}\hfill

    \subfloat[][$\alpha = 50$]
    {\includegraphics[width=.48\linewidth]{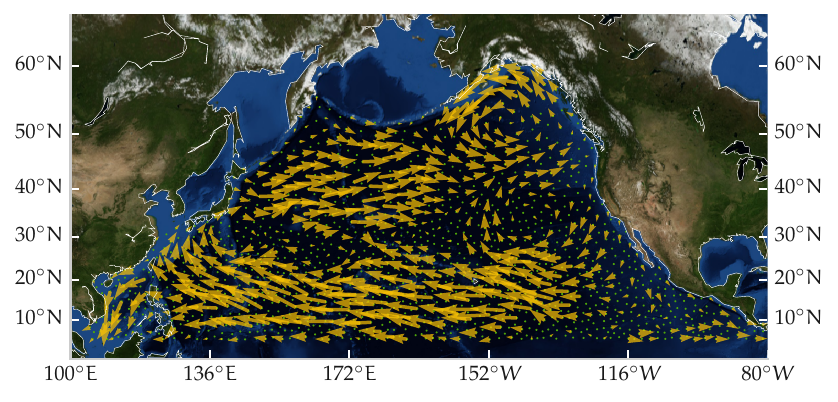}
    \label{fig:smoothed-vec-field-alpha-50}}\hfill
\subfloat[][$\alpha = 500$]
    {\includegraphics[width=.48\linewidth]{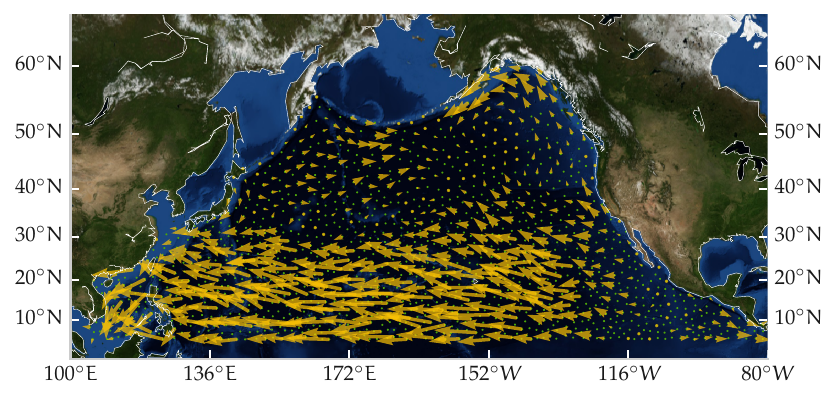}
    \label{fig:smoothed-vec-field-alpha-500}}\hfill

    \caption{Edge flow smoothing on the ocean buoys velocity field dataset with different smoothing constant $\alpha$.}
    \label{fig:gsp-result}
\end{figure}

Figure \ref{fig:gsp-result} shows the result $\hat{\vec\omega}$ of edge
flow smoothing presented in \eqref{eq:edge-flow-smooth}. 
Figure \ref{fig:orig-vec-field} and
\ref{fig:smoothed-vec-field-alpha-5}--\ref{fig:smoothed-vec-field-alpha-500}
show the original and the smoothed vector fields with different smoothing parameters
$\alpha$, respectively.
Several patterns, e.g., North Equatorial and Kuroshio currents, are visible in the original velocity
field. However, a well known North Pacific current at $40^\circ$N is not as apparent as the aforementioned
currents. By contrast, the smoothed flow with $\alpha=50$ makes the North
Pacific Gyre visible. The model with $\alpha=5$ in Figure
\ref{fig:smoothed-vec-field-alpha-5} corresponds to the case of ``under-smoothing'',
while that with $\alpha=500$ represents the case of ``over-smoothing''.
Note that the vector fields are plotted on the Mercator projection purely for
visualization purposes. The $\scx_2$ are constructed on the ECEF system
similar as before.

\section{Conclusion}
The main contribution of the paper is to \romanenu{1}  propose an estimator of the
Helmholtzian $\Delta_1$ of a manifold in the form of a weighted 1-Laplacian of a
two dimensional simplicial complex $\scx_2$, whose vertex set includes points sampled
from a manifold.
With the proposed kernel function for triangles, which is a
core part of the construction of this estimator, \romanenu{2} we further derive
(Section \ref{sec:consistency-analysis}) the infinite sample limit of 1 up-Laplacian 
$\LL_1^\mathrm{up}$ and 1 down-Laplacian $\LL_1^\mathrm{down}$ under the assumption that the points are sampled from a
constant density supported on $\M$, which proves the pointwise consistency and $a=\frac{1}{4}$ in $b=1$ in the Helmholtzian. The spectral consistency of the corresponding $\LL_1^\mathrm{down}$
is also shown using the spectral dependency to the well-studied graph Laplacian.
\romanenu{3} This work opens up avenues for the extensions of the
well-studied node based Laplacian based algorithms to edge flow
learning.  This includes, but is not limited to, semi-supervised
learning (SSL) \cite{BelkinM.N.S+:06} on edge flows and
smoothing/noise reduction for vector fields \cite{SchaubMT.S:18}.
The classical Laplacian-type SSL and graph signal processing algorithms
are thus applied to edge flow learning scenarios with the constructed
weighted Helmholtzian $\LL_1^s$. Furthermore, \romanenu{4} the effectiveness of
the proposed weighted Helmholtzian are shown by comprehensive
experiments on synthetic and real-world datasets.  The proposed
framework is a significant building block for new applications and
algorithms in other domains of study, such as language datasets
\cite{ZhuX:13} with \texttt{word2vec} embedding
\cite{MikolovT.S.C+:13}, (multivariate) time series
\cite{GiustiC.G.B+:16}, and 3D motion datasets \cite{AliS.B.S+:07}.

\section*{Acknowledgements}
The authors acknowledge partial support from the U.S. Department of Energy's Office of Energy Efficiency and Renewable Energy (EERE) under the Solar Energy Technologies Office Award Number DE-EE0008563 (M.M. and Y.C.), from
the NSF DMS PD 08-1269 (M.M. and Y.C.),
and an ARO MURI (I.G.K.). 
Part of this research was performed while the three authors were visiting
the
Institute for Pure and Applied Mathematics (IPAM), which is supported
by the National Science Foundation (Grant No. DMS-1925919). In addition,
M.M. gratefully acknowledges an IPAM Simons Fellowship.
This work was facilitated through the use of advanced computational, storage, and networking infrastructure provided by the Hyak supercomputer system and funded by the STF at the University of Washington.
M.M. and Y.C. thank the Tkatchenko
and Pfaendtner labs and in particular to Stefan Chmiela and Chris Fu for
providing the molecular
dynamics data and for many hours of brainstorming and advice, some of
which took place at IPAM.

\section*{Disclaimer}
The views expressed herein do not necessarily represent the views of the U.S. Department of Energy or the United States Government.

\newcommand{\etalchar}[1]{$^{#1}$}

\bibliographystyle{alpha}

\setupsupp

\twocolumn
\section{Notational table}
\label{sec:notational-table}
\label{sec:first-sec-in-supplement}
\begin{table}[!htb]
\begin{adjustwidth}{-50pt}{0pt}
\caption{Notational table}
\label{tab:notation-summary}
\centering
\begin{tabular}{cl}
\toprule
\multicolumn{2}{c}{Matrix operation} \\
\midrule
$\vec{M}$ & Matrix \\
$\vec{m}_i$ & Vector represents the $i$-th row of $\vec{M}$  \\ 
$\vec{m}_{:,j}^T$ & Vector represents the $j$-th column of $\vec{M}$  \\
$m_{ij}$ & Scalar represents $ij$-th element of $\vec{M}$  \\
$[\vec{M}]_{ij}$ & Scalar, alternative notation for $m_{ij}$  \\
$\vec{M}[\alpha, \beta]$ & Submatrix of $\vec{M}$ of index sets $\alpha, \beta$ \\
$\vec{v}$ & Column vecor   \\
$v_{i}$ & Scalar represents $i$-th element of vecor $\vec{v}$  \\
$[\vec{v}]_{i}$ & Scalar, alternative notation for $v_{i}$   \\
\midrule
\multicolumn{2}{c}{Scalars} \\
\midrule
$n$ & Number of samples (nodes) $(=n_0)$ \\
$n_k$ & Dimension of $k$ co-chain $(=|\Sigma_k|)$ \\
$d$ & Intrinsic dimension \\
$\delta$ & Bandwidth parameter for $\vec L$ \\
$\varepsilon$ & Bandwidth parameter for $\vec W_2$ \\
\midrule
\multicolumn{2}{c}{Vectors \& Matrices} \\
\midrule
$\vec{X}$ & Data matrix \\
$\vec{x}_i$ & Point $i$ in ambient space \\
$\vec{B}_k$ & Boundary operator of $k$ chain \\
$\vec{L}$ & Graph Laplacian \\
$\vec{L}_k$ & $k$ unnormalized Hodge Laplacian \\
$\vec{w}_k$ & $\in\mathbb R^{n_k}$, weights of $k$ (co)chain \\
$\vec{W}_k$ & $ = \diag(\vec{w}_k)$ \\
$\vec{\mathcal L}_k$ & $k$ random walk Hodge Laplacian \\
$\vec{\mathcal L}_k^s$ & $k$ symmetrized Hodge Laplacian \\
$\vec{I}_n$ & Identity matrix $\in\mathbb R^{n\times n}$ \\
$\vec{1}_n$ & All one vecor $\in\mathbb R^n$ \\
$\vec{1}_S$ & $[\vec{1}_S]_i = 1$ if $i\in S$ 0 otherwise\\
$\vec{0}_n$ & All zero vecor $\in\mathbb R^n$ \\
\midrule
\multicolumn{2}{c}{Miscellaneous} \\
\midrule
$\Sigma_k$ & Set of $k$-simplices \\
$V = \Sigma_0$ & Set of nodes  \\
$E = \Sigma_1$ & Set of edges \\
$T = \Sigma_2$ & Set of triangles \\
$\scx_\ell$ & $ = (\Sigma_k)_{k=0}^\ell = (\Sigma_0, \Sigma_1, \cdots, \Sigma_\ell)$ \\
& simplicial complex up to dimension $\ell$ \\
$G(V, E)$ & Graph with vertex set $V$ and edge set $E$ \\
& $ = \scx_1 = (\Sigma_0, \Sigma_1)$ \\
$\mathcal M$ & Data manifold \\
$[s]$ & Set $\{1, \cdots, s\}$ \\
$\epsilon_\pi$ & Levi-Civita symbol for permutation $\pi$ \\
$\mathsf{d}$ & Exterior derivative \\
$\updelta$ & Co-differential operator \\
\bottomrule
\end{tabular}
\label{tb:notation-summaries}
\end{adjustwidth}
\end{table}

\newpage
\section{Pseudocodes}
\label{sec:pseudocodes}
\begin{algorithm}[!htb]
    \setstretch{1.15}
    \SetKwInOut{Input}{Input}
    \SetKwInOut{Output}{Return}
    \SetKwComment{Comment}{$\triangleright $\ }{}
    \Input{Set of $k-1$ and $k$ simplices $\Sigma_{k-1}$, $\Sigma_k$}
    $\vec{B}_k\gets \vec{0}_{n_{k-1}}\vec{0}_{n_{k}}^\top\in\mathbb R^{n_{k-1}\times n_k}$ \\
    \For{every $\sigma_{k-1}\in \Sigma_{k-1}$}{
        \For{every $\sigma_k \in \Sigma_k$}{
            \uIf{$\sigma_{k-1}$ is a face of $\sigma_k$}{
                $[\vec{B}_k]_{\sigma_{k-1},\sigma_k} \gets \epsilon_{i_j,\iksetremovej}$
                \Comment{See \eqref{eq:boundary-map-rigour-def}.}
            }
            \Else{$[\vec{B}_k]_{\sigma_{k-1}, \sigma_k} \gets 0$}
        }
    }
    
    \Output{Boundary map for $k$-chain $\vec{B}_k$}
    \caption{{\boundarymap}}
    \label{alg:boundary-map}
\end{algorithm}
\begin{algorithm}[!htb]
    \SetKwInOut{Input}{Input}
    \SetKwInOut{Output}{Return}
    \SetKwComment{Comment}{$\triangleright $\ }{}
	\Input{Data matrix $\vX\in\rrr^{n\times D}$, radius \texttt{max\_dist} $\delta$, \texttt{max\_dim}
	$k$}
	Build graph $G(V, E)$ with $V = [n]$ and $E = \{(i, j)\in V^2: \|\vx_i - \|\vx_j\| < \delta\}$ \\
	\For{$\ell = 2\to k$}{
		$\Sigma_\ell = \{(i_0,\cdots,i_\ell)\in V^\ell: e \in E \,\forall\, e \in \binom{(i_0,\cdots,i_\ell)}{2}\}$
		\Comment{Clique complex of $G$}
	}
	\Output{VR complex $\scx_k = (V, E,\cdots, \Sigma_k)$}
    \caption{\vrcomplex}
    \label{alg:vr-complex}
\end{algorithm}
\begin{algorithm}[!htb]
    \SetKwInOut{Input}{Input}
    \SetKwInOut{Output}{Return}
    \SetKwComment{Comment}{$\triangleright $\ }{}
    \Input{Data matrix $\vX\in\rrr^{n\times D}$, \texttt{max\_dim} $k$, \texttt{num\_bootstrap\_samples} $B$,
    significance level $\alpha$}
    Compute persistent diagram $\mathcal P \gets \text{\sc PD}(\vX, k)$ \\
    \For{$i=1\to B$}{
    	Sample $[\tilde\vx_i]_{i=1}^n$ with replacement from $\vX$ \\
    	Bootstrapped PD $\tilde{\mathcal P}_i \gets \text{\sc PD}(\tilde\vX, k)$ \\
    	\For{$\ell=0\to k$}{
    		$\mathfrak D^\ell_i \gets \text{\sc BottleneckDist}(\mathcal P, \tilde{\mathcal P}_i, \ell)$
    	}
    }
    \For{$\ell=0\to k$}{
	    $b_\alpha^\ell \gets \text{\sc Percentile}(\mathfrak D^\ell_i, 1-\alpha)$
    }
    \Output{Confidence band $\{b_\alpha^\ell\}_{\ell=0}^k$}
    \caption{\pdbootstrap}
    \label{alg:bootstrap-pd}
\end{algorithm}
\newpage
\onecolumn \section{Rigorous definitions}
\label{sec:background-ext-calculus}

First we define the Levi-Civita notation and permutation parity. This is useful for the definition of
boundary operator $\vB_k$.
\begin{definition}[Permutation parity]
Given a finite set $\{j_0,j_1,\cdots,j_k\}$ with $k\geq 1$ and $j_\ell < j_m$ if $\ell < m$, the parity of a permutation
$\varsigma(\{j_0,\cdots,j_k\}) = \{i_0, i_1,\cdots,i_k\}$ is defined to be
\begin{equation}
	\epsilon_{i_0,\cdots,i_k} = -1^{N(\varsigma)}.
\end{equation}

Here $N(\varsigma)$ is the {\em inversion number} of $\varsigma$. The inversion number is the cardinality of the inversion
set, i.e.,  $N(\varsigma) = \#\{(\ell, m): i_\ell > i_m \text{ if } \ell < m \}$. We say $\varsigma$ is an even permutation
if $\epsilon_{i_0,\cdots,i_k} = 1$ and an odd permutation otherwise.
\end{definition}

\begin{remark}
When $k=1$, the Levi-Civita symbol is
\begin{equation*}
\epsilon_{ij} = \begin{cases}
	+1 \,&\text{if }\, (i,j) = (1, 2), \\
	-1 \,&\text{if }\, (i,j) = (2, 1).
\end{cases}
\end{equation*}

For $k = 2$, the Levi-Civita symbol is
\begin{equation*}
\epsilon_{ijk} = \begin{cases}
	+1 \,&\text{if }\, (i,j,k) \in \{(1, 2, 3), (2, 3, 1), (3, 1, 2)\}, \\
	-1 \,&\text{if }\, (i,j,k) \in \{(3, 2, 1), (1, 3, 2), (2, 1, 3)\}.
\end{cases}
\end{equation*}
\end{remark}

With this in hand, one can define the boundary map as follows.
\begin{definition}[Boundary map \& boundary matrix]
Let $\iksetremovej \coloneqq i_0,\cdots, i_{j-1}, i_{j+1}, \cdots, i_{k}$, and 
$\iksetinsertj$ denote $i_j$ insert into $i_0,\cdots, i_k$ with proper order, 
we define a {\em boundary map (operator)} $\mathcal B_k: \mathcal C_k \to\mathcal C_{k-1}$,
which maps a simplex to its face, by
\begin{equation}
	\mathcal{B}_k([i_0, \cdots, i_k]) = \sum_{j=0}^k(-1)^j [\iksetremovej] = \sum_{j=0}^k\epsilon_{i_j,\iksetremovej}[\iksetremovej].
\end{equation}
Here $i_j,\iksetremovej\coloneqq i_j,i_0,\cdots, i_{j-1},i_{j+1}, \cdots, i_k$. 

The corresponding {\em boundary matrix}
$\vec{B}_k \in\{0,\pm 1\}^{n_{k-1}\times n_k}$ can be defined as follows.
\begin{equation}
	(\vec{B}_k)_{\sigma_{k-1},\sigma_k} \begin{cases}
		\epsilon_{i_j,\iksetremovej} \,&\text{if }\, \sigma_k = [i_0,\cdots,i_k],\,\,\sigma_{k-1} = [\iksetremovej], \\
		0 \,&\text{otherwise.}
	\end{cases}
\label{eq:boundary-map-rigour-def}
\end{equation}

$(\vec{B}_k)_{\sigma_{k-1},\sigma_k}$ represents the orientation of $\sigma_{k-1}$ as a face of $\sigma_k$,
or equals 0 when the two are not adjacent.
\end{definition}

\begin{example}
For the simplicial complex $\scx_2$ in Figure \ref{fig:example-of-a-scx2}, the corresponding $\vB_1$ is in Table
\ref{tab:b1-example-scx2} while $\vB_2$ is in Table \ref{tab:b2-example-scx2}.

\usetikzlibrary{patterns}
\begin{minipage}{0.75\textwidth}
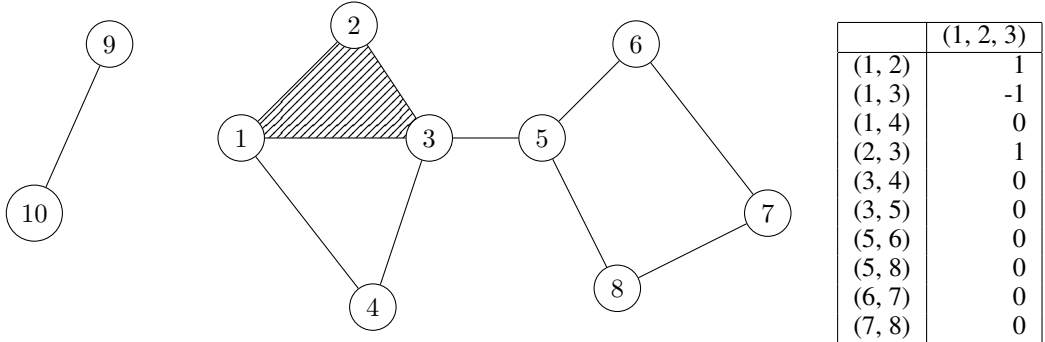
\begin{figure}[H]
\pgfdeclarelayer{nodelayer}
\pgfdeclarelayer{edgelayer}
\pgfsetlayers{nodelayer,edgelayer}   \centering
\begin{tikzpicture}
	\begin{pgfonlayer}{nodelayer}
		\fill[pattern=north east lines, pattern color=black]  (0, 0) -- (1.5, 1.5) -- (2.5, 0) -- (0, 0);
		\node [circle,draw=black, fill=white] (0) at (0, 0) {$1$};
		\node [circle,draw=black, fill=white] (1) at (1.5, 1.5) {$2$};
		\node [circle,draw=black, fill=white] (2) at (2.5, 0) {$3$};
		\node [circle,draw=black, fill=white] (3) at (1.75, -2.25) {$4$};
		\node [circle,draw=black, fill=white] (4) at (4, 0) {$5$};
		\node [circle,draw=black, fill=white] (5) at (5.25, 1.25) {$6$};
		\node [circle,draw=black, fill=white] (6) at (7, -1) {$7$};
		\node [circle,draw=black, fill=white] (7) at (5, -2) {$8$};
		\node [circle,draw=black, fill=white] (8) at (-1.75, 1.25) {$9$};
		\node [circle,draw=black, fill=white] (9) at (-2.75, -1) {$10$};
	\end{pgfonlayer}
	\begin{pgfonlayer}{edgelayer}
		\draw (0) to (1);
		\draw (1) to (2);
		\draw (2) to (3);
		\draw (3) to (0);
		\draw (2) to (4);
		\draw (4) to (5);
		\draw (5) to (6);
		\draw (6) to (7);
		\draw (7) to (4);
		\draw (8) to (9);
		\draw (0) to (2);
	\end{pgfonlayer}
\end{tikzpicture}
\caption{Illustration of a $\scx_2 = (\Sigma_0, \Sigma_1, \Sigma_2)$, shaded region denotes that the triangle $t \in \Sigma_2$.}
\label{fig:example-of-a-scx2}
\end{figure}
\end{minipage}
\hfill
\begin{minipage}{0.22\textwidth}
\vspace{3ex}
\centering
\begin{tabular}{|l|r|}
\hline
        & (1, 2, 3) \\ \hline
(1, 2)  & 1         \\
(1, 3)  & -1        \\
(1, 4)  & 0         \\
(2, 3)  & 1         \\
(3, 4)  & 0         \\
(3, 5)  & 0         \\
(5, 6)  & 0         \\
(5, 8)  & 0         \\
(6, 7)  & 0         \\
(7, 8)  & 0         \\
\hline
\end{tabular}
\captionof{table}{$\vB_2$ of $\scx_2$}
\label{tab:b2-example-scx2}
\end{minipage}

\begin{table}[!htb]
\centering
\begin{tabular}{|l|rrrrrrrrrrr|}
\hline
   & (1, 2) & (1, 3) & (1, 4) & (2, 3) & (3, 4) & (3, 5) & (5, 6) & (5, 8) & (6, 7) & (7, 8) & (9, 10) \\ \hline
1  & 1      & 1      & 1      & 0      & 0      & 0      & 0      & 0      & 0      & 0      & 0       \\
2  & -1     & 0      & 0      & 1      & 0      & 0      & 0      & 0      & 0      & 0      & 0       \\
3  & 0      & -1     & 0      & -1     & 1      & 1      & 0      & 0      & 0      & 0      & 0       \\
4  & 0      & 0      & -1     & 0      & -1     & 0      & 0      & 0      & 0      & 0      & 0       \\
5  & 0      & 0      & 0      & 0      & 0      & -1     & 1      & 1      & 0      & 0      & 0       \\
6  & 0      & 0      & 0      & 0      & 0      & 0      & -1     & 0      & 1      & 0      & 0       \\
7  & 0      & 0      & 0      & 0      & 0      & 0      & 0      & 0      & -1     & 1      & 0       \\
8  & 0      & 0      & 0      & 0      & 0      & 0      & 0      & -1     & 0      & -1     & 0       \\
9  & 0      & 0      & 0      & 0      & 0      & 0      & 0      & 0      & 0      & 0      & 1       \\
10 & 0      & 0      & 0      & 0      & 0      & 0      & 0      & 0      & 0      & 0      & -1      \\
\hline
\end{tabular}
\caption{Boundary matrix $\vB_1$ (incident matrix) of $\scx_2$ in Figure \ref{fig:example-of-a-scx2}.}
\label{tab:b1-example-scx2}
\end{table}

In this example, triangle is not always filled, e.g., $[1, 3, 4] \notin \Sigma_2$. However, for the construction
of VR complex described in Algorithm \ref{alg:manifold-helmholtzian-learning},
every triangle in $\scx_2$ is filled.
\end{example}

Below we provide the definition of the Hodge star operator and $\dd$, $\updelta$.
\begin{definition}[Hodge star]
    Hodge star on 1 form (dual of basis function) is defined as follow, 
    \begin{equation}
        \star (\mathsf{d}s_{i_1} \wedge\cdots\wedge \mathsf{d}s_{i_k}) = \epsilon_{IJ} \mathsf{d}s_{j_1}\wedge\cdots\wedge \mathsf{d}s_{j_{d-k}}
    \end{equation}
    
    Here $J = \{j_1\cdots, j_{d-k}\}$ is the (ordered) complement of $I = \{i_1, \cdots, i_k\}$. E.g., if
    $I = \{1, 5\}$ in $d=7$ dimensional space, we have
    $J=\{2, 3, 4, 6, 7\}$.
    $\epsilon_{IJ}$ is the Levi-Civita symbol of the permutation
    $IJ = \{i_1\cdots i_k, j_1\cdots,j_{n-k}\}$.
\end{definition}

\begin{definition}[Differential \& Co-differential]
    The differential of $k$ form $\zeta_k$ with $\zeta_k = f ds_{i_1} \wedge\cdots\wedge ds_{i_k}$ is
    \begin{equation}
        \mathsf{d}\zeta = \sum_{j=1}^d \frac{\partial f}{\partial s_i} \mathsf{d}s_j\wedge (\mathsf{d}s_{i_1} \wedge\cdots\wedge \mathsf{d}s_{i_k}).
    \end{equation}
    
    The co-differential $\updelta$ on a $k$ form is defined as
    \begin{equation}
        \updelta = (-1)^{d(k-1)+1} \star \mathsf{d}\star.
    \end{equation}
\end{definition}

 \section{Technical lemmas of exterior calculus}
\label{sec:lemmas-exterior-calculus}
In this section, we derive the closed form of 1-Laplacian $\Delta_1 = \dd\updelta + \updelta\dd$
on local coordinate system (tangent plane) when metric tensor at each point is identity.
Note that the assumption is sufficient for current work since the Simplicial complex
$\scx_2$ is built in the {\em ambient space}.
We start with some useful identities list as follows.
\begin{lemma}[Identities of permutation parity]
	Given the Levi-Civita symbol $\epsilon$, one has the following two identities.
    \begin{equation}
        \epsilon_{\{i\}\{-i,-j\}} = \epsilon_{ij}\epsilon_{i, -i}.
        \label{eq:levi-civita-1}
    \end{equation}
    {\em \begin{proof}
        Consider the case when $j>i$, i.e., $\epsilon_{ij} = 1$. We have $\epsilon_{\{i\}\{-i,-j\}} = \epsilon_{i, -i}$, for the inversion number of $\{i\}\{-i,-j\}$
        is $i-1$, which is identical to the inversion number of permutation $i, -i$.
        This implies the parity of the permutation is $\epsilon_{i, -i}$.
        Consider the case $i>j$, i.e., $\epsilon_{ij} = -1$, the inversion number
        of the permutation $\{i\}\{-i,-j\}$ is $i-2$ since there is $i-2$ elements
        (excluding $j$) that is smaller than $i$ in $\{-i, -j\}$. This implies that the parity
        of $\{i\}\{-i,-j\}$ is $-1\cdot \epsilon_{i, -i}$.
        This completes the proof.
    \end{proof}}

    Let $i' < j'$, the second identity states that,
    \begin{equation}
        \epsilon_{\{i',j'\}\{-i',-j'\}} = -\epsilon_{i', -i'}\epsilon_{j', -j'}.
        \label{eq:levi-civita-2}
    \end{equation}
    {\em \begin{proof}
    	The inversion number of the permutation $\{i',j'\}\{-i',-j'\}$
    	is $(j'-2) + (i'-1) = i'+j'-3$; the inversion numbers of the permutations $\{j'\}\{-i',-j'\}$ and $i', -i'$
    	are $j'-2$ and $i'-1$, respectively.
    	One can show that $\epsilon_{\{i',j'\}\{-i',-j'\}} = \epsilon_{\{j'\}\{-i',-j'\}}\cdot \epsilon_{i', -i'}$.
        From \eqref{eq:levi-civita-1}, we have $\epsilon_{\{j'\}\{-i',-j'\}} = \epsilon_{j'i'}\epsilon_{j',-j'} = -\epsilon_{j',-j'}$. This completes the proof.
    \end{proof}}
    \label{thm:permutation-parity-lemmas}
\end{lemma}

The following lemma presents the codifferential of a 2-form.
\begin{lemma}[Codifferential of a 2-form]
    Let $\zeta_2 = \sum_i\sum_{j\neq i} A_{ij} ds_j\wedge ds_i$ be a 2-form. The codifferential operator $\updelta$ acting on
    $\zeta_2$ is
    \begin{equation}
        \updelta \zeta_2 = \sum_i\sum_{j\neq i} \left(\frac{\partial A_{ji}}{\partial s_j} - \frac{\partial A_{ij}}{\partial s_j}\right) \mathsf{d}s_i.
    \end{equation}
\label{thm:codifferential-2-form}
\end{lemma}

\begin{proof}
    \begin{equation*}
        \star \mathsf{d}\star \zeta_2 = \sum_i\sum_{j\neq i} \star \mathsf{d}A_{ij}\star\left(\mathsf{d}s_j\wedge \mathsf{d}s_i\right) = \sum_i\sum_{j\neq i}\sum_{k\in\{i, j\}} \frac{\partial A_{ij}}{\partial s_k} \star\left(\mathsf{d}s_k\wedge \star(\mathsf{d}s_j\wedge \mathsf{d}s_i)\right).
    \end{equation*}

    The last summation is over $k\in\{i, j\}$ otherwise it will produce zero. Next step is to derive the
    exact form of $\star\left(\mathsf{d}s_k\wedge\star(\mathsf{d}s_j\wedge \mathsf{d}s_i)\right)$. Consider the case when $k = i$, we have
    \begin{equation}
    \begin{split}
        \star\left(\mathsf{d}s_i\wedge\star(\mathsf{d}s_j\wedge \mathsf{d}s_i)\right) \,&=\, \epsilon_{ji}\cdot \star\left(\mathsf{d}s_i\wedge\star(\mathsf{d}s_{i'}\wedge \mathsf{d}s_{j'})\right) \\
        \,&=\,\epsilon_{ji}\epsilon_{\{i',j'\}\{-i',-j'\}} \star \left(\mathsf{d}s_i\wedge \bigwedge_{\ell\in\{-i',-j'\}} \mathsf{d}s_\ell \right) \\
        \,&=\, \epsilon_{ji}\epsilon_{\{i',j'\}\{-i',-j'\}}\epsilon_{\{i\}\{-i,-j\}}\epsilon_{-j,j} \mathsf{d}s_j = \underbrace{\epsilon_{j,-j}\epsilon_{-j, j}}_{=(-1)^{d+1}} \mathsf{d}s_j.
    \end{split}
    \nonumber
    \end{equation}
    Last equality holds from Lemma \ref{thm:permutation-parity-lemmas} and $\epsilon_{j,-j}\epsilon_{-j, j} = (-1)^{d+1}$. Consider the case when $k=j$,
    \begin{equation}
    \begin{split}
        \star\left(\mathsf{d}s_j\wedge\star(\mathsf{d}s_j\wedge \mathsf{d}s_i)\right) \,&=\, \epsilon_{ji}\cdot  \star\left(\mathsf{d}s_j\wedge\star(\mathsf{d}s_{i'}\wedge \mathsf{d}s_{j'})\right) \\
        \,&=\,\epsilon_{ji}\epsilon(\{i',j'\}\{-i',-j'\}) \star \left(\mathsf{d}s_j\wedge \bigwedge_{\ell\in\{-i',-j'\}} \mathsf{d}s_\ell \right) \\
        \,&=\, \epsilon_{ji}\underbrace{\epsilon_{\{i',j'\}\{-i',-j'\}}}_{=-\epsilon_{i,-i}\epsilon_{j,-j}}\underbrace{\epsilon_{\{j\}\{-i,-j\}}}_{=\epsilon_{ji}\epsilon_{j,-j}}\epsilon_{-i,i} \mathsf{d}s_i =  (-1)^{d+1} \cdot (-\mathsf{d}s_i).
    \end{split}
    \nonumber
    \end{equation}
    Putting things together, we have
    \begin{equation*}
    \begin{split}
        \updelta \zeta_2 \,&=\, (-1)^{d+1} \star \mathsf{d}\star \zeta_2 = \sum_i\sum_{j\neq i} \frac{\partial A_{ij}}{\partial s_i} \mathsf{d}s_j - \frac{\partial A_{ij}}{\partial s_j} \mathsf{d}s_i \\
        \,&=\, \sum_i\sum_{j\neq i} \left(\frac{\partial A_{ji}}{\partial s_j} - \frac{\partial A_{ij}}{\partial s_j}\right) \mathsf{d}s_i.
    \end{split}
    \end{equation*}
    Last equality holds by changing the index of summing.
\end{proof}

With Lemma \ref{thm:codifferential-2-form} in hand, we can start deriving the closed for of $\Delta_1$
in the local coordinate system.
\begin{proposition}[1-Laplacian in local coordinate system]
	Let $\zeta_1 = \sum_{i=1}^d f_i \mathsf{d} s_i$ be a 1-form.
    The up-Laplacian $\Delta_1^\mathrm{down} = \updelta\dd$ operates on $\zeta_1$ (in local coordinate system) is
    \begin{equation}
        \updelta \mathsf{d} \zeta_1 = \sum_i\sum_{j\neq i} \left(\frac{\partial^2 f_j}{\partial s_j\partial s_i} - \frac{\partial^2 f_i}{\partial s_j^2}\right) \mathsf{d}s_i.
        \label{eq:hodge-1-up-laplacian-local-coord}
    \end{equation}

    The down-Laplacian $\Delta_1^\mathrm{down} = \dd\updelta$ operates on $\zeta_1$ (in local coordinate system) is
    \begin{equation}
        \mathsf{d}\updelta \zeta_1 = -\sum_i \sum_j \frac{\partial^2 f_j}{\partial s_i\partial s_j} \mathsf{d}s_i.
        \label{eq:hodge-1-down-laplacian-local-coord}
    \end{equation}
    \label{thm:hodge-1-laplacian-local-coord}
\end{proposition}

\begin{proof}
    We first consider the up-Laplacian $\updelta\mathsf{d}\zeta_1 = \updelta \zeta_2$ on 1 form $\zeta_1$ with
    \begin{equation*}
        \zeta_2 = \mathsf{d}\zeta_1 = \mathsf{d}\left(\sum_{i=1}^d f_i \mathsf{d}s_i\right) = \sum_i \sum_{j\neq i} \frac{\partial f_i}{\partial s_j} \mathsf{d}s_j\wedge \mathsf{d}s_i.
    \end{equation*}

    From Lemma \ref{thm:codifferential-2-form} and let $A_{ij} = \frac{\partial f_i}{\partial s_j}$, we have
    \begin{equation*}
        \updelta \mathsf{d} \zeta_1 = \sum_i\sum_{j\neq i} \left(\frac{\partial^2 f_j}{\partial s_j\partial s_i} - \frac{\partial^2 f_i}{\partial s_j^2}\right) \mathsf{d}s_i.
    \end{equation*}

    Consider the case of the down-Laplacian and note that $\updelta = -\star \dd\star$. Since the co-differential
    is now act on $1$ form rather than $k=2$, we have
    \begin{equation*}
    \begin{split}
        \mathsf{d}\updelta \zeta_1 \,&=\, - \mathsf{d}\star \mathsf{d}\star \zeta_1 = -\sum_i \mathsf{d}\star \mathsf{d} \left(f_i \star \mathsf{d}s_i\right) = -\sum_i \mathsf{d}\star \sum_j \frac{\partial f_i}{\partial s_j} \mathsf{d}s_j\wedge \star \mathsf{d}s_i \\
        \,&\stackrel{\text{(\romannum{1})}}{=}\, -\sum_i \mathsf{d}\star \frac{\partial f_i}{\partial s_i} \bigwedge_{\ell=1}^d \mathsf{d}s_\ell \stackrel{\text{(\romannum{2})}}{=} -\sum_i \sum_j \frac{\partial^2 f_i}{\partial s_i\partial s_j} \mathsf{d}s_j.
    \end{split}
    \end{equation*}

    Equality (\romannum{1}) holds since $\mathsf{d}s_j\wedge \star \mathsf{d}s_i = 1$ if $i=j$ and $0$ otherwise. Equality
    (\romannum{2}) holds for $\star \bigwedge_{\ell=1}^d \mathsf{d}s_\ell = 1$ (a 0-form).
\end{proof}

\begin{remark}[Sanity check on 3D]
    Note that in 3D, $\mathrm{curl} = \star \mathsf{d}$, therefore $\updelta \mathsf{d} = \star \mathsf{d}\star \mathsf{d} = \mathrm{curl}\,\mathrm{curl}$. For a vector
    field (1-form) $\zeta_1 = f_1 \mathsf{d}s_1 + f_2 \mathsf{d}s_2 + f_3 \mathsf{d}s_3$,
    one has
    \begin{equation}
        \nabla\times \zeta_1 = \left(\frac{\partial f_3}{\partial s_2} - \frac{\partial f_2}{\partial s_3}\right) \mathsf{d}s_1 + \left(\frac{\partial f_1}{\partial s_3} - \frac{\partial f_3}{\partial s_1} \right) \mathsf{d}s_2 + \left(\frac{\partial f_2}{\partial s_1} - \frac{\partial f_1}{\partial s_2}\right) \mathsf{d}s_3.
        \nonumber
    \end{equation}
    \begin{equation}
    \begin{split}
        \nabla\times(\nabla\times \zeta_1) \,&=\, \left(\frac{\partial^2 f_2}{\partial s_1\partial s_2} - \frac{\partial^2 f_1}{\partial s_2^2} - \frac{\partial^2 f_1}{\partial s_3^2} + \frac{\partial^2 f_3}{\partial s_1\partial s_3}\right) \mathsf{d}s_1\\
        \,&=\, \left(\frac{\partial^2 f_3}{\partial s_2\partial s_3} - \frac{\partial^2 f_2}{\partial s_3^2} - \frac{\partial^2 f_2}{\partial s_1^2} + \frac{\partial^2 f_1}{\partial s_1\partial s_2}\right) \mathsf{d}s_2\\
        \,&=\, \left(\frac{\partial^2 f_1}{\partial s_1\partial s_3} - \frac{\partial^2 f_3}{\partial s_1^2} - \frac{\partial^2 f_3}{\partial s_2^2} + \frac{\partial^2 f_2}{\partial s_2\partial s_3}\right) \mathsf{d}s_3 \\
        \,&=\, \sum_i\sum_{j\neq i} \left(\frac{\partial^2 f_j}{\partial s_j\partial s_i} - \frac{\partial^2 f_i}{\partial s_j^2}\right) \mathsf{d}s_i.
    \end{split}
    \nonumber
    \end{equation}
\end{remark}

\begin{corollary}[Relation to Laplace-Beltrami]
    We have the following,
    \begin{equation}
        \Delta_1 f = (\updelta \mathsf{d} + \mathsf{d}\updelta) \zeta_1 = - \sum_i \sum_j \frac{\partial^2 f_i}{\partial s_j^2} \mathsf{d}s_i = -\sum_i \Delta_0 f_i \mathsf{d}s_i.
    \end{equation}
    Here $\Delta_0 f_i = \sum_j \frac{\partial^2 f_i}{\partial s_j^2}$ is the Laplacian on 0-form.
\end{corollary}
\begin{proof}
    Can be obtained by applying the result from Proposition \ref{thm:hodge-1-laplacian-local-coord}.
\end{proof}

\begin{remark}[Vector Laplacian in 3D]
    Note that in 3D case, $\Delta_1 \zeta_1 = -\sum_i \Delta_0 f_i ds_i = -\nabla^2 f$.
    Here $\nabla^2$ is vector Laplacian. This implies that vector Laplacian in 3D is essentially 1-Laplacian up to a sign change.
\end{remark}

\begin{corollary}[1-Laplacian on pure curl \& gradient vector fields]
    If the vector field $\zeta_1$ is a pure curl or gradient vector field, then
    \begin{equation}
        \Delta_1 \zeta_1 = (\updelta \mathsf{d} + \mathsf{d}\updelta) f = - \sum_i \sum_j \frac{\partial^2 f_i}{\partial s_j^2} \mathsf{d}s_i = - \sum_i \Delta_0 f_i \mathsf{d}s_i.
    \end{equation}
\label{thm:pure-curl-grad-1-Laplacian}
\end{corollary}

\begin{proof}
    Consider the case when $\zeta_1$ is curl flow ($\mathsf{d}\updelta \zeta_1 = 0$). This implies
    \begin{equation*}
        \sum_j \frac{\partial^2 f_j}{\partial s_i\partial s_j} = 0 \,\forall\, i\in [d].
    \end{equation*}

    Hence we have $\sum_{j\neq i}\frac{\partial^2 f_j}{\partial s_i\partial s_j} = -\frac{\partial^2 f_i}{\partial s_i^2}$. Plugging into \eqref{eq:hodge-1-up-laplacian-local-coord}, we have
    \begin{equation*}
        \Delta_1 \zeta_1 = \updelta \mathsf{d} \zeta_1 = - \sum_i\sum_j \frac{\partial^2 f_i}{\partial s_j^2} \mathsf{d}s_i = - \sum_i \Delta_0 f_i \mathsf{d}s_i.
    \end{equation*}

    Consider the case when $\zeta_1$ is gradient flow, which implies $\updelta \mathsf{d}f = 0$, or
    \begin{equation*}
        \sum_{j\neq i} \left(\frac{\partial^2 f_j}{\partial s_j\partial s_i} - \frac{\partial^2 f_i}{\partial s_j^2}\right) = 0 \,\forall\, i\in [d].
    \end{equation*}

    Therefore the identity $\sum_{j\neq i} \frac{\partial^2 f_j}{\partial s_j\partial s_i} = \sum_{j\neq i}\frac{\partial^2 f_i}{\partial s_j^2}$ holds. Plugging into \eqref{eq:hodge-1-down-laplacian-local-coord}, one has
    \begin{equation*}
        \Delta_1 \zeta_1 = \mathsf{d}\updelta \zeta_1 = -\sum_i \left(\left(\sum_{j\neq i} \frac{\partial^2 f_i}{\partial s_j^2}\right) + \frac{\partial^2 f_i}{\partial s_i^2} \right)\mathsf{d}s_i = -\sum_i \Delta_0 f_i \mathsf{d}s_i.
    \end{equation*}
    This completes the proof.
\end{proof}
\section{Proofs of the pointwise convergence of the {\em up} Helmholtzian}
\label{sec:pointwise-consistency-L1-up}

\subsection{Outline of the proof}
This section investigates the continuous limit of the discrete operator
$\LL_1^\mathrm{up} =
\vec{B}_{T}\vec{W}_T\vec{B}_{T}^\top\vec{W}_E^{-1}$.

\subsection{Proof of Lemma \ref{thm:discrete-form-L-up}}
\label{sec:proof-thm-discrete-form-L-up}
\begin{proof}
First, note that
\begin{equation*}
    [\vec{B}_T \vec{W}_T\vec{B}_T^\top \vec\omega]_{[x, y]} = w_E(\vecxy) \omega_{[x, y]} + (\cdots),
\end{equation*}
where $w_E(\vecxy) = [\vW_E]_{xy,xy}$.
There are six different cases to consider in the $(\cdots)$ part above.
Assume $[x', y', z']$ is the canonical permutation/ordering of $x, y, z$, i.e.,
$x'<y'<z'$. Note that by the definition of $w_T(\vecxyz)$, one has
$w_T(\vx',\vy',\vz') = w_T(\vecxyz{})$. Therefore, each part of the 
$(\cdots)$ term is
\begin{enumerate}
    \item $[x', y']\times [x', z'] \to w_T(\vecxyz)\cdot -1\cdot \omega_{[x', z']}$
    \item $[x', y']\times [y', z'] \to w_T(\vecxyz)\cdot 1\cdot \omega_{[y', z']}$
    \item $[x', z']\times [x', y'] \to w_T(\vecxyz)\cdot -1\cdot \omega_{[x', y']}$
    \item $[x', z']\times [y', z'] \to w_T(\vecxyz)\cdot -1\cdot \omega_{[y', z']}$
    \item $[y', z']\times [x', y'] \to w_T(\vecxyz)\cdot 1\cdot \omega_{[x', y']}$
    \item $[y', z']\times [x', z'] \to w_T(\vecxyz)\cdot -1\cdot \omega_{[x', z']}$
\end{enumerate}
Grouping 1 and 2 together, one obtains
\begin{equation*}
    w_T(\vecxyz) (\omega_{[y', z']} - \omega_{[x', z']}) = w_T(\vecxyz) (f_{x'y'z'} - \omega_{[x', y']}).
\end{equation*}
Similarly, for 3 and 4 as well as 5 and 6, we have
\begin{equation*}
\begin{gathered}
    w_T(\vecxyz) (-\omega_{[x', y']} - \omega_{[y', z']}) = w_T(\vecxyz) (-f_{x'y'z'} - \omega_{[x', z']}); \\
    w_T(\vecxyz) (\omega_{[x', y']} - \omega_{[x', z']}) = w_T(\vecxyz) (f_{x'y'z'} - \omega_{[y', z']}).
\end{gathered}
\end{equation*}
Here, $f_{x'y'z'} = \omega_{[x',y']}+\omega_{[y', z']}-\omega_{[x',z']}$.
To sum up, the $(\cdots)$ part becomes
\begin{equation*}
w_T(x, y, z)(\sigma_{xy, xyz} f_{{x'y'z'}} - f_{[x, y]}).
\end{equation*}
Note that $\sigma_{xy, xyz} f_{x'y'z'} = \omega_{xy} + \omega_{yz} + \omega_{zx} = f_{xyz}$, therefore,
\begin{equation*}
\begin{split}
    [\vB_T\vW_T\vB_T^\top \vec\omega]_{[x, y]} &\,=\, w_E(\vecxy) \omega_{[x, y]} + \sum_{z\notin \{x, y\}} w_T(\vecxyz)(\sigma_{xy, xyz} f_{x'y'z'} - \omega_{[x, y]}) \\
    &\,=\, \cancel{w_E(\vecxy) \omega_{[x, y]}} - \cancel{\sum_{z\notin \{x, y\}} w_T(\vecxyz)\omega_{[x, y]}} +  \cdots \\
    &\,=\, \sum_{z\notin \{x, y\}} w_T(\vecxyz)f_{xyz}.
\end{split}
\end{equation*}

In the above, we use the fact that $\vW_E = \diag(|\vB_T|\vW_T\vec{1}_{n_T})$.
This completes the proof.
\end{proof}

\subsection{Proof of Proposition \ref{thm:asymptotic-expansion-L-1-up}}
\label{sec:proof-thm-asymptotic-expansion-L-1-up}
We are interested in the asymptotic expansion of $\int_\mathcal{M} w_T(\vecxyz) f_{xyz} \mathsf{d}\mu(\vec z)$. Specifically, our goal is to show the following expansion given some constant $c = c(\vecxy)$.
\begin{equation}
    \int_\mathcal{M} w_T(\vecxyz) f_{xyz} \mathsf{d}\mu(\vec z) = c\int_{x\to y}(\Delta_1 f(\vec\gamma(t)))\vec\gamma'(t) \dd t + \mathcal O(\varepsilon^4).
    \label{eq:goal-of-proof-thm-3}
\end{equation}
Here, $\vec\gamma(t)$ is the parameterization of the geodesic curve
connecting $\vx$ and $\vy$ on the manifold $\M$.
From Corollary \ref{thm:pure-curl-grad-1-Laplacian} in \ref{sec:lemmas-exterior-calculus}, one has
$\Delta_1\vec{\upsilon} = \sum_i\sum_j \frac{\partial^2 f_i}{\partial s_j^2} \mathsf{d}s_i = \sum_i \Delta_0 f_i \mathsf{d}s_i$
in local coordinate $(s_1, \cdots, s_d)$ with $\vec{\upsilon} = \sum_i f_i \mathsf{d}s_i$.
Therefore, it is sufficient to show that the LHS of \eqref{eq:goal-of-proof-thm-3} 
can be expanded to terms consisting of only the coordinate-wise $\Delta_0$.

First, we show the upper bound for the error by integrating
the integral operator
around $\varepsilon^\gamma$ balls around $x$ and $y$.
This lemma is the modification of a similar technique
that appeared in Lemma 8 of \cite{CoifmanRR.L:06}.
\begin{lemma}[Error bound for the localization of an exponential decay kernel]
Let $0<\gamma<1$ and $g$ be a bounded function,
the integration of the integral operator
$\int_\M \kappa_\varepsilon(\vec x, \vec z)\kappa_\varepsilon(\vec y, \vec z) g(\vec z) \mathsf{d}\vec z$
over $z$ that is $\varepsilon^\gamma$
far away from points $x, y\in\M$ can be bounded above by $\mathcal O(\varepsilon^{4+d})$.
\label{thm:error-bound-localization}
\end{lemma}
\begin{proof}
First, we focus on the domain of the integral. Points $\vec z\in\M$ that are $\varepsilon^\gamma$
far away from both $\vec x$ and $\vec y$ form a set
$\{\vec z\in\M: \min(\|\vec z-\vec x\|, \|\vec z-\vec y\|) > \varepsilon^{\gamma}\}$.
Because the kernel has exponential decay, one can follow
the same technique of Lemma 8 in \cite{CoifmanRR.L:06} to bound the integration by
\begin{equation*}
\begin{split}
    \,&\,\left| \inv{\varepsilon^d} \int\limits_{\substack{\vec z\in\mathcal M \\ \min(\|\vec z-\vec x\|, \|\vec z-\vec y\|) > \varepsilon^{\gamma}}} \kappa\left(\frac{\|\vec z-\vec x\|^2}{\varepsilon^2}\right)\kappa\left(\frac{\|\vec z-\vec y\|^2}{\varepsilon^2}\right) g(\vec z) \mathsf{d}\vec z \right| \\
    \,&\leq\, \frac{\|g\|_\infty}{\varepsilon^d} \int\limits_{\substack{\vec z\in\mathcal M \\ \min(\|\vec z-\vec x\|, \|\vec z-\vec y\|) > \varepsilon^{\gamma}}}\left|\kappa\left(\frac{\|\vec z-\vec x\|^2}{\varepsilon^2}\right)\right|\left|\kappa\left(\frac{\|\vec z-\vec y\|^2}{\varepsilon^2}\right)\right| \mathsf d\vec z \\
    \,&\leq\, \frac{\|g\|_\infty\|\kappa\|_\infty}{\varepsilon^d} \int\limits_{\substack{\vec z\in\mathcal M \\ \|\vec z-\vec x\|>\varepsilon^{\gamma}}}\left|\kappa\left(\frac{\|z-x\|^2}{\varepsilon^2}\right)\right| \mathsf d\vec z \\
    \,&\leq\,\|g\|_\infty\|\kappa\|_\infty \int\limits_{\substack{\vec z\in\mathcal M \\ \|\vec z\|>\varepsilon^{\gamma-1}}}\left|\kappa_\varepsilon(\|\vec z\|^2)\right| \mathsf d\vec z \leq C\|g\|_\infty Q\left(\varepsilon^{1-\gamma}\right) \exp\left(-\varepsilon^{\gamma -1}\right).
\end{split}
\end{equation*}

The last inequality holds by using the exponential decay of the kernel. Here, $Q$ is some
polynomial. Since $0<\gamma<1$, the term is exponentially small and is bounded by
$\mathcal O(\varepsilon^{4+d})$.
\end{proof}

Therefore, the original integral operator (LHS of \eqref{eq:goal-of-proof-thm-3}) becomes
\begin{equation}
    \varepsilon^{-d} \vec{\mathcal L}_{1}^{\text{up}} f_{xy} = \varepsilon^{-d} \int\limits_{\substack{\vec z\in\M \\ \max(\|\vec z-\vec y\|, \|\vec z-\vec x\|) \leq \varepsilon^\gamma}} \left(\kappa\left(\frac{\|\vec z-\vec x\|^2}{\varepsilon^2} \right)\kappa\left(\frac{\|\vec z-\vec y\|^2}{\varepsilon^2} \right) \oint_{\Gamma(x,y,z)} \vec{\upsilon}\right) \mathsf d\vec z + \mathcal O(\varepsilon^4).
    \nonumber
\end{equation}

We introduce the last lemma that is useful in removing the bias term
in the line integral before proving Proposition \ref{thm:asymptotic-expansion-L-1-up}.
\begin{lemma}[Line integral approximation]
Assume $f\in\mathcal C^2$ and 
let $\vu(t) = \vx + (\vy - \vx)t$ be a parameterization of the straight line
connecting nodes $x$ and $y$. We have the following error bound
\begin{equation}
	f(\vec u(t)) = f(\vec x) + ((f(\vec y) - f(\vec x))t + \mathcal O(\|\vec y-\vec x\|^2).
\end{equation}
\label{thm:linear-approx-of-line-int}
\end{lemma}
\begin{proof}
	Note that $\vec u(t) = \vec x + (\vec y-\vec x)t$. By Taylor expansion on $f$, one has
	\begin{equation*}
		f(\vec x+(\vec y-\vec x)t) = f(\vec x) + (\vec y-\vec x)^\top \nabla f(\vec x)\cdot t + \mathcal O(\|\vec y-\vec x\|^2).
	\end{equation*}
	Additionally, $(\vec y-\vec x)^\top \nabla f(\vec x)$ is the directional derivative, which can be
	approximate by $f(\vec y) - f(\vec x) + \mathcal O(\|\vec y-\vec x\|^2)$ by Taylor expansion. Therefore
	\begin{equation*}
		f(\vec u(t)) = f(\vec x+(\vec y-\vec x)t) = f(\vec x) + (f(\vec y) - f(\vec x))t + \mathcal O(\|\vec y-\vec x\|^2).
	\end{equation*}
	This completes the proof.
\end{proof}

The outline of the proof is as follows. We first prove the asymptotic expansion
of the integral operator
$\int w_T(\vecxyz) f_{xyz} \dd\mu(\vz)$.
Later on, we bound the error of approximating the manifold by tangent bundles.
Lastly, the asymptotic expansion of the integral operator
is obtained by incorporating the error terms of tangent plane approximation
and the expansion in $\rrr^d$.
The following lemma is the first step, i.e., the asymptotic expansion
in $\mathbb R^d$.

\begin{lemma}[Asymptotic expansion of $\vec{\mathcal L}_1^\mathrm{up}$ in $\rrr^d$]
Under Assumption \ref{assu:manifold-data-assu}--\ref{assu:kernel-assu},
and further assume With the choice of exponential kernel $\kappa_\varepsilon(u) = \exp(-u)$ and $\vec{\upsilon} = (f_1, \cdots, f_d) \in\mathcal C^4(\M)$.
Let $\vu(t) = \vx + (\vy - \vx)t$ for $t$ in $[0, 1]$ be a parameterization of the straight line between nodes
$x, y$ and $\vu'(t) = \dd\vu(t) / \dd t$.
With the choice of exponential kernel $\kappa(u)=\exp(-u)$, one has the following asymptotic expansion
\begin{equation}
\begin{split}
	 &\int_{\vz\in\rrr^d} w_T(\vecxyz) f_{xyz} \mathsf{d}\mu(\vec z) \\
	 &=\, \frac{2}{3}\varepsilon^{2+d}C_2\int_0^1 \sum_i\sum_j(\frac{\partial^2 f_j(\vu(t))}{\partial S_i\partial S_j}-\frac{\partial^2 f_i(\vu(t))}{\partial S_j^2})(\vy-\vx)_i\dd t + \mathcal O(\varepsilon^{d}\delta^3)+\mathcal O(\varepsilon^{4+d})+\mathcal O(\varepsilon^{2+d}\delta^2)
	 \\&= \frac{2}{3}\varepsilon^{2+d}C_2\int_0^1[(\updelta \dd\Vec{\upsilon(\vu(t))})]^\top \vec u'(t)\dd t+ \mathcal O(\varepsilon^{d}\delta^3)+\mathcal O(\varepsilon^{4+d})+\mathcal O(\varepsilon^{2+d}\delta^2).
\end{split}
\label{eq:asymptotic-expansion-L-1-up-result-rd}
\end{equation}
Here, $C_2=\varepsilon^{-(4+d)}\int_{\vz\in \rrr^d}\kappa^2(\frac{\|\vz\|^2}{\varepsilon^2})\vz_1^2\vz_2^2\dd\vz$ and $z_i$ is the i-th coordinate of $\vz$.
\label{thm:asymptotic-expansion-L-1-up-Rd}
\end{lemma}

\begin{proof}
First, we define $\vu(t), \vv(t), \vw(t)$ as in the diagram below
with $\vu(0) = \vx,  \vu(1) =  \vy$, $\vv(0) = \vy, \vv(1) = \vz$, and
$\vw(0) = \vz, \vw(1) = \vx$.

\begin{figure}[H]
    \begin{center}
	\pgfdeclarelayer{nodelayer}
	\pgfdeclarelayer{edgelayer}
	\pgfsetlayers{nodelayer,edgelayer}   
    \begin{tikzpicture}[scale=1, transform shape]
        \begin{pgfonlayer}{nodelayer}
            \node  (0) at (-0.75, 0) {};
            \node  (1) at (3, 0) {};
            \node  (2) at (0, 1.5) {};
            \node  (3) at (0.5, 0) {};
            \node  (4) at (1, 1) {};
            \node  (5) at (-0.5, 0.5) {};
            \node  (6) at (-0.75, -0.15) {$~~x\phantom{'}$};
            \node  (7) at (0.5, -0.15) {$~~u\phantom{'}$};
            \node  (8) at (3, -0.15) {$~~y\phantom{'}$};
            \node  (9) at (1.05, 1.25) {$v$};
            \node  (10) at (0, 1.75) {$z$};
            \node  (11) at (-0.7, 0.65) {$w\phantom{'}$};
            \node  (12) at (1.5, 0) {};
            \node  (13) at (1.8, 0.6) {};
            \node  (14) at (-0.3, 0.9) {};
            \node  (15) at (1.5, -0.15) {$~~u'$};
            \node  (16) at (2, 0.9) {$v'$};
            \node  (17) at (-0.5, 1.05) {$w'$};
        \end{pgfonlayer}
        \begin{pgfonlayer}{edgelayer}
            \draw (2.center) to (1.center);
            \draw (1.center) to (0.center);
            \draw (0.center) to (2.center);
            \draw [densely dashed] (5.center) to (3.center);
            \draw [densely dashed] (3.center) to (4.center);
            \draw [densely dashed] (13.center) to (12.center);
            \draw [densely dashed] (12.center) to (14.center);
        \end{pgfonlayer}
    \end{tikzpicture}
    \end{center}
\end{figure}
For notational simplicity, we do not use the conventional ({\em unit speed}) parametrization of a curve in our analysis. 
One can always use the convention without
changing any conclusions. Further define $\Gamma(x,y,z)$ be the loop connecting nodes $x, y, z$,
i.e., $\Gamma(x,y,z) = \{\vu(t), \vv(t), \vw(t)\}$.
The loop integral $f_{xyz}$ becomes
\begin{equation*}
    \oint_{\Gamma(x,y,z)} \vec{\upsilon} = \int_0^1 \vf(\vu(t))^\top (\vy-\vx) \dd t + \int_0^1 \vf(\vv(t))^\top (\vz-\vy) \dd t + \int_0^1 \vf(\vw(t))^\top (\vx-\vz) \dd t.
\end{equation*}
One can do the following coordinate-wise expansion up to second-order terms.
With slight abuse of notation, let
$\vv = \vv(t)$, $\vu = \vu(1-t)$, and $s_1,\cdots,s_d$ represent the coordinate system. Expanding $f_i(\vv)$ and $f_i(\vw)$ from $f_i(\vu)$ results in 
\begin{equation}
\label{taylyor expansion}
\begin{gathered}
    f_i(\vv(t)) = f_i(\vu(1-t)) + \sum_j \frac{\partial f_i(\vu)}{\partial s_j} (\vv-\vu)_j + \sum_j\sum_k \frac{\partial^2 f_i(\vu)}{\partial s_j\partial s_k} (\vv-\vu)_j(\vv-\vu)_k + \mathcal O(\|\vz-\vy\|^3); \\
    f_i(\vw(t)) = f_i(\vu(1-t)) + \sum_j \frac{\partial f_i(\vu)}{\partial s_j} (\vw-\vu)_j + \sum_j\sum_k \frac{\partial^2 f_i(\vu)}{\partial s_j\partial s_k} (\vw-\vu)_j(\vw-\vu)_k + \mathcal O(\|\vz-\vx\|^3).
\end{gathered}
\end{equation}
With the above expansion, we denote the \nth{0}-, \nth{1}-, and \nth{2}-order terms to be the following
\begin{equation}
	\int_{\vz\in \rrr^d}w^{(2)}(\vx,\vy,\vz)f_{xyz}dz =  \int_{\vz\in \rrr^d}w^{(2)}(\vx,\vy,\vz)(f_{xyz}^{(0)}+f_{xyz}^{(1)}+f_{xyz}^{(2)})dz  + \mathcal O(\varepsilon^{4+d}).
	\label{eq:fxyz-taylor-expansion}
\end{equation}

\begin{claim}
	The loop integral of the constant (\nth{0}-order) term $f^{(0)}_{xyz}$ is zero.
\end{claim}
The above claim is true because

\begin{equation*}
\begin{split}
    f_{xyz}^{(0)} \,&=\, \int_0^1 \sum_i f_i(\vu(t)) (\vy-\vx)_i \dd t + \int_0^1 \sum_i f_i(\vu(1-t)) (\vz-\vy)_i \dd t + \int_0^1 \sum_i f_i(\vu(1-t)) (\vx-\vz)_i \dd t \\
    \,&=\, \int_0^1 \sum_i f_i(\vu(t)) (\vy-\vx)_i \dd t - \int_1^0 \sum_i f_i(\vu(t)) (\vz-\vy)_i \dd t - \int_1^0 \sum_i f_i(\vu(t)) (\vx-\vz)_i \dd t \\
    \,&=\, \int_0^1 \sum_i f_i(u) \cdot (\vy-\vx + \vz-\vy + \vx-\vz)_i \dd t = 0.
\end{split}
\end{equation*}
The second equation holds by changing the variable $1-t \to t$ for the second and third terms.

\begin{claim}
	The integral of the first-order term $\int_{\vz\in \rrr^d}w^{(2)}(\vx,\vy,\vz)f^{(1)}_{xyz}\dd\vz$ is $\mathcal O(\varepsilon^{2+d}\delta)$.
\end{claim}

The integral of the first order term is 
\begin{equation*}
    \begin{split}
        \int_{\vz\in \rrr^d}w^{(2)}(\vx,\vy,\vz)f^{(1)}_{xyz}\dd\vz=\int_{\vz\in \rrr^d}w^{(2)}(\vx,\vy,\vz)\int_0^1 \sum_i\sum_j \frac{\partial f_i(\vu(1-t))}{\partial S_j}(\vv-\vu(1-t))_j(\vz-\vy)_i\dd t     \\
+\int_{z\in R^d}w^{(2)}(\vx,\vy,\vz)\int_0^1 \sum_i\sum_j \frac{\partial f_i(\vu(1-t))}{\partial S_j}(\vw-\vu(1-t))_j(\vx-\vz)_i \dd t \\
=\int_{\vz\in \rrr^d}w^{(2)}(\vx,\vy,\vz)\int_0^1 \sum_i\sum_j \frac{\partial f_i(\vu(t))}{\partial S_j}(\vz-\vx)_j(\vz-\vy)_i(1-t) \dd t\\
-\int_{\vz\in \rrr^d}w^{(2)}(\vx,\vy,\vz)\int_0^1 \sum_i\sum_j \frac{\partial f_i(\vu(1-t))}{\partial S_j}(\vz-\vy)_j(\vz-\vx)_i(1-t)\dd t.
    \end{split}
\end{equation*}
 We can introduce a mirror node $z'$ of $z$ as follows.
\begin{figure}[H]
\pgfdeclarelayer{nodelayer}
\pgfdeclarelayer{edgelayer}
\pgfsetlayers{nodelayer,edgelayer}   
\begin{center}
\begin{tikzpicture}[scale=1, transform shape]
    \begin{pgfonlayer}{nodelayer}
        \node (0) at (-0.75, 0) {};
        \node (1) at (3, 0) {};
        \node (2) at (0, 1.5) {};
        \node (6) at (-0.8, -0.2) {$x$};
        \node (8) at (3.25, -0.2) {$y$};
        \node (10) at (0, 1.75) {$z$};
        \node (11) at (2.25, -1.5) {};
        \node (12) at (2.25, -1.75) {$z'$};
    \end{pgfonlayer}
    \begin{pgfonlayer}{edgelayer}
        \draw (2.center) to (1.center);
        \draw (1.center) to (0.center);
        \draw (0.center) to (2.center);
        \draw (0.center) to (11.center);
        \draw (1.center) to (11.center);
    \end{pgfonlayer}
\end{tikzpicture}
\end{center}
\end{figure}
The mirroring node has a property that $\vz' - \vx = -(\vz-\vy)$ and
$\vz' - \vy = -(\vz-\vx)$. Using mirror node, the integral of the first-order term is 
\begin{equation*}
    \begin{split}
        \int_{\vz\in \rrr^d}w^{(2)}(\vx,\vy,\vz)f^{(1)}_{xyz}\dd\vz=
        \int_{\vz\in \rrr^d}w^{(2)}(\vx,\vy,\vz)\int_0^1 \sum_i\sum_j \frac{\partial f_i(\vu(t))}{\partial S_j}(\vz-\vx)_j(\vz-\vy)_i(1-t) \dd t \\
-\int_{\vz\in \rrr^d}w^{(2)}(\vx,\vy,\vz)\int_0^1 \sum_i\sum_j \frac{\partial f_i(\vu(1-t))}{\partial S_j}(\vz-\vx)_j(\vz-\vy)_i(1-t)\dd t\\
=
\int_{\vz\in \rrr^d}w^{(2)}(\vx,\vy,\vz)\int_0^1 \sum_i\sum_j (\frac{\partial f_i(\vu(t))}{\partial S_j}-\frac{\partial f_i(\vu(1-t))}{\partial S_j})(\vz-\vx)_j(\vz-\vy)_i(1-t) dt.
    \end{split}
\end{equation*}
Since $\frac{\partial f_i(\vu(t))}{\partial S_j}-\frac{\partial f_i(\vu(1-t))}{\partial S_j}=-\sum_k \frac{\partial^2 f_i(\vu(t))}{\partial S_j\partial S_k}(\vy-\vx)_k(1-2t)+\mathcal O(\delta^2)$, then the integral becomes
\begin{equation*}
    \begin{split}
        -\int_{\vz\in \rrr^d}w^{(2)}(\vx,\vy,\vz)\int_0^1 \sum_i\sum_j \sum_k\frac{\partial^2 f_i(\vu(t))}{\partial S_j\partial S_k}(\vz-\vx)_j(\vz-\vy)_i(\vy-\vx)_k(1-t)(1-2t)\dd t +\mathcal O(\varepsilon^{2+d}\delta^2).
    \end{split}
\end{equation*}
We can decompose 
\begin{equation*}
    \begin{split}
    \sum_i \sum_{j}\sum_k(\vz-\vx)_j(\vz-\vy)_i(\vy-\vx)_k=\sum_i \sum_{j\neq i}\sum_k(\vz-\vx)_j(\vz-\vx)_i(\vy-\vx)_k\\+\sum_i \sum_k(\vz-\vx)_{i}^2(\vy-\vx)_k +\sum_i \sum_{j}\sum_k(\vz-\vx)_j(\vx-\vy)_i(\vy-\vx)_k.
 \end{split}
\end{equation*}
\textbf{(1)}we firstly consider $\sum_i \sum_{j\neq i}\sum_k(\vz-\vx)_j(\vz-\vx)_i(\vy-\vx)_k$ term.\\

When $\delta$ a relative small term for $\varepsilon$. The Taylor expansion of $k(\|\vz-\vy\|^2/\varepsilon^2)$ is
$$k(\|\vz-\vy\|^2/\varepsilon^2))=k(\|\vz-\vx\|^2/\varepsilon^2))+2\varepsilon^{-2}k'(\|\vz-\vx\|^2/\varepsilon^2))(\vz-\vx)^T(\vx-\vy)+\mathcal O\left(\frac{\delta^2\|\vz-\vx\|^2}{\varepsilon^4}\right).$$

Since $\int_{\vz\in \rrr^d}k^2(\|\vz-\vx\|^2/\varepsilon^2))(\vz-\vx)_i(\vz-\vx)_j(\vy-\vx)_k=0$ when $i\neq j$.

Additionally, $\int_{\vz\in \rrr^d}k(\|\vz-\vx\|^2/\varepsilon^2))k'(\|\vz-\vx\|^2/\varepsilon^2))(\vz-\vx)_i(\vz-\vx)_j\sum_k(\vz-\vx)_k(\vx-\vy)_k=0$ since it is odd function.

The last term 
$\int_{\vz\in \rrr^d}k(\|\vz-\vx\|^2/\varepsilon^2))(\vz-\vx)_i(\vz-\vx)_j(\vy-\vx)_k\mathcal O\left(\frac{\delta^2\|\vz-\vx\|^2}{\varepsilon^4}\right)=\mathcal O(\varepsilon^d\delta^3)$.

\bigskip
\textbf{(2)}
Then consider $\sum_i \sum_{j}\sum_k(\vz-\vx)_j(\vx-\vy)_i(\vy-\vx)_k$ term, which is $\mathcal O(\varepsilon^d\delta^3)$, the second term in the expanison of $k(\|\vz-\vy\|^2/\varepsilon^2))$ will result in $\mathcal O(\varepsilon^d\delta^3)$.
\begin{equation*}
    \begin{split}
        \int_{\vz\in \rrr^d}k(\|\vz-\vx\|^2/\varepsilon^2))\frac{2}{\varepsilon^2}\sum_{l,i,j,k}(\vz-\vx)_l(\vx-\vy)_l(\vz-\vx)_j(\vx-\vy)_i(\vy-\vx)_k\dd t\dd\vz=\mathcal O(\varepsilon^d\delta^3).
    \end{split}
\end{equation*}
And the first term of the expansion will result in 0 and the second term will also result in $\mathcal O(\varepsilon^d\delta^3)$.

\textbf{(3)}
Next, we consider 
\begin{equation*}
    \begin{split}
        \int_{\vz\in \rrr^d}w^{(2)}(\vx,\vy,\vz)\int_0^1 \sum_i \sum_k\frac{\partial^2 f_i(\vu(t))}{\partial S_j\partial S_k}(\vz-\vx)_i^2(\vy-\vx)_k(1-t)(1-2t)\dd t \dd\vz.
    \end{split}
\end{equation*}

We can define
\begin{equation*}
    \begin{split}
        C_1 = \frac{1}{\varepsilon^{2+d}}\int_{\vz\in R^d}k^2(\|\vz-\vx\|^2/\varepsilon^2)(\vz-\vx)_i^2\dd\vz.
    \end{split}
\end{equation*}

And the second term of expansion of $k(||z-y||^2/\varepsilon^2)$ would result in 0 and third term will result in $\mathcal O(\varepsilon^d\delta^3)$. Hence,
\begin{equation*}
    \begin{split}
        \int_{z\in R^d}w^{(2)}(x,y,z)\int_0^1 \sum_i \sum_k\frac{\partial^2 f_i(u(t))}{\partial S_j\partial S_k}(z-x)_i^2(y-x)_k(1-t)(1-2t)\dd t \dd z\\
=\varepsilon^{2+d}C_1\int_0^1 \sum_i \sum_j\frac{\partial^2 f_j(u(t))}{\partial S_i\partial S_j}(y-x)_i(1-t)(1-2t)\dd t+\mathcal O(\varepsilon^d\delta^3).
    \end{split}
\end{equation*}

In short, the first order term 
\begin{equation*}
    \begin{split}
        \int_{\vz\in \rrr^d}w^{(2)}(\vx,\vy,\vz)f^{1}_{xyz}dz=-\varepsilon^{2+d}C_1\int_0^1 \sum_i \sum_j\frac{\partial^2 f_j(\vu(t))}{\partial S_i\partial S_j}(\vy-\vx)_i(1-t)(1-2t)\dd t+\mathcal O(\varepsilon^d\delta^3).
    \end{split}
\end{equation*}

\begin{claim}
	The integral of the second-order term $\int_{\vz\in \rrr^d}w^{(2)}(\vx,\vy,\vz)f^{(2)}_{xyz}\dd\vz$ is $\mathcal O(\varepsilon^{2+d}\delta)$.
\end{claim}

The integral of the second-order term is 
\begin{equation*}
    \begin{split}
        &\int_{\vz\in\rrr^d}w^{(2)}(\vx,\vy,\vz)f_{xyz}^{(2)}\dd\vz \\
        &=\int_{\vz\in\rrr^d}w^{(2)}(\vx,\vy,\vz)\int_0^1\sum_i\sum_j \sum_k\frac{\partial^2f_i(\vu(1-t))}{\partial S_j\partial S_k}(\vv-\vu(1-t))_j(\vv-\vu(1-t))_k(\vz-\vy)_i\dd t\dd\vz \\
        &+\int_{\vz\in\rrr^d}w^{(2)}(\vx,\vy,\vz)\int_0^1\sum_i\sum_j \sum_k\frac{\partial^2 f_i(\vu(1-t))}{\partial S_j\partial S_k}(\vw-\vu(1-t))_j(\vw-\vu(1-t))_k(\vx-\vz)_i\dd t\dd\vz \\
         &=\int_{\vz\in\rrr^d}w^{(2)}(\vx,\vy,\vz)\int_0^1\sum_i\sum_j \sum_k\frac{\partial^2 f_i(\vu(t))}{\partial S_j\partial S_k}(\vz-\vx)_j(\vz-\vx)_k(\vz-\vy)_i(1-t)^2\dd t\dd\vz\\
         &-\int_{\vz\in\rrr^d}w^{(2)}(\vx,\vy,\vz)\int_0^1\sum_i\sum_j \sum_k\frac{\partial^2 f_i(\vu(t))}{\partial S_j\partial S_k}(\vz-\vy)_j(\vz-\vy)_k(\vz-\vx)_it^2\dd t\dd\vz.
  \end{split}
\end{equation*}
We can decompose the second-order term into the following terms
\begin{equation}
    \begin{split}
        &=\int_{\vz\in\rrr^d}w^{(2)}(\vx,\vy,\vz)\int_0^1\sum_i\sum_j \sum_k\frac{\partial^2 f_i(\vu(t))}{\partial S_j\partial S_k}(\vz-\vx)_j(\vz-\vx)_k(\vz-\vx)_i(1-t)^2\dd t\dd\vz  \label{eq:1}\\ 
        \end{split}
        \end{equation}
        \begin{equation}
        \begin{split}
         &-\int_{\vz\in \rrr^d}w^{(2)}(\vx,\vy,\vz)\int_0^1\sum_i\sum_j \sum_k\frac{\partial^2 f_i(\vu(t))}{\partial S_j\partial S_k}(\vz-\vx)_j(\vz-\vx)_k(\vy-\vx)_i(1-t)^2\dd t\dd\vz\label{eq:2}\\
         \end{split}
        \end{equation}
        \begin{equation}
        \begin{split}
        &-\int_{\vz\in \rrr^d}w^{(2)}(\vx,\vy,\vz)\int_0^1\sum_i\sum_j \sum_k\frac{\partial^2f_i(\vu(t))}{\partial S_j\partial S_k}(\vz-\vy)_j(\vz-\vy)_k(\vz-\vy)_it^2\dd t\dd\vz\label{eq:3}\\
        \end{split}
        \end{equation}
        \begin{equation}
        \begin{split}
        &-\int_{\vz\in \rrr^d}w^{(2)}(\vx,\vy,\vz)\int_0^1\sum_i\sum_j \sum_k\frac{\partial^2 f_i(\vu(t))}{\partial S_j\partial S_k}(\vz-\vy)_j(\vz-\vy)_k(\vy-\vx)_it^2\dd t\dd\vz\label{eq:4}.
    \end{split}
\end{equation}

Using property of mirror nodes $\vz-\vy=-(\vz'-\vx)$, we can combine \eqref{eq:1} and \eqref{eq:3} to get
\begin{equation*}
    \begin{split}
        \int_{\vz\in \rrr^d}w^{(2)}(\vx,\vy,\vz)\int_0^1\sum_i\sum_j \sum_k\frac{\partial^2 f_i(\vu(t))}{\partial S_j\partial S_k}(\vz-\vx)_j(\vz-\vx)_k(\vz-\vx)_i(t^2+(1-t)^2)\dd t\dd\vz.
    \end{split}
\end{equation*}

Recall that the Taylor expansion of $\kappa(\|\vz-\vy\|^2/\varepsilon^2)$, the first term will be 0  and the third term will result in $\mathcal O(\delta^2\varepsilon^{d+1})$, we only need to consider the second term, the sum of \eqref{eq:1} and \eqref{eq:3} becomes 
\begin{equation*}
    \begin{split}
        2\varepsilon^{-2}\int_{\vz\in \rrr^d}\kappa\kappa'\int_0^1\sum_{i,j,k,l}\frac{\partial^2 f_i(\vu(t))}{\partial S_j\partial S_k}(\vz-\vx)_i(\vz-\vx)_j(\vz-\vx)_k(\vz-\vx)_l(\vx-\vy)_l(t^2+(1-t)^2)\dd t\dd\vz.
    \end{split}
\end{equation*}

The following four conditions will make the above integral non-zero:
\romanenu{1} $i=j\neq k=\ell$, \romanenu{2} $i=k\neq j=\ell$, \romanenu{3} $i=\ell\neq j=k$, and \romanenu{4} $i=j=k=\ell$.

Firstly, if $\kappa_\varepsilon(\cdot)$ is the exponential kernel, we define 
\begin{equation*}
    \begin{split}
        &C_2 = \varepsilon^{-(4+d)}\int_{\vz\in \rrr^d}\kappa^2\left(\frac{\|\vz\|^2}{\varepsilon^2}\right)\vz_1^2\vz_2^2\\
        &C_3=\varepsilon^{-(4+d)}\int_{\vz\in \rrr^d}\kappa^2(\frac{\|\vz\|^2}{\varepsilon^2})\vz_1^4.
    \end{split}
\end{equation*}
Since we have
\begin{equation*}
\begin{gathered}
	\iint \exp(-2x^2 - 2y^2) x^4 \dd x\dd y = \frac{3\pi}{32}; \\
	\iint \exp(-2x^2 - 2y^2) x^2 y^2 \dd x\dd y = \frac{\pi}{32}; \\
     \iint \exp(-2x^2 - 2y^2) x^2  \dd x\dd y = \frac{\pi}{8}.
\end{gathered}
\end{equation*}
It implies that $C_3 - 3C_2$ is zero when $d=2$. For general dimensions ($d>2$), one can use the
identity
\begin{equation*}
	\int_{\vx\in\rrr^d} \exp(-2\|\vx\|^2) g(x_1, x_2) \dd\vx = \left(\frac{\sqrt{2\pi}}{4} \right)^{d-2} \iint \exp(-2x_1^2-2x_2^2) g(x_1,x_2) \dd x_1 \dd x_2.
\end{equation*}
Therefore, $C_3 - 3C_2 = \varepsilon^{-(4+d)}\int_{\vz\in\rrr^d} \exp(-2\|\vz\|^2) (z_1^4 - 3z_1^2z_2^2) \dd\vz = 0$ with the exponential kernel. 

Similarly , we have 
\begin{equation}
\label{eq:c1-c2}
         C_1-4C_2=\varepsilon^{-(d+2)}\int_{\vz\in \rrr^d}\exp(-2\|\vz\|^2)z_1^2-4\varepsilon^{-(d+4)}\int_{\vz\in \rrr^d}\exp(-2\|\vz\|^2)z_1^2z_2^2=0.
\end{equation}

Then we inspect condition \romanenu{1}. The integration is
$$2\varepsilon^{2+d}C_2\int_0^1\sum_i\sum_{j\neq i}\frac{\partial^2 f_j(\vu(t))}{\partial S_i\partial S_j}(\vy-\vx)_i(t^2+(1-t)^2)\dd t\dd\vz.$$

In \romanenu{2}, the integral is the same as the integral of \romanenu{1}, so the sum of \romanenu{1} and \romanenu{2} is 
$$4\varepsilon^{2+d}C_2\int_0^1\sum_i\sum_{j\neq i}\frac{\partial^2 f_j(\vu(t))}{\partial S_i\partial S_j}(\vy-\vx)_i(t^2+(1-t)^2)\dd t\dd\vz.$$

For condition \romanenu{3}, the integral is 
\begin{equation*}
    2\varepsilon^{2+d}C_2\int_0^1\sum_i\sum_{j\neq i}\frac{\partial^2 f_i(\vu(t))}{\partial S_j^2}(\vy-\vx)_i(t^2+(1-t)^2)\dd t\dd\vz.
\end{equation*}

For condition \romanenu{4}, the integral is
\begin{equation*}
    2\varepsilon^{2+d}C_3\int_0^1\sum_i\frac{\partial^2 f_i(\vu(t))}{\partial S_i^2}(\vy-\vx)_i(t^2+(1-t)^2)\dd t\dd\vz.
\end{equation*}

From the fact $C_3=3C_2$, the sum of \eqref{eq:1} and \eqref{eq:3} is 
\begin{equation*}
    \begin{split}
        &4\varepsilon^{2+d}C_2\int_0^1\sum_i\sum_{j}\frac{\partial^2 f_j(\vu(t))}{\partial S_i\partial S_j}(\vy-\vx)_i(t^2+(1-t)^2)\dd t\dd\vz\\
        &+2\varepsilon^{2+d}C_2\int_0^1\sum_i\sum_{j}\frac{\partial^2 f_i(\vu(t))}{\partial S_j^2}(\vy-\vx)_i(t^2+(1-t)^2)\dd t\dd\vz.
    \end{split}
\end{equation*}

The sum of \eqref{eq:2} and \eqref{eq:4} is 
\begin{equation*}
    \begin{split}
        -\int_{\vz\in \rrr^d}w^{(2)}(\vx,\vy,\vz)\int_0^1\sum_i\sum_j \sum_k\frac{\partial^2 f_i(\vu(t))}{\partial S_j\partial S_k}(\vz-\vx)_j(\vz-\vx)_k(\vy-\vx)_i((1-t)^2+t^2)\dd t\dd\vz.
    \end{split}
\end{equation*}

Similar to what we do in the first order term, we can get the sum of \eqref{eq:2} and \eqref{eq:4} is 
$$-\varepsilon^{2+d}C_1\int_0^1\sum_i\sum_{j}\frac{\partial^2 f_i(\vu(t))}{\partial S_j^2}(\vy-\vx)_i(t^2+(1-t)^2)\dd t\dd\vz.$$

Therefore, sum the first order and second order term up, we obtain
\begin{equation*}
    \begin{split}
    &2\varepsilon^{2+d}C_2\int_0^1\sum_i\sum_{j}\frac{\partial^2 f_j(\vu(t))}{\partial S_i\partial S_j}(\vy-\vx)_i(t^2+(1-t)^2)\dd t\dd\vz\\&+\varepsilon^{2+d}C_2\int_0^1\sum_i\sum_{j}\frac{\partial^2 f_i(\vu(t))}{\partial S_j^2}(y-x)_i(t^2+(1-t)^2)\dd t\dd\vz\\&-\frac{1}{2}\varepsilon^{2+d}C_1\int_0^1\sum_i\sum_{j}\frac{\partial^2 f_i(\vu(t))}{\partial S_j^2}(\vy-\vx)_i(t^2+(1-t)^2)\dd t\dd\vz\\&-\varepsilon^{2+d}C_1\int_0^1 \sum_i \sum_j\frac{\partial^2 f_j(\vu(t))}{\partial S_i\partial S_j}(\vy-\vx)_i(1-t)(1-2t)\dd t.
    \end{split}
\end{equation*}

Note that 
\begin{equation*}
    \begin{split}
        C_1-4C_2=\varepsilon^{-(d+2)}\int_{\vz\in \rrr^d}\exp(-2\|\vz\|^2)z_1^2-4\varepsilon^{-(d+4)}\int_{\vz\in \rrr^d}\exp(-2\|\vz\|^2)z_1^2z_2^2=0.
    \end{split}
\end{equation*}

The $t^2 + (1-t)^2 = 1 - 2t(1-t)$ term and $(1-t)(1-2t)$ term in the line integral can be removed with the
following technique.
Let $g_{ij}(\vu(t)) = \frac{\partial^2 f_i(\vu(t))}{\partial s_j^2}$ and
$\tau = t - 1/2$. Considering the integral term contains $t(1-t)$, we have
\begin{equation*}
	\int_0^1 g_{ij}(\vu(t)) t(1-t) \dd t = \int_0^1 g_{ij}\left(\frac{\vx+\vy}{2} + (\vy-\vx)\tau\right) \left(\inv{4} - \tau^2\right) \dd\tau.
\end{equation*}
The $1/4$ terms become an unbiased line integral, i.e., $\inv{4}\int_0^1 g_{ij}(\vu(t)) \dd t$.
Therefore, we can focus only on the $\tau^2$ part. Taylor expanding the $g_{ij}$ yields
\begin{equation*}
	\int_{-1/2}^{1/2} \tau^2 \left[g_{ij}\left(\frac{\vx+\vy}{2}\right) + \cancel{\tau(\vy-\vx)^\top\nabla g_{ij}} + \mathcal O(\|\vy-\vx\|^2) \right] \dd\tau = \inv{12}\cdot g_{ij}\left(\frac{\vx+\vy}{2}\right).
\end{equation*}
The first-order term becomes zero using the odd function symmetry, and
the term $1/12$
comes from $\int_{-1/2}^{1/2} \tau^2 \dd\tau$. By Jensen's inequality,
we have
$g_{ij}((\vx+\vy)/2) = \inv{2} (g_{ij}(\vx) + g_{ij}(\vy)) + \mathcal O(\|\vy-\vx\|^2)$. Therefore,
\begin{equation*}
\begin{split}
	\sum_i \int_{-\inv{2}}^{\inv{2}} g_{ij}(\vu(\tau))(\vy-\vx)_i \tau^2 \dd\tau \,&=\,\inv{12}\cdot \inv{2}\left(g_{ij}(\vx) + g_{ij}(\vy)\right)(\vy-\vx)_i + \mathcal O(\delta^2) \\
	\,&=\, \inv{12} \sum_i \int_0^1 [g_{ij}(\vx) + (g_{ij}(\vy) - g_{ij}(\vx)) t](\vy-\vx)_i \dd t +\mathcal O(\delta^2) \\
	\,&=\, \inv{12}\sum_i\int_0^1 g_{ij}(\vu(t)) (\vy-\vx)_i \dd t + \mathcal O(\delta^2).
\end{split}
\end{equation*}
The 
last equality holds from Lemma \ref{thm:linear-approx-of-line-int}. Now we have,
\begin{equation*}
	\int_0^1 \sum_i g_{ij}(\vu(t)) (\vy-\vx)_i t(1-t) \dd t = \left(\inv{4} - \inv{12}\right) \int_0^1 \sum_i g_{ij}(\vu(t)) (\vy-\vx)_i \dd t + \mathcal O(\delta^2).
\end{equation*}
Finally, using the fact that $t^2 + (1-t)^2 = 1 - 2t(1-t)$, we get 
\begin{equation*}
\begin{split}
	\,&\pequal\,\varepsilon^{2+d} C_2 \int_0^1 \sum_i\sum_j \frac{\partial^2 f_i}{\partial S_j^2}(\vu(t)) (\vy-\vx)_i (t^2 + (1-t)^2) \dd t \\
	\,&=\,  \frac{2}{3} \varepsilon^{2+d} C_2 \int_0^1 \sum_i\sum_j \frac{\partial^2 f_i}{\partial S_j^2}(\vu(t)) (\vy-\vx)_i \dd t + \mathcal O(\varepsilon^{2+d}\delta^2).
\end{split}
\end{equation*}
Similarly, we can remove $(1-t)(1-2t)$ term by
\begin{equation*}
    \begin{split}
        \,&\pequal\,\varepsilon^{2+d} C_1 \int_0^1 \sum_i\sum_j \frac{\partial^2 f_j}{\partial S_i\partial S_j}(\vu(t)) (\vy-\vx)_i (1-t)(1-2t) \dd t \\
	\,&=\,  \frac{1}{6} \varepsilon^{2+d} C_2 \int_0^1 \sum_i\sum_j \frac{\partial^2 f_j}{\partial S_i\partial S_j}(\vu(t)) (\vy-\vx)_i \dd t + \mathcal O(\varepsilon^{2+d}\delta^2).
    \end{split}
\end{equation*}
\begin{remark}
Note that the $2/3$ term corresponds to $\int_0^1 t^2 + (1-t)^2 dt$. The result can be
interpreted as follows. When $\vx$ is sufficiently close to $\vy$, the
$g_{ij}(\vu(t))$ term is roughly
a constant. Therefore, the integral can be approximately done separately, i.e.,
$\int_0^1 g_{ij}(\vu(t)) (t^2 + (1-t)^2) \dd t \approx \int_0^1 (t^2 + (1-t)^2) \dd t \int_0^1 g_{ij}(\vu(t)) \dd t = \frac{2}{3} \int_0^1 g_{ij}(\vu(t)) \dd t$.
\end{remark}

The proof is thus completed by putting things together.
\begin{equation*}
\begin{split}
	 \int_{\vz\in\rrr^d} w_T(\vecxyz) f_{xyz} \mathsf{d}\mu(\vec z)&=\frac{2}{3}\varepsilon^{2+d}C_2\int_0^1 \sum_i\sum_j(\frac{\partial^2 f_j(\vu(t))}{\partial S_i\partial S_j}-\frac{\partial^2 f_i(\vu(t))}{\partial S_j^2})(\vy-\vx)_i\dd t \\&+ \mathcal O(\varepsilon^{d}\delta^3)+\mathcal O(\varepsilon^{4+d})+\mathcal O(\varepsilon^{2+d}\delta^2).
\end{split}
\end{equation*}
Using Proposition S7 $\updelta\dd \zeta_1=\sum_i\sum_{j\neq i}(\frac{\partial^2 f_j}{\partial S_i\partial S_j}-\frac{\partial^2 f_i}{\partial S_j^2})d s_i$, one can get following asymptotic expansion
\begin{equation*}
\begin{split}
    \int_{\vz\in\rrr^d} w_T(\vecxyz) f_{xyz} \mathsf{d}\mu(\vec z)&=\frac{2}{3}\varepsilon^{2+d}C_2\int_0^1[(\updelta \dd\Vec{\upsilon(\vu(t))})]^\top \vec u'(t)\dd t\\
    &+ \mathcal O(\varepsilon^{d}\delta^3)+\mathcal O(\varepsilon^{4+d})+\mathcal O(\varepsilon^{2+d}\delta^2).
    \end{split}
\end{equation*}

This completes the proof.
\end{proof}

The second step is to provide the error terms induced by the change of variables from
the ambient space $\M \subseteq \rrr^D$ to the local tangent coordinate
$\T_\vx\M\in\rrr^d$ defined by $\vx$.
In the following, there are two coordinate systems that we mainly focus on: the
{\em normal coordinate} and {\em tangent plane coordinate} at $\vx$.
First, we define
$\vx, \vy, \vz \in\M\subseteq\rrr^D$ to be the points in the {\em ambient space}.
We then let $\vx_s, \vy_s, \vz_s \in\rrr^d$ be the same set of points in
the {\em normal coordinate}
in the neighborhood of point $\vx$.
Points in the {\em tangent plane coordinate} defined by tangent plane
$\mathcal T_\vx\M$ of $\vx$
are denoted $\vx_p, \vy_p, \vz_p\in\rrr^d$.
Note that by definition, the origin of these two coordinate systems is $\vx$, implying that
$\vx_s$, $\vx_p$ are zero vectors.
The following lemma generalizes the
result in \cite{CoifmanRR.L:06} for triangular relations $\vx, \vy, \vz$.

\begin{lemma}[Error terms induced from the change of coordinates]
Define $Q_{\vx, m}(\cdot)$ be a homogenous polynomial of order $m$ with
coefficient defined by $\vx$.
Further, let $\vec\gamma$ be a geodesic curve in $\M$ connecting two points $\va \in \{\vx, \vy, \vz\}$ and $\vb \in\{\vecxyz\}$.
If $\vy, \vz\in\M$ are in a Euclidean ball of
radius $\varepsilon$ around $\vx$, then for $\varepsilon$ sufficiently small,
we have the following four approximations.
\begin{subequations}
\begin{gather}
	[\vz_s]_i = [\vz_p]_i + \mathcal O(\varepsilon^3); \label{eq:lemma-coordinate-change-points}\\
	\mathrm{det}\left(\frac{\dd \vz}{\dd \vz_p}\right) = 1 + Q_{\vx, 2}(\vz_p) + \mathcal O(\varepsilon^3); \label{eq:lemma-coordinate-change-volumes} \\
	\|\va - \vb\|^2 = \|\va_p - \vb_p\| + Q_{\vx,4}(\va_p, \vb_p) + \mathcal O(\varepsilon^5); \label{eq:lemma-coordinate-change-metrics} \\
	\vec\gamma(t) = \va_p + (\vb_p-\va_p)t + \mathcal O(\varepsilon^3). \label{eq:lemma-coordinate-change-geodesics}
\end{gather}
\end{subequations}
\label{thm:coordinate-change-errors}
\end{lemma}

\begin{proof}
\eqref{eq:lemma-coordinate-change-points} and \eqref{eq:lemma-coordinate-change-volumes}
follow naturally from Lemma 6 and 7 of \cite{CoifmanRR.L:06} since there is no triplet-wise
($\vx, \vy, \vz$) relationship.

For \eqref{eq:lemma-coordinate-change-metrics}, if either $\va=\vx$ or
$\vb = \vx$ the result follows from Lemma 7 of \cite{CoifmanRR.L:06}.
Now we consider the case that neither $\va$ nor $\vb$ are equal to $\vx$.
Without loss of generality, let $\va = \vz$ and $\vb = \vy$.
Note that the submanifold in the ambient space is locally
parameterized by $(\vv_p, g(\vv_p))\in\rrr^D$ for $g: \rrr^d \to \rrr^{D-d}$.
Since $\vz$ is not the original of the
normal coordinate system,
one can do a Taylor expansion of $g$ from $\vx_p$.
Denote by $\vs = (s_1,\cdots, s_d)$ the local coordinate of $\T_\vx\M$ at point $\vx$.
By definition, we have
$g_i(\vx_p) = 0$ and $\frac{\partial g_i(\vx_p)}{\partial s_j} = 0$
for $i, j \in [d]$.
We write $g_i(\vz_p) = H_{i,\vx}(\vz_p) + \mathcal O(\varepsilon^3)$,
and $g_i(\vy) = H_{i,\vx}(\vy_p) + \mathcal O(\varepsilon^3)$ by Taylor expansion.
Here, $H$ is the Hessian of $g_i$ at the origin.
Hence, we have
$\|\vz-\vy\|^2  = \|\vz_p-\vy_p\|^2 + \sum_{i=d+1}^D (g_i(\vz_p) - g_i(\vy_p))^2 = \|\vz-\vy\|^2 + Q_{\vx, 4}(\vy_p, \vz_p) + \mathcal O(\varepsilon^5)$.

To prove \eqref{eq:lemma-coordinate-change-geodesics}, first, note that
one can project the geodesic onto $\T_\vx\M$ with $\mathcal O(\varepsilon^3)$
by \eqref{eq:lemma-coordinate-change-points}, i.e.,
$\vec\gamma = \vec\gamma_{\T_\vx\M}(t) + \mathcal O(\varepsilon^3)$.
Therefore, we only need to consider the error term caused by approximating
the projected geodesic
$\vec\gamma_{\T_\vx\M}(t)$ by a straight line $\va_p + (\vb_p-\va_p)t$.
Denote $\vec\gamma_{\T_\vx\M} = \vec\gamma_\T$ for simplicity,
further let  $\mathrm{dist}_\T(\vy, \vz)$ be the arc length of the
projected geodesic $\vec\gamma_\T$ connecting $\va_p, \vb_p$,
from Taylor expansion, one has
\begin{equation}
	\vec\gamma_\T(t) = \vec\gamma_\T(0) + t\vec\gamma_\T'(0)\mathrm{dist}_\T(\va_p, \vb_p) + \inv{2}t^2\mathrm{dist}_\T^2(\va_p, \vb_p) \vec\gamma_\T''(0) + \mathcal O(\varepsilon^3).
	\label{eq:taylor-expansion-proj-geodesic}
\end{equation}
The first step is to show that $\gamma''_\T(0) = \mathcal O(\varepsilon)$.
Since $\mathrm{dist}_\T(\va_p, \vb_p) = \mathcal O(\varepsilon)$,
if $\gamma''_\T(0) = \mathcal O(\varepsilon)$, one can bound the
second-order term by $\mathcal O(\varepsilon^3)$.
Note that if $\va = \vx$, by the definition of geodesic (covariant derivative is
normal to $\M$), we have $\gamma''_\T(0) = 0$. Hence, the error term is bounded
by $\mathcal O(\varepsilon^3)$. If $\vb = \vx$, one can switch $\vb$ with $\va$,
and the same proof can go through. We now deal with
the situation when $\va, \vb\neq \vx$.
One can show that the {\em Levi-Civita connection} of the local (orthonormal) basis vector
can be written as a linear combination of the basis vectors
by the {\em method of moving frame} \cite{ClellandJN:17}, 
with coefficients determined by the local curvature/torsion at that point.
The first-order approximation of the local basis vector $[\ve_\va]_i \in\rrr^D$
for $i \in [d]$ at $\va$ can be approximated
by $[\ve_\va]_i = [\ve_\vx]_i + \mathcal O(\varepsilon)$ using the {\em method of moving frame} along
the geodesic connecting $\vx \to \va$.
By the Gram-Schmidt process, the basis of the normal space $[\vn_\va]_j \in \rrr^D$
at $\va$ for $j = d+1,\cdots, D$ can also be written as a first-order approximation of the normal basis at $\vx$, i.e., $[\vn_\va]_j = [\vn_\vx]_j + \mathcal O(\varepsilon)$.
Since the second derivate of geodesic $\vec\gamma_\M''(0)$ at $\va$ is in the normal space $\T^\perp_\va\M$,
the covariant derivative of the projected geodesic
$\vec\gamma''(0)$ at $\va$ onto $\T_\vx\M$ can be bounded by
$\mathcal O(\varepsilon)$ if the manifold $\M$ has bounded curvature.
The term including $\vec\gamma''(0)$ can thus be bounded by $\mathcal O(\varepsilon^3)$, given that $\va$ and $\vb$ are
sufficiently close to $\vx$. Hence, \eqref{eq:taylor-expansion-proj-geodesic} becomes
\begin{equation*}
	\vec\gamma_\T(t) = \vec\gamma_\T(0) + t\vec\gamma_\T'(0)\mathrm{dist}_\T(\va_p, \vb_p) + \mathcal O(\varepsilon^3).
\end{equation*}
Plugging in $t = 1$ gives us $\vec\gamma'(0)\mathrm{dist}_\T(\va_p, \vb_p) = \vb_p-\va_p + \mathcal O(\varepsilon^3)$.
The following approximation of $\vec\gamma(t)$ thus holds.
\begin{equation*}
	\vec\gamma(t) = \va_p + (\vb_p-\va_p)t + \mathcal O(\varepsilon^3).
\end{equation*}
This completes the proof.
\end{proof}

With the above two lemmas in hand, we can finally start to prove Proposition
\ref{thm:asymptotic-expansion-L-1-up}.

\begin{proofof}{Proposition \ref{thm:asymptotic-expansion-L-1-up}}
Note that from \eqref{eq:lemma-coordinate-change-geodesics}, the geodesic
curve can be approximated by a straight line in the local tangent plane $\T_\vx\M$
with error $\mathcal O(\varepsilon^3)$. Therefore,
$f_{xyz} = \oint_\M \vec{\upsilon} = f_{x_py_pz_p} + \mathcal O(\varepsilon^4)$.
A similar expansion as in Lemma \ref{thm:asymptotic-expansion-L-1-up-Rd} thus holds.
Let $\kappa_\varepsilon(\vz, \vx) = \kappa\left(\frac{\|\vz-\vx\|^2}{\varepsilon^2}\right)$,
from a similar analysis as Lemma \ref{thm:asymptotic-expansion-L-1-up-Rd}, we have
\begin{align*}
	\,&\pequal\,\int_{\vz\in\M} w_T(\vecxyz) f_{xyz} \dd\vz \stackrel{\romanenu{1}}{=} \int_{\substack{\vz\in\M \\ \max(\|\vz-\vy\|, \|\vz-\vx\|) < \varepsilon^\gamma}} w_T(\vecxyz) f_{xyz} \dd\vz +\mathcal O(\varepsilon^{4+d}) \\
	\,&\stackrel{\romanenu{2}}{=}\,\int_{\vz\in\rrr^d}  \Bigg[\left(\kappa_\varepsilon(\vz, \vx) \kappa_\varepsilon(\vz, \vy) + \left(\frac{Q_{\vx, 4}(\vz, \vy)}{\varepsilon^2}\right)\left(\kappa'_\varepsilon(\vz, \vx) \kappa_\varepsilon(\vz, \vy)+\kappa_\varepsilon(\vz, \vx) \kappa'_\varepsilon(\vz,\vy)\right) \right) \\
	\,&\pequal\, \phantom{\int_{z\in\rrr^d}\Bigg[} \cdot f_{xyz} \cdot \left((1 + Q_{\vx, 2}(\vz)\right)\Bigg] \dd\vz + \mathcal O(\varepsilon^{4+d}) \\
	\,&=\frac{2}{3}\varepsilon^{2+d}C_2\int_0^1[(\updelta \dd\Vec{\upsilon(\vu(t))})]^\top \vec u'(t)\dd t+ \mathcal O(\varepsilon^{d}\delta^3)+\mathcal O(\varepsilon^{4+d})+\mathcal O(\varepsilon^{2+d}\delta^2).
\end{align*}
Equality $\romanenu{1}$ holds by Lemma \ref{thm:error-bound-localization}.
Equality $\romanenu{2}$ is valid by first projecting $\vecxyz$ to
$\T_\vx\M$, changing the variables $\vx_p,\vy_p,\vz_p \to \vecxyz$,
and using Lemma \eqref{thm:error-bound-localization} again.
Terms consisting of $Q_{p,4}(\vx, \vy)$ or $Q_{p, 2}(\vx,\vy)$ are in the high order , hence, they can be merged into the last error term $\mathcal O(\varepsilon^{4+d})$.
In Lemma \ref{thm:asymptotic-expansion-L-1-up-Rd},
the differentiation is in the
local coordinate residing on $\T_\vx\M$.
One can change the differentiation to the partial derivative on the normal coordinate system
by \eqref{eq:lemma-coordinate-change-volumes}.
Additionally, one can again approximate the line integral
$\int_0^1 (\updelta\dd\vec{\upsilon})(\vu(t))^\top \vu'(t) \dd t = \int_0^1 (\updelta\dd\vec{\upsilon})(\vec\gamma(t))^\top \vec\gamma'(t) \dd t + \mathcal O(\varepsilon^{3+d}\delta)$
by \eqref{eq:lemma-coordinate-change-geodesics}, where $\vec\gamma$ is the
geodesic connecting $\vx, \vy$. Since $\delta$ decays faster than $\varepsilon$, we can merge  $\mathcal O(\varepsilon^{3+d}\delta)$ to $\mathcal O(\varepsilon^{4+d})$. It implies
\begin{equation*}
\begin{split}
	\int_\M w_T(\vecxyz)f_{xyz}\mathsf{d}\mu(\vec z) =& 
\frac{2}{3}\varepsilon^{2+d}C_2\int_0^1 [\updelta\dd\vec{\upsilon}(\vec\gamma_\M(t))]^\top \vec\gamma_\M'(t) \dd t \\&+\mathcal O(\varepsilon^{4+d}) + \mathcal O(\varepsilon^{d+2}\delta^2)+\mathcal O(\varepsilon^{d}\delta^3).
 \end{split}
\end{equation*}
 This completes the proof.
\end{proofof}

\subsection{Proof of Theorem \ref{thm:pointwise-bias-L1-up}}
Firstly, we need the following lemma to measure $ w_E(\vecxy)$
\begin{lemma}[Kernel weight for edge]
\label{lemma:weight for edge}
\begin{equation*}
    \begin{split}
        w_E(\vy,\vz)&=\int_{\vz\in \rrr^d}w_T(\vy,\vz,\vv)d\mu(\vv)=\int_{\vz\in \rrr^d}\exp\left(-\frac{\|\vv-\vy\|^2}{\varepsilon^2}\right)\exp\left(-\frac{|\|\vv-\vz\|^2}{\varepsilon^2}\right)\dd\vz\\ &=\int_{\vz\in \rrr^d}\exp\left(-\frac{\sum_i(y_i-v_i)^2+\sum_i(z_i-v_i)^2}{\varepsilon^2}\right)\dd\vz\\ &=\int_{\vz\in \rrr^d}\exp\left(-\frac{2\sum_i(v_i-\frac{y_i+z_i}{2})^2+\frac{1}{2}\sum_i(y_i-z_i)^2}{\varepsilon^2}\right)\dd\vz\\ &=C_0\exp\left(-\frac{\sum_i(y_i-z_i)^2}{2\varepsilon^2}\right)=\varepsilon^dC_0\exp\left(-\frac{\|\vy-\vz\|^2}{2\varepsilon^2}\right),
    \end{split}
\end{equation*}
where we define $C_0 = \frac{1}{\varepsilon^d}\int_{\vz\in \rrr^d}\exp\left(-\frac{2|\vz\|^2}{\varepsilon^2}\right)\dd\vz$.
\end{lemma}
\label{sec:proof-of-pointwise-bias-L1-up}
\begin{proofof}{Theorem \ref{thm:pointwise-bias-L1-up}}
Using $\LL_1^\mathrm{up} = \vW_E^{-1}\cdot \vB_T\vW_T\vB_T^\top$ and Lemma
\ref{thm:discrete-form-L-up}, the expected value becomes
\begin{equation}
\begin{split}
	\,&\pequal\,  \mathbb E\left[(\LL_1^\mathrm{up})_{[x,y]}\right] = \mathbb E_{z,v} \left[\frac{\frac{1}{n}\sum_{z\notin\{x,y\}} w_T(\vecxyz) f_{xyz}}{\frac{1}{n}\sum_{v\notin\{x,y\}} w_T(\vecxy,\vv)}\right] \\
	\,&=\, \mathbb E_z\left[\frac{1}{n}\sum_{z\notin\{x,y\}} w_T(\vecxyz) f_{xyz} \right] \cdot \mathbb E_v \left[\frac{w_E(\vecxy)}{\frac{\varepsilon^{-d}}{n}\sum_{v\notin\{x,y\}} w_T(\vecxy,\vv)}\right] \cdot \frac{1}{w_E(\vx,\vy)}.
\end{split}
\label{eq:pointwise-bias-proof-exp-decomp}
\end{equation}
The last equality holds by the independence of random variables $z, v$.
By {\em Monte-Carlo} approximation,
\begin{equation*}
\begin{split}
	\mathbb E\left[\inv{n} \sum_z w_T(\vec x, \vec y, \vec z) f_{xyz}\right] \,&=\frac{2}{3}\varepsilon^{2+d}C_2\int_0^1 [\updelta\dd\vec{\upsilon}(\vec\gamma_\M(t))]^\top \vec\gamma_\M'(t) \dd t \\&+\mathcal O(\varepsilon^{4+d}) + \mathcal O(\varepsilon^{d+2}\delta^2)+\mathcal O(\varepsilon^{d}\delta^3).
\end{split}
\end{equation*}
The last equality holds by Proposition \ref{thm:asymptotic-expansion-L-1-up}.
It is not difficult to show that
$\mathbb E\left[\inv{n} \sum_z w_T(\vec x, \vec y, \vec z) \right] = w_E(\vecxy)$ following the proof of
Proposition \ref{thm:asymptotic-expansion-L-1-up}, implying that the $\mathbb E_v$ term
in \eqref{eq:pointwise-bias-proof-exp-decomp}
is $1 + \mathcal O(n^{-1})$ using the standard ratio estimator. Therefore, with $w_E(\vx,\vy)=\varepsilon^dC_0\exp(-\frac{\|\vy-\vx\|^2}{2\varepsilon^2})$,
one has
\begin{equation*}
\begin{split}
     \mathbb E\left[(\LL_1^\mathrm{up})_{[x,y]}\right]& = \frac{2}{3}\varepsilon^{2}\frac{C_2}{C_0}\int_0^1 [\updelta\dd\vec{\upsilon}(\vec\gamma_\M(t))]^\top \vec\gamma_\M'(t) \dd t \\&+\mathcal O(\varepsilon^{4}) + \mathcal O(\varepsilon^{2}\delta^2)+\mathcal O(\delta^3)+\mathcal O(n^{-1}\varepsilon^{2}\delta).
  \end{split}
\end{equation*}
Let $\delta=\mathcal O(\varepsilon^{\frac{3}{2}})$, one can get
\begin{equation}
    \begin{split}
        \mathbb E\left[(\LL_1^\mathrm{up})_{[x,y]}\right]& = \frac{2}{3}\varepsilon^{2}\frac{C_2}{C_0}\int_0^1 [\updelta\dd\vec{\upsilon}(\vec\gamma_\M(t))]^\top \vec\gamma_\M'(t) \dd t +\mathcal O(\varepsilon^{4})+\mathcal O(\varepsilon^{\frac{7}{2}}n^{-1}).
    \end{split}
\end{equation}
This completes the proof.
\end{proofof}

\section{Proofs of the pointwise convergence of the {\em down} Helmholtzian}
\label{sec:pointwise-consistency-L1-down}
    \subsection{Proof of Proposition \ref{thm:asymptotic-expansion-L-1-down}}
    Similar to the proof of Proposition  \ref{thm:asymptotic-expansion-L-1-up}, we are interested in showing asymptotic expansion 
    \begin{equation*}
        \begin{split}
             \int_{\vz\in \rrr^d}w_E(\vy,\vz)f_{zy}d\mu(z)-\int_{\vz\in \rrr^d}w_E(\vx,\vz)f_{zx}d\mu(\vz)\to c\int_0^1 [\dd\updelta \vec{\upsilon}(\vec\gamma_\M(t))]^\top \vec\gamma_\M'(t) \dd t.
        \end{split}
    \end{equation*}
From Corollary \ref{thm:pure-curl-grad-1-Laplacian} in  \ref{sec:lemmas-exterior-calculus}, one has
$\dd\updelta\vec{\upsilon} = -\sum_i\sum_j \frac{\partial^2 f_j}{\partial S_i\partial S_j} \mathsf{d}s_i$
in local coordinate $(s_1, \cdots, s_d)$ with $\vec{\upsilon} = \sum_i f_i \mathsf{d}s_i$.
\begin{lemma}[Asymptotic expansion of $\vec{\mathcal L}_1^\mathrm{down}$ in $\rrr^d$]
\label{thm:asymptotic-expansion-L-1-down-Rd}
Under Assumption \ref{assu:manifold-data-assu}--\ref{assu:kernel-assu},
and further assume With the choice of exponential kernel $\kappa_\varepsilon(u) = \exp(-u)$ and $\vec{\upsilon} = (f_1, \cdots, f_d) \in\mathcal C^4(\M)$.
Let $\vu(t) = \vx + (\vy - \vx)t$ for $t$ in $[0, 1]$ be a parameterization of the straight line between nodes
$x, y$ and $\vu'(t) = \dd\vu(t) / \dd t$.
With the choice of exponential kernel $\kappa(u)=\exp(-u)$, one has the following asymptotic expansion
\begin{equation}
\label{eq:asymptotic expansion of down}
    \begin{split}
        &\int_{\vz\in \rrr^d}\bigg[w^{(1)}(\vy,\vz)f_{yz}-w^{(1)}(\vx,\vz)f_{xz}\bigg]dz\\=&\frac{2}{3}C_0C_1\varepsilon^{2d+2}\int_0^1[(\dd\updelta \Vec{\upsilon(\vu(t))})]^\top \vec u'(t)\dd t+\mathcal O(\varepsilon^{2d+4})+\mathcal O(\varepsilon^{2d+1}\delta^2).
    \end{split}
\end{equation}
\end{lemma}
\begin{proof} 
Similar to what we have done in Lemma \ref{thm:asymptotic-expansion-L-1-up-Rd}, from expansion of \ref{taylyor expansion} of Lemma \ref{thm:asymptotic-expansion-L-1-up-Rd},With the above expansion, we denote the \nth{0}-, \nth{1}-, and \nth{2}-order terms to be the following

\begin{equation*}
    f_{yz}=f_{yz}^{(0)}+f_{yz}^{(1)}+f_{yz}^{(2)}+\mathcal O(\varepsilon^3).
\end{equation*}

And 
\begin{equation}
\begin{split}
	&\int_{\vz\in \rrr^d}w^{(1)}(\vy,\vz)f_{zy}\dd\vz-\int_{\vz\in \rrr^d}w^{(1)}(\vx,\vz)f_{zx}\dd\vz \\&=  \int_{\vz\in \rrr^d}[w^{(1)}(\vy,\vz)(f_{yz}^{(0)}+f_{yz}^{(1)}+f_{yz}^{(2)})-w^{(1)}(\vx,\vz)(f_{xz}^{(0)}+f_{xz}^{(1)}+f_{xz}^{(2)})]dz + \mathcal O(\varepsilon^{4+2d}).
	\label{eq:fxyz-taylor-expansion}
 \end{split}
\end{equation}
\begin{claim}[constant term] The integral of constant term is 0.\\
The above claim is true since
\begin{equation}
    \begin{split}
    &\int_{\vz\in \rrr^d}[w^{(1)}(\vy,\vz)f_{yz}^{(0)}-w^{(1)}(\vx,\vz)f_{xz}^{(0)}]\dd\vz\\=&\int_{\vz\in \rrr^d}w^{(1)}(\vy,\vz)\int_0^1\sum_if_i(\vu(1-t))(\vy-\vz)_i\dd t\dd\vz\\&-\int_{\vz\in \rrr^d}w^{(1)}(\vx,\vz)\int_0^1\sum_if_i(\vu(1-t))(\vx-\vz)_i\dd t\dd\vz=0.
    \end{split}
\end{equation}
\end{claim}
\begin{claim}[First order term]
The integral of the first order term is 0. Since  
\begin{equation*}
    \begin{split}
        \int_{\vz\in \rrr^d}w^{(1)}(\vy,\vz)\int_0^1\sum_{i,j}\frac{\partial f_i(\vx+t(\vy-\vx))}{\partial S_j}(\vz-\vy+\vy-\vx)_j(\vy-\vz)_i(1-t)\dd t\dd\vz\\
        -\int_{\vz\in \rrr^d}w^{(1)}(\vx,\vz)\int_0^1\sum_{i,j}\frac{\partial f_i(\vx+t(\vy-\vx))}{\partial S_j}(\vz-\vx+\vx-\vy)_j(\vx-\vz)_it\dd t\dd\vz=0.
    \end{split}     
\end{equation*}

This claim holds from $\int_{\vz\in\rrr^d}w^{(1)}(\vy,\vz)(\vy-\vx)_j(\vy-\vz)_i(1-t)\dd\vz=0$ and 
$\int_{\vz\in\rrr^d}w^{(1)}(\vy,\vz)(\vz-\vy)_j(\vy-\vz)_i(1-t)\dd\vz=\int_{\vz\in\rrr^d}w^{(1)}(\vx,\vz)(\vz-\vx)_j(\vx-\vz)_i(1-t)\dd\vz$
\end{claim}

\smallskip
\begin{claim}[Second order term]
Similar tricks applied to the second order term. The integral is 
\begin{equation*}
    \begin{split}
        &\int_{\vz\in \rrr^d}w^{(1)}(\vy,\vz)\int_0^1\sum_{i,j,k}\frac{\partial^2f_i(\vx+t(\vy-\vx))}{\partial S_j \partial S_k}(\vz-\vx)_j(\vz-\vx)_k(\vy-\vz)_i(1-t)^2\dd t\dd\vz\\
        &-\int_{\vz\in \rrr^d}w^{(1)}(\vx,\vz)\int_0^1\sum_{i,j,k}\frac{\partial^2f_i(\vx+t(\vy-\vx))}{\partial S_j \partial S_k}(\vz-\vy)_j(\vz-\vy)_k(\vx-\vz)_it^2\dd t\dd\vz\\
        &=\int_{\vz\in \rrr^d}w^{(1)}(\vy,\vz)\int_0^1\sum_{i,j,k}\frac{\partial^2f_i(\vx+t(\vy-\vx))}{\partial S_j \partial S_k}(\vz-\vy+\vy-\vx)_j(\vz-\vy+\vy-\vx)_k(\vy-\vz)_i(1-t)^2\dd t\dd\vz\\
        &-\int_{\vz\in \rrr^d}w^{(1)}(\vx,\vz)\int_0^1\sum_{i,j,k}\frac{\partial^2f_i(\vx+t(\vy-\vx))}{\partial S_j \partial S_k}(\vz-\vx+\vx-\vy)_j(\vz-\vx+\vx-\vy)_k(\vx-\vz)_it^2\dd t\dd\vz.
    \end{split}
\end{equation*}

We can drop odd function terms since 
$\int_{\vz\in \rrr^d}w^{(1)}(\vy,\vz)(\vz-\vy)_j(\vz-\vy)_k(\vy-\vz)_i\dd\vz=0$. We also have $\int_{\vz\in \rrr^d}w^{(1)}(\vy,\vz)(\vy-\vx)_j(\vy-\vx)_k(\vy-\vz)_i\dd\vz=\varepsilon^dC_0\cdot \mathcal O(\varepsilon^{d+1}\delta^2)=\mathcal O(\varepsilon^{2d+1}\delta^2)$. So the integral of second order term is 
\begin{equation}
\label{eq:5}
    =\int_{\vz\in \rrr^d}w^{(1)}(y,z)\int_0^1\sum_{i,j,k}\frac{\partial^2f_i(x+t(y-x))}{\partial S_j \partial S_k}(z-y)_j(y-x)_k(y-z)_i(1-t)^2dtdz
\end{equation}
        \begin{equation}
        \label{eq:6}
+\int_{\vz\in \rrr^d}w^{(1)}(\vy,\vz)\int_0^1\sum_{i,j,k}\frac{\partial^2f_i(\vx+t(\vy-\vx))}{\partial S_j \partial S_k}(\vy-\vx)_j(\vz-\vy)_k(\vy-\vz)_i(1-t)^2\dd t\dd\vz
\end{equation}
        \begin{equation}
        \label{eq:7}
-\int_{\vz\in \rrr^d}w^{(1)}(\vx,\vz)\int_0^1\sum_{i,j,k}\frac{\partial^2f_i(\vx+t(\vy-\vx))}{\partial S_j \partial S_k}(\vz-\vx)_j(\vx-\vy)_k(\vx-\vz)_it^2\dd t\dd\vz
\end{equation}
        \begin{equation}
        \label{eq:8}
-\int_{\vz\in \rrr^d}w^{(1)}(\vx,\vz)\int_0^1\sum_{i,j,k}\frac{\partial^2f_i(\vx+t(\vy-\vx))}{\partial S_j \partial S_k}(\vx-\vy)_j(\vz-\vx)_k(\vx-\vz)_it^2\dd t\dd\vz
\end{equation}
\begin{equation*}
    +\mathcal O(\varepsilon^{2d+1}\delta^2).
\end{equation*}
\end{claim}

In \eqref{eq:5} and \eqref{eq:7}, the integrals do not equal to 0 when $i=j$, in \eqref{eq:6} and \eqref{eq:8}, integrals do not equal to 0 when $i=k$.
 Then the whole second order term  becomes
 \begin{equation}
     \begin{split}
         -\int_{\vz\in \rrr^d}w^{(1)}(\vx,\vz)\int_0^1\sum_{i,j}\frac{\partial^2f_j(\vx+t(\vy-\vx))}{\partial S_i \partial S_j}(\vy-\vx)_i(t^2+(1-t)^2)(\vz-\vx)_j^2\dd t\dd\vz.
     \end{split}
 \end{equation}
 
From Lemma \ref{lemma:weight for edge}, we can use following integral
\begin{equation}
    \begin{split}
        \int_{\vz\in \rrr^d}w^{(1)}(\vz,\vx)(\vz-\vx)_j^2\dd\vz=\int_{\vz\in \rrr^d}\varepsilon^dC_0\exp(-\frac{\|\vz-\vx\|^2}{2\varepsilon^2})(\vz-\vx)_j^2\dd\vz=\varepsilon^dC_0\cdot4\varepsilon^{d+2}C_1.
    \end{split}
\end{equation}

Recall that $\dd\updelta \zeta_1=-\sum_{i,j}\frac{\partial^2f_j}{\partial S_i \partial S_j}ds_i$ from Proposition S7, and we can remove $(t^2+(1-t)^2)$ by using the same method in \ref{thm:asymptotic-expansion-L-1-up-Rd}.
Therefore, putting things together we can get 
\begin{equation}
    \begin{split}
        &\int_{\vz\in \rrr^d}[w^{(1)}(\vy,\vz)f_{yz}^{(0)}-w^{(1)}(\vx,\vz)f_{xz}^{(0)}]dz\\=&\frac{8}{3}C_0C_1\varepsilon^{2d+2}\int_0^1[(\dd\updelta \Vec{\upsilon(\vu(t))})]^\top \vec u'(t)\dd t+\mathcal O(\varepsilon^{2d+4})+\mathcal O(\varepsilon^{2d+1}\delta^2).
    \end{split}
\end{equation}
\end{proof}

The second step is to provide the error terms induced by the change of variables with lemma \ref{thm:coordinate-change-errors}. Then we can prove Proposition \ref{thm:asymptotic-expansion-L-1-down}
\begin{proofof}{Propostion \ref{thm:asymptotic-expansion-L-1-down}}
The geodesic curve can be approximated by a straight line in the local tangent plane $\mathcal{T}_x\mathcal{M}$ with error $\mathcal O(\varepsilon^3)$. Therefore $f_{yz}=f_{y_pz_p}+\mathcal O(\varepsilon^4)$. With choice of exponential kernel, we have 
\begin{align*}
    \,&\pequal\,\int_{\vz\in\M} (w_E(\vy,\vz) f_{yz}-w_E(\vx,\vz) f_{xz}) \dd\vz \\ &\stackrel{\romanenu{1}}{=} \int_{\substack{\vz\in\M \\ \max(\|\vz-\vy\|, \|\vz-\vx\|) < \varepsilon^\gamma}} (w_E(\vy,\vz) f_{yz}-w_E(\vx,\vz) f_{xz}) \dd\vz +\mathcal O(\varepsilon^{4+2d}) \\
	\,&\stackrel{\romanenu{2}}{=}\,\int_{\vz\in\rrr^d}  \Bigg[\left(\varepsilon^dC_0(\exp(-\frac{\|\vy-\vz\|^2}{2\varepsilon^2})-\frac{Q_{x,4}(\vz,\vy)}{\varepsilon^2}\exp(-\frac{\|\vy-\vz\|^2}{2\varepsilon^2})) \right)\cdot f_{yz}\\
    &-\left(\varepsilon^dC_0(\exp(-\frac{\|\vx-\vz\|^2}{2\varepsilon^2})-\frac{Q_{x,4}(\vz,\vx)}{\varepsilon^2}\exp(-\frac{\|\vx-\vz\|^2}{2\varepsilon^2})) \right)\cdot f_{xz}\bigg] \cdot \left((1 + Q_{\vx, 2}(\vz)\right) \dd\vz + \mathcal O(\varepsilon^{4+2d}) \\
	\,&=\frac{8}{3}\varepsilon^{2+2d}C_0C_1\int_0^1[(\dd\updelta \Vec{\upsilon(\vu(t))})]^\top \vec u'(t)\dd t+\mathcal O(\varepsilon^{4+2d})+\mathcal O(\varepsilon^{1+2d}\delta^2).
\end{align*}

Equality $\romanenu{1}$ holds by Lemma \ref{thm:error-bound-localization}.
Equality $\romanenu{2}$ is valid by first projecting $\vecxyz$ to
$\T_\vx\M$, changing the variables $\vx_p,\vy_p,\vz_p \to \vecxyz$,
and using Lemma \eqref{thm:error-bound-localization} again.
Terms consisting of $Q_{p,4}(\vx, \vy)$ or $Q_{p, 2}(\vx,\vy)$ are in the high order , hence, they can be merged into the last error term $\mathcal O(\varepsilon^{4+2d})$. Then we can again approximate the line integral
$\int_0^1 (\dd\updelta \vec{\upsilon})(\vu(t))^\top \vu'(t) \dd t = \int_0^1 (\dd\updelta \vec{\upsilon})(\vec\gamma(t))^\top \vec\gamma'(t) \dd t + \mathcal O(\varepsilon^3\delta)$
by \eqref{eq:lemma-coordinate-change-geodesics}, where $\vec\gamma$ is the
geodesic connecting $\vx, \vy$. Therefore, it implies that 
\begin{equation}
    \begin{split}
        \,&\pequal\,\int_{\vz\in\M} (w_E(\vy,\vz) f_{yz}-w_E(\vx,\vz) f_{xz}) \dd\vz \\
        &=\frac{8}{3}\varepsilon^{2d+2}C_0C_1\int_0^1 (\dd\updelta \vec{\upsilon})(\vec\gamma_\M'(t))^\top \vec\gamma_\M'(t) \dd t +\mathcal O(\varepsilon^{4+2d})+\mathcal O(\varepsilon^{1+2d}\delta^2).
    \end{split}
\end{equation}
\end{proofof}
\subsection{proof of theorem \ref{thm:L_1-down}}
Using $\LL_1^\mathrm{down}=W_1\cdot B_1W_0^{-1}B_1^T$, the expectation becomes
\begin{equation*}
    \begin{split}
    &\mathbb E\left[(\LL_1^\mathrm{down})_{[x,y]}\right]=W_1\cdot B_1W_0^{-1}B_1^T=E_{z,v}\bigg[\frac{\frac{1}{n}\sum_{\vz\neq\vy}w^{(1)}(\vz,\vy)f_{yz}}{\frac{1}{n}\sum_{\vv\neq\vx}w^{(1)}(\vv,\vy)}\bigg]-E_{z,v}\bigg[\frac{\frac{1}{n}\sum_{\vz\neq\vx}w^{(1)}(\vz,\vx)f_{xz}}{\frac{1}{n}\sum_{\vv\neq\vx}w^{(1)}(\vv,\vx)}\bigg]\\
    &=E_{z}\bigg[\frac{1}{n}\sum_{\vz\neq\vy}w^{(1)}(\vz,\vy)f_{yz}\bigg]\cdot E_v\Bigg[\frac{w^{(0)}(\vy)}{\frac{1}{n}\sum_{\vv\neq\vy}w^{(1)}(\vv,\vy)}\bigg]\cdot \frac{1}{w^{(0)}(\vy)}
    \\
    &-E_{z}\bigg[\frac{1}{n}\sum_{\vz\neq\vx}w^{(1)}(\vz,\vx)f_{xz}\bigg]\cdot E_v\Bigg[\frac{w^{(0)}(\vx)}{\frac{1}{n}\sum_{\vv\neq\vx}w^{(1)}(\vv,\vx)}\bigg]\cdot \frac{1}{w^{(0)}(\vx)}.
    \end{split}
\end{equation*}
The last equality holds by the independence of random variables $\vz,\vv$. By $E[\frac{1}{n}\sum_{\vv\neq\vy}w^{(1)}(\vv,\vy)]=w^{(0)}(\vy)=4\varepsilon^{2d}C_0^2$ and $w^{(0)}(\vy)=w^{(0)}(\vx)$, we can have similar analysis to Theorem \ref{thm:asymptotic-expansion-L-1-up}.With Monte-Carlo approximation and Lemma \ref{thm:asymptotic-expansion-L-1-down-Rd}, one has
\begin{equation}
    \begin{split}
        &\mathbb E\left[(\LL_1^\mathrm{down})_{[x,y]}\right]= \frac{2}{3}\varepsilon^{2}\frac{C_1}{C_0}\int_0^1 [\dd\updelta \vec{\upsilon}(\vec\gamma_\M(t))]^\top \vec\gamma_\M'(t) \dd t + \mathcal O(\varepsilon^4)+\mathcal O(\varepsilon\delta^2)+\mathcal O(n^{-1}\varepsilon^2\delta).
    \end{split}
\end{equation}

\section{Proofs of the spectral consistency of down Laplacian}
\label{sec:spectral-consistency-down-lapla}
\subsection{Outline of the proof}

The proof for spectral consistency of 1 down-Lapalcian is outlined in Figure
\ref{fig:proof-flow-down-laplacian}.
Instead of directly showing the consistency of 1 down-Laplacian (dashed line),
one can use the existing spectral consistency of the Laplace-Beltrami operator
$\Delta_0$ to show the spectral consistency of $\Delta_1$.
The first step of the proof is to show the {\em spectral dependency} of down Helmholtzian
$\LL_1^\mathrm{down} = \vec B_1^\top \vec W_0^{-1}\vec B_1 \vec W_1$
to the corresponding graph Laplacian, i.e.,
$\LL = \vec W_0^{-1}\vec B_1 \vec W_1 \vec B_1^\top$;
or more generally, the spectral dependency between $k-1$ up and $k$ down Laplacians.
Proposition 1.2 of \cite{PostO:09} first showed the spectra of $\vec L_1^\mathrm{down}$ and $\vec L_0$
away from 0 agree including multiplicity.
\cite{HorakD.J:13} later pointed out the aforementioned agreement between non-zero spectra could be extended to
$k$ unweighted Laplacian ($\vec L_k^\mathrm{down}$ and $\vec L_{k-1}^\mathrm{up}$).
Here we provide an extension of the results to the weighted $k$ Laplacians.

\begin{lemma}[Spectral dependency of $\LL_k$]
	Let $\mathcal S(\vec A)$ be the non-zero spectrum of matrix $\vec A$. The non-zero spectra of
	$\LL_k^\mathrm{down}$ and $\LL_{k-1}^\mathrm{up}$ agree including multiplicity, i.e.,
	$\mathcal S(\LL_k^\mathrm{down}) = \mathcal S(\LL_{k-1}^\mathrm{up})$.
\label{thm:spectral-dependency-agree-spectrum}
\end{lemma}
Proof of the lemma can be found in Supplement \ref{sec:full-thm-spectral-dependency} as well as
corollaries on finding the eigenvector of $\vL_{k-1}^\mathrm{up}$ from $\vL_{k}^\mathrm{down}$ (or vice-versa).
This lemma points out that the non-zero spectrum of $\LL_k^\mathrm{down}$ is identical to the non-zero spectrum
of $\LL_{k-1}^\mathrm{up}$. It indicates that the down graph Helmholtzian $\LL_1^\mathrm{down}$ is
consistent if the corresponding {\em random walk} Laplacian
$\LL = \vec{I}_n - \vec{D}^{-1}\vec{K} = \vec W_0^{-1}\vec B_1\vec W_1 \vec{B}_1^\top$ is consistent,
with $\vec W_0$, $\vec W_1$ constructed from $\vec W_2$ as discussed in Section \ref{sec:algorithms}.
The next Proposition  investigates the consistency of $\LL$.

\begin{proposition}[Consistency of the random walk Laplacian $\LL$ with kernel \eqref{eq:vr-kernel-general-form}]
Under Assumption \ref{assu:manifold-data-assu}--\ref{assu:kernel-assu},
same spectral consistency result as in Theorem 5 of \cite{BerryT.S:19} can be obtained for
the corresponding
$\LL = \vec W_0^{-1}\vec B_1\vec W_1 \vec{B}_1^\top$,
with weights calculated by $\vec W_k = \diag(|\vec B_{k+1}| \vec W_{k+1}\vec{1}_{n_{k+1}})$ for $k = 0, 1$.
\label{thm:consistency-L0-with-w1}
\end{proposition}
\begin{proofsk}
We first investigate the scenario when $w^{(2)}(\vecxyz)$ is constant, i.e.,
$w^{(2)}(\vecxyz) = \mathds{1}\left(\|\vec x-\vec y\|<\varepsilon\right)\mathds{1}\left(\|\vec x-\vec y\|<\varepsilon\right)\mathds{1}\left(\|\vec x-\vec y\|<\varepsilon\right)$.
If $\kappa(\cdot)$ has exponential decay, one can show that the corresponding
$w^{(1)}(\vecxy)$ has exponential decay, implying the consistency of
\cite{BerryT.S:19}. The above analysis can be naturally extended to general kernel $\kappa(\cdot)$
for $\kappa(\cdot)$ is upper bounded by the indicator kernel.
\end{proofsk}

The last part of the proof is to show the agreement in the spectra of the continuous
operators $\Delta_0$ and $\Delta_1$.
It was shown by using {\em Courant-Fischer-Weyl min-max principle} on $\Delta_1$
that the non-zero spectrum of $\Delta_1^\mathrm{down}$
(also known as the spectrum of the {\em co-exact}/curl 1-form) is a copy of non-zero
eigenvalues of $\Delta_0$ \cite{DodziukJ.M:95,ColboisB:06}.
With the right arrow completed in Figure \ref{fig:proof-flow-down-laplacian},
the down Helmholtzian $\LL_1^\mathrm{down}$ is hence
shown to converge {\em spectrally} to the spectrum of
$\Delta_1^\mathrm{down} = \mathsf{d}_0\updelta_1$.

\subsection{Spectral dependency and related corollaries}
\label{sec:full-thm-spectral-dependency}
\begin{proofof}{Lemma \ref{thm:spectral-dependency-agree-spectrum}}
From \cite{SchaubMT.B.H+:20}, the spectra of the {\em random walk}
$k$-Laplacian and of the symmetrized $k$-Laplacian are identical,
implying that one can study the spectrum of $\LL_k^s$ (symmetric)
instead of $\LL_k$. Following the proof of \cite{PostO:09}, one has
$\vW_{k-1}^{-1/2}\vB_k\vW_k^{1/2} \LL_k^{s,\mathrm{down}} =
\LL_{k-1}^{s,\mathrm{up}}\vW_{k-1}^{-1/2}\vB_k\vW_k^{1/2}$ and
$\vW_k^{1/2}\vB_k^\top \vW_{k-1}^{-1/2}\LL_{k-1}^{s,\mathrm{up}} =
\LL_k^{s,\mathrm{down}}\vW_k^{1/2}\vB_k^\top \vW_{k-1}^{-1/2}$.  This
implies that the mapping between the images of
$\LL_k^{s,\mathrm{down}}$ and $\LL_{k-1}^{s,\mathrm{up}}$ are
isomorphisms. In mathematical terms, if ${\tilde \vB}_k =
\vW_{k-1}^{-1/2}\vB_k\vW_k^{1/2}: \img(\LL_k^{s,\mathrm{down}}) \to
\img(\LL_{k-1}^{s,\mathrm{up}})$ and ${\tilde \vB}^*_k =
\vW_k^{1/2}\vB_k^\top \vW_{k-1}^{-1/2}:
\img(\LL_{k-1}^{s,\mathrm{up}}) \to \img(\LL_{k}^{s,\mathrm{down}})$,
we have ${\tilde \vB}_k$ and ${\tilde \vB}^*_{k}$ are isomorphisms.
Since $\LL_k^{s,\mathrm{down}} = {\tilde \vB}^*_k{\tilde \vB}_k$ and
$\LL_{k-1}^{s,\mathrm{up}} = {\tilde \vB}_k{\tilde \vB}^*_k$, the
isomorphisms of two operator implies
$\dim(\img(\LL_k^{s,\mathrm{down}})) =
\dim(\img(\LL_{k-1}^{s,\mathrm{up}}))$ and $\mathcal
S(\LL_k^\mathrm{down}) = \mathcal S(\LL_{k-1}^\mathrm{up})$. This
completes the proof.
\end{proofof}

Below are some corollaries of Lemma
\ref{thm:spectral-dependency-agree-spectrum} which are not used in our analysis but are useful in practice.
These lemmas connect the eigenvectors of $\vL_k$ (or $\LL_k$) with the eigenvectors of $\vL_{k-1}$ (or $\LL_{k-1}$).

\begin{corollary}[Eigenvectors of $\vec{L}_k$ and $\vL_{k-1}$]
    \begin{enumerate}[label=\underline{C.\Roman*}]
        \item []
        \item \label{enu:thm-spec-depen-unnorm-lifting} Let $\phi_{k-1}\in\mathbb R^{n_{k-1}}$ be the eigenvector of $\vec{L}_{k-1}$ with
        eigenvalues $\lambda$, then $\vec{B}_k^\top\phi_{k-1}$ is the eigenvector of $\vec{L}^\text{down}_{k}$ with same eigenvalue.
        \item \label{enu:thm-spec-depen-unnorm-droping} Let $\phi_k\in\mathbb R^{n_k}$ be the eigenvector of $\vec{L}_k$ with
        eigenvalues $\lambda$, then $\vec{B}_k\phi_k$ is the eigenvector of $\vec{L}^\text{up}_{k-1}$ with same eigenvalue.
    \end{enumerate}
    \label{thm:spectral-dependency-unnorm}
\end{corollary}

\begin{proof}
For \ref{enu:thm-spec-depen-unnorm-lifting},
Let $\phi_{k-1}$ be the non-trivial eigenfunction of $\vec{L}_{k-1}$ with eigenvalue $\lambda$, this implies $\vec{L}_{k-1}\phi_{k-1} = \lambda\phi_{k-1}$, therefore
\begin{equation}
    \lambda \vec{B}_k^\top \phi_{k-1} = \vec{B}_k^\top \vec{L}_{k-1} \phi_{k-1} = \left(\cancel{\vec{B}_k^\top\vec{B}_{k-1}^\top}\vec{B}_{k-1} + \vec{B}_k^\top\vec{B}_k\vec{B}_k^\top\right) \phi_{k-1} = \vec{L}_k^\text{down}\vec{B}_k^\top \phi_{k-1}.
    \nonumber
\end{equation}

Therefore, $\vec{B}_k^\top\phi_{k-1}$ will be the eigenfunction of $\vec{L}_k^\text{down}$.
Similarly, for \ref{enu:thm-spec-depen-unnorm-droping},
\begin{equation}
    \lambda\vec{B}_k\phi_k = \vec{B}_k\vec{L}_k\phi_k = \left(\vec{B}_k\vec{B}_k^\top\vec{B}_k + \cancel{\vec{B}_k \vec{B}_{k+1}}\vec{B}_{k+1}^\top\right)\phi_k = \vec{L}_{k-1}^\text{up}\vec{B}_k\phi_k.
    \nonumber
\end{equation}
This completes the proof.
\end{proof}

\begin{corollary}[Eigenvectors of $\vec{\mathcal L}_k$ and $\vec{\mathcal L}_{k-1}$]
    \begin{enumerate}[label=\underline{C.\Roman*}]
        \item []
        \item \label{enu:thm-spec-depen-norm-lifting} Let $\phi_{k-1}\in\mathbb R^{n_{k-1}}$ be the eigenvector of $\vec{\mathcal L}_{k-1}$ with
        eigenvalues $\lambda$, then $\vec{B}_k^\top\phi_{k-1}$ is the eigenvector of $\vec{\mathcal L}^\text{down}_{k}$ with same eigenvalue.
        \item \label{enu:thm-spec-depen-norm-droping} Let $\phi_k\in\mathbb R^{n_k}$ be the eigenvector of $\vec{\mathcal L}_k$ with
        eigenvalues $\lambda$, then $\vec{W}_{k-1}^{-1}\vec{B}_k \vec{W}_k\phi_k$ is the eigenvector of $\vec{\mathcal L}^\text{up}_{k-1}$ with same eigenvalue.
    \end{enumerate}
    \label{thm:spectral-dependency-norm}
\end{corollary}

\begin{proof}
For \ref{enu:thm-spec-depen-norm-lifting},
since $\phi_{k-1}$ is the non-trivial eigenfunction of $\vec{\mathcal L}_{k-1}$ with eigenvalue $\lambda$, we have
\begin{equation}
    \lambda\vec{B}^\top_k\phi_{k-1} = \vec{B}_k^\top \vec{\mathcal L}_{k-1}\phi = \left(\cancel{\vec{B}_k^\top \vec{B}_{k-1}^\top} \vec{W}_{k-2}^{-1}\vec{B}_{k-1}\vec{W}_{k-1}  + \vec{B}_k^\top \vec{W}_{k-1}^{-1} \vec{B}_{k} \vec{W}_{k} \vec{B}_{k}^\top\right) \phi = \vec{\mathcal L}_k^\text{down} \vec{B}_k^\top \phi_k.
    \nonumber
\end{equation}

For the case in \ref{enu:thm-spec-depen-norm-droping}, we need some little algebraic
tricks to get rid of the weights $\vec{W}$ before boundary operator.
\begin{equation}
\begin{split}
    \lambda \vec{W}_{k-1}^{-1}\vec{B}_k \vec{W}_k\phi_k \,&\,= \vec{W}_{k-1}^{-1}\vec{B}_k \vec{W}_k\vec{\mathcal L}_k\phi_k \\
    \,&\,= \Big(\underbrace{\vec{W}_{k-1}^{-1}\vec{B}_k \vec{W}_k \vec{B}_k^\top}_{\vec{\mathcal L}_{k-1}^\text{up}} \vec{W}_{k-1}^{-1} \vec{B}_k \vec{W}_{k} + {\vec{W}_{k-1}^{-1}\underbrace{\vec{B}_k \cancel{\vec{W}_k \vec{W}_{k}^{-1}} \vec{B}_{k+1}}_{=0} \vec{W}_{k+1} \vec{B}_{k+1}^\top}\Big)\phi_k \\
    \,&\,= \vec{\mathcal L}_{k-1}^\text{up}\vec{W}_{k-1}^{-1}\vec{B}_k \vec{W}_k\phi_k.
\end{split}
\nonumber
\end{equation}
This completes the proof.
\end{proof}

\subsection{Proof of Proposition \ref{thm:consistency-L0-with-w1}}
\label{sec:proof-consistency-L0-with-w1}
We first start with the following lemma that derives the closed form of $w^{(1)}(\vecxy)$
when the weight on the triangles is an indicator function.
Note that in the construction below, we ignore the $\kappa(\vecxy)$ factor,
i.e., we assume that $w^{(2)}(\vecxyz) = \kappa(\vec x, \vec z)\kappa(\vec y, \vec z)$,
for a more concise notation; the $\kappa(\vecxy)$ factor will be added back later.

\begin{lemma}[The integral form of constant triangular weight]
Let $\varepsilon$ be a bandwidth parameter.
Further assume a constant triangular weight, i.e.,
$w^{(2)}(\vecxyz) = \mathds{1}(\|\vz-\vx\| < \varepsilon)\mathds{1}(\|\vz-\vy\| < \varepsilon)$,
then
\begin{equation}
\begin{split}
	w^{(1)}(\vecxy) \,&=\, \mathds{1}(\|\vx-\vy\|<\delta) \int_{\vz\in\M} w^{(2)}(\vecxyz) \dd\vz \\
	\,&=\, \mathds{1}(\|\vx-\vy\|<\delta) C\cdot I_{1 - \frac{\|\vec x-\vec y\|^2}{4(\gamma\varepsilon)^2}}\left(\frac{d+1}{2}, \inv{2}\right) + \mathcal O(\varepsilon^2).
\end{split}
\label{eq:spherical-cap}
\end{equation}
Where $C = \varepsilon^d \cdot p\cdot C_d$,
$C_d$ is the volume of unit $d$-ball, i.e., $C_d = \pi^{d/2} / \Gamma(d/2 + 1)$, and $I_{\vec{x}}(a, b)$
is the regularized incomplete beta function.
\label{thm:integral-form-constant}
\end{lemma}
\begin{proof}
In the continuous limit, with constant sampling density $p$, we have
\begin{equation*}
\begin{split}
	\,&\pequal\, \varepsilon^{-d}\int_{\vz\in\M} w^{(2)} \dd\vz = \varepsilon^{-d}\int_{\vz\in\rrr^d} w^{(2)} \dd\vz + \mathcal O(\varepsilon^2) \\
    \,&=\, p\cdot 2\cdot \mathrm{Vol}_\mathrm{cap}\left(\varepsilon - \frac{\|\vx-\vy\|_2}{2}; \varepsilon, d\right) = p\cdot C_d I_{1 - \frac{\|\vx-\vy\|^2}{4\varepsilon^2}}\left(\frac{d+1}{2}, \inv{2}\right).
\end{split}
\end{equation*}

The first equality holds from projecting $\vecxyz$ onto $\T_\vx\M$
and using Lemma \ref{thm:coordinate-change-errors}, $\mathcal O(\varepsilon^2)$
is from $Q_{\vx, 4}(\vz,\vy)$ and $Q_{\vx, 2}(\vz)$ of
\eqref{eq:lemma-coordinate-change-metrics} and
\eqref{eq:lemma-coordinate-change-volumes}, respectively.
Last equality holds because
\begin{equation*}
\begin{split}
    \mathrm{Vol}_{\mathrm{cap}} (h; r, d) \,&=\, \int_0^\phi C_{d-1}r^{d-1}\sin^{d-1} \theta r\sin\theta d\theta = C_{d-1} r^d \int_0^t \nu^{\frac{d-1}{2}}(1-\nu)^{-\inv{2}} d\nu \\
    \,&=\,C_{d-1}r^d B\left(\frac{d+1}{2},\inv{2}\right)I_{(2rh - h^2) /r^2}\left(\frac{d+1}{2},\inv{2}\right) \\
    \,&=\, C_d r^d I_{(2rh - h^2) /r^2}\left(\frac{d+1}{2},\inv{2}\right).
\end{split}
\end{equation*}

And,
\begin{equation}
C_{d-1}\cdot B\left(\frac{d+1}{2},\inv{2}\right) = \frac{\pi^{(d-1)/2}}{\Gamma\left(\frac{d-1}{2}+1 \right)} \frac{\Gamma\left(\frac{d+1}{2}\right)\Gamma(1/2)}{\Gamma\left(\frac{d}{2}+1 \right)} = C_d.
\nonumber
\end{equation}

This completes the proof.
\end{proof}

\begin{proofof}{Proposition \ref{thm:consistency-L0-with-w1}}
It suffices to prove that the corresponding $w^{(1)}(\vecxy)$ has
exponential decay, and the $\mathcal O(\varepsilon^2)$ error term can be
ignored in the asymptotic expansion of graph Laplacian operator.
Let $w^{(1)}_{\mathds{1}}\left(\|x-y\|; r \right)$
be \eqref{eq:spherical-cap}. The integral operator of general
kernel $\kappa$
can be decomposed into two parts, i.e.,
\begin{equation*}
\begin{split}
	w^{(1)}(\vecxy) = \int_\M \kappa(\vx, \vz)\kappa(\vy, \vz) \dd \vz \,&=\, \int\limits_{\substack{z\in\M \\ \max(\|\vz-\vy\|, \|\vz-\vx\|) \leq \varepsilon^\gamma}} \kappa(\vx, \vz)\kappa(\vy, \vz) \dd\vz \\
	\,&+\, \int\limits_{\substack{z\in\M \\ \min(\|\vz-\vy\|, \|\vz-\vx\|) > \varepsilon^\gamma}} \kappa(\vx, \vz)\kappa(\vy, \vz) \dd\vz.
\end{split}
\end{equation*}

With $\min(\|\vz-\vx\|, \|\vz-\vy\|) < \varepsilon^\gamma$,
one has $\kappa(\vx, \vy) \leq \mathds{1}(\|\vx-\vy\| < \varepsilon^\gamma)$. Therefore,
the first term can be bounded by $w^{(1)}_{\mathds{1}}(\vx, \vy; \varepsilon^\gamma)$.
The second term can be bounded by $\mathcal O(\varepsilon^3)$ using Lemma
\ref{thm:error-bound-localization} with $g(\cdot) = 1$. Note that
the result can be generalized to manifold with boundaries by the following.
For the points that is
$\varepsilon^\gamma$ within the boundary $\partial \M$, the above inequality is
still valid, for one can use a modified kernel $\kappa'(\vx, \vy)$ with
$\kappa'(\vx, \vy) = \kappa(\vx, \vy)$ if $\vy\in \M$ and $0$ otherwise.
Putting in $\kappa(\vx, \vy)$, one has,

\begin{equation*}
	w^{(1)}(\vx, \vy) \leq \kappa(\vx, \vy) w^{(1)}_\mathds{1}(\vx, \vy; \varepsilon^\gamma) + \kappa(\vx, \vy)\cdot \mathcal O(\varepsilon^2) \leq C\kappa(\vx, \vy) + \kappa(\vx, \vy)\cdot \mathcal O(\varepsilon^2).
\end{equation*}
Last inequality holds since $w^{(1)}_\mathds{1}(\vx, \vy; \varepsilon^\gamma) \leq C$ from
Lemma \ref{thm:integral-form-constant}.
The above inequality shows that $w^{(1)}(\vecxy)$ can be decomposed into a term that has fast
enough decay and another term which is bounded by $\mathcal O(\varepsilon^2)$.
Note that the graph is built with radius $\delta$, the second order expansion
of the graph Laplacian integral operator \cite{BerryT.S:19} has a $\delta^2$
term, implying that $\kappa(\vecxy)\cdot \mathcal O(\varepsilon^2)$
term can be bounded by $\mathcal O(\varepsilon^2\delta^2) = \mathcal O(\delta^{10/3})$.
Hence, spectral consistency as in Theorem 5 of \cite{BerryT.S:19}
with bias \& variance determined by
$\mathcal O(\delta^{4/3}) = \mathcal O(\varepsilon^2)$ can be achieved.
More specifically, since points are sampled with constant density from
the manifold $\M$, the spectrum of $\LL = \vec W_0^{-1}\vec B_1\vec W_1\vec B_1^\top$
with weight $\vec w_1 = |\vec B_2|\vec w_2$ and $\vec w_0 = |\vec B_1|\vec w_1$ converges
to the spectrum of Laplace-Beltrami operator $\Delta_0$ with bias \& variance in the order of $\mathcal O(\varepsilon^2)$.
\end{proofof}
 \section{Effects of the weights $a,b$ between up/down Laplacians}
\label{sec:discussion-choice-of-a-b}

Even though in Section \ref{sec:consistency-analysis} it was established that for the consistency of $\LL_1$, $a=\inv{4}$ and $b=1$, here we consider the possibility of chosing other non-negative values for $a,b$. This can be useful from a machine learning perspective, allowing one to emphasize either the gradient or the curl subspaces of $\LL_1$. 

From HHD, the space of cochains $\rrr^{n_1}$ can be decomposed into three different
orthogonal subspaces: the image of $\LL_1^\mathrm{down}$
(gradient), the image of $\LL_1^\mathrm{up}$ (curl), and the kernel of both
$\LL_1^\mathrm{down}$ and $\LL_1^\mathrm{up}$ (harmonic). Since these subspaces are
orthogonal to each other, rescaling $\LL_1^\mathrm{up}$ and $\LL_1^\mathrm{down}$ with some
constants $a, b$ will only scale the spectra accordingly without altering the
eigenvectors. We first
investigate the spectrum of the rescaled $\LL_1$ w.r.t. $a, b$ with the following
Corollary.

\begin{corollary}[Spectrum of new $\LL_1$]
The range of the spectra of $\LL_1$ is $\lambda(\LL_1) \in [0, \max(2a, 3b)]$.
\label{thm:spectrum-rescaled-norm-L1}
\end{corollary}
\begin{proof}
    From \cite{HorakD.J:13}, $\lambda(\vW_k^{-1}\vB_{k+1}\vW_{k+1}\vB_{k+1}^\top) \in [0, k+2]$.
    From Lemma \ref{thm:spectral-dependency-agree-spectrum}, one has
    $\mathcal S(\LL_k^\mathrm{down}) = \mathcal S(\LL_{k-1}^\mathrm{up})$.
    Thus we have $\lambda(\vB_1^\top\vW_0^{-1}\vB_1\vW_1)\in[0, 2]$ and
    $\lambda (\vW_1^{-1}\vB_2\vW_2\vB_2^\top) \in [0, 3]$.
    From HHD, an eigenvector can only be either curl, gradient, or harmonic flow.
    Thus the non-zero spectrum of $\LL_1$ will simply be the union of two disjoint
    eigenvalue set.
    Since rescaling rescaling the matrix by a constant will only change the scales of
    the eigenvalues, the union of the (rescaled) down and up Laplacian will therefore
    be in the range of $[0, \max(2a, 3b)]$.
    This completes the proof.
\end{proof}

Note that by choosing $a = \inv{2}$ and $b = \inv{3}$, the spectra of
$\LL_1$, $\LL_1^\mathrm{down}$, and $\LL_1^\mathrm{up}$ are all
upper bounded by 1.

\begin{figure}[!htb]
    \subfloat[][$a = \inv{4}$, $b = 1$]
    {\includegraphics[width=0.3\linewidth]{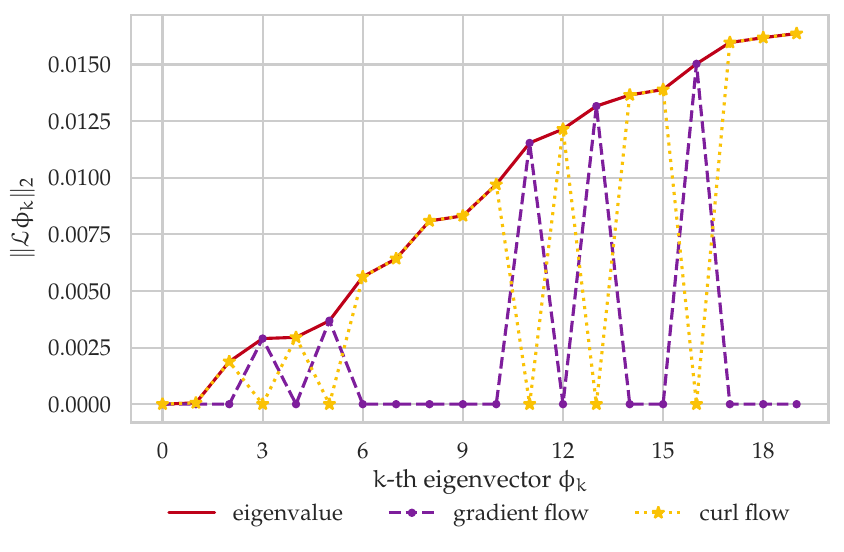}
    \label{fig:hhd-shift-ethanol-1-4}}\hfill
\subfloat[][$a = \inv{2}$, $b = \inv{3}$]
    {\includegraphics[width=0.3\linewidth]{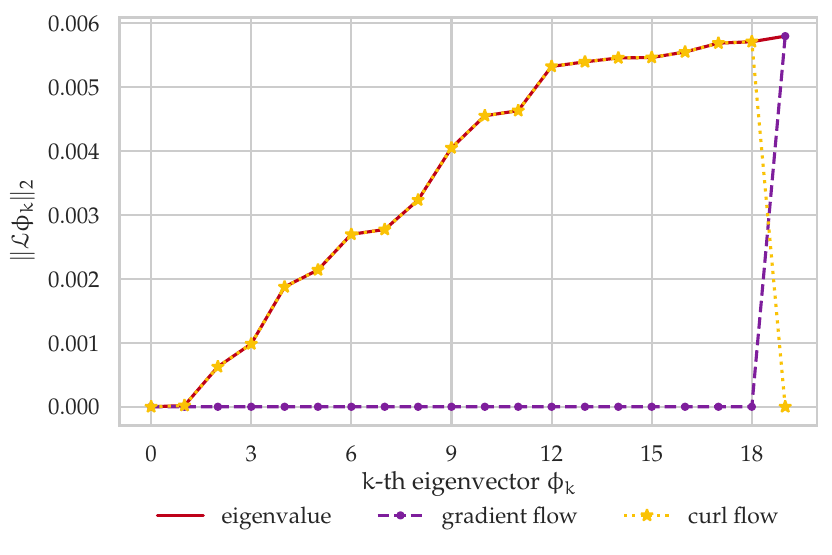}
    \label{fig:hhd-shift-ethanol-3-2}}\hfill
\subfloat[][$a = b =1$]
    {\includegraphics[width=0.3\linewidth]{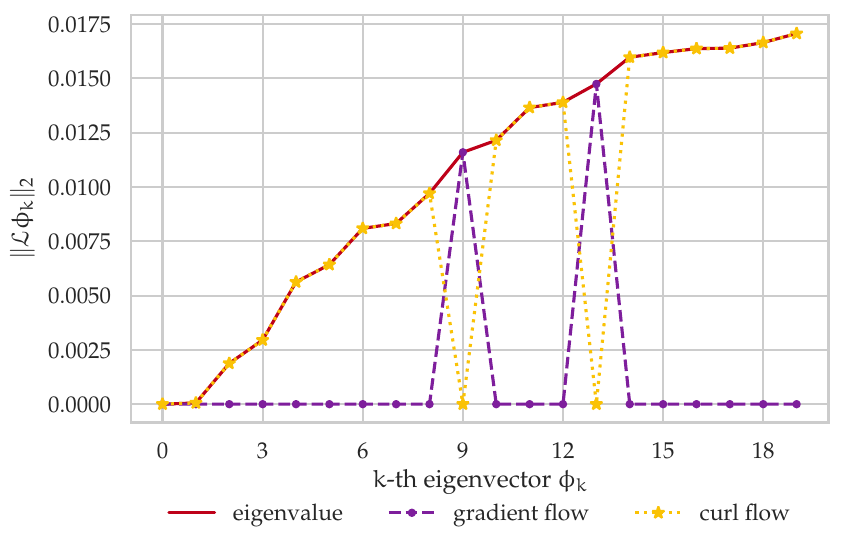}
    \label{fig:hhd-shift-ethanol-1-1}}\hfill
\caption{Shift in the rankings of the $\LL_1^\mathrm{up}$, $\LL_1^\mathrm{down}$ spectrum with different choices of $a, b$ values for ethanol dataset. The eigenvector corresponds to the third eigenvalue in \protect\subref{fig:hhd-shift-ethanol-1-4} is identical to that corresponds to the 18th in \protect\subref{fig:hhd-shift-ethanol-3-2} and the ninth in \protect\subref{fig:hhd-shift-ethanol-1-1}. Note that the rankings within gradient or curl flows will not change by different choices of $a, b$, which can be shown by comparing the curl flows (yellow) between \protect\subref{fig:hhd-shift-ethanol-1-4}--\protect\subref{fig:hhd-shift-ethanol-1-1}.}
    \label{fig:hhd-shift}
\end{figure}

Based on the discussion above, different choices of $a, b$ constants will shift
the rankings between the curl flows ($\LL_1^\mathrm{up}$) and gradients flows
($\LL_1^\mathrm{down}$). This effect can be seen in Figure
\ref{fig:hhd-shift}, with the first two gradient flows (in purple) in Figure
\ref{fig:hhd-shift-ethanol-1-4} corresponds to the ninth and the thirteenth
eigenvalues in Figure \ref{fig:hhd-shift-ethanol-1-1}.
Considering the case $a = \inv{2}; b=\inv{3}$ when the spectra of
$\LL_1^\mathrm{down}$ and $\LL_1^\mathrm{up}$ are both upper bounded by 1. Since
there are only $n_0 - \beta_0$ non-zero eigenvalues in $\LL_1^\mathrm{down}$
(\# of edges needed to form a spanning tree) compared to
$n_1 - (n_0 - \beta_0) - \beta_1$ non-zero eigenvalues in $\LL_1^\mathrm{up}$
(\# of independent triangles),
the density of the gradient flows
will be $\mathcal O(n_1 / n_0)$ less than those of curl flow.
That is to say, we will observe more curl
flows than gradients flow for a fixed number of eigenvalues as shown in
Figure \ref{fig:hhd-shift-ethanol-3-2}. Choosing a smaller $a$ value increases
the density of the gradient flow in the low frequency region (see
a smaller choice of $a$ in \ref{fig:hhd-shift-ethanol-1-1} and an even smaller $a$
in \ref{fig:hhd-shift-ethanol-1-4})
It creates a more balanced distribution of flows in the low frequency regime.
Not that the choice $a = \inv{4}; b = 1$ creates the most balanced spectrum
within the first 20 eigenvalues, as shown in Figure \ref{fig:hhd-shift}.

One can also analyze the random walk in the finite simplicial complex
as in \cite{SchaubMT.B.H+:20}. By letting $a = \inv{2}; b = \inv{3}$ with
$\vW_2 = \vI_{n_2}$, they showed the constructed Helmholtzian
corresponds to a finite random walk with equal probability ($p =
\inv{2}$) of performing up (diffuse to upper adjacent edges by common
triangle) and down (diffuse to lower adjacent edges by common nodes)
random walk. One can easily extend their analysis to non-constant
weights $\vW_2$ and different $a, b$ values. This results in a random
walk with probability $\frac{2a}{2a + 3b}$ in performing lower random
walk, while performing upper random walk with probability
$\frac{3a}{2a+3b}$. Similarly, $a = \inv{2}; b = \inv{3}$ will result
in an equal probability of upper/lower random walks, as suggested in
\cite{SchaubMT.B.H+:20}.  However, it might not be optimal for the
transition probability when performing lower random walk (depending on
$w_1$), which is much larger than the transition probability of the upper adjacent walk (depending on $w_2$).  Hence, one might
need to choose a smaller $a$ value to ensure a more balanced random
walk across all neighboring edges.

\section{Velocity field and cochain processing}
\label{sec:vector-field-cochain-map}
\subsection{Obtaining a 1-cochain}
From the discussion in Section \ref{sec:background}, the $k$-cochain is obtained by
$\vec{\omega}^k(\sigma_i) = \int_{\sigma_i} \zeta_k$. Given only the vector field
$\zeta(\vec x_i) = \vec f(\vec x_i) \in\mathbb R^D \,\forall\, i\in [n]$, the 1-cochain
$\vec \omega$ on edge $e = (i, j)$ can be computed by $\omega_e = \int_0^1 \vec f(\vec\gamma(t)) \vec\gamma'(t) \dd t$.
With $\vec\gamma(t) \approx \vec x_i + (\vec x_j - \vec x_i) t$, and
$\vec\gamma'(t) = \dd \vec u(t) / \dd t \approx (\vec x_j - \vec x_i)$, one can approximate
$\vec f(\vec u(t)) \approx \vec f(\vec x_i) + (\vec f(\vec x_j) - \vec f(\vec x_i)) t$
by Lemma \ref{thm:linear-approx-of-line-int},
\begin{equation}
\begin{split}
	\omega_e \,&=\, \int_0^1 \vec f^\top(\vec\gamma(t)) \vec\gamma'(t) \dd t \approx \int_0^1 \left[\vec f(\vec x_i) + (\vec f(\vec x_j) - \vec f(\vec x_i)) t\right]^\top (\vec x_j-\vec x_i) \dd t \\
	\,&=\, \inv{2}(\vec f(\vec x_i) + \vec f(\vec x_j))^\top(\vec x_j - \vec x_i).
\end{split}
\label{eq:linear-interpolation-edge-flow}
\end{equation}

Note that \eqref{eq:linear-interpolation-edge-flow} can be written in a more concise
form using boundary operator $\vec B_1$. Let $\vec F\in\mathbb R^{n\times D}$ with
$\vec f_i = \vec F_{i, :} = \vec f(\vec x_i)$, we have
$[|\vec B_1^\top|\vec F]_{[i,j]} = \vec f(\vec x_i) + \vec f(\vec x_j)$.
Additionally,
we have $[-\vec B_1^\top\vec X]_{[i,j]} = \vec x_j - \vec x_i$.
Therefore,
\begin{equation}
	\vec \omega = -\inv{2}\diag(\vec B_1^\top \vec X\vec F^\top|\vec B_1|).
\end{equation}

One can follow the procedure stated below to obtain the point-wise vector field from 1-cochain.
Define $\vec X_E = -\vec B_1^\top\vec X$, and
$\vec \chi_E = (\vec X_E)^{\circ 2} \vec{1}_D \in\rrr^{n_1}$,
where $\vec M^{\circ p}$ is the {\em Hadamard power} of matrix $\vec M$,
i.e., $[\vec M^{\circ p}]_{ij} = M_{ij}^p$. Further let $[\vec\chi_E]_{[i,j]}$ represent
the norm of $\vec x_j - \vec x_i$, i.e.,
$[\vec\chi_E]_{[i,j]} = \|\vec x_j - \vec x_i\|^2_2$. Given the 1-cochain $\vec\omega$,
one can solve the following $D$ least squares problem to estimate the vector field
$\vec F$ on each point $\vec x_i$.

\begin{equation}
	\hat{\vec v}_\ell = \argmin_{\vec v_\ell \in \mathbb R^{n}}\left\{ \left\||\vec B_1^\top|\vec v_\ell - \left([\vec X_E]_{:, \ell} \oslash \vec\chi_E\right)\circ\vec\omega\right\|^2_2\right\} \,\forall\, \ell=1,\cdots,D.
\label{eq:linear-reverse-interpolation-edge-flow}
\end{equation}

Where $\circ$, $\oslash$ correspond to {\em Hadamard product} and {\em Hadamard division},
respectively.
The solution to the $\ell$-th least squares problem corresponds to the estimate of
$f_\ell(\vec x_i)$ from $\inv{2}(f_\ell(\vec x_i) + f_\ell(\vec x_j))$
as in \eqref{eq:linear-interpolation-edge-flow}.
More specifically,
\begin{equation*}
	\inv{2}(f_\ell^\parallel (\vec x_i) + f_\ell^\parallel(\vec x_j)) = [\left([\vec X_E]_{:, \ell} \oslash \vec\chi_E\right)\circ\vec\omega]_{[i,j]} = \frac{(x_{j,\ell} - x_{i,\ell})\omega_{ij}}{\|\vx_j - \vx_i\|^2}.
\end{equation*}

The estimated vector field is
thus the concatenation of the $D$ least squares solutions,
\begin{equation}
	\hat{\vec F} =
\left[
  \begin{array}{cccc}
    \vrule & \vrule &        & \vrule \\
    \hat{\vec v}_1    & \hat{\vec v}_2    & \ldots & \hat{\vec v}_D    \\
    \vrule & \vrule &        & \vrule
  \end{array}
\right] \in \rrr^{n\times D}.
\end{equation}

\subsection{Smoother vector field from the 1-cochain by a damped least square}
\label{sec:reverse-interpolation-damped-lsqrt}
Since the linear system in \eqref{eq:linear-reverse-interpolation-edge-flow} is
overdetermined ($n_1$ is oftentimes greater than $n_0$), one can obtain a smoother
estimated vector field from the 1-cochain using a damped least squares. 
That is to say, one can change the aforementioned loss function to the following,

\begin{equation}
	\hat{\vF} = \argmin_{\vF \in \mathbb R^{n\times D}}\left\{ \left\||\vec B_1^\top|\vF - \left(\vec X_E \oslash \vec\chi_E\right)\circ\vec\omega\right\|^2_F+ \lambda \|\vF\|_F\right\}.
\label{eq:linear-reverse-interpolation-edge-flow-damped}
\end{equation}

Here $\|\cdot\|_F$ represents the Frobenius norm. 
\eqref{eq:linear-reverse-interpolation-edge-flow-damped} is
essentially a multi-output Ridge regression problem.
Figure \ref{fig:rev-int-comp} shows the estimated field from the 1-cochain constructed
by the simulated field (shown in Figure \ref{fig:rev-int-comp-ground-truth-1-2})
with different damping parameter $\lambda$'s. 
A larger $\lambda$ yields a smoother (estimated) field in the original space, which can be seen by comparing 
Figures \ref{fig:rev-int-comp-damp-0} ($\lambda=0$), \ref{fig:rev-int-comp-damp-1e3}
($\lambda = 1000$), and \ref{fig:rev-int-comp-damp-1e5} ($\lambda = 10^5$). 
This is because a larger $\lambda$ results in
narrower ``band'' after proper scaling,
as presented in the parity plot in Figure \ref{fig:rev-int-comp-parity-damp-0}--\ref{fig:rev-int-comp-parity-damp-1e5}.
Cross validation (CV) can be used to choose the damping constant $\lambda$.
Different scoring criteria used in the validation set will result in different chosen
$\lambda$ values. 
The scoring function we used throughout this paper is the
Fisher z-transformed Pearson correlation value,
for we care more about the relative relations of the vector field
rather than the absolute scales.
The selected regularization parameter (denoted $\lambda^*_\rho$) using this criteria tends to
be larger than that chosen by the mean squared error (denoted $\lambda^*_\text{MSE}$),
thus resulting in a smoother vector field
(see e.g., Figure \ref{fig:rev-int-comp-damp-cv-corr} v.s. Figure \ref{fig:rev-int-comp-damp-cv-mse}).
Figures \ref{fig:rev-int-comp-exps-1} and \ref{fig:rev-int-comp-exps-2}
show all the vector fields reported in this paper estimated from the same
1-cochains with different regularization parameter. More specifically, estimated velocity fields with $\lambda = 0$,
$\lambda^*_{\rho}$, and $\lambda^*_{\text{MSE}}$.
As clearly shown in these Figures, we gain interpretability by having a smoother
vector field without having too much structural changes using $\lambda^*_{\rho}$.

\begin{figure}[!htb]
    \subfloat[][Ground truth view x, y]
    {\includegraphics[width=0.3\linewidth]{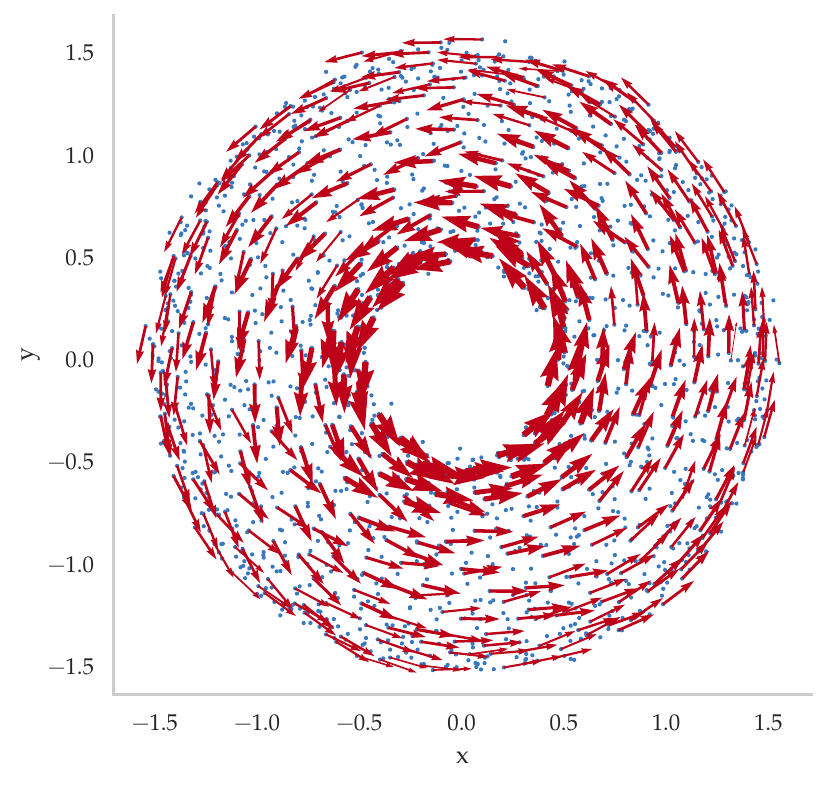}
    \label{fig:rev-int-comp-ground-truth-1-2}}\hfill
\subfloat[][$\lambda = 0$]
    {\includegraphics[width=0.3\linewidth]{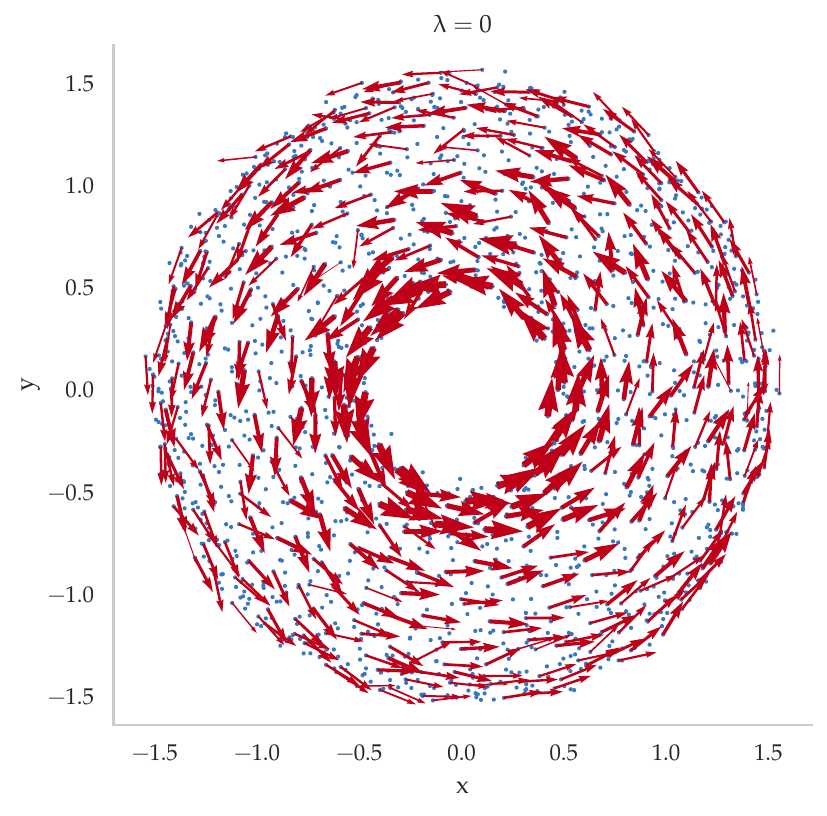}
    \label{fig:rev-int-comp-damp-0}}\hfill
\subfloat[][$\lambda = 100$]
    {\includegraphics[width=0.3\linewidth]{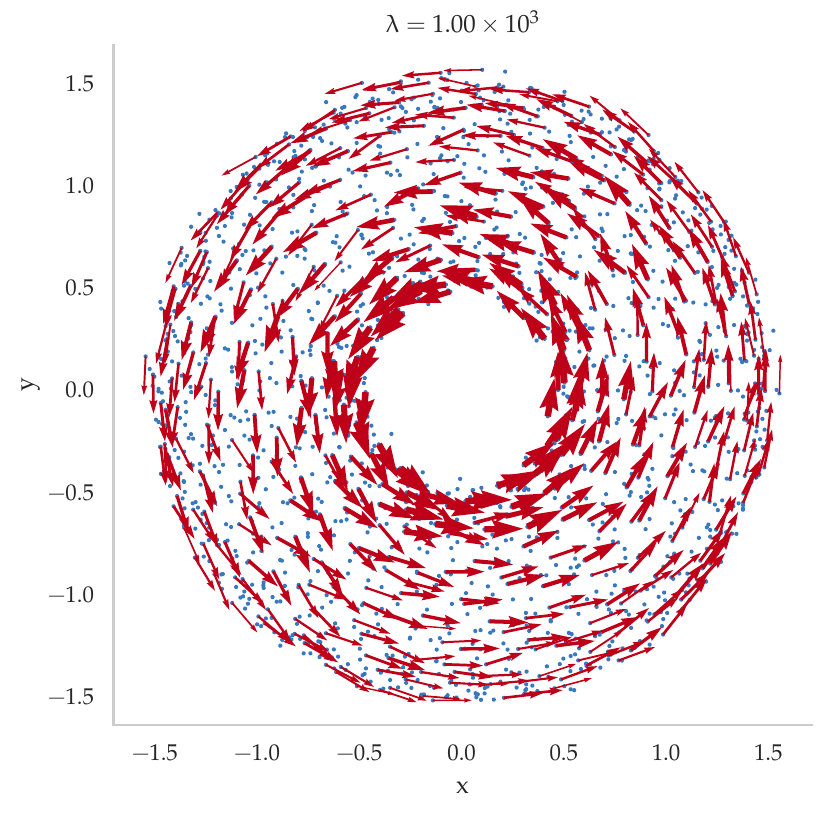}
    \label{fig:rev-int-comp-damp-1e3}}\hfill

    \subfloat[][Ground truth view x, z]
    {\includegraphics[width=0.3\linewidth]{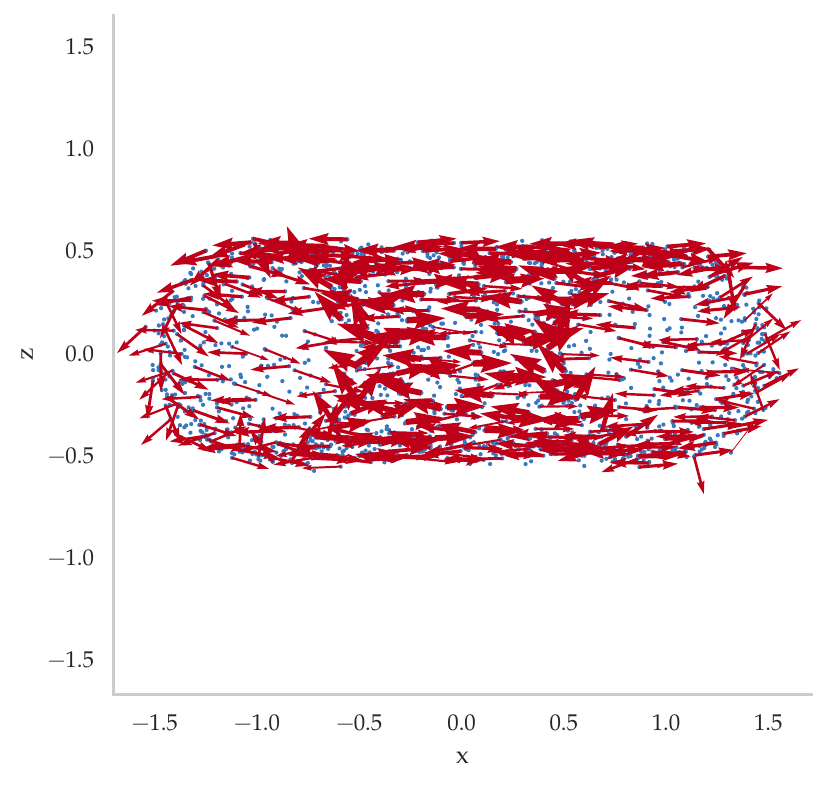}
    \label{fig:rev-int-comp-ground-truth1-3}}\hfill
\subfloat[][Parity plot with $\lambda = 0$]
    {\includegraphics[width=0.3\linewidth]{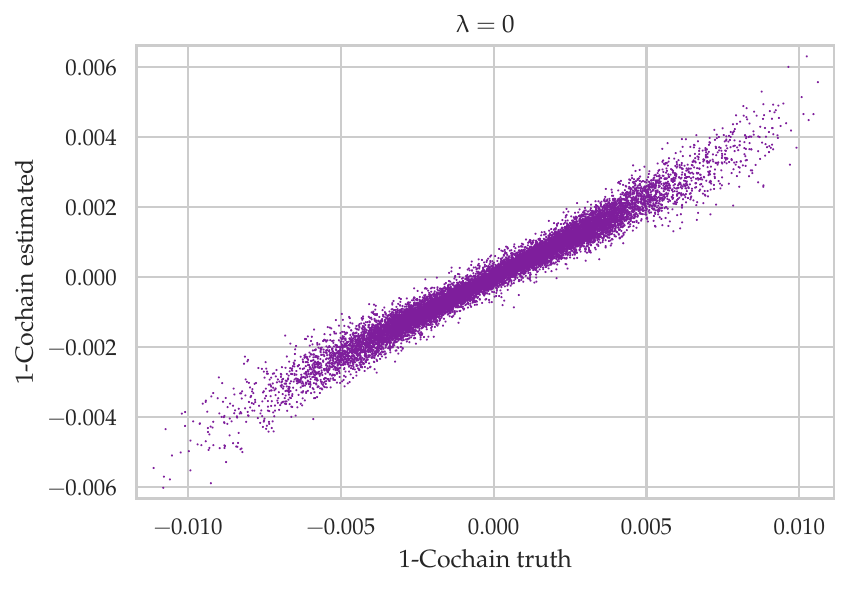}
    \label{fig:rev-int-comp-parity-damp-0}}\hfill
\subfloat[][Parity plot with $\lambda = 100$]
    {\includegraphics[width=0.3\linewidth]{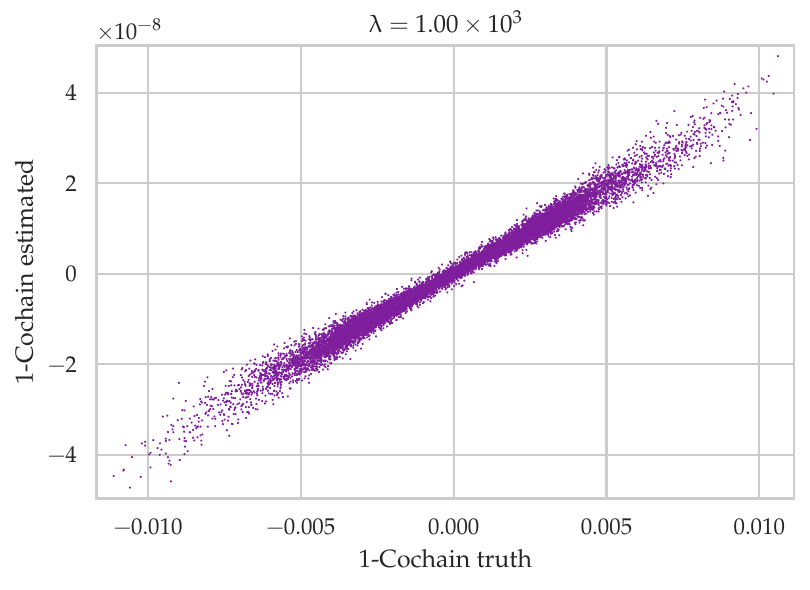}
    \label{fig:rev-int-comp-parity-damp-1e3}}\hfill

    \subfloat[][$\lambda = 10^5$]
    {\includegraphics[width=0.3\linewidth]{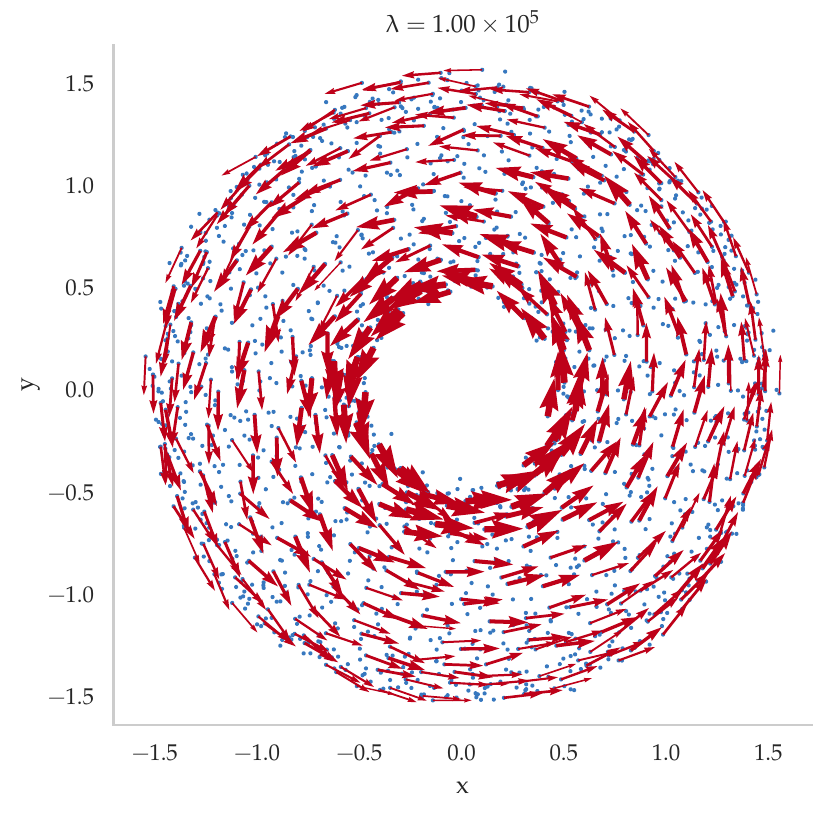}
    \label{fig:rev-int-comp-damp-1e5}}\hfill
\subfloat[][$\lambda^*_\rho$]
    {\includegraphics[width=0.3\linewidth]{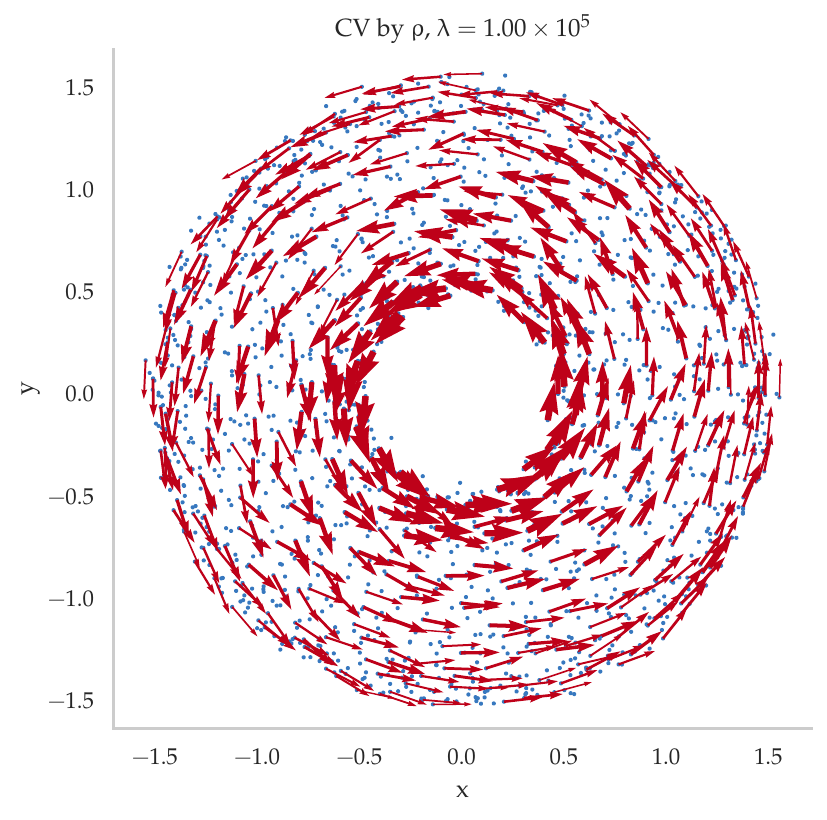}
    \label{fig:rev-int-comp-damp-cv-corr}}\hfill
\subfloat[][$\lambda^*_\text{MSE}$]
    {\includegraphics[width=0.3\linewidth]{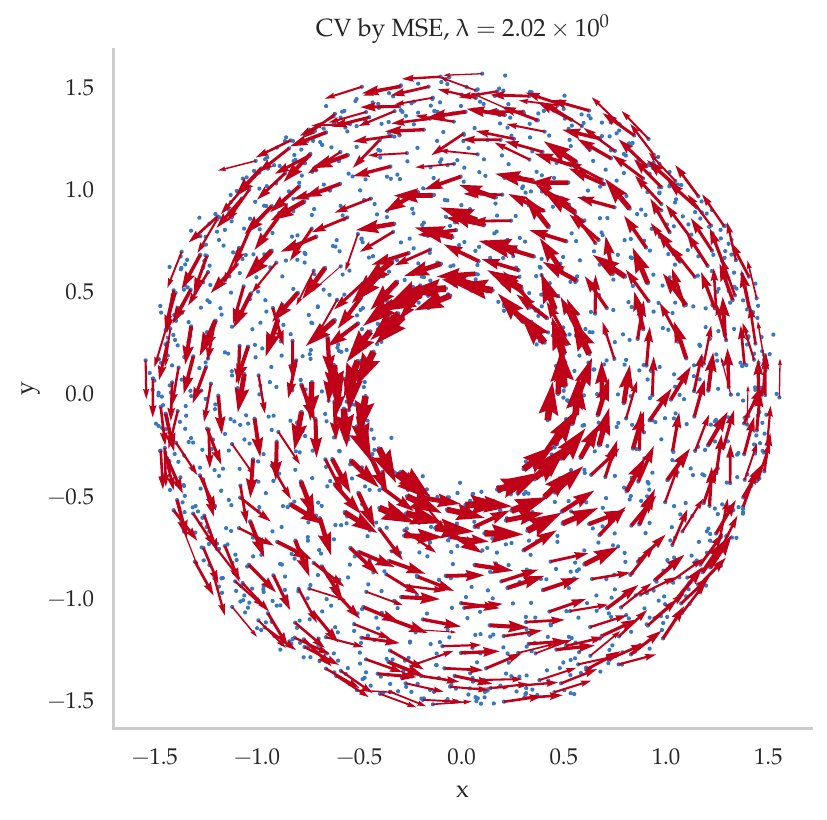}
    \label{fig:rev-int-comp-damp-cv-mse}}\hfill

    \subfloat[][Parity plot $\lambda = 10^5$]
    {\includegraphics[width=0.3\linewidth]{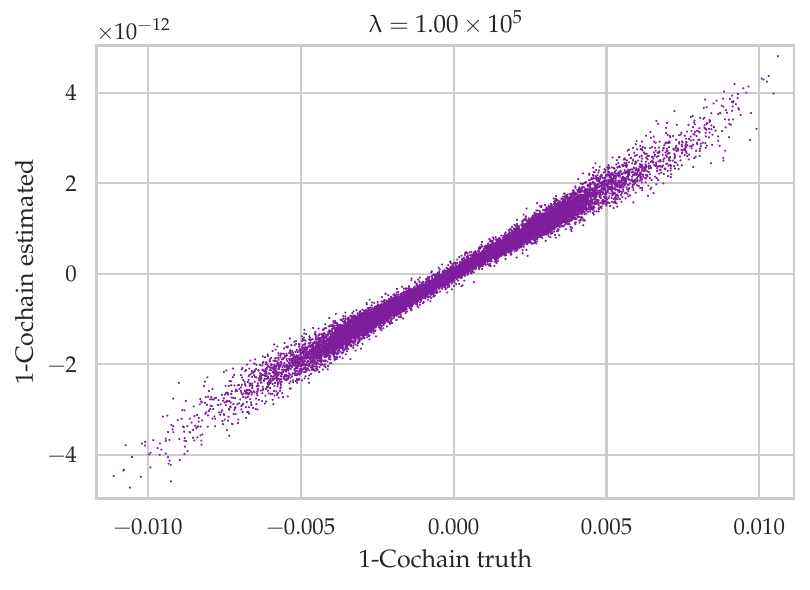}
    \label{fig:rev-int-comp-parity-damp-1e5}}\hfill
\subfloat[][Parity plot with $\lambda^*_\rho$]
    {\includegraphics[width=0.3\linewidth]{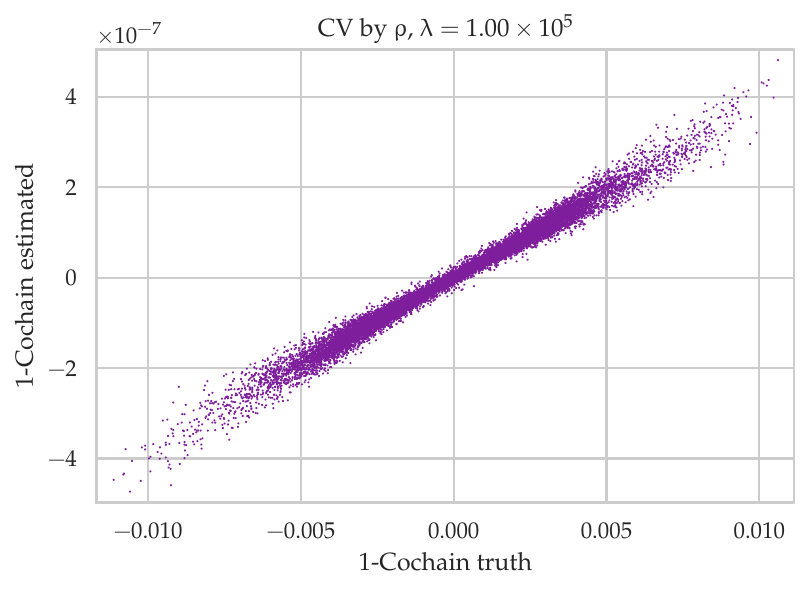}
    \label{fig:rev-int-comp-parity-damp-cv-corr}}\hfill
\subfloat[][Parity plot with $\lambda^*_\text{MSE}$]
    {\includegraphics[width=0.3\linewidth]{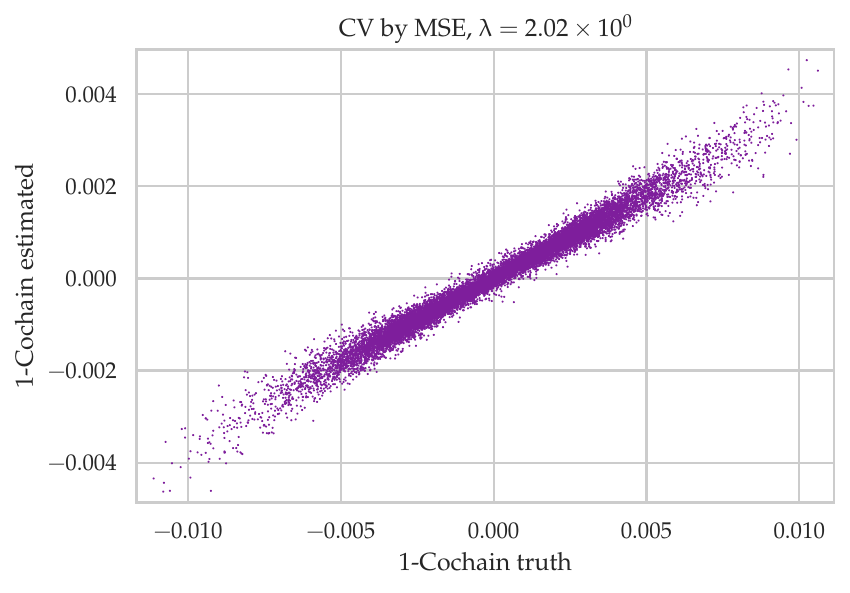}
    \label{fig:rev-int-comp-parity-damp-cv-mse}}\hfill

    \caption{Estimated velocity field from the 1-cochain with different choices of damping parameter $\lambda$'s. \protect\subref{fig:rev-int-comp-ground-truth-1-2}, and \protect\subref{fig:rev-int-comp-ground-truth1-3} shows the synthetic field which cycles around the outer loop of a torus. $\lambda^*_\rho$ and $\lambda^*_\text{MSE}$ represent the damping constant chosen by cross validation (CV) with scoring function be fisher z-transformed Pearson correlation value and mean squared error, respectively.}
    \label{fig:rev-int-comp}
\end{figure}

\clearpage
\subsection{Velocity field mapping between representations}
\label{sec:vec-field-mapping}
Given a set of points $[\vx_i]_{i=1}^n$ in $\rrr^D$ sampled from a
Manifold $\M$, vectors $\vv_i$ in the tangent subspace of $\M$ at each
data point, and a mapping $\varphi: \rrr^D \to \rrr^d$ from the
ambient space to another representation; we are interested in obtaining
the vector field $\vu_i\in\rrr^d$ of each point in the new
representation space $\vec\phi_i = \varphi(\vx_i)$. This problem can
be solved by writing out explicitly the definition of the velocity
field in the new representation space, i.e.,
\begin{equation}
    u_{ij} = \lim_{t\to 0} \frac{\varphi_j(\vx_i + \vv_i t) - \varphi_j(\vx_i)}{t}
    = \left(\nabla_\vx \varphi_i (\vx_i)\right)^\top \vv_i.
    \label{eq:vec-field-mapping-def-of-velocity}
\end{equation}

The $j$-th component of the vector $\vu_i$ is essentially the directional
derivative of the mapping $\varphi$ along  $\vv_i$ in the original space. 
Let $\mathfrak J_\vx\varphi \in \rrr^{d\times D}$ be the Jacobian matrix,
one can turn \eqref{eq:vec-field-mapping-def-of-velocity} to,
\begin{equation}
    \vu_i = \mathfrak J_\vx\varphi (\vx_i) \vv_i \, \text{ with }\, \mathfrak J_\vx \varphi(\vx_i) = \begin{bmatrix}
        \rotvert (\nabla_\vx\varphi_1(\vx_i))^\top \rotvert \\
        \rotvert (\nabla_\vx\varphi_2(\vx_i))^\top \rotvert \\
        \rowsvdots \\
        \rotvert (\nabla_\vx\varphi_d(\vx_i))^\top \rotvert \\
    \end{bmatrix} \vv_i.
    \label{eq:vec-field-mapping-jacobian}
\end{equation}

The velocity field mapping problem mapping now becomes a gradient estimation problem,
which can be solved using any gradient estimation methods, e.g.,
\cite{LuoC.S.W+:09,MukherjeeS.W:06}. In this work, we use the gradient estimation
method by \cite{MukherjeeS.S:16} which aims to solve the 
(local) weighted linear regression on the local tangent plane. More specifically,
the gradient of $f:\rrr^D\to\rrr$ at point $\vx_i$, denoted as $\nabla_{\vx} f(\vx_i)$,
is the minimizer of the following least squares problem,

\begin{equation*}
	\nabla_\vx f(\vx_i) = \argmin_{\vg \in\rrr^D} \sum_{j\sim i} w_{ij} \left\|\left(f(\vx_j) - f(\vx_i)\right) - \vg^\top (\vx_j - \vx_i) \right\|_2.
\end{equation*}

The $w_{ij}$ can be estimated by the weights used in the Local PCA \cite{Chen2013}.
Note that if the target embedding is not in Euclidean space, e.g., mapping the small
molecule dataset from the ambient space $\vX$ to the torsion space as in Figure
\ref{fig:ethanol-first-two-eigenform-tau} and \ref{fig:mda-first-two-eigenform-tau},
one has to use the proper boundary condition when calculating $f(\vx_j) - f(\vx_i)$.
That is to say, the angular distance in the torsion space should be used
(distance between $\pi/2$ and $2\pi$ is $-\pi/2$ rather than $3\pi/2$) to get smooth
estimation of the gradients.

 \section{Datasets}
\label{sec:datasets}
\subsection{Synthetic datasets}
\label{sec:exp-detail-synthetic}
\paragraph{Torus data}
Let the parameterization of a torus be
\begin{equation*}
\begin{gathered}
	x = (a+b\cos\alpha)\cos\beta; \\
	y = (a+b\cos\alpha)\sin\beta; \\
	z = a + b\sin\alpha.
\end{gathered}
\end{equation*}
The torus dataset is generated by sampling $n = 2,000$
points from a grid in $(\alpha,\beta)\in[0, 2\pi)^2$ space and mapping them
it into $\rrr^3$ by the above torus parametrization; the outer radius of torus is $a = 1$
and the inner radius is $b = 0.5$.
Random gaussian noise is added in the first three coordinates.
Additional 10 dimensional gaussian noise is added to each data point.
The first two eigenforms (point-wise velocity field in Figure \ref{fig:torus-est-vec-field-evect}.) are obtained by solving the
linear system as in \eqref{eq:linear-reverse-interpolation-edge-flow}
using the first two eigenvectors $\vec\phi_0$ and $\vec\phi_1$ of $\LL_1^s$.

\begin{figure}[!htb]
    \hfill
    \subfloat[][The first pointwise eigen-form from $\vec{\phi}_0$]
    {\includegraphics[width=0.45\linewidth]{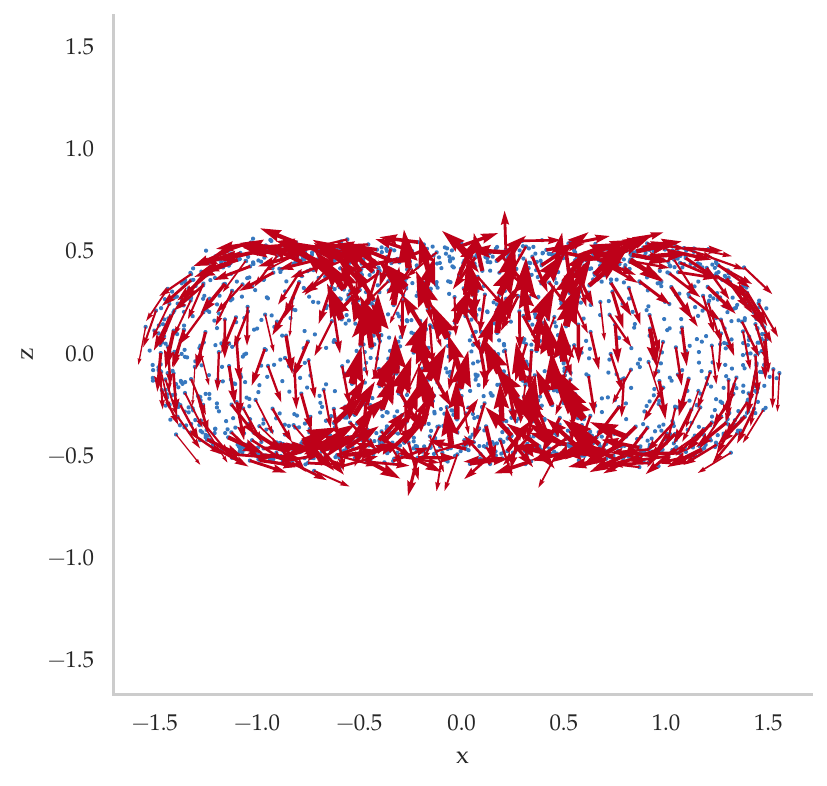}
    \label{fig:torus-est-vec-field-evect-0}}\hfill
\subfloat[][The second pointwise eigen-form from $\vec{\phi}_1$]
    {\includegraphics[width=0.45\linewidth]{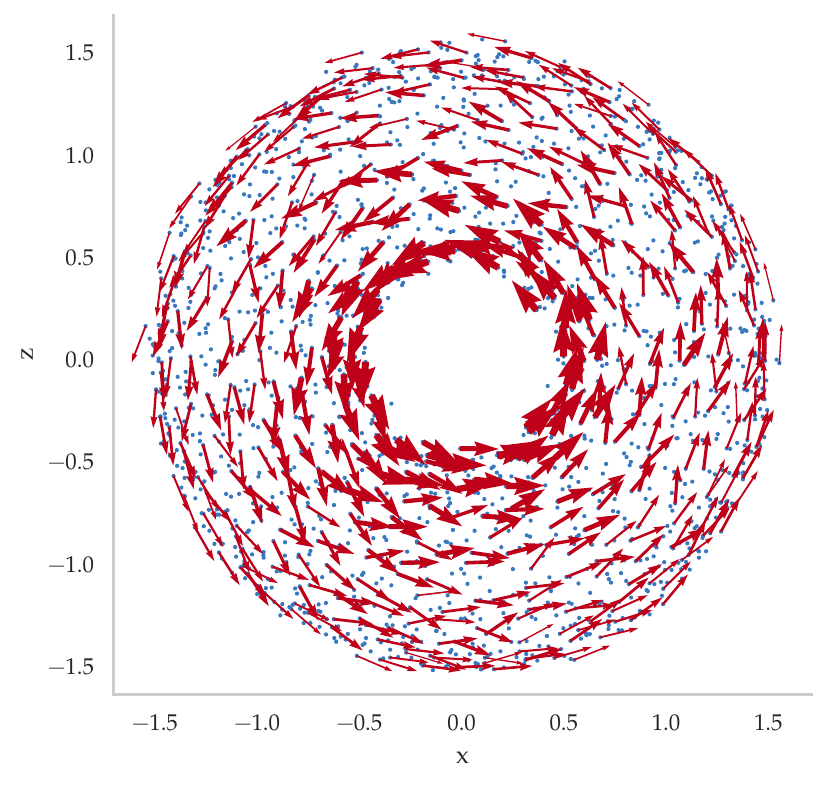}
    \label{fig:torus-est-vec-field-evect-1}}\hfill
\caption{The first two {\em interpolated} eigenforms from the $1$-cochain eigenvectors ($\vec\phi_0$ and $\vec\phi_1$) of $\LL_1^s$.}
    \label{fig:torus-est-vec-field-evect}
\end{figure}

\paragraph{2D strip and synthetic vector field  for SSL}
We sampled $n = 5,000$ points from the grid in $[-2, 2]^2$.
We then generate the vector field with
$\zeta = 0.3\zeta_\text{grad} + 0.7\zeta_\text{curl}$, where the
analytical form of $\zeta_\text{grad}$ and $\zeta_\text{curl}$ are
\begin{equation*}
\begin{gathered}
	\zeta_\text{curl}(x, y) = [x^2y, -xy^2]; \\
	\zeta_\text{grad}(x, y) = [-x, -y].
\end{gathered}
\end{equation*}

Note that because we have the analytical form of the synthetic vector field,
we do not need to use linear approximation of integration as in
\eqref{eq:linear-interpolation-edge-flow} to generate the 1-cochain
$[\vec\omega]_{xy} = \int_{x\to y} \zeta(\vec\gamma(t))\vec\gamma'(t)\dd t$.

\subsection{Small molecule datasets}
\label{sec:exp-detail-ethanol}
\label{sec:exp-detail-mda}

\begin{figure}[!htb]
    \subfloat[][$0^\text{th}$ eigenform on $\vX$]
    {\includegraphics[width=0.26\linewidth]{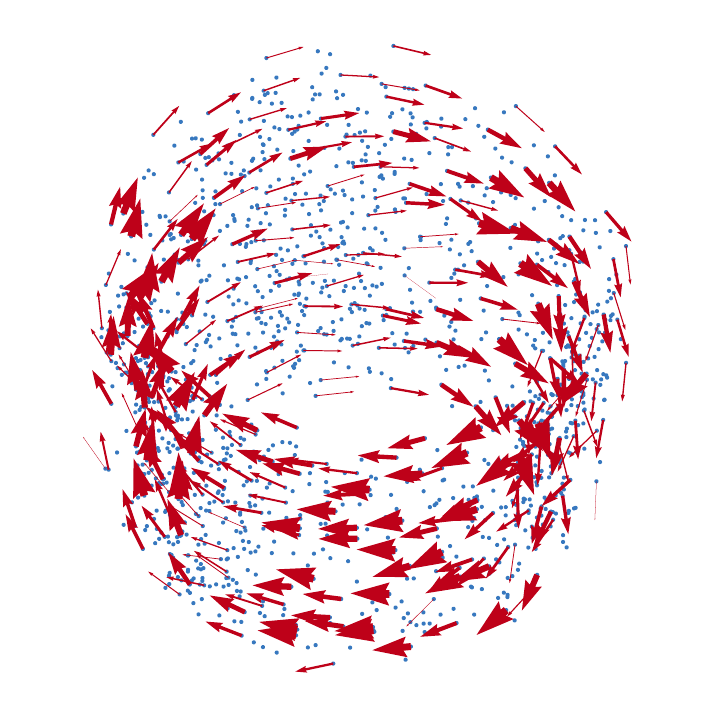}
    \label{fig:eth-pca-0}}\hfill
\subfloat[][$1^\text{st}$ eigenform on $\vX$]
    {\includegraphics[width=0.26\linewidth]{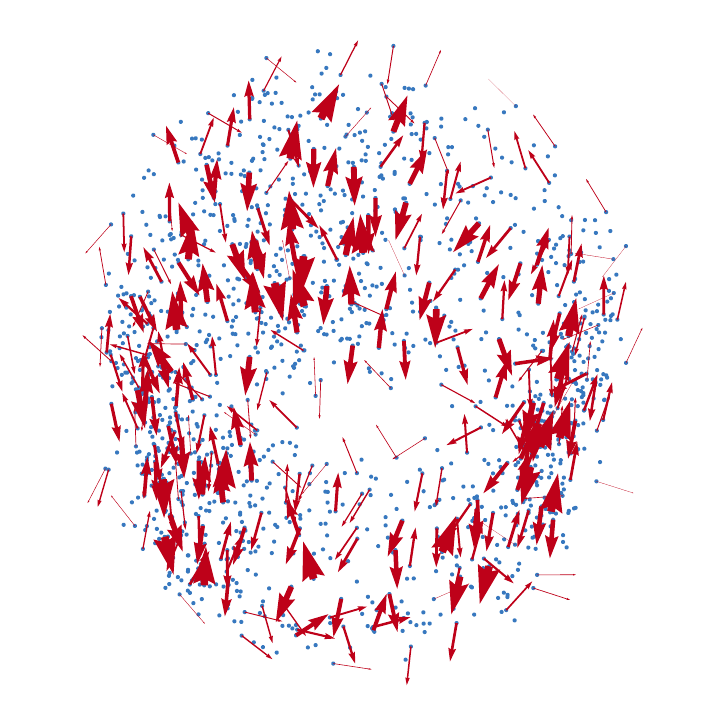}
    \label{fig:eth-pca-1}}\hfill
\subfloat[][HHD result]
    {\includegraphics[width=0.35\linewidth]{figs/ethanol/grad_curl_harmonic_ness_scale_1_4}
    \label{fig:eth-hhd}}\hfill

    \subfloat[][$2^\text{nd}$ eigenform on $\vX$]
    {\includegraphics[width=0.24\linewidth]{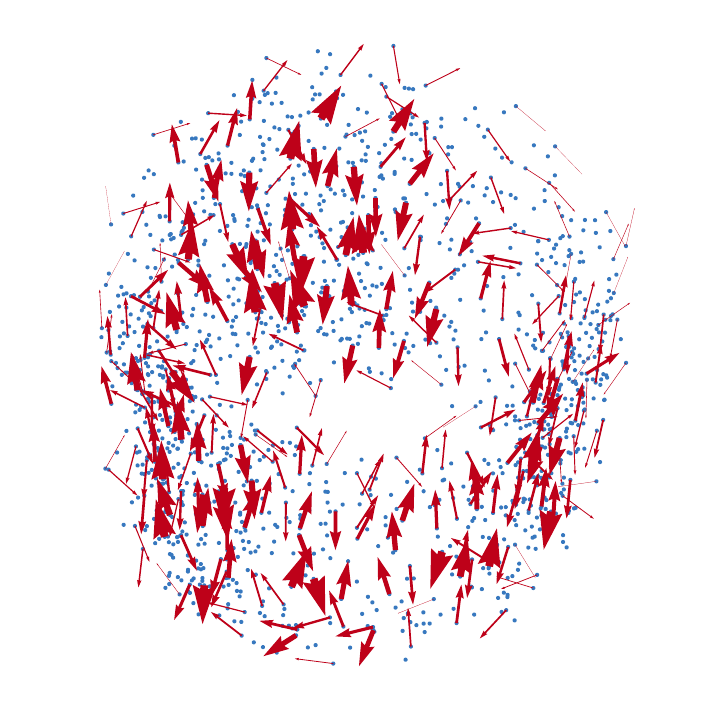}
    \label{fig:eth-pca-2}}\hfill
\subfloat[][$3^\text{rd}$ eigenform on $\vX$]
    {\includegraphics[width=0.24\linewidth]{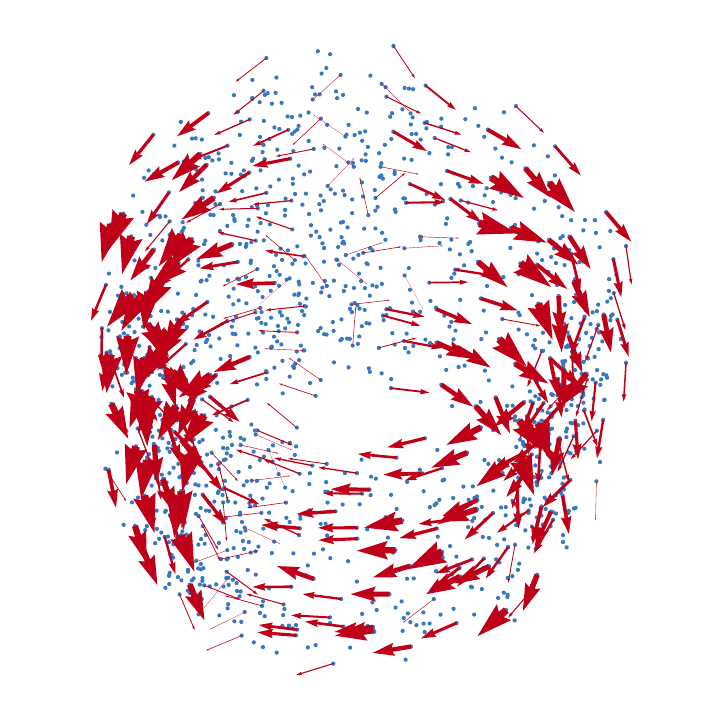}
    \label{fig:eth-pca-3}}\hfill
\subfloat[][$4^\text{th}$ eigenform on $\vX$]
    {\includegraphics[width=0.24\linewidth]{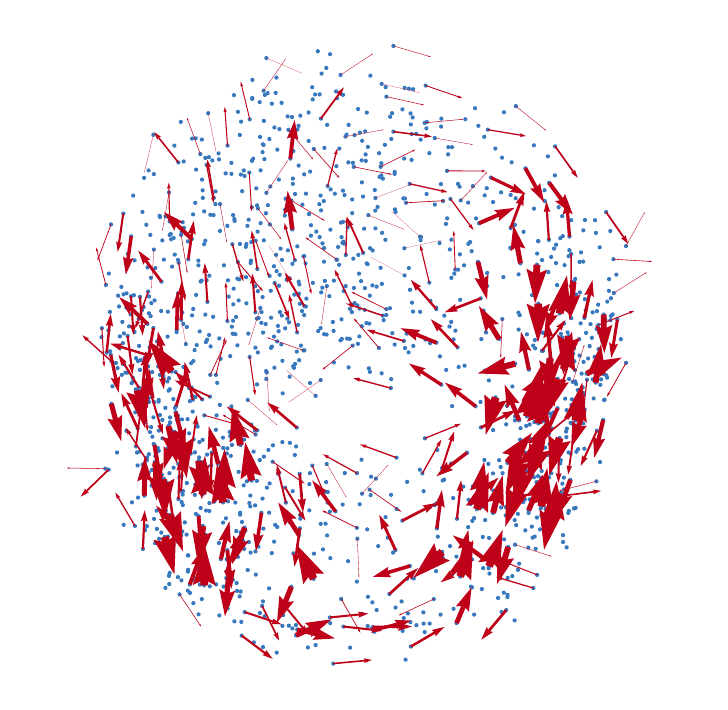}
    \label{fig:eth-pca-4}}\hfill
\subfloat[][$5^\text{th}$ eigenform on $\vX$]
    {\includegraphics[width=0.24\linewidth]{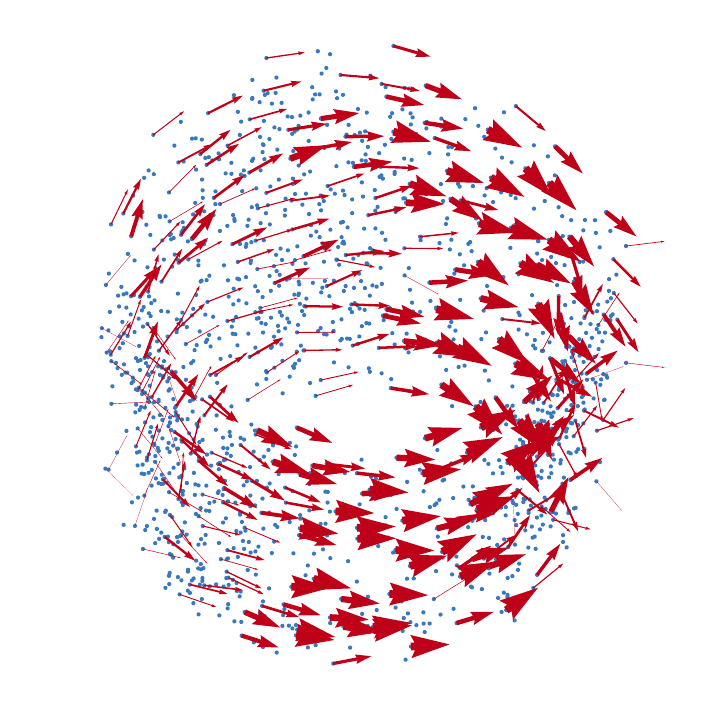}
    \label{fig:eth-pca-5}}\hfill

    \subfloat[][$6^\text{th}$ eigenform on $\vX$]
    {\includegraphics[width=0.24\linewidth]{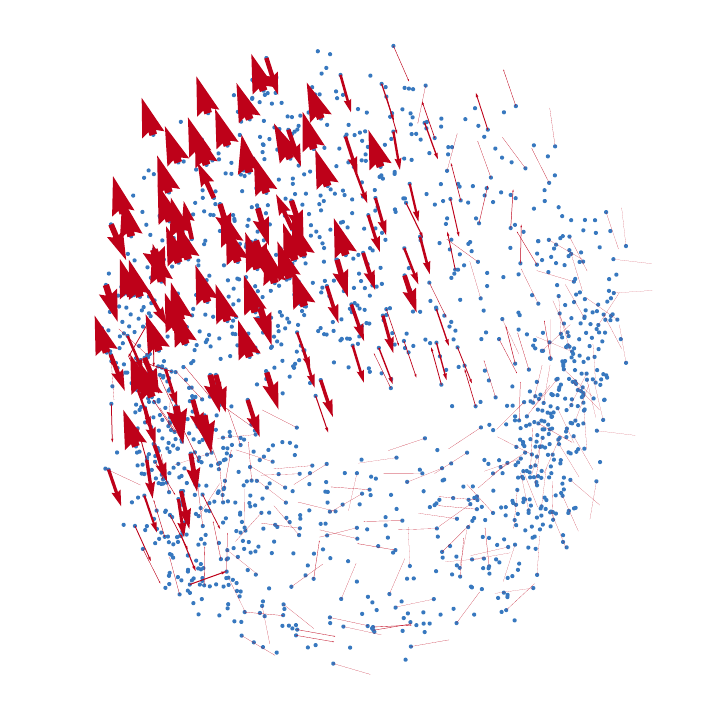}
    \label{fig:eth-pca-6}}\hfill
\subfloat[][$7^\text{th}$ eigenform on $\vX$]
    {\includegraphics[width=0.24\linewidth]{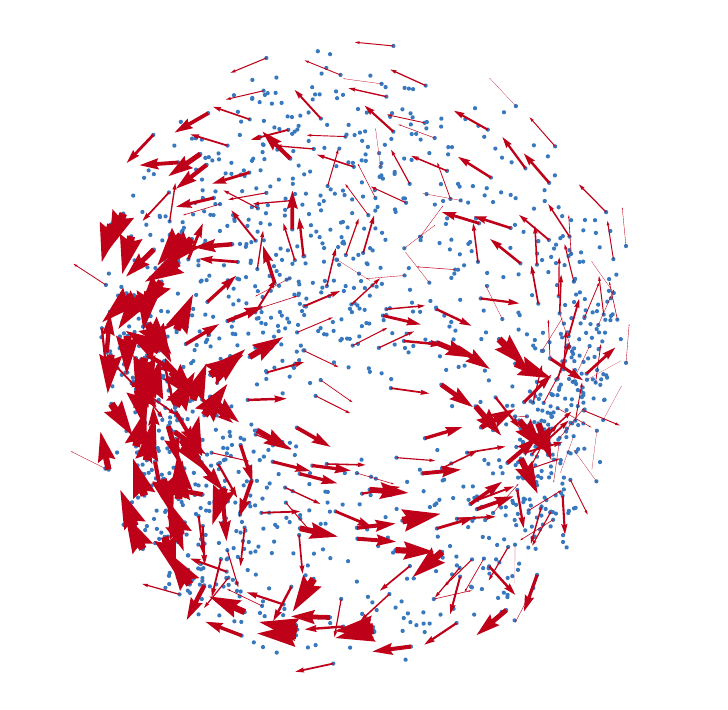}
    \label{fig:eth-pca-7}}\hfill
\subfloat[][$8^\text{th}$ eigenform on $\vX$]
    {\includegraphics[width=0.24\linewidth]{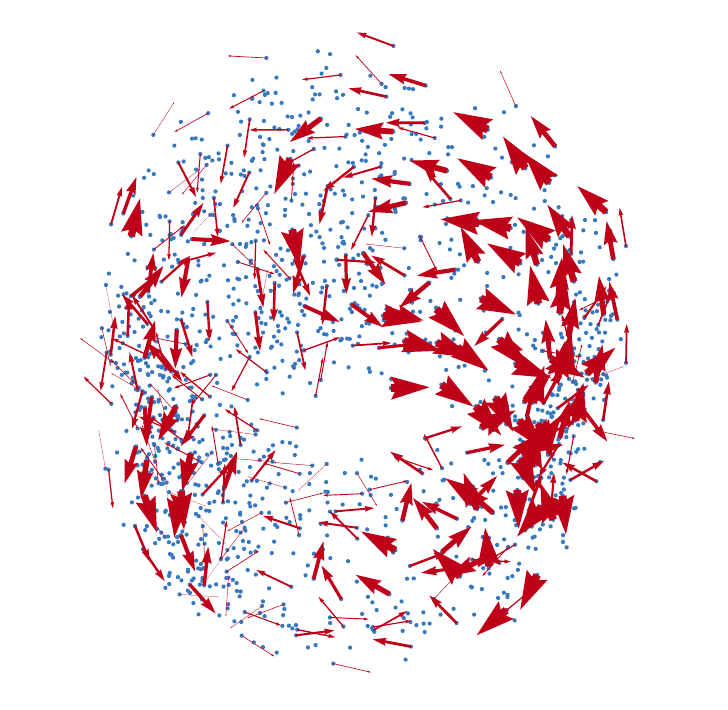}
    \label{fig:eth-pca-8}}\hfill
\subfloat[][$9^\text{th}$ eigenform on $\vX$]
    {\includegraphics[width=0.24\linewidth]{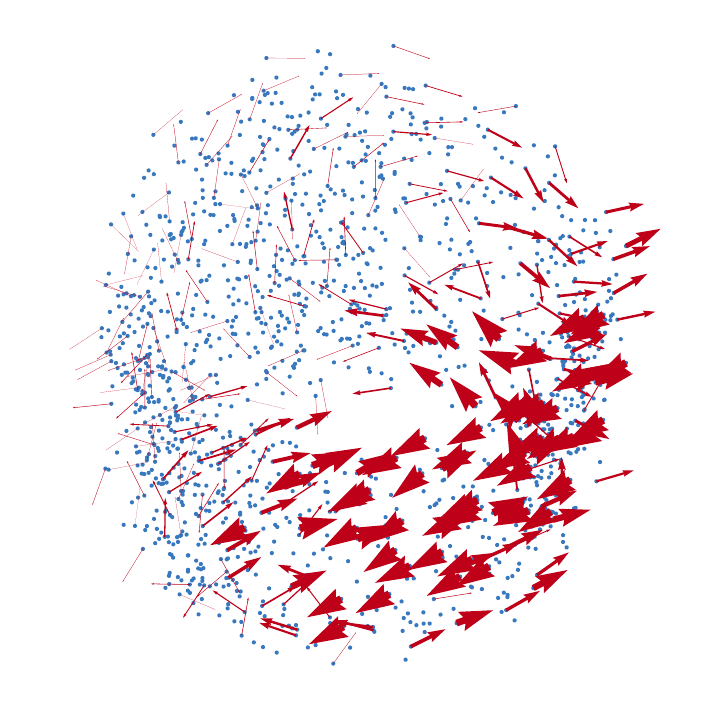}
    \label{fig:eth-pca-9}}\hfill

    \caption{The first 10 estimated vertex-wise eigenforms on the original dataset $\vX$ by solving the linear system \eqref{eq:linear-reverse-interpolation-edge-flow-damped}. Figure \protect\subref{fig:eth-hhd} is the HHD on the first 10 eigenforms, showing that the first two eigenflows are harmonic; the third, fifth, eleventh, thirteenth, and the sixteenth eigenflows are gradient flow, while the rest are curl flows.}
    \label{fig:eth-extra-experiment}
\end{figure}

The ethanol and malondialdehyde (MDA) datasets \cite{ChmielaS.T.S+:17} consist of
molecular dynamics (MD) trajectories
with different amounts of conformational degrees of freedom. A point in the dataset
corresponds to a molecular configuration, which is recorded in the xyz format. More specifically,
if the molecule has $N$ atoms, then a configuration can be specified by a
$N\times 3$ matrix. To remove the translational and rotational symmetry in the original
configuration space, we preprocess the data by considering two of the angles of every
triplet of atoms in a molecule (under the observation that
two angles are sufficient to determine a triangle up to a constant). The linear relation is removed
by applying {\em Principal component analysis} (PCA) with the unexplained variance ratio
less than $10^{-4}$.
This generates the original data $\vX$ with {\em ambient dimension} upper bounded by
$D \leq 2\cdot \binom{N}{3}$.
In Figure \ref{fig:eth-extra-experiment}, we estimate the first 10 eigenforms by
solving the linear system \eqref{eq:linear-reverse-interpolation-edge-flow-damped}.
The 0-th
eigenform clearly represents the bigger loop parameterized by Hydroxyl rotor as
shown in the inset scatter plot of Figure \ref{fig:lambda-baselines-ethanol}.
In Figure \ref{fig:eth-pca-1}, it is difficult
to tell whether the first eigenflow corresponds to Methyl rotor or not. One can overcome this
issue by mapping the first eigenform to the torsion space with prior knowledge
as illustrated in Figure \ref{fig:ethanol-first-two-eigenform-tau}.
Harmonic flows often represent a global structure; by contrast, flows in Figure
\ref{fig:eth-pca-2}--\ref{fig:eth-pca-4} are more localized, implying that the eigenflows
are not harmonic.
The Helmholtz-Hodge decomposition of the eigenvectors of $\LL_1$ in Figure
\ref{fig:eth-hhd} confirms this.

Figure \ref{fig:mda-eigenform-plots} shows the scatter plot of the first three
principal components of the MDA dataset. As clearly shown in the figures, it is
difficult to
make sense of the topological structure for the manifold of such dataset is a torus
embedded in a 4
dimensional space. With the aid of the first two harmonic eigenforms, one can
infer that the first loop travels in the direction of northwest to southeast, while
the second loop goes diagonally from northeast to southwest. With proper prior knowledge,
one can map the harmonic eigenforms to the torsion space to get a better visualization,
as shown in Figure \ref{fig:mda-first-two-eigenform-tau}.

\begin{figure}[!htb]
    \subfloat[][]
    {\includegraphics[width=0.24\linewidth]{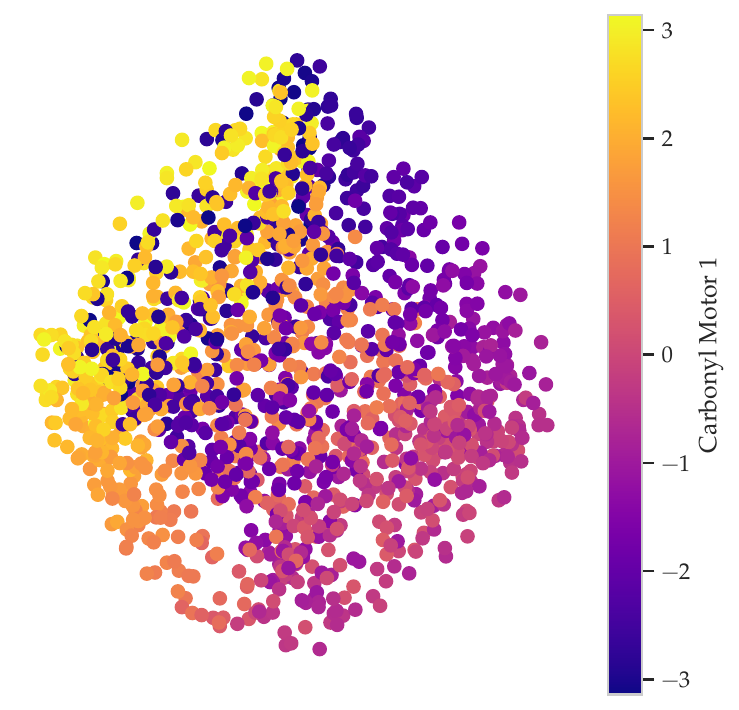}
    \label{fig:mda-eigenform-plots-tau-1}}\hfill
\subfloat[][]
    {\includegraphics[width=0.24\linewidth]{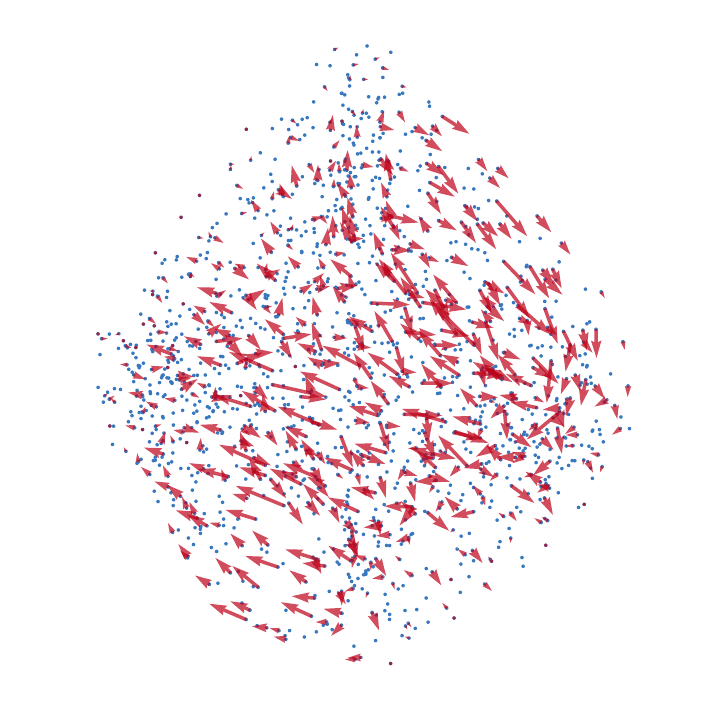}
    \label{fig:mda-eigenform-plots-1}}\hfill
\subfloat[][]
    {\includegraphics[width=0.24\linewidth]{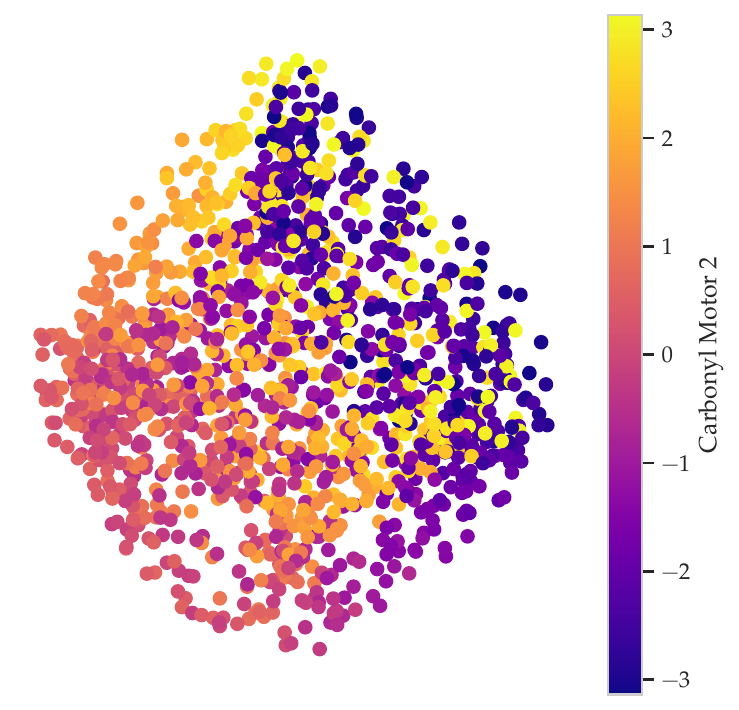}
    \label{fig:mda-eigenform-plots-tau-2}}\hfill
\subfloat[][]
    {\includegraphics[width=0.24\linewidth]{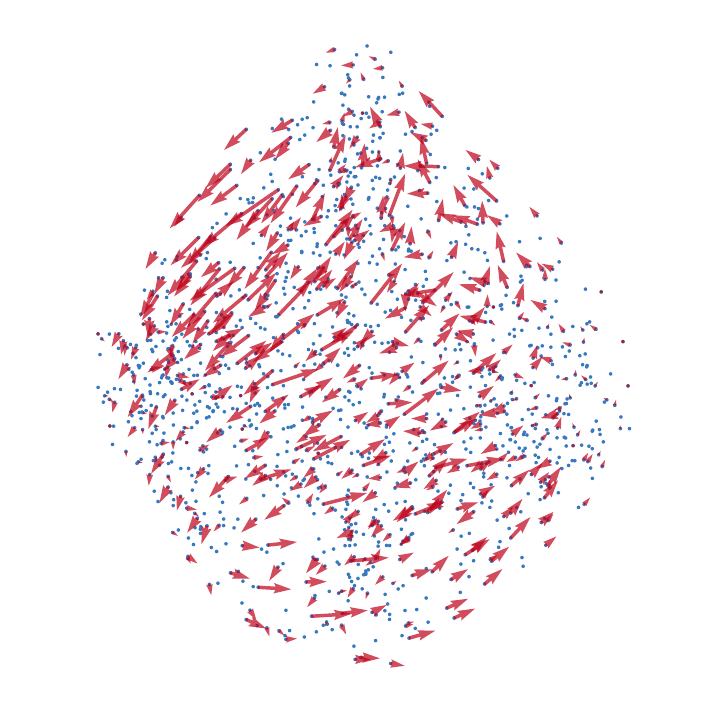}
    \label{fig:mda-eigenform-plots-2}}\hfill
\caption{The harmonic eigenflows of the MDA dataset. \protect\subref{fig:mda-eigenform-plots-tau-1} and \protect\subref{fig:mda-eigenform-plots-tau-2} are the scatter plot of the first three PCs colored by the first and second Carbonyl rotors (purple and yellow in the inset of \ref{fig:lambda-baselines-mda}, see also \ref{fig:mda-first-two-eigenform-tau}). \protect\subref{fig:mda-eigenform-plots-1} and \protect\subref{fig:mda-eigenform-plots-2} represent the first two harmonic eigenforms estimated from the eigenvectors of $\LL_1^s$. The zeroth eigenform in \protect\subref{fig:mda-eigenform-plots-1} parameterizes the first carbonyl rotor in \protect\subref{fig:mda-eigenform-plots-tau-1}, while the first eigenform in \protect\subref{fig:mda-eigenform-plots-2} represents the second carbonyl rotor as in \protect\subref{fig:mda-eigenform-plots-tau-2}. See also Figure \ref{fig:mda-first-two-eigenform-tau} for a better visualization.}
    \label{fig:mda-eigenform-plots}
\end{figure}

\subsection{Ocean drifter dataset}
\label{sec:exp-detail-ocean}
The ocean drifter data, also known as {\em Global Lagrangian Drifter Data},
were collected by NOAA's Atlantic Oceanographic and Meteorological Laboratory.
The dataset was used in \cite{FroylandG.P:15} to analyze the Lagrangian
coherent structures in the ocean current, showing that certain
flow structures stay coherent over time.
Each point in the dataset is a buoy at a certain time, with buoy ID,
location (in latitude \& longitude), date/time, velocity,
and water temperature available to the practitioner.
We extract the buoys that were in the North Pacific ocean dated between
2010 to 2019. The original sample size is around $3$ million,
we sampled $1,500$ furthest buoy that meet the above criteria.

The velocity field in the original data depends a lot on the events in the shorter
time scale,
e.g., wind or faster changing ocean current.
In comparison, ocean motions at longer time scale
oftentimes are more interesting to the scientists.
Therefore, we discard the short-term velocity field in the original data and
calculate the the velocity field as follows.
We first compute the finite difference of the current location and the next location
of the same buoy.
The velocity of the buoy at current point is obtained by dividing this quantity
by the time difference.
The 1-cochain can be constructed by linear approximate of integral as in
\eqref{eq:linear-interpolation-edge-flow} after
obtaining the velocity field of each point.

\subsection{Single cell RNA velocity data}
All of the RNA velocity data (Chromaffin, Mouse hippocampus, and human
glutamatergic neuron cell differentiation dataset)
 as well as methods/codes to preprocess them can be found in \cite{LaMannoG.S.Z+:18}.
The RNA velocity is a point-wise vector field in the ambient space
predicting the future evolution of the cell. To generate the 1-cochain,
we apply a linear interpolation as in \eqref{eq:linear-interpolation-edge-flow}.
The number of cells in the Chromaffin, mouse hippocampus, and human glutamatergic
neuron cell differentiation datasets are $n'=384$, $n'=18,213$, and $n'=1,720$,
respectively. PCA is applied on the RNA expression, resulting in the ambient dimensions
of the aforementioned three datasets being $D = 5$, $10$, and $2$, respectively.
We choose the furthest $n = 800$
and $n = 600$ cells for the mouse hippocampus and the human
glutamatergic neuron data
when building the simplicial complex. Since the sample size of the Chromaffin dataset
is small, we used all the cells in our analysis.

\begin{figure}[!htb]
    \subfloat[][$k = 1$]
    {\includegraphics[width=0.32\linewidth]{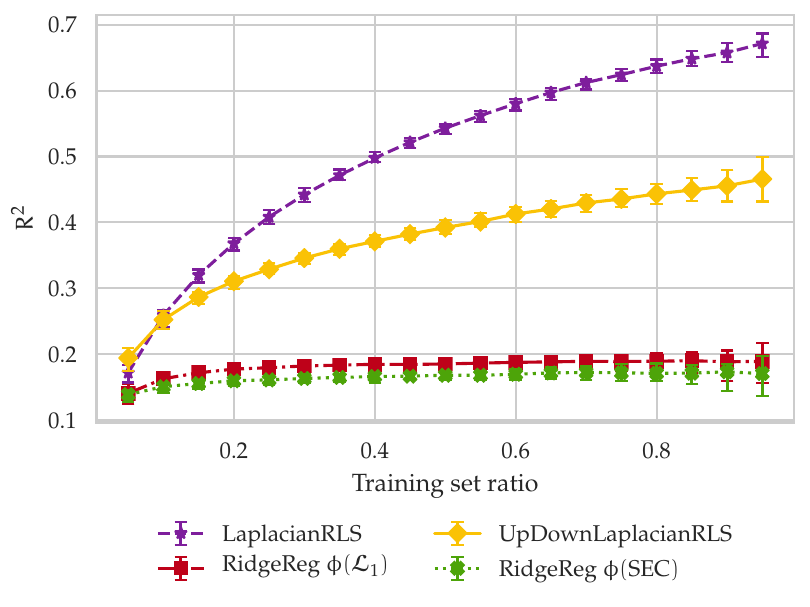}
    \label{fig:hippocampus-knn-1}}\hfill
\subfloat[][$k = 50$]
    {\includegraphics[width=0.32\linewidth]{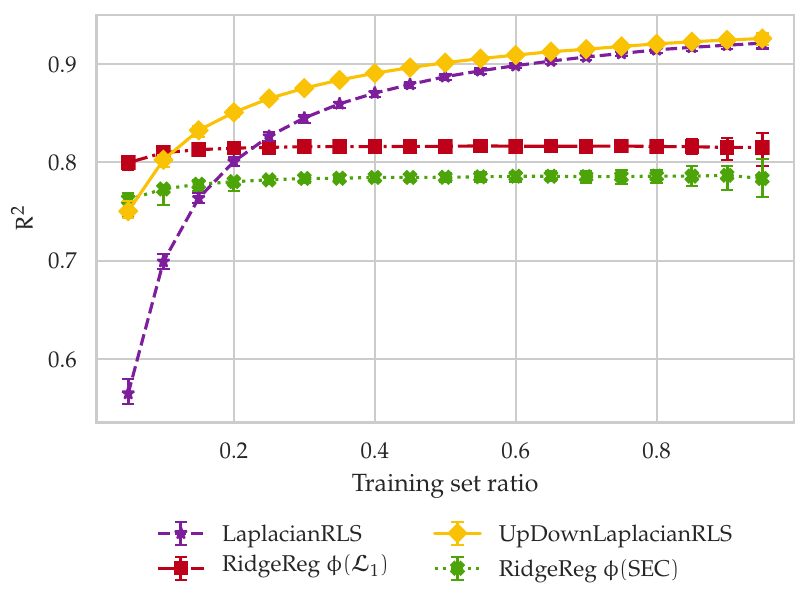}
    \label{fig:hippocampus-knn-50}}\hfill
\subfloat[][$k = 500$]
    {\includegraphics[width=0.32\linewidth]{figs/rna_hippocampus/k_500/r2_score}
    \label{fig:hippocampus-knn-500}}\hfill
\caption{SSL results of the velocity fields of mouse hippocampus cell differentiation dataset with different smoothing parameter $k$.}
    \label{fig:hippocampus-knn}
\end{figure}

In the RNA velocity framework, one can control how smooth the generated RNA velocity
is by specifying the number of nearest neighbors $k$. The larger this
parameter is, the smoother the vector field will become.
We show in Figure \ref{fig:hippocampus-knn} that the proposed
algorithm out-performed other algorithms for several choices of smoothness values.
Note that in theory, the UpDownLaplacianRLS should be at least as good as
LaplacianRLS algorithm, for the first one is the extension of the second method.
However, since we only choose the hyperparameter when train ratio is $0.2$,
it is possible that the under-performance of UpDownLaplacian seen in Figure \ref{fig:hippocampus-knn-1} is due to suboptimal choices of parameters.

\section{SSL experiments on the divergence free flows}
\label{sec:urban-traffic-and-comp-jiaj}

\begin{figure}[!htb]
	\subfloat[][]
    {\includegraphics[width=0.24\linewidth]{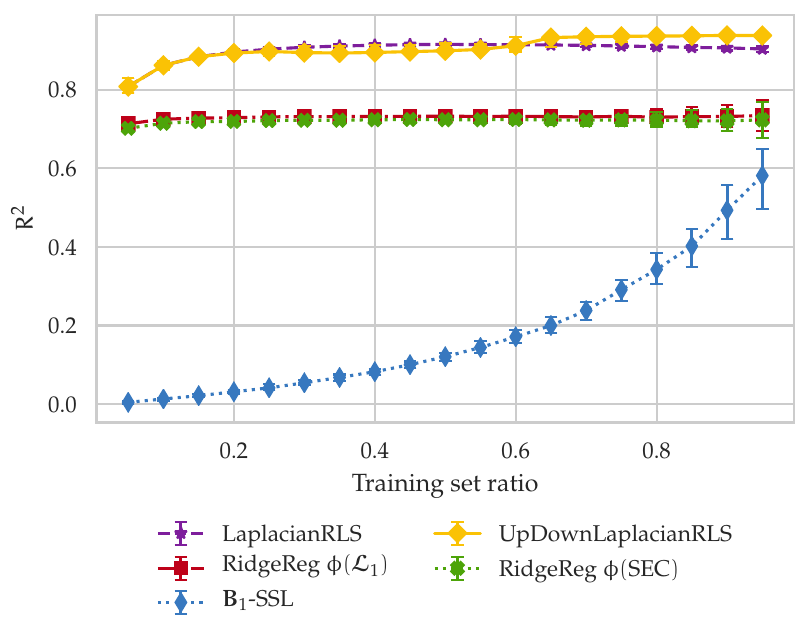}
    \label{fig:comp-with-jssb19-2d-r2-with-b1}}\hfill
\subfloat[][]
    {\includegraphics[width=0.24\linewidth]{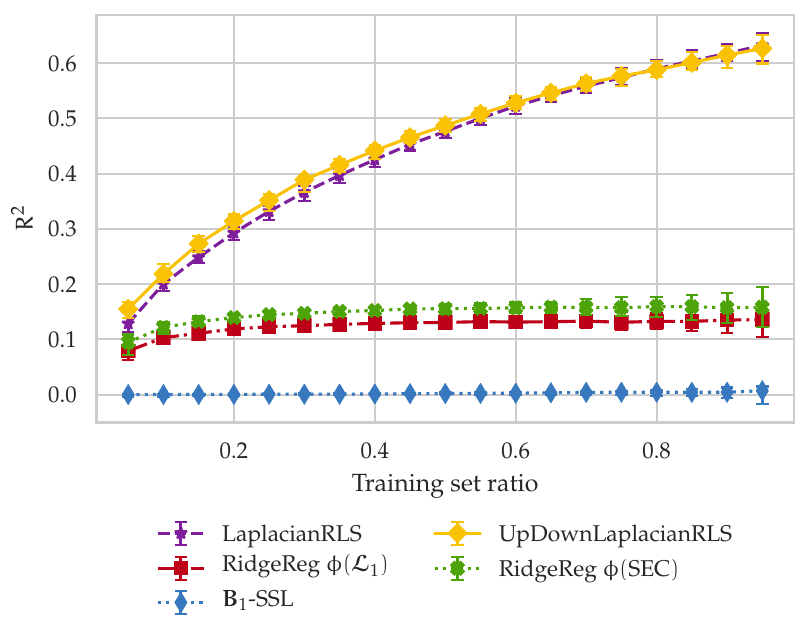}
    \label{fig:comp-with-jssb19-ocean-farest-r2-with-b1}}\hfill
\subfloat[][]
    {\includegraphics[width=0.24\linewidth]{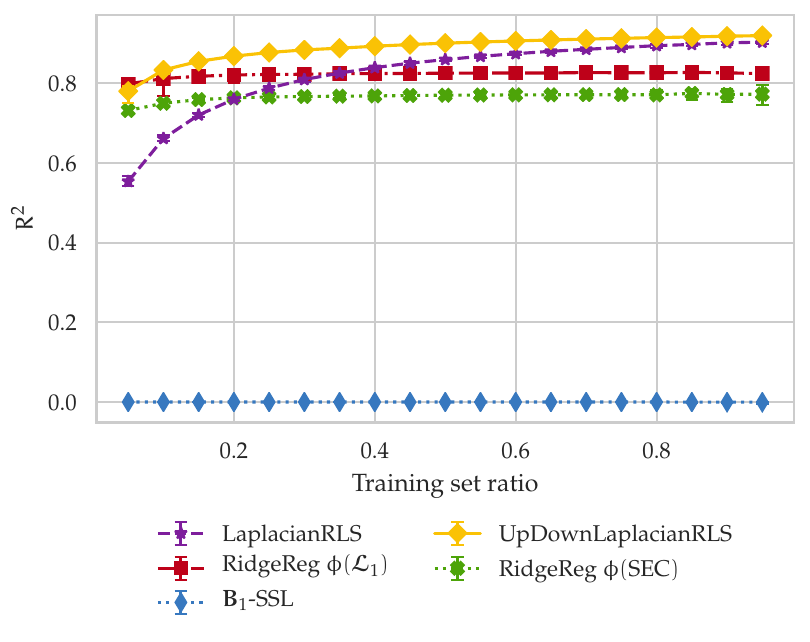}
    \label{fig:comp-with-jssb19-chromaffin-r2-with-b1}}\hfill
\subfloat[][]
    {\includegraphics[width=0.24\linewidth]{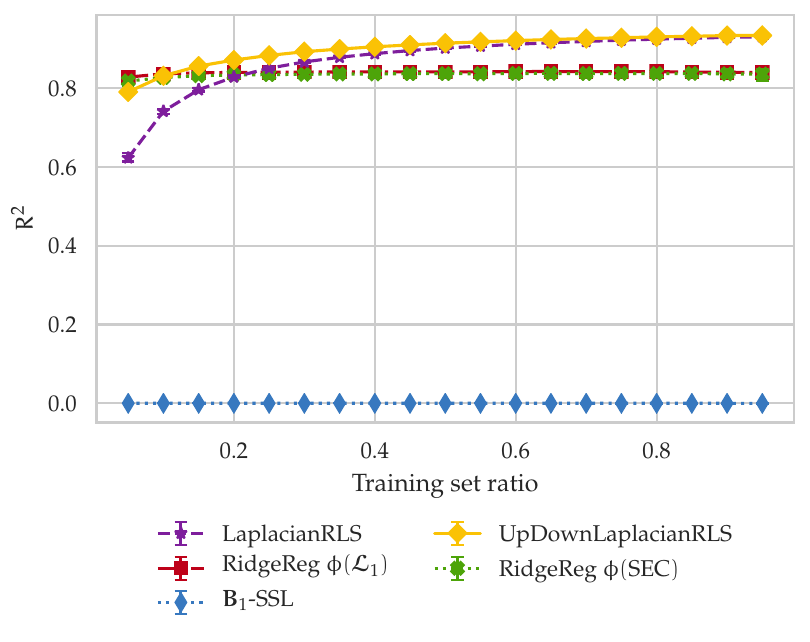}
    \label{fig:comp-with-jssb19-hippocampus-r2-with-b1}}\hfill

    \subfloat[][]
    {\includegraphics[height=1.95cm]{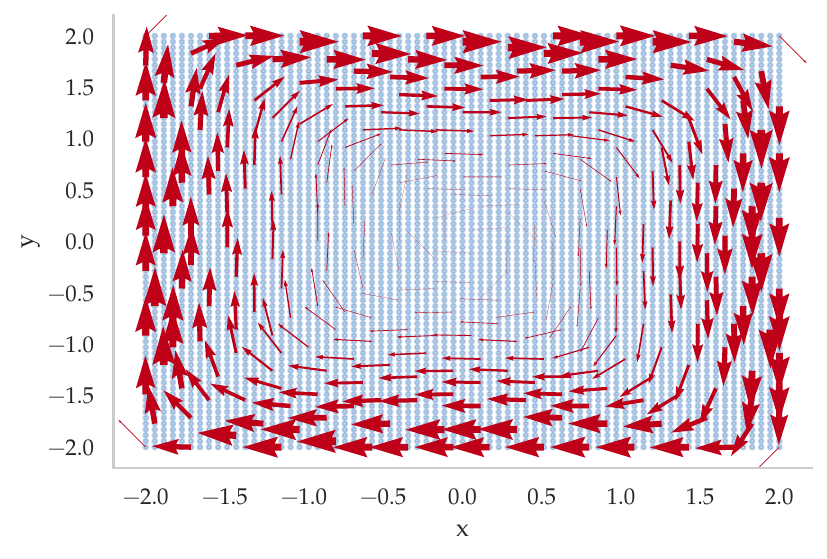}
    \label{fig:comp-with-jssb19-2d}}\hfill
\subfloat[][]
    {\includegraphics[height=1.95cm]{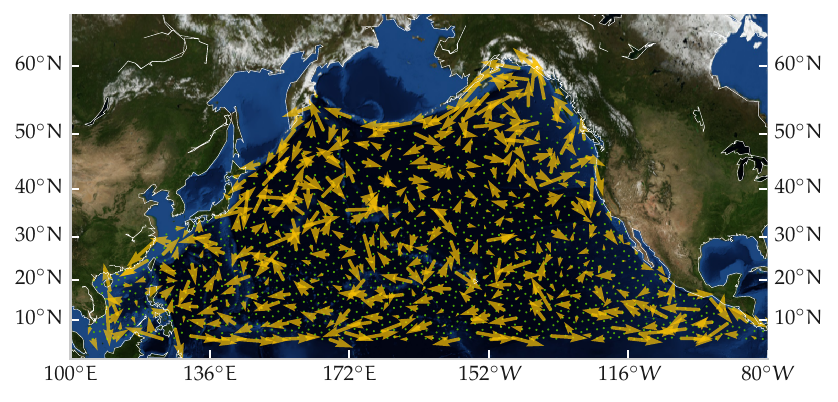}
    \label{fig:comp-with-jssb19-ocean-farest}}\hfill
\subfloat[][]
    {\includegraphics[height=1.95cm]{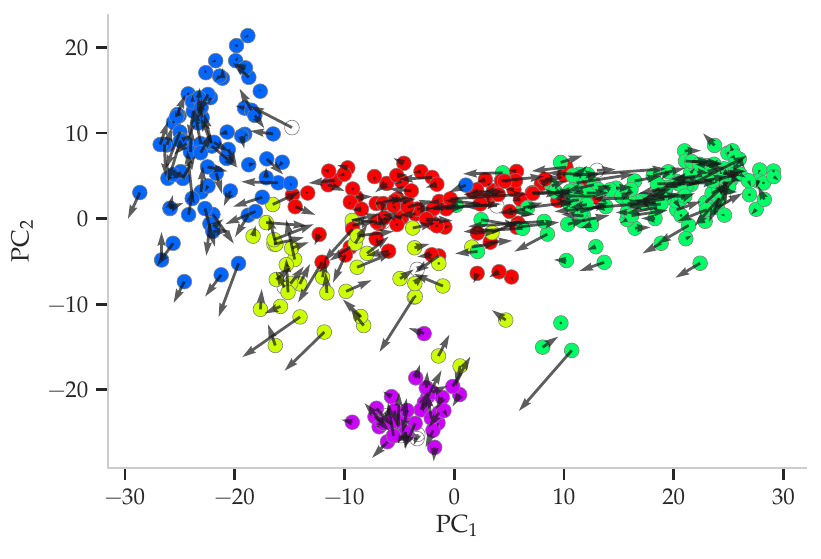}
    \label{fig:comp-with-jssb19-chromaffin}}\hfill
\subfloat[][]
    {\includegraphics[height=1.95cm]{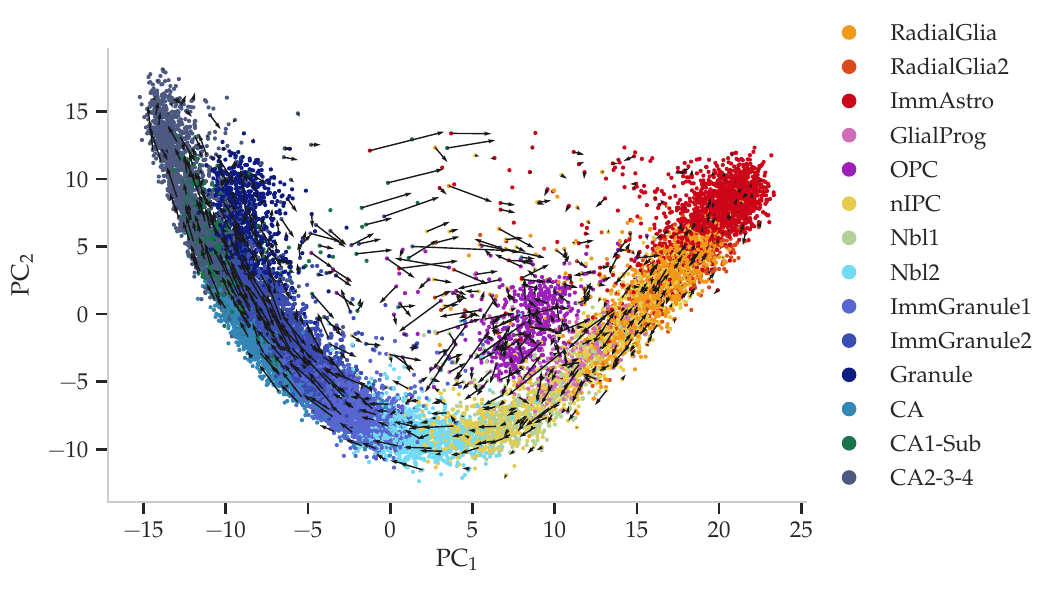}
    \label{fig:comp-with-jssb19-hippocampus}}\hfill

	\subfloat[][]
    {\includegraphics[width=0.24\linewidth]{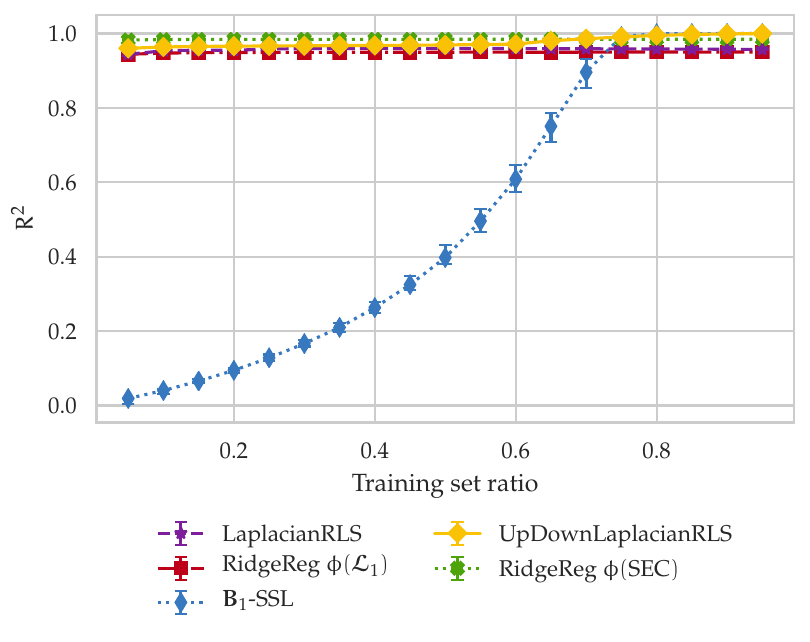}
    \label{fig:comp-with-jssb19-2d-r2-curl}}\hfill
\subfloat[][]
    {\includegraphics[width=0.24\linewidth]{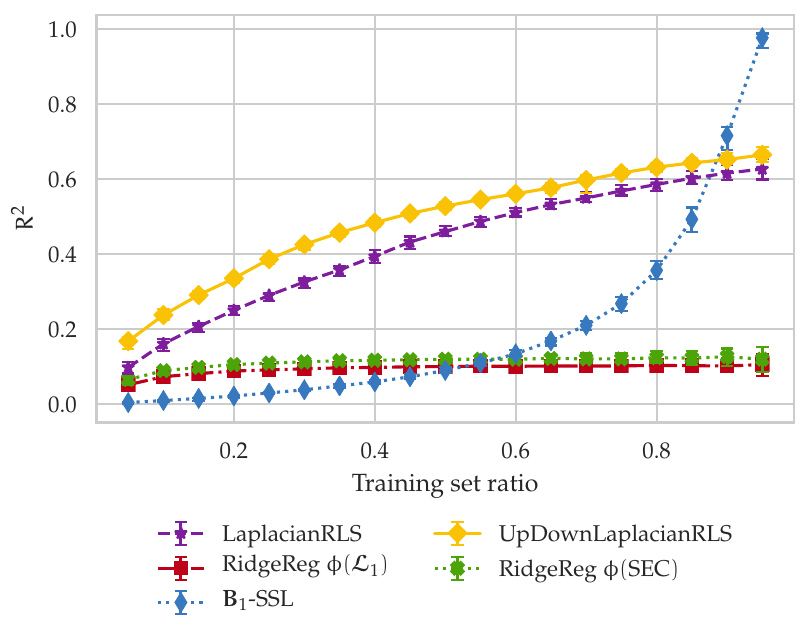}
    \label{fig:comp-with-jssb19-ocean-farest-r2-curl}}\hfill
\subfloat[][]
    {\includegraphics[width=0.24\linewidth]{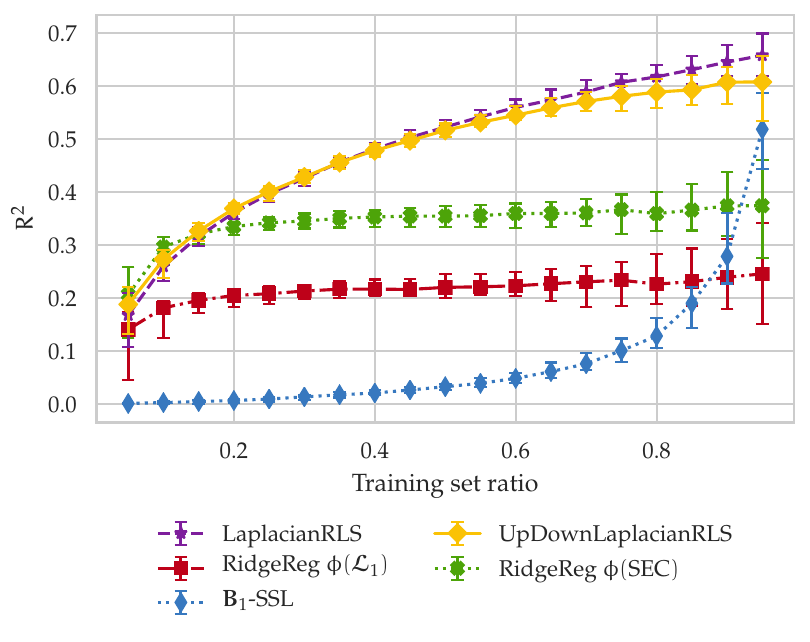}
    \label{fig:comp-with-jssb19-chromaffin-r2-curl}}\hfill
\subfloat[][]
    {\includegraphics[width=0.24\linewidth]{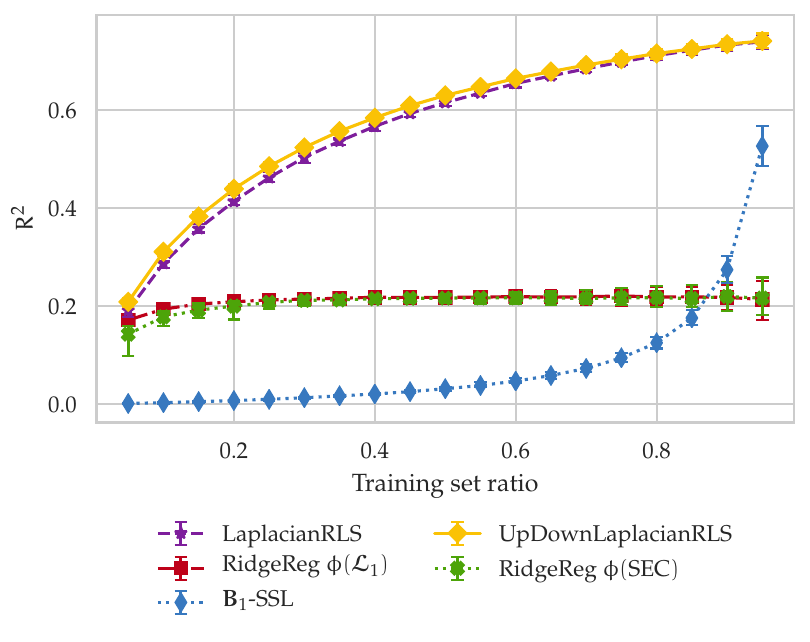}
    \label{fig:comp-with-jssb19-hippocampus-r2-curl}}\hfill
    \caption{SSL results on various datasets with the $\vB_1$-SSL proposed by \cite{JiaJ.S.S+:19}. Columns from left to right correspond to the results of 2D strip, ocean buoy, Chromaffin cell differentiation, and mouse hippocampus cell differentiation dataset. The top row shows the SSL results on the original velocity field with the result of $\vB_1$-SSL (blue curve) added. The second row represents the curl component of the flows in Figure \ref{fig:ssl-result} using HHD. The third row are the SSL results of the data with the curl flow shown in the second row.}
    \label{fig:comp-with-jssb19}
\end{figure}

The $\vB_1$-SSL algorithm proposed by \cite{JiaJ.S.S+:19} works on
(approximately) divergence-free flow. However, the assumption is not
always valid for the flows observed in the many real datasets are
often times a mixture of gradient, curl, and harmonic flows.  Applying
the $\vB_1$-SSL algorithm by \cite{JiaJ.S.S+:19} do not results in a
good result, as shown in Figures
\ref{fig:comp-with-jssb19-2d-r2-with-b1}--\ref{fig:comp-with-jssb19-hippocampus-r2-with-b1}.
Note that these
figures are identical to Figures \ref{fig:ssl-r2-2D-strip}--\ref{fig:ssl-r2-hippocampus-k500}, but with the SSL results from
\cite{JiaJ.S.S+:19} (blue curves) added.  Except for the synthetic
flow, the performances of $\vB_1$-SSL are as bad as random guess.

To further evaluate $\vB_1$-SSL of \cite{JiaJ.S.S+:19}, in comparison with
our $\LL_1$ based SSL algorithms, we artificially create data that
satisfies the $\vB_1$-SSL assumptions. Namely, we extract the curl
component from the computed 1-cochain using HHD.  In mathematical
terms, we first solve the following linear system to get the vector
potential $\hat{\vv} = \argmin_{\vv\in\rrr^{n_2}} \|\vB_2 \vv -
\vec\omega\|_2$.  The curl component is obtained by projecting
1-cochain $\vec\omega$ onto the image of $\vB_2$, i.e.,
$\vec\omega_{\mathrm{curl}} = \vB_2\hat{\vv}$.
The estimated velocity field from the curl cochains for each datasets
can be found in Figures \ref{fig:comp-with-jssb19-2d}--\ref{fig:comp-with-jssb19-hippocampus}. 
As shown in Figure \ref{fig:comp-with-jssb19-2d-r2-curl}--\ref{fig:comp-with-jssb19-hippocampus-r2-curl},
the proposed
algorithms based on both the up and down Laplacian out-perform
\cite{JiaJ.S.S+:19} for small train/test ratio. In fact, the
performance of $\vB_1$-SSL is always weak until the proportion of labeled
examples exceeds about $0.7$.  For manifolds with simple structure,
i.e., 2D plane and ocean dataset, \cite{JiaJ.S.S+:19} can achieve
almost perfect predictions ($R^2 \approx 1$) when train-test ratio
$\geq 0.9$.  However, this is not the case for manifolds with complex
structures, e.g., RNA velocity datasets.
 \section{Choice of regularization parameter $\lambda$ for estimating point-wise velocity field from 1-cochain---vector field from all experiments}
\begin{figure}[!htb]
    \subfloat[][]
    {\includegraphics[width=0.3\linewidth]{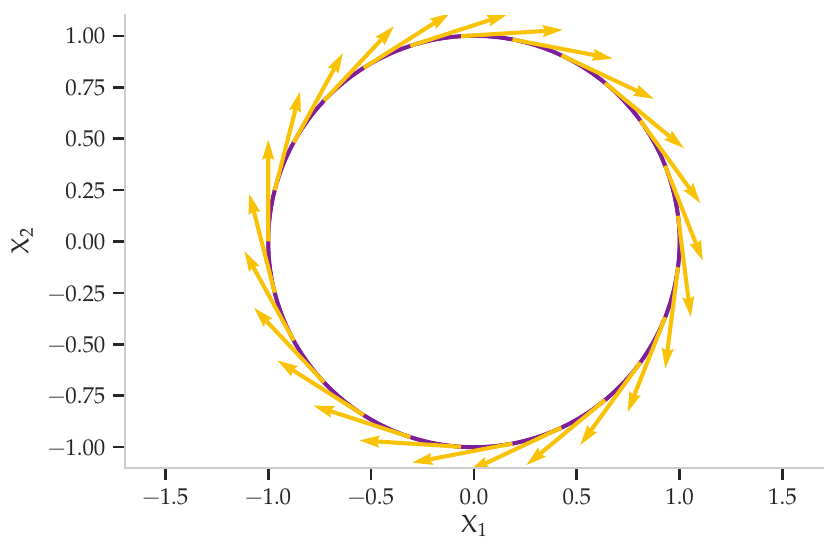}
    \label{fig:rev-int-comp-exps-circle-damp0}}\hfill
\subfloat[][]
    {\includegraphics[width=0.3\linewidth]{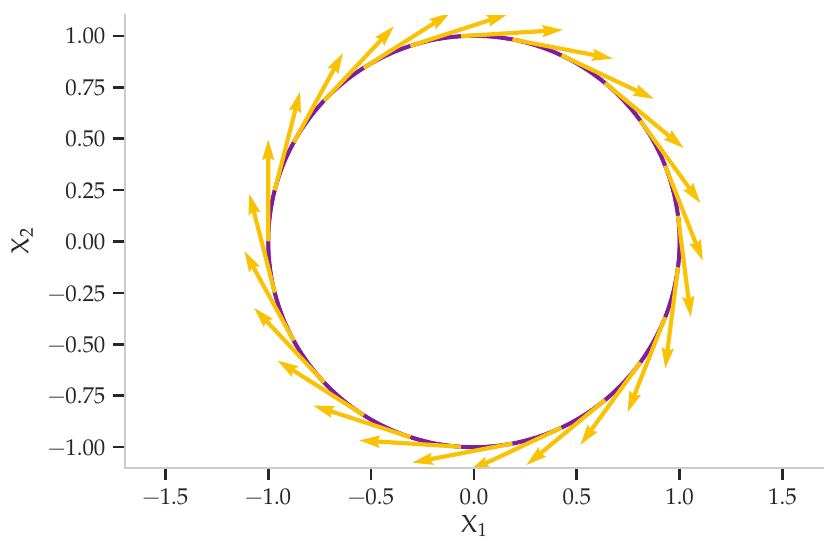}
    \label{fig:rev-int-comp-exps-circle-corr}}\hfill
\subfloat[][]
    {\includegraphics[width=0.3\linewidth]{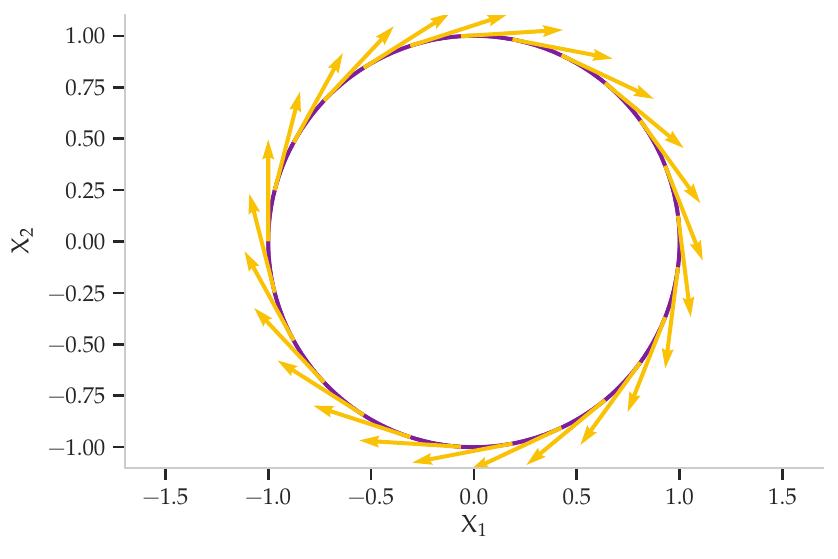}
    \label{fig:rev-int-comp-exps-circle-mse}}\hfill

    \subfloat[][]
    {\includegraphics[width=0.3\linewidth]{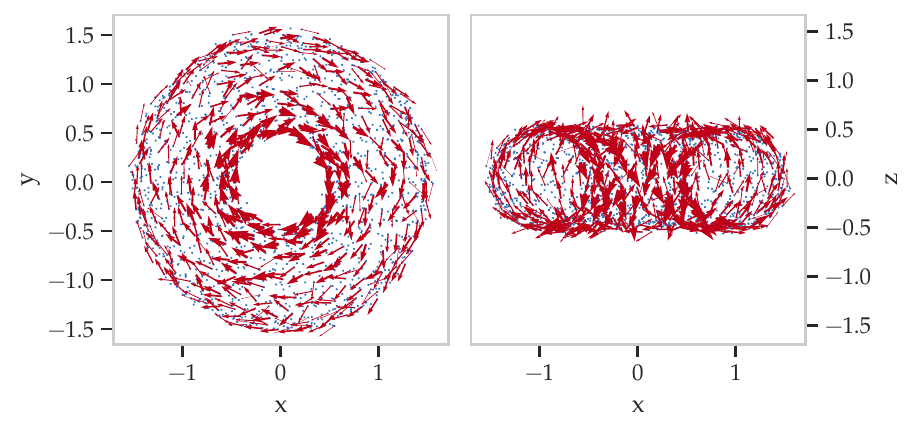}
    \label{fig:rev-int-comp-exps-torus-damp0}}\hfill
\subfloat[][]
    {\includegraphics[width=0.3\linewidth]{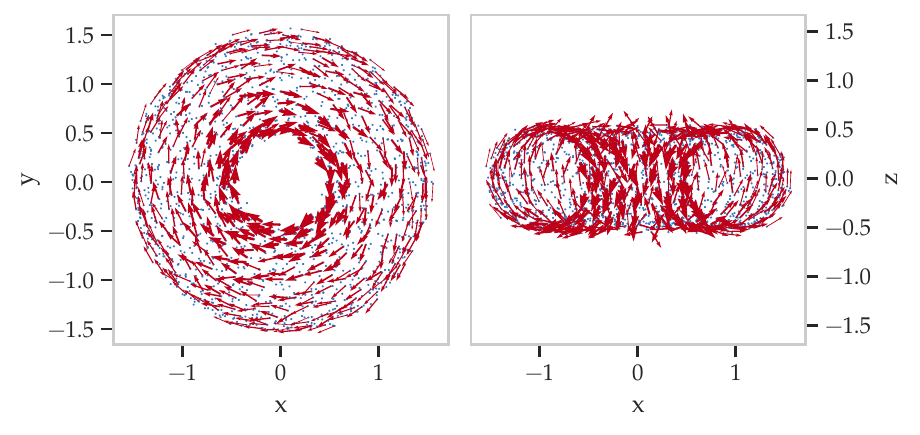}
    \label{fig:rev-int-comp-exps-torus-corr}}\hfill
\subfloat[][]
    {\includegraphics[width=0.3\linewidth]{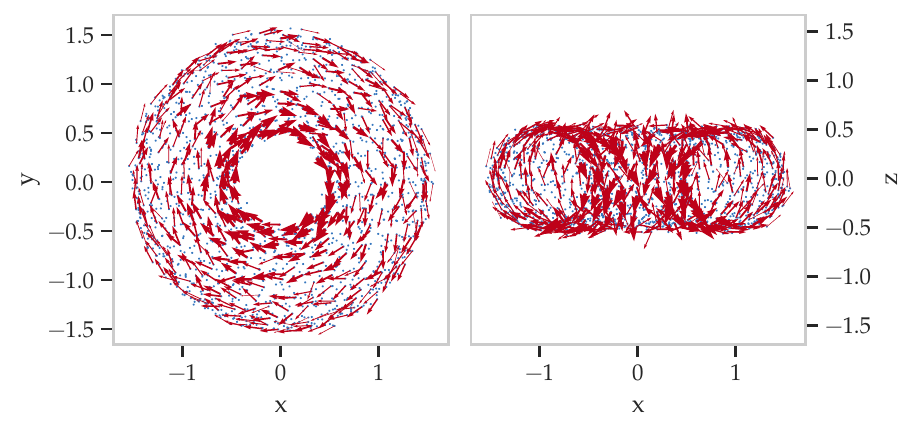}
    \label{fig:rev-int-comp-exps-torus-mse}}\hfill

    \subfloat[][]
    {\includegraphics[width=0.3\linewidth]{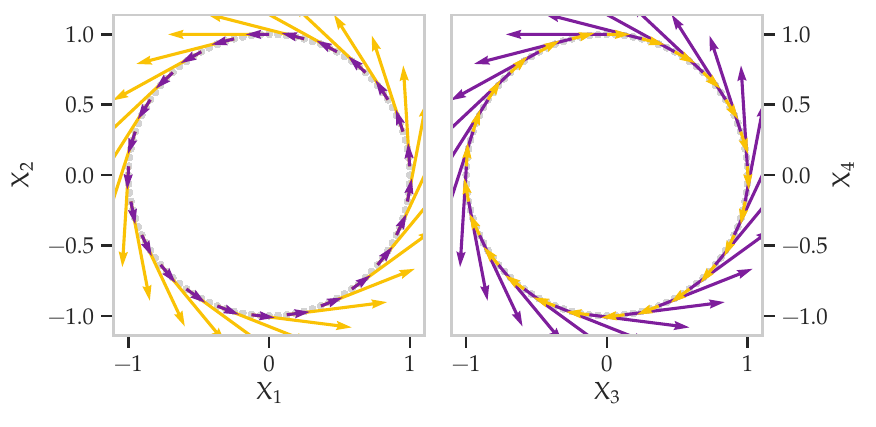}
    \label{fig:rev-int-comp-exps-flat-torus-damp0}}\hfill
\subfloat[][]
    {\includegraphics[width=0.3\linewidth]{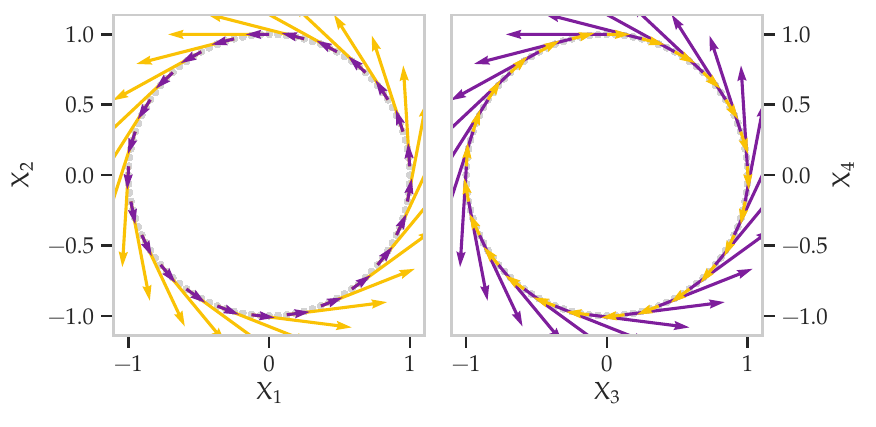}
    \label{fig:rev-int-comp-exps-flat-torus-corr}}\hfill
\subfloat[][]
    {\includegraphics[width=0.3\linewidth]{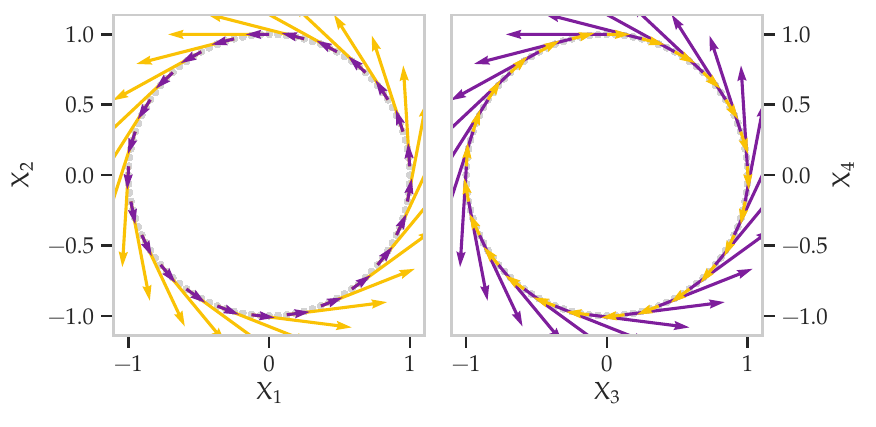}
    \label{fig:rev-int-comp-exps-flat-torus-mse}}\hfill

    \subfloat[][]
    {\includegraphics[width=0.3\linewidth]{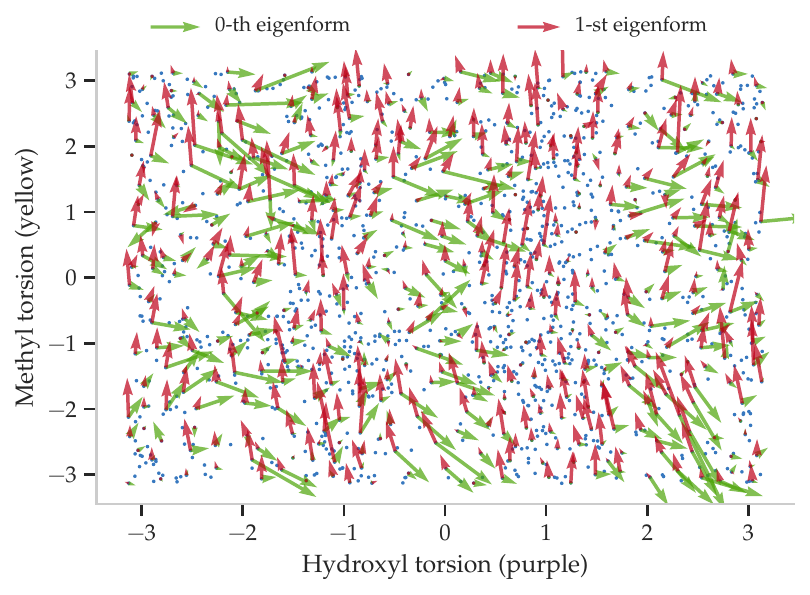}
    \label{fig:rev-int-comp-exps-ethanol-damp0}}\hfill
\subfloat[][]
    {\includegraphics[width=0.3\linewidth]{figs/ethanol/rev_int_exp/eigenform_tau_space_label_description}
    \label{fig:rev-int-comp-exps-ethanol-corr}}\hfill
\subfloat[][]
    {\includegraphics[width=0.3\linewidth]{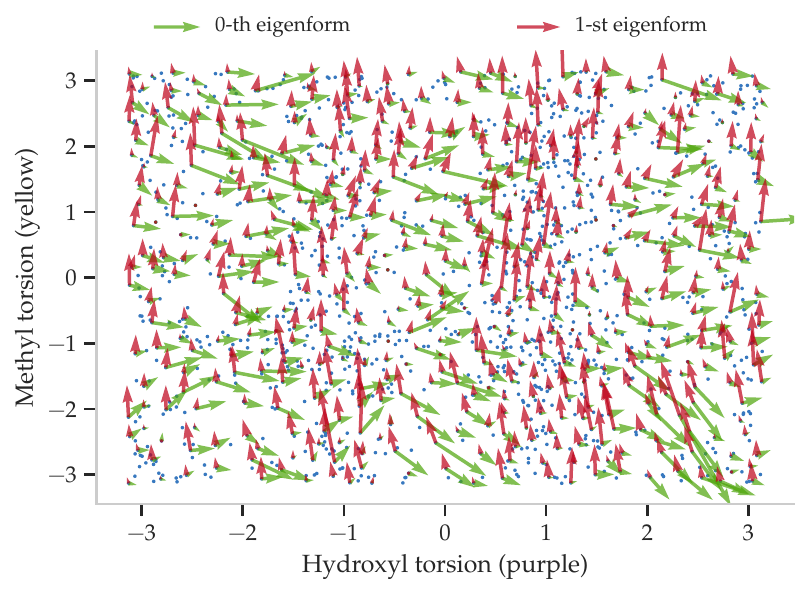}
    \label{fig:rev-int-comp-exps-ethanol-mse}}\hfill

    \subfloat[][]
    {\includegraphics[width=0.3\linewidth]{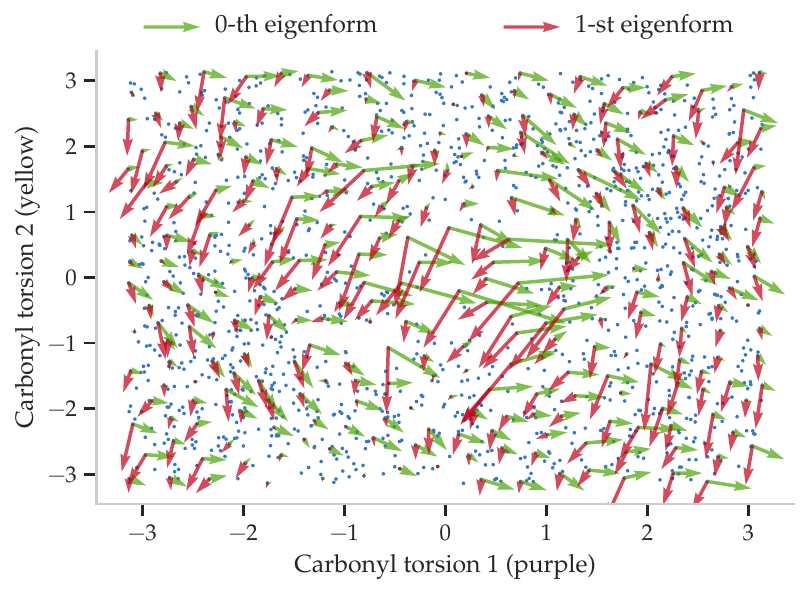}
    \label{fig:rev-int-comp-exps-mda-damp0}}\hfill
\subfloat[][]
    {\includegraphics[width=0.3\linewidth]{figs/MDA/rev_int_exp/eigenform_tau_space_label_description}
    \label{fig:rev-int-comp-exps-mda-corr}}\hfill
\subfloat[][]
    {\includegraphics[width=0.3\linewidth]{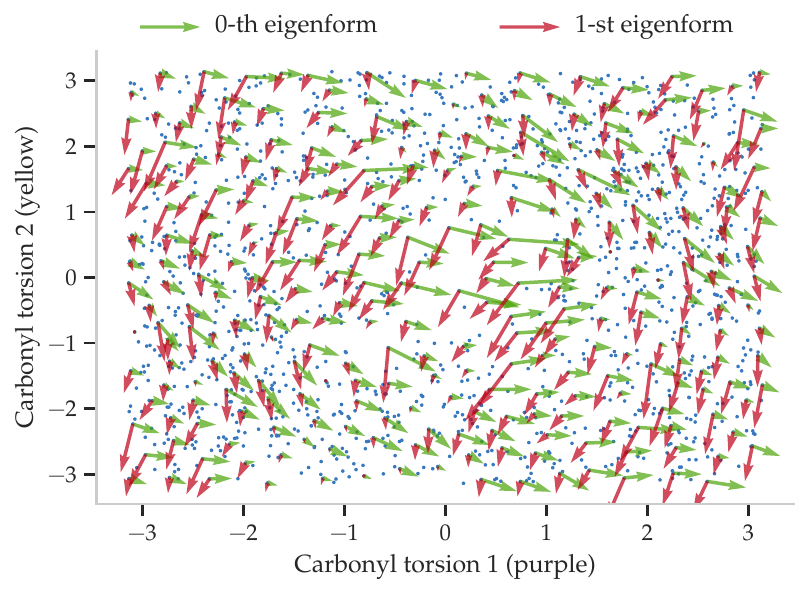}
    \label{fig:rev-int-comp-exps-mda-mse}}\hfill

    \caption{Comparisons of the estimated velocity fields from 1-cochain for various datasets with different regularization coefficient $\lambda$'s. Columns from left to right are $\lambda = 0$, $\lambda^*_\rho$, and $\lambda^*_\text{MSE}$ (see more detail in Supplement \ref{sec:reverse-interpolation-damped-lsqrt} and Figure \ref{fig:rev-int-comp}). Rows from top to bottom correspond to the eigenflows of synthetic circle, eigenflows of synthetic torus, eigenflows of synthetic flat torus, eigenflows of ethanol in torsion space, and eigenflows of MDA in torsion space. The last two vector fields are mapped from original PCA space $\vX$ to torsion space using \eqref{eq:vec-field-mapping-jacobian}.}
    \label{fig:rev-int-comp-exps-1}
\end{figure}

\begin{figure}[!htb]
    \subfloat[][]
    {\includegraphics[width=0.3\linewidth]{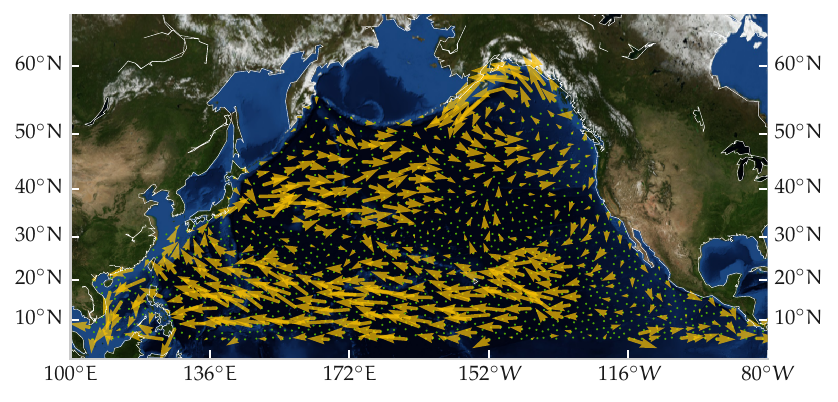}
    \label{fig:rev-int-comp-exps-ocean-smooth-damp0}}\hfill
\subfloat[][]
    {\includegraphics[width=0.3\linewidth]{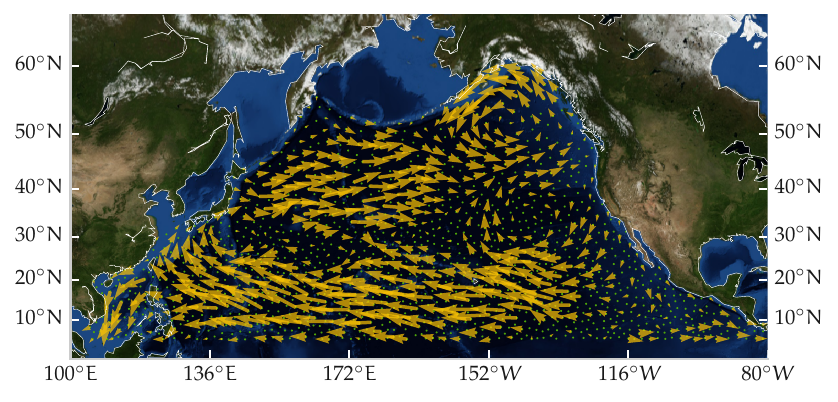}
    \label{fig:rev-int-comp-exps-ocean-smooth-corr}}\hfill
\subfloat[][]
    {\includegraphics[width=0.3\linewidth]{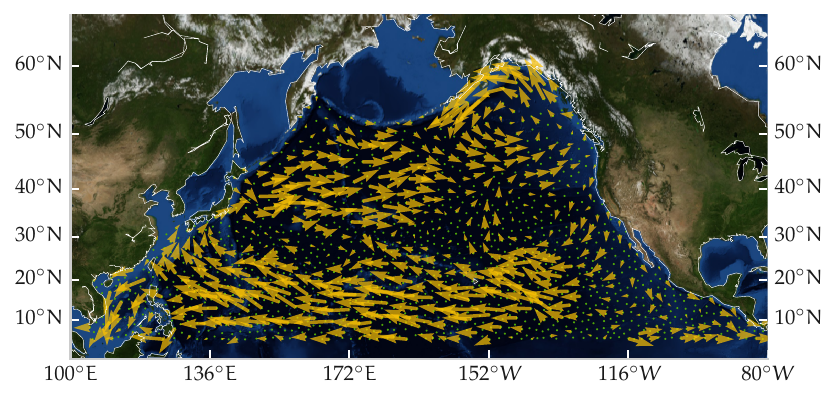}
    \label{fig:rev-int-comp-exps-ocean-smooth-mse}}\hfill

    \subfloat[][]
    {\includegraphics[width=0.3\linewidth]{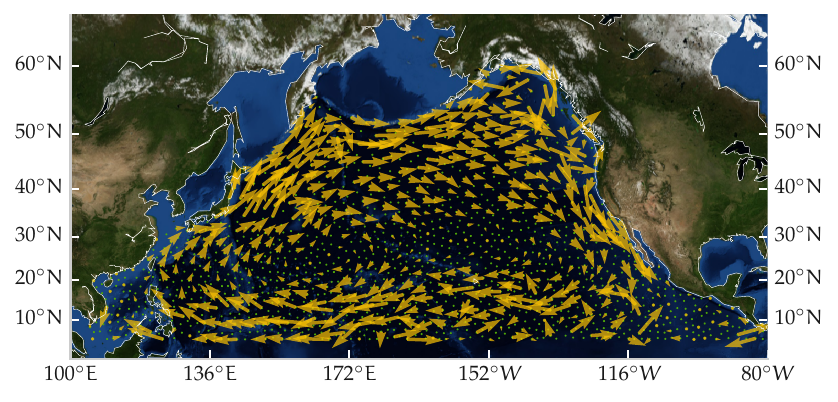}
    \label{fig:rev-int-comp-exps-ocean-ip-damp0}}\hfill
\subfloat[][]
    {\includegraphics[width=0.3\linewidth]{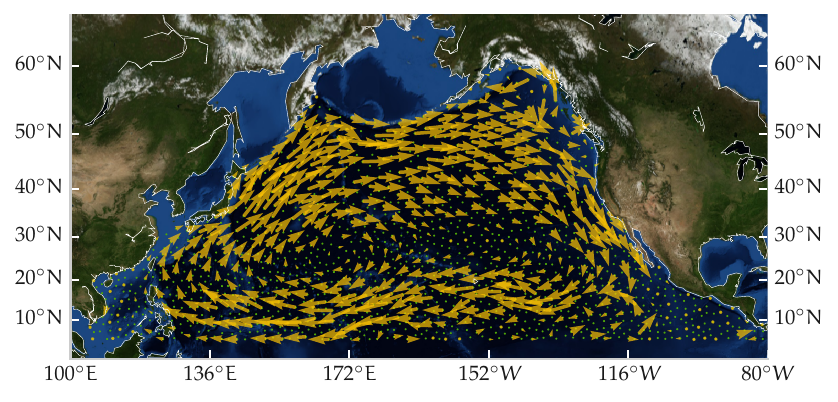}
    \label{fig:rev-int-comp-exps-ocean-ip-corr}}\hfill
\subfloat[][]
    {\includegraphics[width=0.3\linewidth]{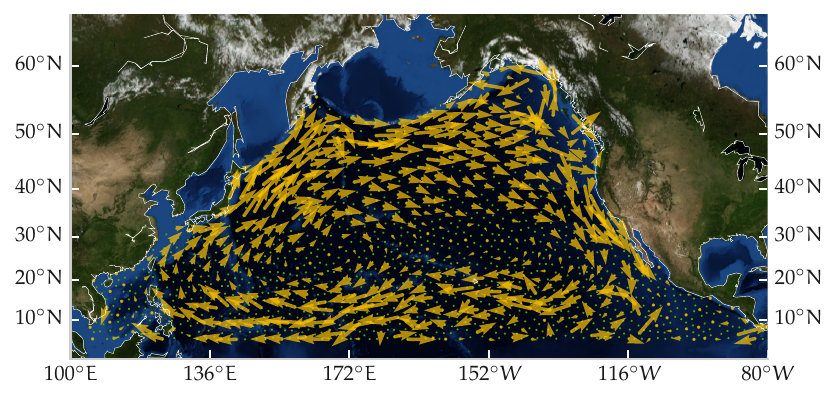}
    \label{fig:rev-int-comp-exps-ocean-ip-mse}}\hfill

    \subfloat[][]
    {\includegraphics[width=0.3\linewidth]{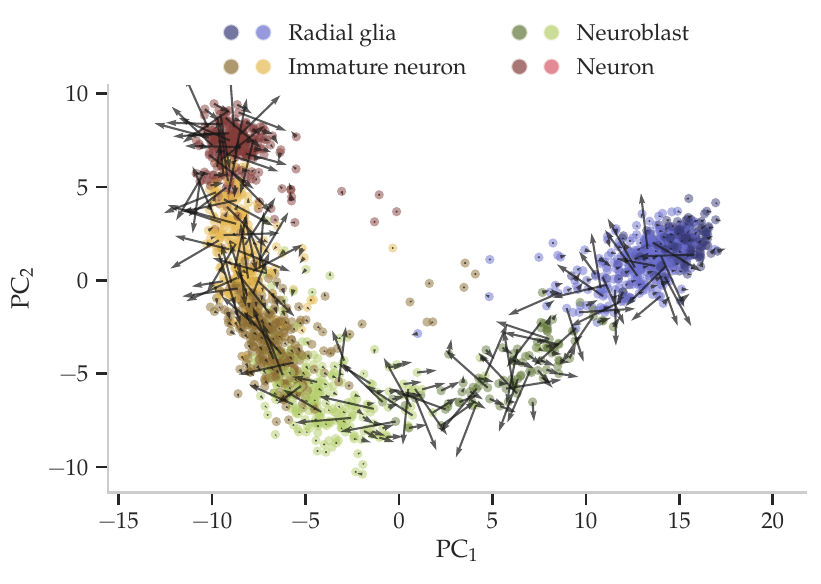}
    \label{fig:rev-int-comp-exps-rna-embryo-count-damp0}}\hfill
\subfloat[][]
    {\includegraphics[width=0.3\linewidth]{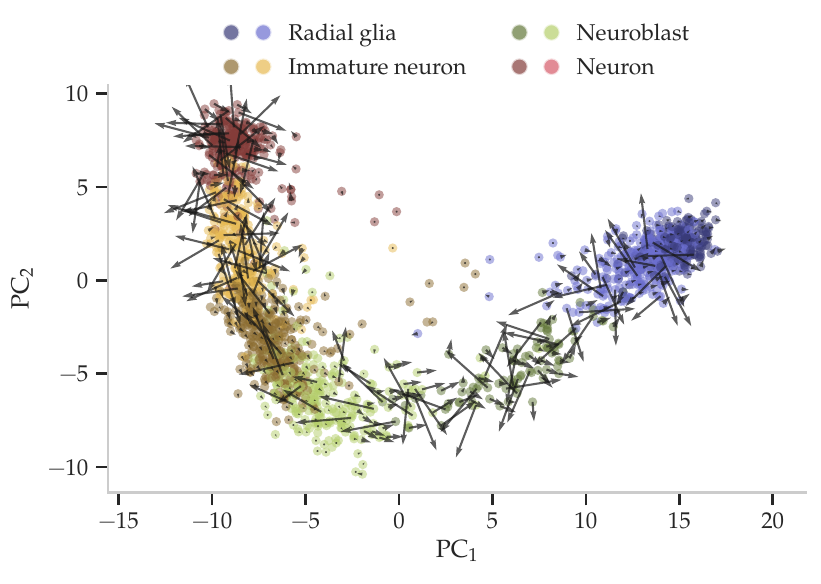}
    \label{fig:rev-int-comp-exps-rna-embryo-count-corr}}\hfill
\subfloat[][]
    {\includegraphics[width=0.3\linewidth]{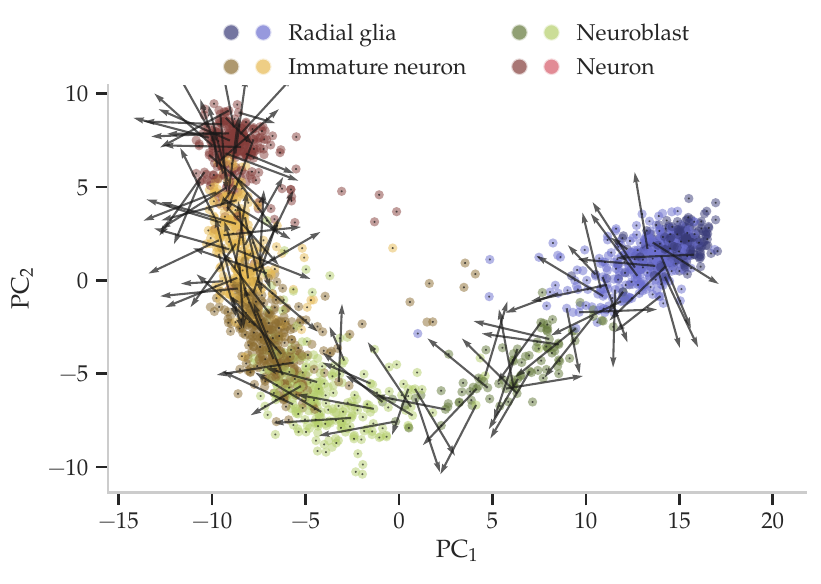}
    \label{fig:rev-int-comp-exps-rna-embryo-count-mse}}\hfill

    \subfloat[][]
    {\includegraphics[width=0.3\linewidth]{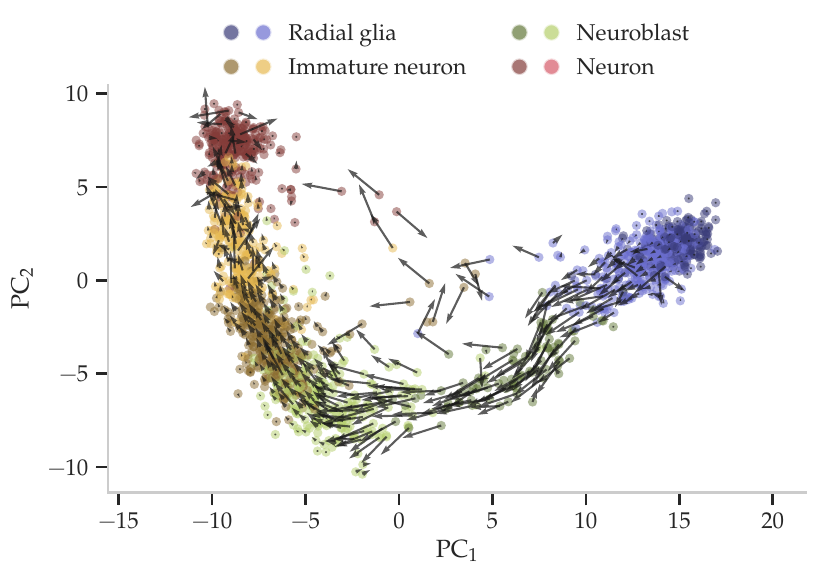}
    \label{fig:rev-int-comp-exps-rna-embryo-ip-damp0}}\hfill
\subfloat[][]
    {\includegraphics[width=0.3\linewidth]{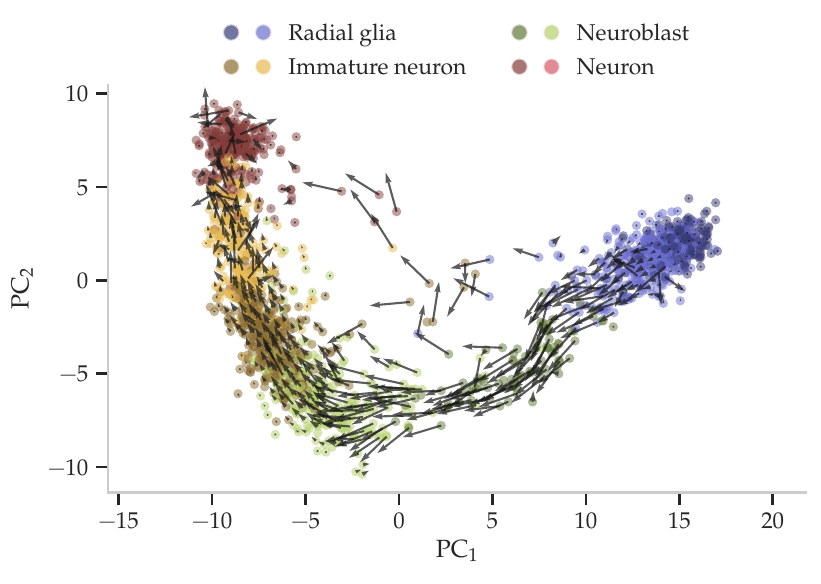}
    \label{fig:rev-int-comp-exps-rna-embryo-ip-corr}}\hfill
\subfloat[][]
    {\includegraphics[width=0.3\linewidth]{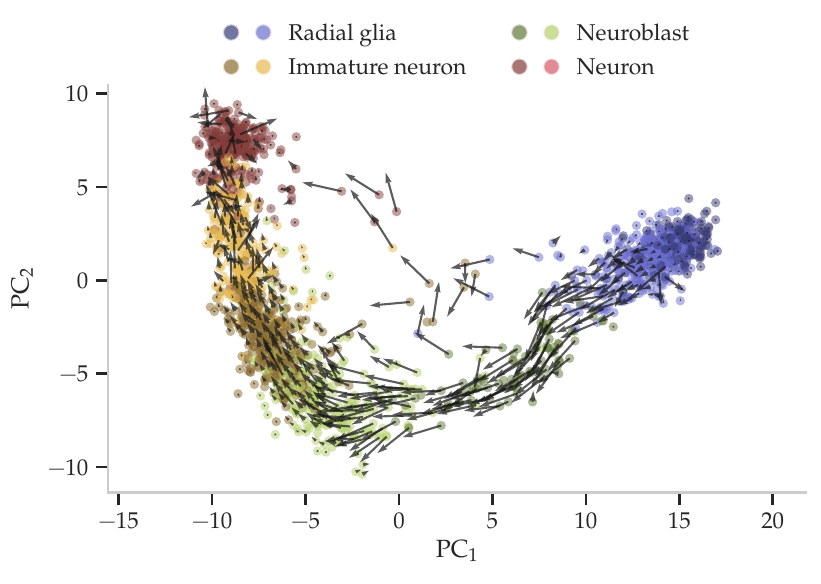}
    \label{fig:rev-int-comp-exps-rna-embryo-ip-mse}}\hfill

    \caption{Comparisons of the estimated velocity fields from 1-cochain for various datasets with different regularization coefficient $\lambda$'s. Columns from left to right are $\lambda = 0$, $\lambda^*_\rho$, and $\lambda^*_\text{MSE}$ (see more detail in Supplement \ref{sec:reverse-interpolation-damped-lsqrt} and Figure \ref{fig:rev-int-comp}). Rows from top to bottom correspond to the smoothed velocity field of ocean buoy dataset, estimated velocity field from the inverse problem on ocean buoy data, estimated velocity field of the zero-padded cochain on the human glutamatergic neuron data, and the estimated velcoity from the inverse problem on human glutamatergic neuron dataset.}
    \label{fig:rev-int-comp-exps-2}
\end{figure}

\end{document}